\documentclass[a4paper, 11pt]{article}
\usepackage{latexsym,amsfonts, amsmath, amssymb, amsthm}
\usepackage[utf8]{inputenc}
\usepackage[final=true,protrusion=true,expansion=true]{microtype}
\usepackage[noadjust]{cite}
\usepackage{enumitem}
\usepackage{color}
\usepackage{mathtools}

\usepackage{booktabs}
\usepackage{multirow}

\usepackage{subcaption}
\usepackage[labelfont=bf]{caption}
\usepackage[affil-it]{authblk}

\usepackage{geometry}
\usepackage{hyperref}
\hypersetup{
	final,
    colorlinks=true,
    linkcolor=blue,
    filecolor=magenta,
    urlcolor=cyan,
}

\usepackage{cleveref}

\theoremstyle{plain}
\newtheorem{theorem}{Theorem}[section]
\newtheorem{proposition}[theorem]{Proposition}
\newtheorem{lemma}[theorem]{Lemma}

\newtheorem{assumption}[theorem]{Assumption}

\theoremstyle{plain}
\newtheorem{definition}[theorem]{Definition}

\theoremstyle{definition}
\newtheorem{remark}[theorem]{Remark}

\numberwithin{equation}{section}

\long\def\acks#1{\vskip 0.3in\noindent{\large\bf Acknowledgments}\vskip 0.2in
\noindent #1}
\newcommand{\JN}[1]{\mathcal{J}^N_{#1}}
\newcommand{\J}[1]{\mathcal{J}^*_{#1}}
\newcommand{\Q}[1]{\mathcal{Q}^*_{#1}}

\newcommand{\uN}[1]{u^N_{#1}}
\newcommand{\uNv}[2]{u^N_{#1}({#2})}
\newcommand{\up}[1]{u^*_{#1}}
\newcommand{\upu}[2]{u^{*,#2}_{#1}}

\newcommand{\upp}[2]{u^*_{#1}({#2})}
\newcommand{\uppu}[3]{u^{*,#2}_{#1}({#3})}
\newcommand{\upt}[1]{\widetilde{u}^*_{#1}}
\newcommand{\uNt}[1]{\widetilde{u}^N_{#1}}
\newcommand{\uptu}[2]{\widetilde{u}^{*,#2}_{#1}}

\newcommand{\uNhat}[1]{\widehat{u}^N_{#1}}
\newcommand{\uNhatt}[1]{\widehat{\widetilde{u}}^N_{#1}}
\newcommand{\uNhatv}[2]{\widehat{u}^N_{#1}(#2)}
\newcommand{\uhatp}[1]{\widehat{u}^*_{#1}}
\newcommand{\uppt}[2]{\widetilde{u}^*_{#1}({#2})}
\newcommand{\uNpt}[2]{\widetilde{u}^N_{#1}({#2})}
\newcommand{\upptu}[3]{\widetilde{u}^{*,#2}_{#1}({#3})}
\newcommand{\uhatpu}[2]{\widehat{u}^{*,#2}_{#1}}

\newcommand{\uhatpp}[2]{\widehat{u}^*_{#1}({#2})}

\newcommand{\uhatppt}[2]{\widehat{\widetilde{u}}^*_{#1}({#2})}
\newcommand{\uNhatpt}[2]{\widehat{\widetilde{u}}^N_{#1}({#2})}
\newcommand{\uhatpptu}[3]{\widehat{\widetilde{u}}^{*,#2}_{#1}({#3})}

\newcommand{\uhatpt}[1]{\widehat{\widetilde{u}}^*_{#1}}
\newcommand{\uhatptu}[2]{\widehat{\widetilde{u}}^{*,#2}_{#1}}

\newcommand{\whatp}[1]{\widehat{w}^*_{#1}}
\newcommand{\whatpp}[2]{\widehat{w}^*_{#1}({#2})}
\newcommand{\vhatp}[1]{\widehat{v}^*_{#1}}
\newcommand{\vhatpp}[2]{\widehat{v}^*_{#1}({#2})}
\newcommand{\uhatpReverted}[1]{\underbar{\widehat{u}}^*_{#1}}

\newcommand{\uhatptReverted}[1]{\widehat{\underbar{\widetilde{u}}}^*_{#1}}
\newcommand{\uhatptuReverted}[2]{\widehat{\underbar{\widetilde{u}}}^{*,#2}_{#1}}

\newcommand{\uhatppReverted}[2]{\underbar{\widehat{u}}^*_{#1}({#2})}

\newcommand{\uhatpptuReverted}[3]{\underbar{\widehat{\widetilde{u}}}^{*,#2}_{#1}({#3})}
\newcommand{\vhatpReverted}[1]{\underbar{\widehat{v}}^*_{#1}}
\newcommand{\vhatppReverted}[2]{\underbar{\widehat{v}}^*_{#1}({#2})}

\newcommand{\uhatpptReverted}[2]{\widehat{\underbar{\widetilde{u}}}^*_{#1}({#2})}

\newcommand{\overbar}[1]{\makebox[0pt]{$\phantom{#1}\mkern 1.5mu\overline{\mkern-1.5mu\phantom{#1}\mkern-1.5mu}\mkern 1.5mu$}#1}
\renewcommand{\underbar}[1]{\makebox[0pt]{$\phantom{#1}\mkern 1.5mu\underline{\mkern-1.5mu\phantom{#1}\mkern-1.5mu}\mkern 1.5mu$}#1}

\newcommand{\overbarscript}[1]{\mkern 1.5mu\overline{\mkern-1.5mu#1\mkern-1.5mu}\mkern 0mu}

\newcommand{\dummy}{\mathord{\color{black!33}\bullet}}

\providecommand*{\vol}[1]{\operatorname{vol}\!\left({#1}\right)}

\providecommand{\esssup}{\operatorname*{ess\,sup}}

\providecommand{\limsup}{\operatorname*{lim\,sup}}

\providecommand{\fpartial}[1]{\partial_{#1}}
\providecommand{\fpartiall}[1]{\partial^2_{#1}}
\providecommand{\fd}[1]{\frac{d}{d {#1}}}

\providecommand{\CA}{{\cal A}}
\providecommand{\CB}{{\cal B}}
\providecommand{\CC}{{\cal C}}

\providecommand{\CJ}{{\cal J}}

\providecommand{\CL}{{\cal L}}
\providecommand{\CM}{{\cal M}}
\providecommand{\CN}{{\cal N}}

\providecommand{\CP}{{\cal P}}
\providecommand{\CQ}{{\cal Q}}

\providecommand{\CS}{{\cal S}}
\providecommand{\CT}{{\cal T}}
\providecommand{\CU}{{\cal U}}
\providecommand{\CV}{{\cal V}}

\providecommand{\CX}{{\cal X}}

\providecommand{\bbE}{\mathbb{E}}

\providecommand{\bbN}{\mathbb{N}}

\providecommand{\bbR}{\mathbb{R}}

\providecommand*{\abs}[1]{\left|{#1}\right|}
\providecommand*{\absbig}[1]{\big|{#1}\big|}
\providecommand*{\absnormal}[1]{|{#1}|}
\providecommand*{\N}[1]{\left\|{#1}\right\|} 
\providecommand*{\Nnormal}[1]{\|{#1}\|} 
\providecommand*{\Nbig}[1]{\big\|{#1}\big\|} 

\providecommand*{\abs}[1]{\left|{#1}\right|} 

\usepackage{bbm}
\usepackage{arydshln}
\usepackage{subcaption}
\usepackage[font=small,labelfont=bf]{caption}
\usepackage{graphicx}
\usepackage{xcolor}
\usepackage{ifdraft}
\usepackage{scalerel}
\usepackage{mathrsfs}  
\usepackage{etoc}
\usepackage{comment}

\makeatother

\title{\usefont{OT1}{bch}{b}{n}
	\huge
    Global Convergence of Adjoint-Optimized Neural PDEs%
    \footnote{This article is part of the project ``DMS-EPSRC: Asymptotic Analysis of Online Training Algorithms in Machine Learning: Recurrent, Graphical, and Deep Neural Networks'' (NSF DMS-2311500).}
}

\date{}

\author[1]{Konstantin Riedl\thanks{Email: \texttt{Konstantin.Riedl@maths.ox.ac.uk}}}
\author[1]{Justin Sirignano\thanks{Email: \texttt{Justin.Sirignano@maths.ox.ac.uk}}}
\author[2]{Konstantinos Spiliopoulos\thanks{Email: \texttt{kspiliop@bu.edu}}}

\affil[1]{University of Oxford, Mathematical Institute}
\affil[2]{Boston University, Department of Mathematics \& Statistics}

\begin{document}
\maketitle

\begin{abstract}
\noindent
    Many engineering and scientific fields have recently become interested in modeling terms in partial differential equations (PDEs) with neural networks, which requires solving the inverse problem of learning neural network terms from observed data in order to approximate missing or unresolved physics in the PDE model.
    The resulting neural-network PDE model, being a function of the neural network parameters, can be calibrated to the available ground truth data by optimizing over the PDE using gradient descent, where the gradient is evaluated in a computationally efficient manner by solving an adjoint PDE.
    These neural PDE models have emerged as an important research area in scientific machine learning.
    In this paper, we study the convergence of the adjoint gradient descent optimization method for training neural PDE models in the limit where both the number of hidden units and the training time tend to infinity.
    Specifically, for a general class of nonlinear parabolic PDEs with a neural network embedded in the source term, we prove convergence of the trained neural-network PDE solution to the target data (i.e., a global minimizer).
    The global convergence proof poses a unique mathematical challenge that is not encountered in finite-dimensional neural network convergence analyses due to (i) the neural network training dynamics involving a non-local neural network kernel operator in the infinite-width hidden layer limit where the kernel lacks a spectral gap for its eigenvalues and (ii) the nonlinearity of the limit PDE system,
    which leads to a non-convex optimization problem in the neural network function even in the infinite-width hidden layer limit (unlike in typical neural network training cases where the optimization problem becomes convex in the large neuron limit).
    The theoretical results are illustrated and empirically validated by numerical studies.
\end{abstract}




\section{Introduction}
\label{sec:introduction}

Motivated by the remarkable successes of machine learning and deep learning \cite{lecun2015deep} in speech and image recognition~\cite{hinton2012deep,krizhevsky2012imagenet}, computer vision~\cite{krizhevsky2012imagenet}, natural language processing~\cite{vaswani2017attention}, and biology~\cite{jumper2021highly},
researchers and practitioners have sought to leverage and extend machine learning techniques to scientific disciplines,
leading to the emergence of the field of scientific machine learning (SciML)~\cite{karniadakis2021physics},
where data-driven methods are integrated with physics-based modeling.
SciML seeks to develop machine learning methods with physics-based modeling, integrating the modeling flexibility of neural networks (NNs) and large real-world datasets with well-established partial differential equation (PDE) models derived from physics.
The field has developed a variety of different approaches.
Physics-informed neural networks (PINNs)~\cite{raissi2019physics,karniadakis2021physics,lu2021deepxde,cuomo2022scientific,wang2023expert,jiang2023global,zhao2025convergence}, the deep Galerkin method \cite{sirignano2018dgm}, the
deep Ritz method \cite{yu2018deep}, and neural Q-learning~\cite{cohen2023neural}
exploit the property of an NN as a universal function approximator to approximate the PDE solution of a known PDE with an NN by training the NN parameters to satisfy the differential operator as well as initial and boundary conditions.
In some scenarios, however, the PDE governing the physical phenomenon of interest is either entirely unknown or only partially accessible, leading to an inverse problem.
When the PDE is unknown, operator learning~\cite{boulle2023mathematical,li2020fourier} attempts to learn, leveraging the universality of neural operators~\cite{lu2021learning,kovachki2023neural} as approximators of nonlinear mappings, the PDE solution operator from observed data by minimizing a suitable loss.
In the case of hidden, incomplete or unclosed physics, as in the setting of this paper,
the PDE structure is known while certain terms are unknown (such as coefficients or source terms of the PDE).
Leveraging again the property of NNs as universal function
approximators, NNs can be trained to model the unknown PDE terms using observed data~\cite{brunton2016discovering,schaeffer2017learning,Brunton_Kutz_2019,champion2019data,duraisamy2019turbulence,brunton2020machine,sun2020neupde,srivastava2021generalizable,duraisamy2021perspectives,dong2022optimization,sirignano2023pde,aarset2023learning,dong2024descent,holler2024uniqueness}.

Integration of machine learning with PDEs in science and engineering, combined with
the growing availability of large datasets from field measurements, experiments, and high-fidelity numerical simulations,
can yield more accurate engineering models and inform physical models with data-driven insights~\cite{brunton2016discovering,schaeffer2017learning,champion2019data,brunton2020machine,Brunton_Kutz_2019,sirignano2023pde} across a diverse range of application domains.
NN terms in the PDE can be trained to learn missing, unknown, or unrepresented physics and correct numerical discretization errors.
In computational fluid dynamics, for instance,
NNs are introduced into the governing equations to represent the unclosed terms in PDE models of
turbulent flows such as Reynolds-averaged Navier-Stokes (RANS) and Large-eddy Simulation
(LES) equations~\cite{tracey2015machine,duraisamy2019turbulence,srivastava2021generalizable,kochkov2021machine,duraisamy2021perspectives,sirignano2023deep,hickling2024large,kakka2025neural,nair2023deep,nair2024adjoint}.
Neural PDE or neural-network PDE (NN-PDE) models also have applications
in finance, economics, and biology~\cite{gierjatowicz2020robust,goswami2021data,cohen2023arbitrage,fan2024machine}.

Since the solution of the NN-PDE is a function of parameters of the NN which models certain terms/coefficients in the PDE,
the NN parameters must be calibrated such that the NN-PDE solution matches the available ground truth data as closely as possible.
Such target data may come from real-world experiments or high-fidelity numerical simulations.
In order to solve this inverse problem, i.e., calibrate the NN parameters, one must optimize over the NN-PDE via gradient descent on a suitable loss that quantifies the discrepancy between the NN-PDE solution and the target data.
The gradient descent algorithm requires evaluating the gradient of the objective function, which depends upon the solution of the NN-PDE.
Therefore, the gradient of the NN-PDE solution with respect to the NN parameters needs to be evaluated, which is computationally challenging.
However, this gradient can be efficiently evaluated using the adjoint PDE of the NN-PDE.
Adjoint optimization methods have been developed and applied to NN-PDEs in \cite{sirignano2023pde,sirignano2022online,macart2021embedded,sirignano2020dpm,holland2019towards,srivastava2021generalizable}.
More generally, adjoint optimization has been widely used for PDE optimization~\cite{bosse2014one,brandenburg2009continuous,bueno2012continuous,cagnetti2013adjoint,duta2002harmonic,gauger2012automated,giles2000introduction,giles2010convergence1,giles2010convergence2,hazra2004simultaneous,hazra2007direct,hinze2008optimization,jameson1998optimum,jameson2003aerodynamic,jameson2003reduction,kaland2014one,knopoff2013adjoint,pierce2000adjoint,nadarajah2000comparison,nadarajah2001studies,reuther1996aerodynamic}.
The adjoint gradient descent algorithm solves a linear adjoint PDE at each optimization iteration, evaluates the gradient of the objective function with respect to the NN parameters using the adjoint PDE solution, and then takes a gradient descent step to update the NN parameters.

While adjoint methods have demonstrated effectiveness and efficiency across a wide range of applications, including design and shape optimization, aerodynamics, combustion, and tumor growth modeling in medicine,
a rigorous mathematical analysis of adjoint gradient descent optimization methods in the setting of nonlinear NN-PDEs remains absent.
The analysis in this paper is focused on this topic and provides first-of-its-kind convergence results in the nonlinear NN-PDE regime, where training the NN model leads to a non-convex optimization problem in the NN function, even in the large neuron limit.
This is because the NN-PDE solution, which the loss is a function of, depends nonlinearly on the NN function, which is the design variable of our problem formulation and the quantity being trained. Our global convergence proof must address unique mathematical challenges that are not encountered in finite-dimensional NN convergence analyses due to (i) the NN training dynamics involving a non-local NN kernel operator in the infinite-width hidden layer limit where the kernel lacks a spectral gap for its eigenvalues and (ii) the nonlinearity of the limit PDE system,
which leads to a non-convex optimization problem in the NN function even in the infinite-width hidden layer limit (unlike in typical NN training cases where the optimization problem becomes convex in the large neuron limit~\cite{jacot2018neural,chizat2019lazy}).

We prove that the NN-PDE solution converges weakly to the target data (i.e., a global minimizer) during training. 
The first step is to prove that the adjoint PDE solution vanishes in the weak sense as the training time tends to infinity,
which in turn requires establishing that a quadratic functional of the adjoint involving the positive definite non-local NN kernel operator converges to zero.
The latter is proven by applying a cycle of stopping times analysis.
The cycle of stopping times analysis requires the development of a novel approach for obtaining a regularity bound for this quadratic functional in terms of the learning rate, which is based on a careful PDE analysis of an adjoint PDE system associated with the quadratic functional, thus the analysis of the \emph{adjoint of the original adjoint PDE}.
The derivation of bounds on several norms of the different adjoint PDEs is also required;
see \Cref{sec:contributions,sec:convergence} for more details.
Developing these mathematical methods allows us to prove a much more general and stronger convergence result than done by prior analyses that only proved a substantially weaker notion of convergence for linear PDEs and a very restrictive class of objective functions~\cite{sirignano2023pde};
see \Cref{rem:comparisonconvergencesirignano2023pde} for the details.

In our mathematical analysis, we will consider the second-order semi-linear parabolic neural-network partial differential equation~(NN-PDE)
\begin{alignat}{2}
\label{eq:parabolicPDEN_plain}
\begin{aligned}
    \fpartial{t}\uN{\theta} + \CL\uN{\theta} - q(\uN{\theta})
    &= g_\theta^N
    \qquad&&\text{in }
    D_T, \\
    \uN{\theta}
    &= 0
    \qquad&&\text{on }
    [0,T]\times\partial D, \\
    \uN{\theta}
    &= f
    \qquad&&\text{on }
    \{0\}\times D,
\end{aligned}
\end{alignat}
where $\CL$ denotes a second-order linear PDE operator given in divergence form as
\begin{equation}
    \CL u
    = - \sum_{i,j=1}^d \fpartial{x_j} \left(a^{ij}(t,x) \fpartial{x_i} u\right)
    + \sum_{i=1}^d b^i(t,x) \fpartial{x_i} u 
    + c(t,x)u
\end{equation}
with PDE coefficients $a^{ij}, b^i, c: [0,T]\times D\rightarrow\bbR$,
and where $q:[0,T]\times D\times\bbR\rightarrow\bbR$ denotes the nonlinearity of the PDE.
For notational convenience, we will sometimes (as done in \eqref{eq:parabolicPDEN_plain}) omit writing the physical time and space dependency for the PDE solution as well as the PDE coefficients and terms, i.e., we write $u$ instead of $u(t,x)$ or $q(u)$ instead of $q(t,x,u(t,x))$.
The PDE operator~$\CL$ together with its coefficients~$a^{ij}$, $b^i$ and $c$ as well as the nonlinearity $q$ are assumed to satisfy Assumptions~\ref{asm:PDE_L_parabolicity}, \ref{asm:PDE_coefficients} and \ref{asm:WP_coefficients} as well as Assumptions~\ref{asm:PDE_nonlinearity_q_ubdd}, \ref{asm:PDE_nonlinearity_q_uubdd}, \ref{asm:WP_nonlinearity_q_growth} and \ref{asm:WP_nonlinearity_q_ubdd}, respectively.
On the time-space domain $[0,T]\times D\subset \bbR\times \bbR^d$ we moreover impose Assumptions~\ref{asm:D} and \ref{asm:Dbdd}, and on the initial condition $f$ Assumption~\ref{asm:WP_initialcondition} (see Section \ref{sec:assumptions}).

The PDE~\eqref{eq:parabolicPDEN_plain} is driven by an NN with parameters $\theta$ in the source term.
We design it to be a fully-connected NN $g^N_\theta=g^N_\theta(t,x)$ with a single hidden layer consisting of $N$ neurons, i.e., it takes the form
\begin{equation}
    \label{eq:gN}
    g_\theta^N(t,x)
    = \frac{1}{N^\beta}\sum_{i=1}^N c^i\sigma\big(w^{t,i}t + (w^i)^Tx + \eta^i\big),
\end{equation}
where the NN parameters of the $i$th neuron/unit are collected in the weight vector $\theta^i = (c^i,w^{t,i},w^i,\eta^i) \in \bbR\times\bbR\times\bbR^{d}\times\bbR$
and where $\theta=(\theta^i)_{i=1,\dots,N}$ denotes the collection of all NN parameters,
which are initialized independently according to $\theta_0^i = (c^i_0,w^{t,i}_0,w^i_0,\eta^i_0)\sim\mu_0$ with a measure $\mu_0\in\CP(\bbR\times\bbR\times\bbR^{d}\times\bbR)$ obeying Assumption~\ref{asm:NN_mu0}.
The factor $1/N^\beta$ in \eqref{eq:gN} with $\beta\in(1/2,1)$ is a normalization/scaling, and the NN nonlinearity~$\sigma$ satisfies Assumptions~\ref{asm:NN_sigma} and \ref{asm:NN_sigma'}.

The aforementioned assumptions on the PDE are collected in Assumptions~\ref{def:PDE_assumptions} and \ref{def:WP_assumptions},
and the ones on the NN in Assumption~\ref{def:NN_assumptions}.
They are assumed to hold throughout the manuscript.

Given a continuous target function~$h\in L_2(D_T)$ corresponding to or representing measured or observed data (the ground truth),
we wish to solve the inverse problem of calibrating the NN parameters~$\theta$ so that the solution $\uN{\theta}=\uNv{\theta}{t,x}$ to the NN-PDE~\eqref{eq:parabolicPDEN_plain} closely approximates the prescribed target data~$h$.
For this purpose,
we seek to minimize the loss (least squares loss)
\begin{equation}
    \label{eq:parabolicJ}
    \JN{\theta} = \frac{1}{2} \int_0^T\!\!\!\int_D (\uNv{\theta}{t,x} -h(t,x))^2 \,dxdt
\end{equation}
by training\footnote{Note that the training time, denoted by $\tau$, is distinct from and unrelated to the physical PDE time $t$.} the NN parameters~$\theta$ via continuous-time gradient descent
\begin{equation}
    \label{eq:GD}
    \frac{d}{d\tau} \theta_\tau
    = - \alpha^N_\tau \nabla_{\theta}\JN{\theta_\tau},
\end{equation}
where $\alpha^N_\tau=\frac{\alpha_\tau}{N^{1-2\beta}}$ denotes the learning rate which is assumed to be decreasing in the training time $\tau$ and satisfies the well-known Robbins-Monro conditions~\cite{robbins1951stochastic}
\begin{equation}
    \label{eq:learning_rate}
    \int_{0}^\infty\alpha_\tau \,d\tau=\infty
    \qquad\text{and}\qquad
    \int_{0}^\infty\alpha_\tau^2\,d\tau<\infty.
\end{equation}
Monotonicity as well as the conditions~\eqref{eq:learning_rate} are standard requirements for learning rates in machine learning \cite{bertsekas2000gradient}.
A suitable and classical example fulfilling those conditions is given by $\alpha_\tau=\frac{1}{1+\tau}$.
Computing the gradient $\nabla_{\theta}\JN{\theta}$ w.r.t.\@ the NN parameters~$\theta$ in \eqref{eq:GD} is computationally challenging due to its dependency on the solution~$\uN{\theta}$ of the PDE~\eqref{eq:parabolicPDEN_plain}.
Deriving naively a PDE for $\nabla_{\theta}\uN{\theta}$
by applying the gradient to \eqref{eq:parabolicPDEN_plain}
yields a PDE whose dimension is equal to the number of the NN parameters~$\theta$. Typically, the number of NN parameters is very large (thousands or even hundreds of thousands), leading to a high-dimensional PDE which is computationally costly to solve numerically.
Analogously, trying to estimate the gradient by numerical differentiation with finite differences amounts to an equivalently expensive and thus infeasible task.
A computationally efficient way to evaluate the gradient $\nabla_{\theta}\JN{\theta}$, however, is given by the adjoint method~\cite{giles2000introduction},
which can be regarded as a continuous PDE version of the usual backpropagation algorithm.
As we verify in \Cref{lem:grad_thetaJ},
the gradient $\nabla_{\theta}\JN{\theta}$ w.r.t.\@ the NN parameters~$\theta$ can be computed
according to
\begin{equation}
    \label{eq:nablaJ}
    \nabla_{\theta}\JN{\theta}
    = \int_0^T\!\!\!\int_D \nabla_\theta g_\theta^N(t,x)\uNhatv{\theta}{t,x}\,dxdt,
\end{equation}
which requires solving the adjoint PDE of \eqref{eq:parabolicPDEN_plain},
which is a second-order linear parabolic backward PDE of the form
\begin{alignat}{2}
\label{eq:parabolicadjointN_plain}
\begin{aligned}
    -\fpartial{t}\uNhat{\theta} + \CL^\dagger\uNhat{\theta} - q_u(\uN{\theta})\uNhat{\theta}
    &= (\uN{\theta}-h)
    \qquad&&\text{in }
    D_T, \\
    \uNhat{\theta}
    &= 0
    \qquad&&\text{on }
    [0,T]\times\partial D, \\
    \uNhat{\theta}
    &= 0
    \qquad&&\text{on }
    \{T\}\times D,
\end{aligned}
\end{alignat}
where $\CL^\dagger$ denotes the adjoint of $\CL$ which is given by
\begin{equation}
    \CL^\dagger \widehat{u}
    = - \sum_{i,j=1}^d \fpartial{x_i} \left(a^{ij}(t,x) \fpartial{x_j} \widehat{u}\right)
    - \sum_{i=1}^d b^i(t,x) \fpartial{x_i} \widehat{u} 
    + \left(c(t,x)-\sum_{i=1}^d\fpartial{x_i}b^i(t,x)\right)\widehat{u},
\end{equation}
and where $q_u=\partial_u q:[0,T]\times D\times\bbR\rightarrow\bbR$ denotes the partial derivative of $q$ w.r.t.\@ $u$, i.e., the function $(t,x,u)\mapsto q_u(t,x,u)=\partial_u q(t,x,u)$.

Adjoint optimization is widely used in practice due to its computational efficiency for evaluating the gradient of a PDE objective function.
Computing \eqref{eq:nablaJ} requires solving only one linear PDE of the same dimension as the original PDE~\eqref{eq:parabolicPDEN_plain} that is being optimized, which is computationally tractable.
Evaluating the gradient $\nabla_\theta g_\theta^N$ of the NN~$g_\theta^N$ itself w.r.t.\@ the NN parameters~$\theta$ is typically done efficiently using automatic differentiation frameworks such as TensorFlow or PyTorch.

Before discussing our analytical setup,
let us address the well-posedness (existence and uniqueness) of the PDE system \eqref{eq:parabolicPDEN_plain}\,\&\,\eqref{eq:parabolicadjointN_plain} coupled with the gradient descent update~\eqref{eq:GD} during training.
As we rigorously prove in \Cref{lem:parabolic_wellposedness_N},
there exists, for any finite training time horizon~$\CT$, a unique weak solution~$((\uN{\theta_\tau},\uNhat{\theta_\tau}))_{\tau\in[0,\CT]}$ in the sense of \Cref{def:weak_sol,def:weak_sol_adjoint} in the space $\CC\left([0,\CT],\CS\times\CS\right)$, where $\CS:=L_2([0,T],H^1(D))\cap L_\infty([0,T],L_2(D))$.

For training nonlinear parabolic NN-PDE models of the form~\eqref{eq:parabolicPDEN_plain},
we study in this paper the global convergence of the adjoint gradient descent optimization method~\eqref{eq:GD}, where the gradient of the loss~$\JN{\theta}$ w.r.t.\@ the NN parameters~$\theta$ is computed by solving the adjoint PDE~\eqref{eq:parabolicadjointN_plain} and evaluating formula \eqref{eq:nablaJ}.
We consider the theoretical limit where both the number of neurons~$N$ in the NN~$g_\theta^N$ in \eqref{eq:gN} and the training time~$\tau$ in \eqref{eq:GD} tend to infinity.

As a first step,
let us derive and theoretically justify the limiting training dynamics in the infinite-width hidden layer limit, i.e., as the number of hidden units $N\rightarrow\infty$.
Therefore, denote by
$\mu_\tau^N= \frac{1}{N}\sum_{i=1}^N \delta_{c^i_\tau,w_\tau^{t,i},w^i_\tau,\eta^i_\tau}$
the empirical measure at training time $\tau$ of the NN parameters of our fully-connected NN~\eqref{eq:gN} with a single hidden layer with $N$ neurons and their parameters~$\theta_\tau = (c^i_\tau,w_\tau^{t,i},w^i_\tau,\eta^i_\tau)_{i=1,\dots,N}$.
By computing with chain rule
\begin{equation}
\begin{split}
    \fd{\tau}g_{\theta_\tau}^N(t,x) &= \nabla_\theta g_{\theta_\tau}^N(t,x) \cdot\fd{\tau}\theta_\tau \\
    &= - \alpha^N_\tau \int_0^T\!\!\!\int_D \nabla_\theta g_{\theta_\tau}^N(t,x) \cdot \nabla_\theta g_{\theta_\tau}^N(t',x')\uNhatv{\theta_\tau}{t',x'}\,dx'dt', 
\end{split}
\end{equation}
as done in detail in \eqref{eq:proof:lem:lazytraining:50}--\eqref{eq:proof:lem:lazytraining:53}, we obtain by the fundamental theorem of calculus for the training time evolution of the NN function~$g_{\theta_\tau}^N$ that
\begin{equation}
    \label{eq:parabolicgNtau}
\begin{split}
    g_{\theta_\tau}^N(t,x)
    = g_{\theta_0}^N(t,x) - \int_0^\tau  \alpha_s\big[T_{B(\mu^N_s)}\uNhat{\theta_s}\big](t,x)\,ds
\end{split}
\end{equation}%
with the NN integral operator~$T_{B(\mu)}$ defined as
\begin{equation}
\label{eq:parabolicT_B}
    [T_{B(\mu)}\widehat{u}](t,x)
    = \int_0^T\!\!\!\int_D B(t,x,t',x';\mu)\widehat{u}(t',x') \,dx'dt'
\end{equation}
and
where the symmetric non-local NN kernel (also known as the neural tangent kernel~(NTK)~\cite{jacot2018neural}) is given by
\begin{equation}
\begin{split}
    \label{eq:parabolicB}
    B(t,x,t',x';\mu)
    &= \left\langle k(t,x,t',x';c,w^t,w,\eta),\mu(dc,dw^t,dw,d\eta)\right\rangle
\end{split}
\end{equation}
with
\begin{equation}
\begin{split}
    \label{eq:parabolick}
    k(t,x,t',x';c,w^t,w,\eta)
    &= \sigma(w^tt+w^Tx+\eta)\sigma(w^tt'+w^Tx'+\eta)\\
    &\quad\,+ c^2\sigma'(w^tt+w^Tx+\eta)\sigma'(w^tt'+w^Tx'+\eta)(tt'+x^Tx'+1).
\end{split}
\end{equation}
That means, the NN function~$g_{\theta_\tau}^N$ follows during training the kernel gradient of the least squares loss~\eqref{eq:parabolicJ} using the pre-limit NTK~$B(\mu^N_\tau)$, which is random at initialization and varies during training, as can be seen from \eqref{eq:parabolicgNtau}.
In contrast, in the infinite-width hidden layer limit, i.e., as the number of hidden units $N\rightarrow\infty$ in \eqref{eq:parabolicgNtau},
the kernel becomes deterministic and converges to a limit NTK, which remains constant during training, as can be seen from \eqref{eq:proof:lem:lazytraining:64b}.
This is similar to and in line with the overparameterized training phenomenon~\cite{jacot2018neural,chizat2019lazy} observed for certain scalings in~\eqref{eq:gN},
yet requires, due to the nonlinear PDE setting considered in this manuscript, detailed computations, which we provide in the proof of \Cref{lem:lazytraining} below. \Cref{lem:lazytraining} allows us to represent the limit NN function~$g^*_\tau$ during training by the integro-differential equation\footnote{Note that equation~\eqref{eq:parabolicgtau} can be written after taking the training time derivative equivalently as the infinite-dimensional ODE \mbox{$\fd{\tau} g^*_\tau(t,x) = - \alpha_\tau T_{B_0}\uhatp{\tau} = - \alpha_\tau \int_0^T\!\!\!\int_D B(t,x,t',x';\mu_0)\uhatpp{\tau}{t',x'} \,dx'dt'$} with $\uhatp{\tau}$ depending nonlinearly on $g^*_\tau$ according to the nonlinear PDE system~\eqref{eq:parabolicPDE*}--\eqref{eq:parabolicadjoint*}.}
\begin{equation}
    \label{eq:parabolicgtau}
\begin{split}
    g^*_\tau(t,x)
    &= - \int_0^\tau \alpha_s \big[T_{B_0}\uhatp{s}\big](t,x) \,ds
\end{split}
\end{equation}
with the constant limit NTK $B_0=B(\mu_0)=B(\dummy,\dummy,\dummy,\dummy;\mu_0)$, where $\mu_0$ is the probability distribution for the parameter initialization of the NN.
The representation \eqref{eq:parabolicgtau} of the NN function~$g^*_\tau$ during training reveals a linearization of the NN training dynamics around their initialization. In particular, while the learning rate for individual NN parameters is $\frac{\alpha_\tau}{N^{1-\beta}}$, as can be seen from \eqref{eq:proof:lem:lazytraining:1}, and thus converges to zero as the number of parameters $N\rightarrow\infty$,
the NN function~$g^*_\tau$ itself has the non-zero learning rate~$\alpha_\tau$, as apparent from \eqref{eq:parabolicgtau}. Thus, due to the large number of degrees of freedom in the overparameterized regime, the individual parameters are required to move smaller and smaller distances from their initial locations to achieve a given magnitude change in the neural network output as $N \rightarrow \infty$.

 As the number of hidden units $N\rightarrow\infty$, the NN source term~$g_{\theta_\tau}^N$ of the PDE~\eqref{eq:parabolicPDEN_plain} converges to $g^*_\tau$ while the PDE solution~$\uN{\theta_\tau}$ and the solution to the adjoint PDE~$\uNhat{\theta_\tau}$
converge in $L_2([0,T],H^1(D))$- and $L_\infty([0,T],L_2(D))$-norm to functions $\up{\tau}$ and $\uhatp{\tau}$ solving the PDE system
\begin{alignat}{2}
\label{eq:parabolicPDE*}
\begin{aligned}
    \fpartial{t}\up{\tau} + \CL\up{\tau} - q(\up{\tau})
    &= g^*_\tau
    \qquad&&\text{in }
    D_T, \\
    \up{\tau}
    &= 0
    \qquad&&\text{on }
    [0,T]\times\partial D, \\
    \up{\tau}
    &= f
    \qquad&&\text{on }
    \{0\}\times D,
\end{aligned}
\end{alignat}
and
\begin{alignat}{2}
\label{eq:parabolicadjoint*}
\begin{aligned}
    -\fpartial{t}\uhatp{\tau} + \CL^\dagger\uhatp{\tau} - q_u(\up{\tau})\uhatp{\tau}
    &= (\up{\tau}-h)
    \qquad&&\text{in }
    D_T, \\
    \uhatp{\tau}
    &= 0
    \qquad&&\text{on }
    [0,T]\times\partial D, \\
    \uhatp{\tau}
    &= 0
    \qquad&&\text{on }
    \{T\}\times D,
\end{aligned}
\end{alignat}
which is coupled with the integro-differential equation~\eqref{eq:parabolicgtau} for $g^*_\tau$.
Before making this joint convergence as the number of neurons $N$ tends to infinity mathematically precise in \Cref{lem:lazytraining},
let us address the well-posedness of the PDE system \eqref{eq:parabolicPDE*}--\eqref{eq:parabolicadjoint*} coupled with \eqref{eq:parabolicgtau}.
As we rigorously prove in \Cref{lem:parabolic_wellposedness_infinitewidth},
there exists, for any finite training time horizon~$\CT$, a unique weak solution~$((\up{\tau},\uhatp{\tau}))_{\tau\in[0,\CT]}$ in the sense of \Cref{def:weak_sol,def:weak_sol_adjoint} in the space $\CC\left([0,\CT],\CS\times\CS\right)$.

\begin{theorem}[Overparameterized training regime]
    \label{lem:lazytraining}
    Assume that the learning rate satisfies additionally $\int_{0}^\infty\alpha_\tau^{4/3}\,d\tau<\infty$.
    Let $\CT<\infty$ be a given training time horizon.
    For each $N$, let us denote by $((\uN{\theta_\tau},\uNhat{\theta_\tau}))_{\tau\in[0,\CT]}\in \CC\left([0,\CT],\CS\times\CS\right)$ the unique weak solution to the PDE system \eqref{eq:parabolicPDEN_plain}\,\&\,\eqref{eq:parabolicadjointN_plain} coupled with the gradient descent update~\eqref{eq:GD} in the sense of \Cref{lem:parabolic_wellposedness_N},
    and let us denote by $((\up{\tau},\uhatp{\tau}))_{\tau\in[0,\CT]}\in \CC\left([0,\CT],\CS\times\CS\right)$ the unique weak solution to the PDE system \eqref{eq:parabolicPDE*}--\eqref{eq:parabolicadjoint*} coupled with the integro-differential equation~\eqref{eq:parabolicgtau} in the sense of \Cref{lem:parabolic_wellposedness_infinitewidth}.
    Then, as the number of hidden units $N\rightarrow\infty$,
    \begin{subequations}
        \label{eq:lem:lazytraining}
    \begin{align}
        \sup_{\tau\in[0,\CT]} \bbE \left[\Nbig{\uN{\theta_\tau} - \up{\tau}}_{L_2([0,T],H^1(D))} + \Nbig{\uN{\theta_\tau} - \up{\tau}}_{L_\infty([0,T],L_2(D))}\right]
        & \rightarrow0,\label{eq:lem:lazytraininga} \\
        \sup_{\tau\in[0,\CT]} \bbE \left[\Nbig{\uNhat{\theta_\tau} - \uhatp{\tau}}_{L_2([0,T],H^1(D))} + \Nbig{\uNhat{\theta_\tau} - \uhatp{\tau}}_{L_\infty([0,T],L_2(D))}\right]
        & \rightarrow0, \label{eq:lem:lazytrainingb}\\
        \sup_{\tau\in[0,\CT]} \bbE \Nbig{g_{\theta_\tau}^N - g^*_\tau}_{L_2(D_T)}
        & \rightarrow0. \label{eq:lem:lazytrainingc}
    \end{align}
    \end{subequations}
    Here, the expectation~$\bbE$ is taken w.r.t.\@ to the random initialization of the NN parameters. (The only source of randomness is the random initialization of the NN parameters before training begins.)
\end{theorem}

\Cref{lem:lazytraining} proves the convergence of the solution 
$((\uN{\theta_\tau},\uNhat{\theta_\tau}))_{\tau\in[0,\CT]}$ to $((\up{\tau},\uhatp{\tau}))_{\tau\in[0,\CT]}$ as the number of hidden units $N$ tends to infinity for any finite training time horizon $\CT<\infty$. Therefore, the NN-PDE trained with adjoint gradient descent optimization converges to the solution of the limit PDE system \eqref{eq:parabolicPDE*}--\eqref{eq:parabolicadjoint*} coupled with the integro-differential equation~\eqref{eq:parabolicgtau} as the number of hidden units $N \rightarrow \infty$.
A detailed proof of \Cref{lem:lazytraining} is presented in \Cref{sec:InfiniteWidth}.

\begin{remark}
    \Cref{lem:lazytraining} requires the well-posedness (existence and uniqueness) of both the pre-limit PDE system \eqref{eq:parabolicPDEN_plain}\,\&\,\eqref{eq:parabolicadjointN_plain} coupled with the gradient descent update~\eqref{eq:GD} and the limit PDE system \eqref{eq:parabolicPDE*}--\eqref{eq:parabolicadjoint*} coupled with the integro-differential equation~\eqref{eq:parabolicgtau}. We rigorously state those results in \Cref{lem:parabolic_wellposedness_N,lem:parabolic_wellposedness_infinitewidth}, respectively, with their detailed proofs given in \Cref{sec:WellPosedness_Proof}.

    The additional assumption $\int_{0}^\infty\alpha_\tau^{4/3}\,d\tau<\infty$ on the learning rate, which is satisfied by typical learning rates, is exclusively required for the well-posedness of the pre-limit PDE system \eqref{eq:parabolicPDEN_plain}\,\&\,\eqref{eq:parabolicadjointN_plain} as stated in \Cref{lem:parabolic_wellposedness_N} below and commented on in more details thereafter in \Cref{rem:learning_rate_4/3}. This additional integrability assumption is not used elsewhere in the paper.
\end{remark}

In this large neuron limit,
we show as a second step
that the dynamics~\eqref{eq:parabolicPDE*}--\eqref{eq:parabolicadjoint*} coupled with \eqref{eq:parabolicgtau} converges to a global minimizer of the loss
\begin{equation}
    \label{eq:parabolicJtau}
    \J{\tau}
    = \frac{1}{2} \int_0^T\!\!\!\int_D \left(\upp{\tau}{t,x}-h(t,x)\right)^2 dxdt
\end{equation}
as the training time~$\tau$ tends to infinity.
We thus prove the convergence of the NN-PDE solution~$\up{\tau}$ to the target data~$h$ (i.e., a global minimizer) as $\tau\rightarrow\infty$.
To be more precise, the following main convergence result about the adjoint gradient descent optimization method is proven in this paper.

\begin{theorem}[Global convergence of NN-PDE]
    \label{thm:main}
    Let $((\up{\tau},\uhatp{\tau}))_{\tau\in[0,\infty)}\in \CC\left([0,\infty),\CS\times\CS\right)$ denote the unique weak solution to the PDE system \eqref{eq:parabolicPDE*}--\eqref{eq:parabolicadjoint*} coupled with the integro-differential equation~\eqref{eq:parabolicgtau} in the sense of \Cref{lem:parabolic_wellposedness_infinitewidth} and \Cref{rem:parabolic_wellposedness} on the training time interval $[0,\infty)$.
    Then, the loss $\J{\tau}$ defined in \eqref{eq:parabolicJtau} is monotonically decreasing with $\fd{\tau}\J{\tau} = - \alpha_\tau \Q{\tau} = - \alpha_\tau(\uhatp{\tau},T_{B_0}\uhatp{\tau})_{L_2(D_T)}\leq 0$, and the solution $\up{\tau}$ to \eqref{eq:parabolicPDE*} converges weakly to the target $h$ in $L_2$ as $\tau\rightarrow\infty$, i.e., 
    \begin{equation}
        \up{\tau}\rightharpoonup h
        \text{ in }
        L_2
        \quad
        \text{as }
        \tau\rightarrow\infty.
    \end{equation}
\end{theorem}

The statement follows from \Cref{lem:parabolictimeevolutionJt*,lem:convergence_solution}.
A detailed proof sketch of \Cref{thm:main} is presented in \Cref{sec:convergence}.

\subsection{Contributions}
\label{sec:contributions}

Motivated by the popularity and effectiveness of the adjoint gradient descent optimization method~\eqref{eq:GD}
for training NN-PDE models and thereby solving the inverse problem of learning the neural network modeled terms from observed data as demonstrated in the literature,
our paper develops a rigorous global convergence analysis of this machine learning algorithm for a general class of nonlinear parabolic NN-PDEs of the form~\eqref{eq:parabolicPDEN_plain}.
To calibrate the NN-PDE to available data,
the method trains the NN parameters~$\theta$ embedded within the PDE by running gradient descent on the least squares loss ($L_2$-loss) $\JN{\theta}$ with the gradient being evaluated in a computationally efficient manner by solving an associated adjoint PDE. This is a highly non-convex optimization problem and therefore, for a finite number of hidden units $N$, the trained NN-PDE may only converge to a local minimizer of the objective function. We study the algorithm's asymptotic convergence behavior in the limit where both the number of hidden units~$N$ of the NN $g_\theta^N$ in \eqref{eq:gN} and the training time~$\tau$ in continuous-time gradient descent~\eqref{eq:GD} tend to infinity. First-of-its-kind convergence results to a global minimizer are proven in the nonlinear setting, which go significantly beyond previous analyses that considered much more restrictive classes of linear PDEs, a very restrictive class of objective functions, and a substantially weaker notion of convergence.

Our first result is about the convergence to the infinite-width hidden layer limit as the number~$N$ of neurons tends to infinity.
We prove that, as $N\rightarrow\infty$, the NN function $g_\theta^N$ converges to its infinite-width hidden layer limit $g^*$ in \eqref{eq:parabolicgtau}, which can be represented during training by an integro-differential equation involving a positive definite non-local NN kernel operator~\eqref{eq:parabolicT_B} that remains constant during training but lacks a spectral gap; that is, its eigenvalues do not have a uniform positive lower bound. The NN-PDE solution~$\uN{\theta}$ and the adjoint~$\uNhat{\theta}$, which solve the PDE system \eqref{eq:parabolicPDEN_plain}\,\&\,\eqref{eq:parabolicadjointN_plain} coupled with the gradient descent update~\eqref{eq:GD},
converge to $\up{}$ and $\uhatp{}$ solving the PDE system \eqref{eq:parabolicPDE*}--\eqref{eq:parabolicadjoint*} coupled with the integro-differential equation~\eqref{eq:parabolicgtau}.

Our second result proves global convergence of the trained NN-PDE solution~$\up{\tau}$ to the target data~$h$ (i.e., a global minimizer) as the training time $\tau$ goes to infinity.
Due to the PDE system~\eqref{eq:parabolicPDE*}--\eqref{eq:parabolicadjoint*} coupled with the integro-differential equation~\eqref{eq:parabolicgtau} being both nonlinear and non-local, several mathematical challenges need to be addressed. In particular, due to the nonlinearity of the PDE, training the NN model leads to a non-convex optimization problem in the NN function~$g^*_\tau$ even in the large neuron limit.
This is very different from typical NN limits (e.g., gradient descent training of a standard feedforward fully-connected network) where the training of the infinite-width NN is shown to satisfy the gradient flow of a convex function. 

Furthermore, the aforementioned lack of a spectral gap in the NN kernel of the non-local NN kernel operator in the infinite-width hidden layer limit poses a unique technical complication that is not encountered in finite-dimensional NN convergence analyses (where the eigenvalues of the NN kernel matrix have a positive lower bound).
By showing that the quadratic functional~$\Q{\tau} = (\uhatp{\tau},T_{B_0}\uhatp{\tau})_{L_2(D_T)}$ of the PDE adjoint~$\uhatp{\tau}$,
which appears in the training time derivative $\fd{\tau}\J{\tau}= -\alpha_\tau \Q{\tau}$ of the loss~$\J{}$ and involves the positive definite non-local NN kernel operator~$T_{B_0}$, converges to zero as the training time $\tau\rightarrow\infty$, we establish the weak convergence of the solution~$\uhatp{\tau}$ to the adjoint PDE~\eqref{eq:parabolicadjoint*} to zero as $\tau\rightarrow\infty$. The weak convergence of the adjoint PDE solution can then be used to prove that the original NN-PDE solution~$\up{\tau}$ to the nonlinear PDE~\eqref{eq:parabolicadjoint*} converges weakly to the target data~$h$ as $\tau\rightarrow\infty$. 

To prove that the functional~$\Q{\tau}$ of the adjoint vanishes as the training time~$\tau$ tends to infinity, we apply a cycle of stopping times analysis. This technique crucially requires the development of a novel approach for obtaining a regularity bound for the functional~$\Q{\tau}$ in terms of the learning rate~$\alpha_\tau$, which involves the analysis of an adjoint associated with the functional~$\Q{\tau}$, thus the analysis of an adjoint PDE system of the adjoint PDE~\eqref{eq:parabolicadjoint*}. This is a second-level adjoint system of the original adjoint PDE. In addition, the proof requires carefully establishing uniform (in the training time $\tau$) bounds on several norms of the different adjoint PDEs.

We expect that the developed mathematical methods can be applied to other PDEs and NN architectures in scientific machine learning. For example, we prove a result of independent interest that (strong) limit points of the trained NN-PDE solution are global minimizers of the loss~$\J{}$ for an even more general class of second-order parabolic NN-PDEs.

Numerical studies that illustrate and support our theoretical findings are also presented in the paper.

\subsection{Organization}
\label{sec:organization}

In \Cref{sec:main}, we discuss in detail the main contributions of this paper.
Therefore, after collecting all assumptions made throughout this paper in \Cref{sec:assumptions},
we derive in \Cref{sec:formulagrad} formula~\eqref{eq:nablaJ} for $\nabla_{\theta}\JN{\theta}$, before providing in \Cref{sec:WellPosedness_N,sec:WellPosedness_infinitewidth}, respectively,
well-posedness results for
the NN-PDE training dynamics in both the finite-width hidden layer regime and the infinite-width hidden layer limit, i.e, for 
the PDE system \eqref{eq:parabolicPDEN_plain}\,\&\,\eqref{eq:parabolicadjointN_plain} coupled with the gradient descent update~\eqref{eq:GD}
and
the PDE system \eqref{eq:parabolicPDE*}--\eqref{eq:parabolicadjoint*} coupled with the integro-differential equation \eqref{eq:parabolicgtau}.  Their proofs are provided in \Cref{sec:WellPosedness_Proof}.
Afterwards, we elaborate on and prove in \Cref{sec:InfiniteWidth} our first main theoretical result, \Cref{lem:lazytraining},
which is about the convergence of $(g_{\theta_\tau}^N, \uN{\theta_\tau}, \uNhat{\theta_\tau})$ to their infinite-width hidden layer limit counterparts $(g^*_\tau, \up{\tau}, \uhatp{\tau})$ as the number~$N$ of neurons tends to infinity.
We conclude with \Cref{sec:convergence},
where we discuss and provide an insightful proof sketch of our second main theoretical result, \Cref{thm:main},
which is concerned with the convergence of the limit NN-PDE solution~$\up{\tau}$ to the target data~$h$ during training, i.e., as the training time $\tau\rightarrow\infty$.
Its proof is based on several auxiliary results which we discuss in detail in the thematic \Cref{sec:infinitewidthNN,sec:decayJ,sec:PDEconsiderations,sec:FunctionalCQtau,sec:cyclestoppingtimes,sec:convergences}.

\Cref{sec:experiments} contains numerical examples demonstrating the theoretical results of the paper. We provide the code implementing the adjoint gradient descent optimization method in the GitHub repository
\url{https://github.com/KonstantinRiedl/NNPDEs}.

As discussed, the proofs of the main results are contained in \Cref{sec:infinitewidthNN,sec:decayJ,sec:PDEconsiderations,sec:FunctionalCQtau,sec:cyclestoppingtimes,sec:convergences}.
\Cref{sec:infinitewidthNN} is dedicated to presenting the mathematical tools related to the NN.
In \Cref{sec:decayJ}, we compute the training time derivative of the loss~$\J{\tau}$ and show that $\fd{\tau}\J{\tau}= -\alpha_\tau \Q{\tau}$ with the quadratic functional $\Q{\tau} = (\uhatp{\tau},T_{B_0}\uhatp{\tau})_{L_2(D_T)}$ of the adjoint.
This implies in particular that $\J{\tau}$ is monotonically non-increasing.
Leveraging this property, we provide in \Cref{sec:PDEconsiderations} uniform (in the training time $\tau$) bounds on several norms of the PDE solution $\up{\tau}$ and the adjoint $\uhatp{\tau}$, 
which eventually, by analyzing a second-level adjoint system of the original adjoint PDE in \Cref{sec:FunctionalCQtau}, permit to establish a regularity bound for the functional $\Q{\tau}$ in terms of the learning rate~$\alpha_\tau$.
Adapting a cycle of stopping times analysis \cite{bertsekas2000gradient,sirignano2022online} while leveraging the aforementioned regularity bound,
we eventually prove in \Cref{sec:cyclestoppingtimes} that $\fd{\tau}\J{\tau}= -\alpha_\tau \Q{\tau}$ implies $\Q{\tau}\rightarrow0$ as the training time $\tau\rightarrow\infty$.
With the positive definiteness of the NN kernel operator $T_{B_0}$ we therefrom infer in \Cref{sec:convergences} the weak convergence $\uhatp{\tau}\rightharpoonup 0$ and thus $\up{\tau}\rightharpoonup h$ by definition of the adjoint PDE~\eqref{eq:parabolicadjoint*}.
We conclude \Cref{sec:convergences} by proving as a result of independent interest, that (strong) limit points~$\up{\infty}$ of the trained NN-PDE solution $\up{\tau}$ satisfy $\up{\infty} \equiv h$ a.e.\@, thus being global minimizers of the loss~$\J{}$.

\subsection{Notation}
\label{sec:notation}

We denote by $D\subset\bbR^d$ the spatial domain of the considered parabolic PDE.
Its boundary is $\partial D$.
$T$ denotes the physical time horizon of the PDE.
$D_T := (0,T)\times D$ denotes the time-space domain.
Its lateral surface is $\partial D_T:= [0,T]\times \partial D$, and $\Gamma_T:=\partial D_T \cup \{(t,x):t=0,x\in D\}$.
Moreover, for $\Delta T'>0$, we introduce the notation $D_{T',T'+\Delta T'}:=(T',T'+\Delta T')\times D$.

For a spatial domain $D$, the spaces $L_p(D)$ and $W^k_p(D)$ denote the classical Lebesgue and Sobolev spaces.
They contain all measurable functions $u:D\rightarrow\bbR$ with finite corresponding norm.
For the norms on those spaces it holds $\N{u}^p_{L_p(D)}=\int_D \abs{u(x)}^p dx$ and $\N{u}^p_{W^k_p(D)}=\sum_{\abs{\alpha}\leq k} \N{D^\alpha u}^p_{L_p(D)}$ or $\N{u}_{L_\infty(D)}=\esssup_{x\in D} \abs{u(x)}$ and $\N{u}_{W^k_\infty(D)}=\max_{\abs{\alpha}\leq k} \N{D^\alpha u}_{L_\infty(D)}$ in the case $p=\infty$.
We abbreviate $H^k(D) = W^k_2(D)$ and denote by $H_0^1(D)$ the space of all functions in $H^1(D)$ with zero trace.
$H^{-1}(D)$ denotes the dual space of $H_0^1(D)$.

For a time-space domain $D_T$, the spaces $L_p(D_T)$ denote the classical Lebesgue spaces.
They contain all measurable functions $u:D_T\rightarrow\bbR$ with finite corresponding norm.
For the norms on those spaces it hold $\N{u}^p_{L_p(D_T)}=\int_0^T\!\!\!\int_D \abs{u(t,x)}^p dxdt$ and $\N{u}_{L_\infty(D_T)}=\esssup_{(t,x)\in D_T} \abs{u(t,x)}$ in the case $p=\infty$.

For a function space~$\CX$ on the space $D$, the spaces $L_p([0,T],\CX)$ denote the Bochner spaces.
Let us associate with a function $u:D_T\rightarrow\bbR$ the mapping $\mathbf{u}:[0,T]\rightarrow \CX$ defined by $\mathbf{u}(t) := u(t,\dummy)$.
In what follows we may abuse notation and write $u$ in place of $\mathbf{u}$.
The Bochner spaces contain all strongly (Bochner) measurable functions with finite Bochner norm. 
For those norms it hold $\N{u}_{L_p([0,T],\CX)}^p=\int_0^T \N{u(t,\dummy)}_\CX^p dt$ or $\N{u}_{L_\infty([0,T],\CX)}=\esssup_{t\in[0,T]}\N{u(t,\dummy)}_\CX$ in the case $p=\infty$, see \cite[Section~5.9.2]{evans2010partial}.

A weak solution to the nonlinear parabolic PDE \eqref{eq:parabolicPDEN_plain} in the sense of \cite[Chapter~7]{evans2010partial} is defined as follows.
\begin{definition}[Weak solution of \eqref{eq:parabolicPDEN_plain}]
    \label{def:weak_sol}
    A function $\uN{\theta}\in L_2([0,T],H_0^1(D))$ with weak derivative $\fpartial{t}\uN{\theta}\in L_2([0,T],H^{-1}(D))$ is a weak solution of the PDE \eqref{eq:parabolicPDEN_plain} provided
    \begin{enumerate}[label=(\roman*),labelsep=10pt,leftmargin=35pt]
        \item $\left\langle \fpartial{t}\uNv{\theta}{t,\dummy},v\right\rangle_{H^{-1}(D),H_0^1(D)} + \CB[\uNv{\theta}{t,\dummy},v;t] - (q(\uNv{\theta}{t,\dummy}),v)_{L_2(D)}= \left(g_\theta(t,\dummy),v\right)_{L_2(D)}$
    \end{enumerate}
    for each $v\in H_0^1(D)$ and a.e.\@ time $t\in[0,T]$, where the bilinear form $\CB$ is given by
    \begin{equation}
        \label{eq:CB}
        \CB[u,v;t]
        := \int_U \sum_{i,j=1}^d  a^{ij}(t,x) \fpartial{x_i} u \fpartial{x_j} v
        + \sum_{i=1}^d b^i(t,x) \fpartial{x_i} u v
        + c(t,x)uv \,dx,
    \end{equation}
    and
    \begin{enumerate}[label=(\roman*),labelsep=10pt,leftmargin=35pt]
        \setcounter{enumi}{1}
        \item $\uNv{\theta}{0,\dummy}=f$.
    \end{enumerate}
\end{definition}

Analogously, we define a weak solution to the linear PDE \eqref{eq:parabolicadjointN_plain} as follows.

\begin{definition}[Weak solution of \eqref{eq:parabolicadjointN_plain}]
    \label{def:weak_sol_adjoint}
    A function $\uNhat{\theta}\in L_2([0,T],H_0^1(D))$ with weak derivative $\fpartial{t}\uNhat{\theta}\in L_2([0,T],H^{-1}(D))$ is a weak solution of the adjoint PDE \eqref{eq:parabolicadjointN_plain} (parabolic backward PDE) provided
    \begin{enumerate}[label=(\roman*),labelsep=10pt,leftmargin=35pt]
        \item \label{item:def:weak_sol_adjoint_1}$\left\langle -\fpartial{t}\uNhatv{\theta}{t,\dummy},v\right\rangle_{H^{-1}(D),H_0^1(D)} + \CB^\dagger[\uNhatv{\theta}{t,\dummy},v;t] - (q_u(\uNv{\theta}{t,\dummy})\uNhatv{\theta}{t,\dummy},v)_{L_2(D)}
        $\\ $\phantom{XXX}= \left(\uNv{\theta}{t,\dummy}-h,v\right)_{L_2(D)}$
    \end{enumerate}
    for each $v\in H_0^1(D)$ and a.e.\@ time $t\in[0,T]$, where $\CB^\dagger$ denotes the adjoint bilinear form satisfying $\CB^\dagger[\widehat{u},u;t]=\CB[u,\widehat{u};t]$,
    and
    \begin{enumerate}[label=(\roman*),labelsep=10pt,leftmargin=35pt]
        \setcounter{enumi}{1}
        \item \label{item:def:weak_sol_adjoint_2}$\uNhatv{\theta}{T,\dummy}=0$.
    \end{enumerate}
\end{definition}

Since we investigate the evolution of the PDE solutions to \eqref{eq:parabolicPDEN_plain}\,\&\,\eqref{eq:parabolicadjointN_plain} during training (see \eqref{eq:GD}),
we are interested in their training time trajectories which we denote by $((\uN{\theta_\tau},\uNhat{\theta_\tau}))_{\tau\in[0,\CT]}$.
The function space $\CC\left([0,\CT],\CS\times\CS\right)$ denotes the space of all such continuous trajectories, i.e., the space of all continuous functions mapping from $[0,\CT]$ to $\CS\times\CS$.

By $C$ we typically denote generic constants, which may vary throughout the proof.
To keep the notation concise,
we indicate 
by $\alpha$ their dependency on $\alpha_0$ or $\int_0^\infty \alpha_\tau^2 \,d\tau$ (see \eqref{eq:learning_rate}),
by $D$ their dependency on $\vol{D}$ or $\abs{D}$ (see Assumption~\ref{asm:Dbdd}),
by $\CL$ their dependency on $\nu$, some norms of $a^{ij}$, $b^i$, $c$, as well as their partial space derivatives, or some norms of $f$ and $h$ (see Assumptions~\ref{asm:PDE_L_parabolicity}, \ref{asm:PDE_coefficients}, \ref{asm:WP_coefficients} and \ref{asm:WP_initialcondition}),
by $q$ their dependency on properties of $q$ (see Assumptions~\ref{asm:PDE_nonlinearity_q_ubdd}, \ref{asm:PDE_nonlinearity_q_uubdd}, \ref{asm:WP_nonlinearity_q_growth} and \ref{asm:WP_nonlinearity_q_ubdd}), 
by $\sigma$ their dependency on properties of the NN nonlinearity $\sigma$ (see Assumptions~\ref{asm:NN_sigma} and \ref{asm:NN_sigma'}),
and
by $\mu_0$ their dependency on properties of $\mu_0$ (see Assumption~\ref{asm:NN_mu0}).

\section{Discussion of the Main Results}
\label{sec:main}

This section is dedicated to the discussion of the main theoretical contributions of this paper.

\subsection{Assumptions}
\label{sec:assumptions}

Let us start by stating all assumptions used throughout this manuscript.
We cluster them into assumptions related to the PDE~\eqref{eq:parabolicPDEN_plain}, which we summarize in Assumptions~\ref{def:PDE_assumptions} and \ref{def:WP_assumptions}, and assumptions on the NN listed thereafter in \Cref{def:NN_assumptions}.

\begin{assumption}[Second-order semi-linear parabolic PDE~\eqref{eq:parabolicPDEN_plain}]
    \label{def:PDE_assumptions}
	Throughout we assume that the time horizon~$T$ of the PDE~\eqref{eq:parabolicPDEN_plain} is finite and that the spatial domain~$D\subset\bbR^d$ of the PDE~\eqref{eq:parabolicPDEN_plain}
	\begin{enumerate}[label=A\arabic*,labelsep=10pt,leftmargin=35pt]
        \item\label{asm:D} is an open connected set with a $C^2$ smooth boundary $\partial D$,
        \item\label{asm:Dbdd} has finite volume $\vol{D}$ and is bounded by $\abs{D}$.
    \end{enumerate}
    Moreover, we assume that
	\begin{enumerate}[label=A\arabic*,labelsep=10pt,leftmargin=35pt,]
        \setcounter{enumi}{2}
        \item\label{asm:PDE_L_parabolicity} the parabolic PDE operator~$\fpartial{t}+\CL$ is uniformly parabolic, i.e., there exists $\nu>0$ such that $\sum_{i,j=1}^d a^{ij}(t,x)\xi_i\xi_j \geq \nu\N{\xi}^2$ for all $(t,x)\in \overbar{D_T}$ and $\xi\in\bbR^d$,
        \item\label{asm:PDE_coefficients} the coefficients $a^{ij},b^i,c \in L_\infty(D_T)$ and $\fpartial{x_k}a^{ij},\fpartial{x_k}b^i\in L_\infty(D_T)$,
        \item\label{asm:PDE_nonlinearity_q_ubdd} the nonlinearity~$q$ is such that $\abs{q_u}\leq c_q$ for a constant $c_q>0$,
        \item\label{asm:PDE_nonlinearity_q_uubdd} the nonlinearity~$q$ is such that $\abs{q_{uu}}\leq c'_{q}$ for a constant $c'_{q}>0$.
    \end{enumerate}
\end{assumption}

\begin{assumption}[Well-posedness of second-order semi-linear parabolic PDE~\eqref{eq:parabolicPDEN_plain}]
    \label{def:WP_assumptions}
    Moreover, we assume that
	\begin{enumerate}[label=W\arabic*,labelsep=10pt,leftmargin=35pt,]
        \item\label{asm:WP_coefficients} the coefficients $a^{ij},b^i,c \in L_\infty(D_T)$ and $\fpartial{x_k}a^{ij},\fpartial{x_k}b^i\in L_\infty(D_T)$ are $(\gamma_1/2,\gamma_1)$-Hölder continuous in $(t,x)$ with $\gamma_1>0$,
        \item\label{asm:WP_nonlinearity_q_growth} the nonlinearity~$q$ is $(\gamma_1/2,\gamma_1)$-Hölder continuous in $(t,x)$ with $\gamma_1>0$ and such that $\abs{q(u)}\leq C_q(1+\abs{u})$ for any $u\in\bbR$ and for a constant $C_q>0$,
        \item\label{asm:WP_nonlinearity_q_ubdd} the nonlinearity~$q$ is such that $q_u$ is continuous,
        \item\label{asm:WP_initialcondition} the initial condition $f\in\CC^2(D)$ with $f|_{\partial D} = 0$ is $\gamma_2$-Hölder continuous.
    \end{enumerate}
\end{assumption}

\begin{remark}
    \label{rem:assumptionsWP}
    The conditions of \Cref{def:WP_assumptions} are required only for the well-posedness proof in \Cref{lem:parabolic_wellposedness_N,lem:parabolic_wellposedness_infinitewidth}.
    If well-posedness as below can be shown under a different set of assumptions, these new assumptions would replace \Cref{def:WP_assumptions}.
\end{remark}

\begin{assumption}[Neural network in \eqref{eq:gN}]\label{def:NN_assumptions}
	Throughout we assume that the NN is such that
	\begin{enumerate}[label=B\arabic*,labelsep=10pt,leftmargin=35pt]
        \item\label{asm:NN_sigma} the nonlinearity~$\sigma$ of the NN is non-constant, bounded (i.e., $\abs{\sigma}\leq C_\sigma$), and $L_\sigma$-Lipschitz continuous,
        \item\label{asm:NN_sigma'} the derivative $\sigma'$ of the nonlinearity~$\sigma$ of the NN is bounded (i.e., $\abs{\sigma'} \leq C_{\sigma'}$) and $L_{\sigma'}$-Lipschitz continuous,
        \item\label{asm:NN_mu0} the randomly initialized NN parameters~$\theta_0^i = (c^i_0,w^{t,i}_0,w^i_0,\eta^i_0)$ are i.i.d.\@ and drawn from a distribution $\mu_0\in\CP(\bbR\times\bbR\times\bbR^{d}\times\bbR)$
        which is such that 
        \begin{enumerate}[label=(\roman*),labelsep=10pt,leftmargin=35pt]
            \item $c^i_0$ is independent from $(w^{t,i}_0,w^i_0,\eta^i_0)$, \label{asm:NN_mu0i}
            \item the marginal distribution $\mu_{0,c}$ of $c^i_0$ is mean-zero and compactly supported,  \label{asm:NN_mu0ii}
            \item the marginal distribution $\mu_{0,(w^t,w,\eta)}$ of $(w^{t,i}_0,w^i_0,\eta^i_0)$ has  bounded $k$th-order moments $\CM_{k}(\mu_{0,(w^t,w,\eta)})$ for $k=\max\{4,d+2\}$, \label{asm:NN_mu0iii}
            \item the marginal distribution $\mu_{0,(w^t,w,\eta)}$ of $(w^{t,i}_0,w^i_0,\eta^i_0)$ assigns positive probability to every set with positive Lebesgue measure. \label{asm:NN_mu0iv}
        \end{enumerate}
    \end{enumerate}
\end{assumption}

\subsection{A Computationally Efficient Formula for \texorpdfstring{$\nabla_{\theta}\JN{\theta}$}{the Gradient of the Loss w.r.t. the NN Parameters}}
\label{sec:formulagrad}

As pointed out in the introduction,
the practicability of the adjoint gradient descent method~\eqref{eq:GD} is thanks to an efficient computation of the gradient~$\nabla_{\theta}\JN{\theta}$.
The following result proves \eqref{eq:nablaJ}.
Its proof is given at the end of \Cref{sec:decayJ}.
\begin{lemma}
\label{lem:grad_thetaJ}
    Let $\uN{\theta},\nabla_\theta g_\theta^N \in L_2(D_T)$ and 
    let $\uNhat{\theta}$ denote a weak solution to \eqref{eq:parabolicadjointN_plain} in the sense of \Cref{def:weak_sol_adjoint}, which satisfies $\fpartial{t}\uNhatv{\theta}{t,\dummy}\in L_2(D)$ for a.e.\@ $t\in[0,T]$.
    Define the loss $\JN{\theta}$ as in \eqref{eq:parabolicJ}.
    Then, the gradient $\nabla_\theta\JN{\theta}$ w.r.t.\@ the NN parameters~$\theta$ can be written as
    \eqref{eq:nablaJ}.
\end{lemma}

\subsection{Well-Posedness of the NN-PDE Training Dynamics in the Finite-Width Hidden Layer Regime}
\label{sec:WellPosedness_N}

We first provide a result about the well-posedness of the NN-PDE training dynamics in the finite-width hidden layer regime, i.e, for the PDE system \eqref{eq:parabolicPDEN_plain}\,\&\,\eqref{eq:parabolicadjointN_plain} coupled with the gradient descent update~\eqref{eq:GD} for the NN parameters of the NN function $g_{\theta_\tau}^N$.

\begin{lemma}[Well-posedness of NN-PDE training dynamics \eqref{eq:parabolicPDEN_plain}\,\&\,\eqref{eq:parabolicadjointN_plain}]
    \label{lem:parabolic_wellposedness_N}
    Let $N\in\bbN$ be fixed and let $\theta_0=(\theta_0^i)_{i=1,\dots,N}$ be initialized such that $\theta_0^i\sim\mu_0$ for each $i=1,\dots,N$. 
    Assume that the learning rate satisfies additionally $\int_{0}^\infty\alpha_\tau^{4/3}\,d\tau<\infty$.
    Let $\CT<\infty$ be a given training time horizon.
    Then there exists a unique weak solution
    \begin{equation}
        \left((\uN{\theta_\tau},\uNhat{\theta_\tau})_{\tau\in[0,\CT]}\right)\in \CC\left([0,\CT],\CS\times\CS\right)
    \end{equation}
    to the PDE system \eqref{eq:parabolicPDEN_plain}\,\&\,\eqref{eq:parabolicadjointN_plain} coupled with the gradient descent update~\eqref{eq:GD} in the sense of \Cref{def:weak_sol,def:weak_sol_adjoint}
    which satisfies $(\partial_t\uNv{\theta_\tau}{t,\dummy},\partial_t\uNhatv{\theta_\tau}{t,\dummy})\in L_2(D) \times L_2(D)$
    for a.e.\@ $t\in[0,T]$ and for every $\tau\in[0,\CT]$.
\end{lemma}

The proof is based on a fixed point argument which allows one to decouple the PDE system \eqref{eq:parabolicPDEN_plain}\,\&\,\eqref{eq:parabolicadjointN_plain} from the gradient descent update~\eqref{eq:GD}.
After invoking classical existence results for the nonlinear PDE system \eqref{eq:parabolicPDEN_plain}\,\&\,\eqref{eq:parabolicadjointN_plain} from \cite{ladyzhenskaia1968linear} for given NN parameter updates $\widetilde{c}_\tau^i,\widetilde{w}_\tau^{t,i},\widetilde{w}_\tau^i$, and $\widetilde{\eta}_\tau^i$,
we eventually employ the Banach fixed point theorem to prove well-posedness of the PDE system \eqref{eq:parabolicPDEN_plain}\,\&\,\eqref{eq:parabolicadjointN_plain} coupled with the gradient descent update~\eqref{eq:GD} on a local training time domain, which is eventually extended by a bootstrapping argument.

\begin{remark}
    \label{rem:learning_rate_4/3}
    The additional assumption $\int_{0}^\infty\alpha_\tau^{4/3}\,d\tau<\infty$ on the learning rate required in \Cref{lem:parabolic_wellposedness_N}
    is slightly stronger than \eqref{eq:learning_rate} but satisfied by typical learning rates such as $\alpha_\tau=\frac{1}{1+\tau}$.
    We leverage this assumption in the proof of \Cref{lem:parabolic_wellposedness_N} to deal with the technical challenges arising from the NN integral operator varying during training.
            
    More precisely, when establishing contractivity estimates for the application of the Banach fixed point theorem, this nonlinearity leads to the appearance of higher-order (up to fourth-order) product terms between the NN parameters and the adjoint PDE solution, see, e.g., \eqref{eq:proof:lem:parabolic_wellposedness_N:CONTRACTION}, where the bound on the right-hand side scales as $C (1+M)^4\CT$ with $M$ denoting a bound on the NN parameters and the adjoint PDE solution up to time $\CT$.
    When extending well-posedness by a bootstrapping argument from a training time interval $[0,\CT_{k-1}]$ to $[0,\CT_{k} = \CT_{k-1}+\Delta\CT_k]$,
    this requires us (see particularly \eqref{Eq:ContractivityFfiniteSystem}) to choose time intervals of the form $\Delta\CT_k \propto 1/(C(1+M_k^4))$, more precisely \eqref{eq:proof:lem:parabolic_wellposedness_N:CTk}. 
    One can see from \eqref{eq:proof:lem:parabolic_wellposedness_N:SELFMAPPINGk} that assuming only \eqref{eq:learning_rate} on the learning rate would inevitably lead to the choice $M_k\propto Ck^{1/2}+C(h,f,\theta_0)$. 
    However, in that case, we  notice that the series $\sum_{k=0}^\infty \Delta\CT_k$ is essentially a geometric series, which does not diverge. Therefore,  this would not allow to extend well-posedness to arbitrary time horizons.
    In contrast, the additional assumption on the learning rate allows us to control more stringently the worst-case growth of the NN parameters, the NN function, and the NN-PDE solution during training, see, e.g., \eqref{eq:proof:lem:parabolic_wellposedness_N:SELFMAPPING} and \eqref{eq:proof:lem:parabolic_wellposedness_N:Mk}, where we estimate that they grow as $M_k\propto Ck^{1/4}+C(h,f,\theta_0)$.
    This can be exploited to balance the appearance of the higher-order product terms with the slower growth,
    enabling us to eventually get that the series $\sum_{k=0}^\infty \Delta\CT_k$ is essentially a harmonic series, which diverges, therefore allowing for a global well-posedness result, see, e.g., \eqref{eq:proof:lem:parabolic_wellposedness_N:harmonicseries}.
\end{remark}

Due to its technical nature, the proof is deferred to \Cref{sec:WellPosedness_Proof_N}.

\subsection{Well-Posedness of the NN-PDE Training Dynamics in the Infinite-Width Hidden Layer Limit}
\label{sec:WellPosedness_infinitewidth}

Let us now provide a result about the well-posedness of the NN-PDE training dynamics in the infinite-width hidden layer limit, i.e, for the PDE system \eqref{eq:parabolicPDE*}--\eqref{eq:parabolicadjoint*} coupled with the integro-differential equation~\eqref{eq:parabolicgtau} for $g^*_\tau$.

\begin{lemma}[Well-posedness of NN-PDE training dynamics \eqref{eq:parabolicPDE*}--\eqref{eq:parabolicadjoint*}]
    \label{lem:parabolic_wellposedness_infinitewidth}
    Let $\CT<\infty$ be a given training time horizon.
    Then there exists a unique weak solution
    \begin{equation}
        \left((\up{\tau},\uhatp{\tau})_{\tau\in[0,\CT]}\right)\in \CC\left([0,\CT],\CS\times\CS\right)
    \end{equation}
    to the PDE system \eqref{eq:parabolicPDE*}--\eqref{eq:parabolicadjoint*} coupled with \eqref{eq:parabolicgtau} in the sense of \Cref{def:weak_sol,def:weak_sol_adjoint}
    which satisfies $(\partial_t\upp{\tau}{t,\dummy},\partial_t\uhatpp{\tau}{t,\dummy})\in L_2(D) \times L_2(D)$
    for a.e.\@ $t\in[0,T]$ and for every $\tau\in[0,\CT]$.
\end{lemma}

The proof resembles the one of \Cref{lem:parabolic_wellposedness_N}
and is based again on a fixed point argument which allows decoupling the PDE system \eqref{eq:parabolicPDE*}--\eqref{eq:parabolicadjoint*} from the integro-differential equation~\eqref{eq:parabolicgtau} as before.
After invoking classical existence results for the nonlinear PDE system \eqref{eq:parabolicPDE*}--\eqref{eq:parabolicadjoint*} from \cite{ladyzhenskaia1968linear} for a given right-hand side $\widetilde{g}_\tau$,
we eventually employ the Banach fixed point theorem to prove well-posedness of the PDE system \eqref{eq:parabolicPDE*}--\eqref{eq:parabolicadjoint*} coupled with the integro-differential equation~\eqref{eq:parabolicgtau} on a local training time domain, which is eventually extended by a bootstrapping argument.

Due to its technical nature, the proof is deferred to \Cref{sec:WellPosedness_Proof_infintiewidth}.

\begin{remark}
    \label{rem:parabolic_wellposedness}
    With the statements of \Cref{lem:parabolic_wellposedness_N} and \Cref{lem:parabolic_wellposedness_infinitewidth} being valid for arbitrary training time horizons~$\CT$,
    we can infer well-posedness of the NN-PDE training dynamics \eqref{eq:parabolicPDEN_plain}\,\&\,\eqref{eq:parabolicadjointN_plain} as well as well-posedness of the NN-PDE training dynamics \eqref{eq:parabolicPDE*}--\eqref{eq:parabolicadjoint*} on the infinite training time interval $[0,\infty)$. In particular, it is proven in Step~2e in the proofs of \Cref{lem:parabolic_wellposedness_N} and \Cref{lem:parabolic_wellposedness_infinitewidth} that the Banach fixed point theorem gives existence of the corresponding solution globally in the training time.
\end{remark}

\subsection{Infinite-Width Neural Network Perspective}
\label{sec:InfiniteWidth}

Our first main theoretical result, \Cref{lem:lazytraining}, rigorously
proves that the PDE system \eqref{eq:parabolicPDE*}--\eqref{eq:parabolicadjoint*} coupled with the integro-differential equation~\eqref{eq:parabolicgtau} is indeed the correct limit of the PDE system \eqref{eq:parabolicPDEN_plain}\,\&\,\eqref{eq:parabolicadjointN_plain} coupled with the gradient descent update~\eqref{eq:GD} as the number of neurons $N\rightarrow\infty$.

\begin{proof}[Proof of \Cref{lem:lazytraining}]
    \textbf{Step 1: Boundedness of gradient descent updates.}
    Let us first prove that the gradient descent updates~\eqref{eq:GD} are uniformly bounded in $N$ and in the training time $\tau$ for $\tau\in[0,\CT]$.
    According to formula~\eqref{eq:GD} and using the definition of the learning rate $\alpha^N_\tau$ it holds for continuous-time gradient descent $\fd{\tau} \theta_\tau = - \frac{\alpha_\tau}{N^{1-2\beta}} \nabla_{\theta}\JN{\theta_\tau}$,
    which allows to explicitly derive expressions for $\fd{\tau} c_\tau^i$, $\fd{\tau} w_\tau^{t,i}$, $\fd{\tau} w_\tau^i$, and $\fd{\tau} \eta_\tau^i$.
    With the fundamental theorem of calculus we therefrom infer
    \begin{subequations}
        \label{eq:proof:lem:lazytraining:1}
    \begin{align}
        c_\tau^i
        &= c_0^i - \frac{1}{N^{1-\beta}} \int_0^\tau \alpha_s \int_0^T\!\!\!\int_D   \sigma\big(w^{t,i}_s t + (w^i_s)^T x + \eta_s^i\big)\uNhatv{\theta_s}{t,x}\, dxdtds, \label{eq:proof:lem:lazytraining:1a} \\
        w_\tau^{t,i}
        &= w^{t,i}_0 - \frac{1}{N^{1-\beta}} \int_0^\tau \alpha_s \int_0^T\!\!\!\int_D   c_s^i\sigma'\big(w^{t,i}_s t + (w^i_s)^T x + \eta_s^i\big)t\uNhatv{\theta_s}{t,x}\, dxdtds, \label{eq:proof:lem:lazytraining:1b}\\
        w_\tau^i
        &= w_0^i - \frac{1}{N^{1-\beta}} \int_0^\tau \alpha_s \int_0^T\!\!\!\int_D   c_s^i\sigma'\big(w^{t,i}_s t + (w^i_s)^T x + \eta_s^i\big)x\uNhatv{\theta_s}{t,x}\, dxdtds, \label{eq:proof:lem:lazytraining:1c}\\
        \eta_\tau^i
        &= \eta_0^i - \frac{1}{N^{1-\beta}} \int_0^\tau \alpha_s \int_0^T\!\!\!\int_D   c^i_s\sigma'\big(w^{t,i}_s t + (w^i_s)^T x + \eta_s^i\big)\uNhatv{\theta_s}{t,x}\, dxdtds. \label{eq:proof:lem:lazytraining:1d}
    \end{align}
    \end{subequations}
    Exploiting that $\sigma$ is bounded as of Assumption \ref{asm:NN_sigma} and that the domain $D$ has bounded volume as of Assumption~\ref{asm:Dbdd},
    we can use \eqref{eq:proof:lem:lazytraining:1a} to bound with Cauchy-Schwarz inequality
    \begin{equation}
        \label{eq:proof:lem:lazytraining:2}
    \begin{split}
        \absbig{c_\tau^i - c_0^i}
        &\leq \frac{1}{N^{1-\beta}} \int_0^\tau \alpha_s \sqrt{\int_0^T\!\!\!\int_D  \big(\sigma\big(w^{t,i}_s t + (w^i_s)^T x + \eta_s^i\big)\big)^2\, dxdt} \N{\uNhat{\theta_s}}_{L_2(D_T)} ds \\
        &\leq \frac{C}{N^{1-\beta}} \int_0^\tau \alpha_s\N{\uNhat{\theta_s}}_{L_2(D_T)} ds
    \end{split}
    \end{equation}
    for a constant $C=C(T,D,\sigma)$.
    
    By following the computations of \textit{Step 1c} in the proof of \Cref{lem:parabolic_wellposedness_infinitewidth} in \Cref{sec:WellPosedness_Proof} that lead to \eqref{eq:proof:lem:parabolic_wellposedness_infinitewidth:NORM_utilde} and \eqref{eq:proof:lem:parabolic_wellposedness_infinitewidth:NORM_uhattilde} for the solutions to the PDE system \eqref{eq:parabolicPDEN_plain}\,\&\,\eqref{eq:parabolicadjointN_plain} coupled with the gradient descent update~\eqref{eq:GD},
    we obtain the bounds
    \begin{equation}
        \label{eq:proof:lem:lazytraining:11}
        \N{\uN{\theta_\tau}}_{L_2([0,T],H^1(D))} + \N{\uN{\theta_\tau}}_{L_\infty([0,T],L_2(D))}
        \leq C\left(\N{f}_{L_2(D)} + \N{g_{\theta_\tau}^N}_{L_2(D_T)} + 1\right)
    \end{equation}
    and
    \begin{equation}
        \label{eq:proof:lem:lazytraining:12}
        \N{\uNhat{\theta_\tau}}_{L_2([0,T],H^1(D))} + \N{\uNhat{\theta_\tau}}_{L_\infty([0,T],L_2(D))}
        \leq C\left(\N{\uN{\theta_\tau}}_{L_2(D_T)} + \N{h}_{L_2(D_T)}\right)
    \end{equation}
    for a constant $C=C(T,\CL,q)$ which is in particular independent of $N$.
    Using the definition of the NN \eqref{eq:gN} we can estimate with Jensen's inequality
    \begin{equation}
        \label{eq:proof:lem:lazytraining:13}
    \begin{split}
        \N{g_{\theta_\tau}^N}_{L_2(D_T)}^2
        &= \int_0^T\!\!\!\int_D \left(\frac{1}{N^\beta}\sum_{i=1}^N c^i_\tau\sigma\big(w^{t,i}_\tau t + (w^i_\tau)^Tx + \eta^i_\tau\big)\right)^2 dxdt \\
        &\leq C\frac{1}{N^{2\beta-2}} \frac{1}{N}\sum_{i=1}^N (c^i_\tau)^2
        = C\frac{1}{N^{2\beta-2}} \gamma_\tau^N,
    \end{split}
    \end{equation}
    for $C=C(T,D,\sigma)<\infty$,
    where we used the boundedness of $\sigma$ as of Assumptions \ref{asm:NN_sigma} and that the domain $D$ has bounded volume as of Assumption~\ref{asm:Dbdd}.
    In the last step, we introduced the notation $\gamma_\tau^N := \frac{1}{N}\sum_{i=1}^N (c^i_\tau)^2$.
    Combining \eqref{eq:proof:lem:lazytraining:11}--\eqref{eq:proof:lem:lazytraining:13},
    we end up with the bound
    \begin{equation}
        \label{eq:proof:lem:lazytraining:14}
        \N{\uNhat{\theta_\tau}}_{L_2([0,T],H^1(D))}^2 + \N{\uNhat{\theta_\tau}}_{L_\infty([0,T],L_2(D))}^2
        \leq C\left(\frac{1}{N^{2\beta-2}} \gamma_\tau^N + \N{f}_{L_2(D)}^2 + \N{h}_{L_2(D_T)}^2 + 1\right)
    \end{equation}
    for a constant $C=C(T,D,\CL,q,\sigma)$.

After squaring both sides of \eqref{eq:proof:lem:lazytraining:2} and using Cauchy-Schwarz inequality we obtain
    \begin{equation}
        \label{eq:proof:lem:lazytraining:31}
    \begin{split}
        \absbig{c_\tau^i - c_0^i}^2
        &\leq \frac{C}{N^{2(1-\beta)}} \int^{\tau}_{0}\alpha_s^{2} \,ds \int_0^\tau \N{\uNhat{\theta_s}}_{L_2(D_T)}^2 ds \\
        &\leq C  \int_0^\tau \gamma_s^N\, ds + \frac{C\tau}{N^{2(1-\beta)}}   \left(\N{f}_{L_2(D)}^2 + \N{h}_{L_2(D_T)}^2 + 1\right),
    \end{split}
    \end{equation}
    where we inserted \eqref{eq:proof:lem:lazytraining:14} and used the second part of \eqref{eq:learning_rate} in the second step.
    Summing over $i=1,\dots,N$ and normalizing by $N$ we can bound
    \begin{equation}
        \label{eq:proof:lem:lazytraining:32}
    \begin{split}
        \gamma_\tau^N
        &\leq 2\gamma_0^N + \frac{2}{N}\sum_{i=1}^N \absbig{c_\tau^i - c_0^i}^2 \\
        &\leq 2\gamma_0^N + C \int_0^\tau \gamma_s^N\, ds + \frac{C\tau}{N^{2(1-\beta)}}  \left(\N{f}_{L_2(D)}^2 + \N{h}_{L_2(D_T)}^2 + 1\right).
    \end{split}
    \end{equation}
    Since $\gamma_0^N$ is compactly supported due to Assumption~\ref{asm:NN_mu0},
    an application of Grönwall's inequality gives the estimate
    \begin{equation}
        \label{eq:proof:lem:lazytraining:33}
        \sup_{\tau\in[0,\CT]} \gamma_\tau^N
        \leq C
    \end{equation}
    for a constant $C=C(\alpha,\CT,T,D,\CL,q,\sigma,\mu_0)$ which is in particular independent of $N$.
    Employing \eqref{eq:proof:lem:lazytraining:33} in \eqref{eq:proof:lem:lazytraining:31} shows after using Young's inequality that
    \begin{equation}
        \label{eq:proof:lem:lazytraining:34}
    \begin{split}
        \absbig{c_\tau^i}^2
        &\leq 2\absbig{c_0^i}^2 + C + \frac{C}{N^{2(1-\beta)}}  \left(\N{f}_{L_2(D)}^2 + \N{h}_{L_2(D_T)}^2 + 1\right),
    \end{split}
    \end{equation}
    for some other, potentially larger, constant $C$.
    Recalling that the parameters~$c_0^i$ are initialized with compact support as of  Assumption~\ref{asm:NN_mu0}, \eqref{eq:proof:lem:lazytraining:34} proves that 
    \begin{equation}
        \label{eq:proof:lem:lazytraining:35}
    \begin{split}
        \sup_{N\in\bbN} \sup_{i=1,\dots,N, \tau\in[0,\CT]} \abs{c_\tau^i}
        &\leq C_c,
    \end{split}
    \end{equation}
    for a constant $C_c=C_c(\alpha,\CT,T,D,\CL,q,\sigma,\mu_0)$ which is in particular independent of $N$.

    Leveraging that the NN parameters~$c^i_\tau$ are uniformly bounded and exploiting that $\sigma'$ is bounded as of Assumption \ref{asm:NN_sigma'} and
    that the domain $D$ is bounded as of Assumption~\ref{asm:Dbdd},
    we can use \eqref{eq:proof:lem:lazytraining:1b}--\eqref{eq:proof:lem:lazytraining:1d} to bound (analogously to \eqref{eq:proof:lem:lazytraining:2}) with Cauchy-Schwarz inequality
    \begin{equation}
        \label{eq:proof:lem:lazytraining:41}
    \begin{split}
        \absbig{w_\tau^{t,i} - w^{t,i}_0} + \Nbig{w_\tau^i - w_0^i} + \absbig{\eta_\tau^i - \eta_0^i}
        & \leq \frac{C}{N^{1-\beta}} \int_0^\tau \alpha_s \N{\uNhat{\theta_s}}_{L_2(D_T)} ds
    \end{split}
    \end{equation}
    for a constant $C=C(T,D,\sigma,C_c)$.
    After squaring both sides of \eqref{eq:proof:lem:lazytraining:41} and using Cauchy-Schwarz inequality we obtain
    \begin{equation}
        \label{eq:proof:lem:lazytraining:42}
    \begin{split}
        & \absbig{w_\tau^{t,i} - w^{t,i}_0}^2 + \Nbig{w_\tau^i - w_0^i}^2 + \absbig{\eta_\tau^i - \eta_0^i}^2
        \leq \frac{C}{N^{2(1-\beta)}} \int^{\tau}_{0}\alpha_{s}^{2} \,ds \int_0^\tau \N{\uNhat{\theta_s}}_{L_2(D_T)}^2 ds \\
        &\qquad\, \leq C  \int_0^\tau \gamma_s^N\, ds + \frac{C\tau}{N^{2(1-\beta)}}  \left(\N{f}_{L_2(D)}^2 + \N{h}_{L_2(D_T)}^2 + 1\right),
    \end{split}
    \end{equation}
    where we inserted \eqref{eq:proof:lem:lazytraining:14} and used the second part of \eqref{eq:learning_rate} in the second step.
    Employing \eqref{eq:proof:lem:lazytraining:33} in \eqref{eq:proof:lem:lazytraining:42} shows after using Young's inequality that for $\tau\in[0,\CT]$ with $\CT<\infty$
    \begin{align}
        \absbig{w_\tau^{t,i}}^2
        &\leq 2\absbig{w^{t,i}_0}^2 + C + \frac{C}{N^{2(1-\beta)}}  \left(\N{f}_{L_2(D)}^2 + \N{h}_{L_2(D_T)}^2 + 1\right),\label{eq:proof:lem:lazytraining:34b} \\
        \N{w_\tau^{i}}^2
        &\leq 2\N{w_0^{i}}^2 + C + \frac{C}{N^{2(1-\beta)}}  \left(\N{f}_{L_2(D)}^2 + \N{h}_{L_2(D_T)}^2 + 1\right),\label{eq:proof:lem:lazytraining:34c} \\
        \absbig{\eta_\tau^i}^2
        &\leq 2\absbig{\eta_0^i}^2 + C + \frac{C}{N^{2(1-\beta)}}  \left(\N{f}_{L_2(D)}^2 + \N{h}_{L_2(D_T)}^2 + 1\right),\label{eq:proof:lem:lazytraining:34d}
    \end{align}
    for a constant $C=C(\alpha,\CT,T,D,\CL,q,\mu_0,C_c)$ which is in particular independent of $N$.

    Recalling that the parameters~$w^{t,i}_0$, $w_0^i$, and $\eta_0^i$ are initialized according to the measure~$\mu_0$, whose marginal distribution $\mu_{0,(w^t,w,\eta)}$ of $(w^{t,i}_0,w^i_0,\eta^i_0)$ has bounded moments as of  Assumption~\ref{asm:NN_mu0}\ref{asm:NN_mu0iii}, \eqref{eq:proof:lem:lazytraining:34b}--\eqref{eq:proof:lem:lazytraining:34d} prove that 
    \begin{equation}
        \label{eq:proof:lem:lazytraining:43}
    \begin{split}
        \sup_{N\in\bbN} \sup_{i=1,\dots,N, \tau\in[0,\CT]} \bbE\left[\abs{w_\tau^{t,i}} +\N{w_\tau^{i}}+\abs{\eta_\tau^{i}}\right]
        \leq C_{w,\eta}
    \end{split}
    \end{equation}
    for a constant $C_{w,\eta}=C_{w,\eta}(\alpha,\CT,T,D,\CL,q,\sigma,\mu_0,C_c)$ which is in particular independent of $N$.

    \textbf{Step 2: Boundedness of the NN~\eqref{eq:gN}.}
    Combining the explicit expressions for $\fd{\tau} c_\tau^i$, $\fd{\tau} w_\tau^{t,i}$, $\fd{\tau} w_\tau^i$, and $\fd{\tau} \eta_\tau^i$,
    we obtain for $\fd{\tau}g_{\theta_\tau}^N$ by taking the training time derivative in \eqref{eq:gN} that
    \begin{equation}  \label{eq:proof:lem:lazytraining:50}
    \begin{split}
        \fd{\tau}g_{\theta_\tau}^N(t,x)
        &= \frac{1}{N^\beta} \sum_{i=1}^N \left(\fd{\tau}c_\tau^i\right)\sigma(\star) + c_\tau^i\sigma'(\star)\left(\left(\fd{\tau}w^{t,i}_\tau \right)t + \left(\fd{\tau}w^i_\tau\right)^T x + \fd{\tau}\eta^i_\tau\right)\\
        &= - \frac{\alpha_\tau}{N} \sum_{i=1}^N \int_0^T\!\!\!\int_D \left(\sigma(\star)\sigma(\star') + (c_\tau^i)^2\sigma'(\star)\sigma'(\star')(tt'+x^Tx'+1)\right) \uNhatv{\theta_\tau}{t',x'}\, dx'dt',
    \end{split}
    \end{equation}
    where we abbreviated $\star=w^{t,i}_\tau t + (w^i_\tau)^T x + \eta^i_\tau$ and $\star'=w^{t,i}_\tau t' + (w^i_\tau)^T x' + \eta^i_\tau$.
    Denoting now by $\mu_\tau^N= \frac{1}{N}\sum_{i=1}^N \delta_{c^i_\tau,w_\tau^{t,i},w^i_\tau,\eta^i_\tau}$ the empirical measure at training time $\tau$ of our fully-connected NN~\eqref{eq:gN} with a single hidden layer with $N$ neuron and their parameters~$\theta_\tau = (c^i_\tau,w_\tau^{t,i},w^i_\tau,\eta^i_\tau)_{i=1,\dots,N}$,
    and using the definition of the NN kernel~$B$ from \eqref{eq:parabolicB},
    we can rewrite the formula for $\fd{\tau}g_{\theta_\tau}^N$ in \eqref{eq:proof:lem:lazytraining:50} as
    \begin{equation}
        \label{eq:proof:lem:lazytraining:52}
    \begin{split}
        \fd{\tau}g_{\theta_\tau}^N(t,x)
        &= -\alpha_\tau\int_0^T\!\!\!\int_D B(t,x,t',x';\mu^N_\tau)\uNhatv{\theta_\tau}{t',x'}\,dx'dt'
        =-\alpha_\tau T_{B(\mu^N_\tau)}\uNhat{\theta_\tau},
    \end{split}
    \end{equation}
    where we used the definition of the NN integral operator~$T_{B}$ from \eqref{eq:parabolicT_B} in the last step.
    Simple integration in the training time~$\tau$ yields by the fundamental theorem of calculus
    \begin{equation}
        \label{eq:proof:lem:lazytraining:53}
    \begin{split}
        g_{\theta_\tau}^N(t,x)
        &= g_{\theta_0}^N(t,x) + \int_0^\tau \fd{s}g_{\theta_s}^N(t,x) \,ds
        = g_{\theta_0}^N(t,x) - \int_0^\tau  \alpha_sT_{B(\mu^N_s)}\uNhat{\theta_s}\,ds.
    \end{split}
    \end{equation} 
    
    It is straightforward to see from the definition of the kernel~$B$ in \eqref{eq:parabolicB} that we can bound 
    \begin{equation}
        \label{eq:proof:lem:lazytraining:54}
    \begin{split}
        \N{B(\mu_\tau^N)}_{L_2(D_T\times D_T)}
        &\leq C \left(1 + \frac{1}{N}\sum_{i=1}^N (c^i_\tau)^2\right)
        \leq C 
    \end{split}
    \end{equation}
    for a constant $C=C(T,D,\sigma,C_c)<\infty$ due to the boundedness  Assumptions \ref{asm:NN_sigma},  \ref{asm:NN_sigma'}, and \ref{asm:Dbdd}, and using \eqref{eq:proof:lem:lazytraining:35} in the last step.
    We can use this in \eqref{eq:proof:lem:lazytraining:53} to bound with Young's inequality, Cauchy-Schwarz inequality and the second part of \eqref{eq:learning_rate},
    \begin{equation}    \label{eq:proof:lem:lazytraining:55}
    \begin{split}
        \N{g_{\theta_\tau}^N}_{L_2(D_T)}^2
        &\leq 2\N{g_{\theta_0}^N}_{L_2(D_T)}^2 + 2\N{\int_0^\tau  \alpha_sT_{B(\mu^N_s)}\uNhat{\theta_s}\,ds}_{L_2(D_T)}^2 \\
        &\leq 2\N{g_{\theta_0}^N}_{L_2(D_T)}^2 + 2\int_0^{\tau}\alpha_{s}^{2}\,ds\int_0^\tau  \N{T_{B(\mu^N_s)}\uNhat{\theta_s}}_{L_2(D_T)}^2 ds \\
        &\leq 2\N{g_{\theta_0}^N}_{L_2(D_T)}^2 + C\int_0^\tau  \N{\uNhat{\theta_s}}_{L_2(D_T)}^2 ds
    \end{split}
    \end{equation}
    for a constant $C=C(\alpha,T,D,\sigma,C_c)<\infty$ which is in particular independent of $N$.
    Using the explicit form \eqref{eq:gN} we notice further that for any $p\geq2$ it holds
    \begin{equation}        \label{eq:proof:lem:lazytraining:56}
    \begin{split}
        \bbE\N{g_{\theta_0}^N}_{L_p(D_T)}^p
        & \leq \frac{C(p)}{N^{p\beta}} \bbE \int_0^T\!\!\!\int_D \left(\sum_{i=1}^N \abs{c_0^i\sigma\big(w^{t,i}_0t + (w_0^i)^Tx + \eta_0^i\big)}^2 \right)^{p/2} dxdt \\
        &\leq \frac{1}{N^{p\beta-p/2}} C \leq C
    \end{split}
    \end{equation}
    for a constant $C=C(p,T,D,\sigma,C_c)<\infty$.
    To obtain the first inequality in \eqref{eq:proof:lem:lazytraining:56} we used the Marcinkiewicz-Zygmund inequality with random variables $z^i(t,x) = c_0^i\sigma\big(w^{t,i}_0t + (w_0^i)^Tx + \eta_0^i\big)$, which are independent thanks to the initial independence of the parameters $\theta_0^i = (c^i_0,w^{t,i}_0,w^i_0,\eta^i_0)$ as of Assumption~\ref{asm:NN_mu0}, mean-zero due to the $c^i_0$ having zero mean and being drawn independently from the other parameters as of Assumptions~\ref{asm:NN_mu0}\ref{asm:NN_mu0ii} and \ref{asm:NN_mu0i}, and have finite $p$th moments due to the $c^i_0$'s being compactly supported as of Assumption~\ref{asm:NN_mu0}\ref{asm:NN_mu0ii} together with the boundedness of $\sigma$ from Assumption~\ref{asm:NN_sigma}.
    The last two reasons also justify the second inequality in \eqref{eq:proof:lem:lazytraining:56}, while the third inequality holds since $\beta\in(1/2,1)$.
    With $p=2$, this allows to conclude \eqref{eq:proof:lem:lazytraining:55} after taking the expectation with the estimate
    \begin{equation}        \label{eq:proof:lem:lazytraining:57}
    \begin{split}
        \sup_{s\in[0,\tau]} \bbE\N{g_{\theta_s}^N}_{L_2(D_T)}^2
        &\leq C \left(1 + \int_0^\tau  \bbE\N{\uNhat{\theta_s}}_{L_2(D_T)}^2 ds \right)
    \end{split}
    \end{equation}
    for a constant $C=C(\alpha,T,D,\sigma,C_c)<\infty$ which is in particular independent of $N$.

    Leveraging this estimate we can bound the norm of the adjoint in expectation as
    \begin{equation}
        \label{eq:proof:lem:lazytraining:58}
    \begin{split}
        &\sup_{s\in[0,\tau]} \bbE\left[\N{\uNhat{\theta_s}}_{L_2([0,T],H^1(D))}^2 + \N{\uNhat{\theta_s}}_{L_\infty([0,T],L_2(D))}^2\right]\\
        &\qquad\,\leq C\left(\sup_{s\in[0,\tau]}\bbE\left[\N{\uN{\theta_s}}_{L_2(D_T)}^2\right] + \N{h}_{L_2(D_T)}^2\right)\\
        &\qquad\,\leq C\left(\sup_{s\in[0,\tau]}\bbE\left[\N{g_{\theta_s}^N}_{L_2(D_T)}^2\right] + \N{f}_{L_2(D)}^2 + \N{h}_{L_2(D_T)}^2 + 1\right)\\
        &\qquad\,\leq C\left( \int_0^\tau  \bbE\N{\uNhat{\theta_s}}_{L_2(D_T)}^2 ds + \N{f}_{L_2(D)}^2 + \N{h}_{L_2(D_T)}^2 + 1\right)\\
        &\qquad\,\leq C\left( \int_0^\tau  \bbE\left[\N{\uNhat{\theta_s}}_{L_2([0,T],H^1(D))}^2 \!+\! \N{\uNhat{\theta_s}}_{L_\infty([0,T],L_2(D))}^2\right] ds \!+\! \N{f}_{L_2(D)}^2\!+\!\N{h}_{L_2(D_T)}^2\!+\!1\right)\!,
    \end{split}
    \end{equation}
    where we used the estimates \eqref{eq:proof:lem:lazytraining:12} and \eqref{eq:proof:lem:lazytraining:11} in the first and second inequality, respectively.
    After employing \eqref{eq:proof:lem:lazytraining:57} to obtain the next-to-last line,
    an application of Grönwall's inequality yields the uniform bound
    \begin{equation}
        \label{eq:proof:lem:lazytraining:59}
        \sup_{s\in[0,\tau]} \bbE\left[\N{\uNhat{\theta_s}}_{L_2([0,T],H^1(D))}^2 + \N{\uNhat{\theta_s}}_{L_\infty([0,T],L_2(D))}^2\right]
        \leq C
    \end{equation}
    for a constant $C=C(\alpha,\CT,T,D,\CL,q,\sigma,C_c)<\infty$ which is in particular independent of $N$.
    Herefore, note that $\bbE\big[\Nnormal{\uNhat{\theta_0}}_{L_2([0,T],H^1(D))}^2 + \Nnormal{\uNhat{\theta_0}}_{L_\infty([0,T],L_2(D))}^2\big] \leq C$ for a constant $C=C(T,D,\CL,q,\sigma,C_c)<\infty$ as of \eqref{eq:proof:lem:lazytraining:12} and \eqref{eq:proof:lem:lazytraining:11} together with the fact that $\bbE\Nnormal{g_{\theta_0}^N}_{L_2(D_T)}^2\leq C$ according to \eqref{eq:proof:lem:lazytraining:56}.
    
    \textbf{Step 3: Convergence as $N\rightarrow\infty$.}
    With the adjoint $\uNhat{\theta_\tau}$ being bounded uniformly (in the number $N$ of NN parameters) in expectation as of \eqref{eq:proof:lem:lazytraining:59},
    we immediately derive from \eqref{eq:proof:lem:lazytraining:2} and \eqref{eq:proof:lem:lazytraining:41} after taking the expectation that
    \begin{equation}
        \label{eq:proof:lem:lazytraining:61}
    \begin{split}
        &\sup_{i=1,\dots,N, \tau\in[0,\CT]} \bbE\left[\absbig{c_\tau^i - c_0^i} + \absbig{w_\tau^{t,i} - w^{t,i}_0} + \Nbig{w_\tau^i - w_0^i} + \absbig{\eta_\tau^i - \eta_0^i}\right]
        \leq \frac{C}{N^{1-\beta}} 
    \end{split}
    \end{equation}
    for a constant $C=C(\alpha,\CT,T,D,\CL,q,\sigma,\mu_{0})<\infty$ which is in particular independent of $N$.

    Let us now prove \eqref{eq:lem:lazytraining}.
    Recall that $(\uN{\theta_\tau},\uNhat{\theta_\tau})$ and $(\up{\tau},\uhatp{\tau})$ are solutions to the PDE system \eqref{eq:parabolicPDEN_plain}\,\&\,\eqref{eq:parabolicadjointN_plain} coupled with the gradient descent update~\eqref{eq:GD} and the PDE system \eqref{eq:parabolicPDE*}--\eqref{eq:parabolicadjoint*} coupled with the integro-differential equation~\eqref{eq:parabolicgtau}, respectively.
    Following the computations of \textit{Step 2c} in the proof of \Cref{lem:parabolic_wellposedness_infinitewidth} in \Cref{sec:WellPosedness_Proof} that lead to i.p.\@ \eqref{eq:proof:lem:parabolic_wellposedness_infinitewidth:CONTRACTIVITY_utilde} and  \eqref{eq:proof:lem:parabolic_wellposedness_infinitewidth:CONTRACTIVITYuinfty5} as well as \eqref{eq:proof:lem:parabolic_wellposedness_infinitewidth:CONTRACTIVITY_uhattilde} we obtain the bounds
    \begin{equation}
        \label{eq:proof:lem:lazytraining:62a}
        \N{\uN{\theta_\tau}-\up{\tau}}_{L_2([0,T],H^1(D))} + \N{\uN{\theta_\tau}-\up{\tau}}_{L_\infty([0,T],L_2(D))}
        \leq C\N{g_{\theta_\tau}^N-g^*_\tau}_{L_2(D_T)},
    \end{equation}
    and, with $p=d+2$ which satisfies $p>d+1$, by employing Morrey's inequality
    \begin{equation}
        \label{eq:proof:lem:lazytraining:62b}
        \N{\uN{\theta_\tau}-\up{\tau}}_{L_\infty(D_T)}
        \leq C\N{g_{\theta_\tau}^N-g^*_\tau}_{L_p(D_T)}
    \end{equation}
    as well as,
    \begin{equation}
        \label{eq:proof:lem:lazytraining:62c}
    \begin{split}
        &\Nbig{\uNhat{\theta_\tau}-\uhatp{\tau}}_{L_2([0,T],H^1(D))} + \Nbig{\uNhat{\theta_\tau}-\uhatp{\tau}}_{L_\infty([0,T],L_2(D))} \\
        &\qquad\, \leq C\left(\N{\uN{\theta_\tau}-\up{\tau}}_{L_2(D_T)} + \Nbig{\uhatp{\tau}}_{L_\infty([0,T], L_2(D))}\N{\uN{\theta_\tau}-\up{\tau}}_{L_\infty(D_T)}\right) \\
        &\qquad\, \leq C\left(\N{\uN{\theta_\tau}-\up{\tau}}_{L_2(D_T)} + \N{\uN{\theta_\tau}-\up{\tau}}_{L_\infty(D_T)}\right)
    \end{split}
    \end{equation}
    at the cost of some other, potentially larger, constant $C=C(\alpha,\CT,T,D,\CL,q,\sigma,\mu_{0},C^B_2)<\infty$.
    Here, we used in the last step the fact that $\uhatp{\tau}$ obeys a deterministic bound, which follows after an application of Grönwall's inequality from the chain of inequalities
    \begin{equation}
    \begin{split}
        \Nnormal{\uhatp{\tau}}_{L_2([0,T],H^1(D))}^2 + \Nnormal{\uhatp{\tau}}_{L_\infty([0,T],L_2(D))}^2
        &\leq C\left(\N{\up{\tau}}_{L_2(D_T)}^2 + \N{h}_{L_2(D_T)}^2\right) \\
        &\leq C\left(\N{g^*_\tau}_{L_2(D_T)}^2 + \N{f}_{L_2(D)}^2 + \N{h}_{L_2(D_T)}^2 + 1 \right) \\
        &\leq C\left(\int_0^\tau \N{\uhatp{s}}_{L_2(D_T)}^2 ds \!+\! \N{f}_{L_2(D)}^2 \!+\! \N{h}_{L_2(D_T)}^2 \!+\! 1 \right),
    \end{split}
    \end{equation}
    where the first two inequalities are obtained by following the computations of \textit{Step 1c} in the proof of \Cref{lem:parabolic_wellposedness_infinitewidth} in \Cref{sec:WellPosedness_Proof} that lead to i.p.\@ \eqref{eq:proof:lem:parabolic_wellposedness_infinitewidth:NORM_uhattilde} and \eqref{eq:proof:lem:parabolic_wellposedness_infinitewidth:NORM_utilde}, while the last step holds true since we have by Cauchy-Schwarz inequality, \Cref{lem:parabolicTB} and the second part of \eqref{eq:learning_rate} that
    $\Nnormal{g^*_\tau}_{L_2(D_T)}^2 = \Nbig{\int_0^\tau \alpha_s T_{B_0}\uhatp{s} \,ds}_{L_2(D_T)}^2 \leq \int_0^{\tau}\alpha_{s}^{2}\,ds\int_0^\tau  \Nbig{T_{B_0}\uhatp{s}}_{L_2(D_T)}^2 ds \leq C\int_0^\tau  \Nbig{\uhatp{s}}_{L_2(D_T)}^2 ds$
    for a constant $C=C(\alpha,C^B_2)$.
    
    Since, as we established in \eqref{eq:proof:lem:lazytraining:53}, the NN~$g_{\theta_\tau}^N$ in the source term of the PDE~\eqref{eq:parabolicadjointN_plain} can be represented during training by an integro-differential equation similarly to the representation of $g^*_\tau$ in \eqref{eq:parabolicgtau}, we can estimate with triangle inequality
    \begin{equation}        \label{eq:proof:lem:lazytraining:63}
    \begin{split}
        \N{g_{\theta_\tau}^N-g^*_\tau}_{L_p(D_T)}
        &= \N{g_{\theta_0}^N - \int_0^\tau  \alpha_s\left(T_{B(\mu^N_s)}\uNhat{\theta_s} - T_{B_0}\uhatp{s} \right)  ds}_{L_p(D_T)} \\
        &\leq \N{g_{\theta_0}^N}_{L_p(D_T)} + \int_0^\tau  \alpha_s \N{\left(T_{B(\mu^N_s)} - T_{B_0} \right) \uNhat{\theta_s}}_{L_p(D_T)}  ds \\
        &\phantom{\leq \; \N{g_{\theta_0}^N}_{L_p(D_T)} }
        + \int_0^\tau \alpha_s \N{T_{B_0}\!\left(\uNhat{\theta_s} - \uhatp{s}\right)}_{L_p(D_T)} ds
    \end{split}
    \end{equation}
    and it remains to control each term individually in expected value.
    To estimate the expectation of the first term on the right-hand side of \eqref{eq:proof:lem:lazytraining:63} we use Jensen's inequality recalling that $p\geq1$ since $p = d+2$ and use \eqref{eq:proof:lem:lazytraining:56} which yields
    \begin{equation}
        \label{eq:proof:lem:lazytraining:64a}
    \begin{split}
        \bbE\N{g_{\theta_0}^N}_{L_p(D_T)}
        \leq\left(\bbE\N{g_{\theta_0}^N}_{L_p(D_T)}^p\right)^{1/p}
        \leq \frac{1}{N^{\beta-1/2}} C
    \end{split}
    \end{equation}
    for a constant $C=C(p,T,D,\sigma,C_c)<\infty$ which is in particular independent of $N$.
    To bound the second term of \eqref{eq:proof:lem:lazytraining:63} in expectation we first note that with triangle inequality it holds
    \begin{equation}        \label{eq:proof:lem:lazytraining:64b_prelim}
    \begin{split}
        &\bbE\N{\left(T_{B(\mu^N_s)} - T_{B_0} \right) \uNhat{\theta_s}}_{L_p(D_T)}\\
        &\qquad\,\leq
        \bbE\N{\left(T_{B(\mu^N_s)} - T_{B(\mu^N_0)} \right) \uNhat{\theta_s}}_{L_p(D_T)}
        + \bbE\N{\left(T_{B(\mu^N_0)} - T_{B_0} \right) \uNhat{\theta_s}}_{L_p(D_T)}.
    \end{split}
    \end{equation}
    For the first term in \eqref{eq:proof:lem:lazytraining:64b_prelim} we bound with two applications of Cauchy-Schwarz inequality
    \begin{allowdisplaybreaks}
    \begin{align}
        &\bbE\N{\left(T_{B(\mu^N_s)} - T_{B(\mu^N_0)} \right) \uNhat{\theta_s}}_{L_p(D_T)} \notag \\
        &\qquad\,= \bbE\left(\int_0^T\!\!\!\int_D  \abs{\int_0^T\!\!\!\int_D \!B(t,x,t',x';\mu^N_s\!-\!\mu^N_0) \uNhatv{\theta_s}{t',x'} \,dx'dt'}^p \!dxdt\right)^{1/p} \notag\\
        &\qquad\,\leq \bbE\left[\left(\int_0^T\!\!\!\int_D  \left(\int_0^T\!\!\!\int_D \big(B(t,x,t',x';\mu^N_s\!-\!\mu^N_0)\big)^2 \,dx'dt'\right)^{p/2} dxdt\right)^{1/p} \N{\uNhat{\theta_s}}_{L_2(D_T)}\right] \notag\\
        &\qquad\,\leq \left(\bbE\left(\int_0^T\!\!\!\int_D  \left(\int_0^T\!\!\!\int_D \big(B(t,x,t',x';\mu^N_s\!-\!\mu^N_0)\big)^2 \,dx'dt'\right)^{p/2} \!\!dxdt\right)^{2/p} \right)^{1/2}\!\!\left(\bbE\N{\uNhat{\theta_s}}_{L_2(D_T)}^2\right)^{1/2} \notag \\
        &\qquad\,\leq C\left(\bbE\left(\int_0^T\!\!\!\int_D  \left(\int_0^T\!\!\!\int_D \big(B(t,x,t',x';\mu^N_s\!-\!\mu^N_0)\big)^2 \,dx'dt'\right)^{p/2} dxdt\right)^{2/p} \right)^{1/2}
    \end{align}
    \end{allowdisplaybreaks}
    for a constant $C=C(\alpha,\CT,T,D,\CL,q,\sigma,C_c)$ after using \eqref{eq:proof:lem:lazytraining:59} in the last step to bound the expected value of the norm of the adjoint.
    Observing further after recalling the definition of $B$ from \eqref{eq:parabolicB} that it holds with \Cref{lem:kLipschitz}
    \begin{equation}
    \begin{split}
        B(t,x,t',x';\mu^N_s\!-\!\mu^N_0)
        &\!\leq\! \frac{1}{N}\sum_{i=1}^N L_k(c^i_s,c^i_0) \left(\absnormal{c^i_s\!-\!c^i_0} \!+\! \absnormal{w^{t,i}_s\!-\!w^{t,i}_0} \!+\! \Nnormal{w^i_s\!-\!w^i_0} \!+\! \absnormal{\eta^i_s\!-\!\eta^i_0}\right) \\
        &\!\leq\! C \frac{1}{N}\sum_{i=1}^N  \left(\absnormal{c^i_s\!-\!c^i_0} \!+\! \absnormal{w^{t,i}_s\!-\!w^{t,i}_0} \!+\! \Nnormal{w^i_s\!-\!w^i_0} \!+\! \absnormal{\eta^i_s\!-\!\eta^i_0}\right)
    \end{split}
    \end{equation}
    for a constant $C=C(T,D,\sigma,C_c)$ after using \eqref{eq:proof:lem:lazytraining:35} in the last step to bound $L_k(c^i_s,c^i_0)$, which is quadratic in $c^i_s$ and $c^i_0$ (see \Cref{lem:kLipschitz}),
    we are left with
    \begin{equation}    \label{eq:proof:lem:lazytraining:64b_prelim1}
    \begin{split}
        &\bbE\N{\left(T_{B(\mu^N_s)} - T_{B(\mu^N_0)} \right) \uNhat{\theta_s}}_{L_p(D_T)}\\
        &\qquad\,\leq C \left(\bbE\left(\frac{1}{N}\sum_{i=1}^N \left(\absnormal{c^i_s-c^i_0} + \absnormal{w^{t,i}_s-w^{t,i}_0} + \Nnormal{w^i_s-w^i_0} + \absnormal{\eta^i_s-\eta^i_0}\right)\right)^{2}\right)^{1/2} \\
        &\qquad\,\leq C \left(\bbE\left(\frac{1}{N}\sum_{i=1}^N \left(\frac{1}{N^{1-\beta}} \int_0^\tau \alpha_s\N{\uNhat{\theta_s}}_{L_2(D_T)} ds\right)\right)^{2}\right)^{1/2}\\
        &\qquad\,= \frac{C}{N^{1-\beta}} \left(\bbE\left( \int_0^\tau \alpha_s\N{\uNhat{\theta_s}}_{L_2(D_T)} ds\right)^{2}\right)^{1/2}\\
        &\qquad\,\leq \frac{C}{N^{1-\beta}} \left(\int_0^\tau \alpha_s^2\,ds \, \bbE  \int_0^\tau \N{\uNhat{\theta_s}}_{L_2(D_T)}^2 ds\right)^{1/2} \\
        &\qquad\,\leq \frac{C}{N^{1-\beta}}. 
    \end{split}
    \end{equation}
    for a constant $C=C(\alpha,\CT,T,D,\CL,q,\sigma,\mu_{0},C_c)$ which is in particular independent of $N$, where we used \eqref{eq:proof:lem:lazytraining:2} and \eqref{eq:proof:lem:lazytraining:41} in the second line, Cauchy-Schwarz inequality in the next-to-last step, and the second part of \eqref{eq:learning_rate} together with the bound~\eqref{eq:proof:lem:lazytraining:59} on the expected value of the norm of the adjoint to obtain the last inequality. 
    In order to tackle the second term in \eqref{eq:proof:lem:lazytraining:64b_prelim}, let us first introduce the random variables $Z^i(t,x,t',x') = k(t,x,t',x';c^i_0,w^{t,i}_0,w^i_0,\eta^i_0)-\int k(t,x,t',x';c,w^t,w,\eta)\,d\mu_0(c,w^t,w,\eta)$, which are independent thanks to the initial independence of the parameters $\theta_0^i = (c^i_0,w^{t,i}_0,w^i_0,\eta^i_0)$ as of Assumption~\ref{asm:NN_mu0}, mean-zero, and have finite $p$th moments due to the $c^i_0$'s being compactly supported as of Assumption~\ref{asm:NN_mu0}\ref{asm:NN_mu0ii} together with the boundedness of $\sigma$ and $\sigma'$ from Assumptions~\ref{asm:NN_sigma} and \ref{asm:NN_sigma'} and the boundedness of the domain $D$ as of Assumption~\ref{asm:Dbdd}.
    After taking the expectation we can estimate with  two applications of Cauchy-Schwarz inequality in the second and third step, \eqref{eq:proof:lem:lazytraining:59} in the fourth step to bound the expected value of the norm of the adjoint, and two applications of Jensen's inequality in the last step (once for the expectation in the setting of a concave function and once for the time-space integral in the setting of a convex function at the cost of a constant depending only on $T\vol{D}$ and $p$; herefore, recall that  $p/2\geq1$ and $2/p\leq1$ since $p = d+2$) that
    \begin{equation}
        \label{eq:proof:lem:lazytraining:64b_prelim2_aux1}
    \begin{split}
        &\bbE\N{\left(T_{B(\mu^N_0)} - T_{B_0} \right) \uNhat{\theta_s}}_{L_p(D_T)}
        \\
        &\qquad\,= \bbE\left(\int_0^T\!\!\!\int_D  \abs{\int_0^T\!\!\!\int_D \!B(t,x,t',x';\mu^N_0\!-\!\mu_0) \uNhatv{\theta_s}{t',x'} \,dx'dt'}^p \!dxdt\right)^{1/p} \\
        &\qquad\,\leq \bbE\left[\left(\int_0^T\!\!\!\int_D  \left(\int_0^T\!\!\!\int_D \big(B(t,x,t',x';\mu^N_0\!-\!\mu_0)\big)^2 \,dx'dt'\right)^{p/2} dxdt\right)^{1/p} \N{\uNhat{\theta_s}}_{L_2(D_T)}\right] \\
        &\qquad\,\leq \left(\bbE\left(\int_0^T\!\!\!\int_D  \left(\int_0^T\!\!\!\int_D \big(B(t,x,t',x';\mu^N_0\!-\!\mu_0)\big)^2 \,dx'dt'\right)^{p/2} dxdt\right)^{2/p} \right)^{1/2}\\
        &\qquad\, \qquad\cdot\left(\bbE\N{\uNhat{\theta_s}}_{L_2(D_T)}^2\right)^{1/2} \\
        &\qquad\,\leq C\left(\bbE\left(\int_0^T\!\!\!\int_D  \left(\int_0^T\!\!\!\int_D \big(B(t,x,t',x';\mu^N_0\!-\!\mu_0)\big)^2 \,dx'dt'\right)^{p/2} dxdt\right)^{2/p} \right)^{1/2} \\
        &\qquad\,= \frac{C}{N}\left(\bbE\left(\int_0^T\!\!\!\int_D  \left(\int_0^T\!\!\!\int_D \left(\sum_{i=1}^N Z^i(t,x,t',x')\right)^2 \,dx'dt'\right)^{p/2} dxdt\right)^{2/p} \right)^{1/2} \\
        &\qquad\,\leq \frac{C}{N}\left(\int_0^T\!\!\!\int_D  \int_0^T\!\!\!\int_D \bbE\abs{\sum_{i=1}^N Z^i(t,x,t',x')}^p \,dx'dt' dxdt \right)^{1/p}
    \end{split}
    \end{equation}
    for a constant $C=C(p,\CT,T,D,\CL,q,\mu_0)$.
    We can now employ the Marcinkiewicz-Zygmund inequality with random variables $Z^i$ to obtain
    \begin{equation}
        \label{eq:proof:lem:lazytraining:64b_prelim2_aux2}
    \begin{split}
        &\bbE\N{\left(T_{B(\mu^N_0)} - T_{B_0} \right) \uNhat{\theta_s}}_{L_p(D_T)}
        \leq \frac{C}{N}\left(\int_0^T\!\!\!\int_D \int_0^T\!\!\!\int_D \bbE \left(\sum_{i=1}^N \abs{Z^i(t,x,t',x')}^2\right)^{p/2} \!\!dx'dt'dxdt\right)^{1/p} \\
        &\qquad\,\leq \frac{C}{N} \left(\int_0^T\!\!\!\int_D \int_0^T\!\!\!\int_D \bbE \left[N^{p/2-1}\sum_{i=1}^N \abs{Z^i(t,x,t',x')}^p\right] dx'dt'dxdt\right)^{1/p} \\
        &\qquad\,\leq \frac{C}{N^{1/2}} \left(\int_0^T\!\!\!\int_D \int_0^T\!\!\!\int_D \bbE \abs{Z^1(t,x,t',x')}^p dx'dt'dxdt \right)^{1/p}
    \end{split}
    \end{equation}
    for some other, potentially larger, constant $C$ after using Hölder's inequality in the next-to-last and the fact that the random variables $Z^i(t,x,t',x')$ are identically distributed in the last step.
    It further holds after recalling the definition of the random variables $Z^i(t,x,t',x')$ with \Cref{lem:kLipschitz} that
    \begin{equation}
    \begin{split}
        \label{eq:proof:lem:lazytraining:64b_prelim2}
        &\bbE\N{\left(T_{B(\mu^N_0)} - T_{B_0} \right) \uNhat{\theta_s}}_{L_p(D_T)}  \\
        &\qquad\,\leq \frac{C}{N^{1/2}} \left(\bbE\abs{\int L_k(c^i_0,c) \left(\absnormal{c^i_0\!-\!c} + \absnormal{w^{t,i}_0\!-\!w^{t}} + \Nnormal{w^i_0\!-\!w} + \absnormal{\eta^i_0\!-\!\eta}\right)d\mu_0(c,w^t,w,\eta)}^p\right)^{1/p} \\
        &\qquad\,\leq \frac{C}{N^{1/2}}
    \end{split}
    \end{equation}
    for a constant $C=C(p,\CT,T,D,\CL,q,\sigma,\mu_0)$ which is in particular independent of $N$.
    In the last step, recalling that $L_k(c^i_0,c)$ is quadratic in $c^i_0$ and $c$ (see \Cref{lem:kLipschitz}), we firstly used that the initial condition~$\mu_0$ has a compactly supported marginal distribution $\mu_{0,c}$ as of Assumptions~\ref{asm:NN_mu0}\ref{asm:NN_mu0ii} and that the parameters~$c_0^i\sim\mu_{0,c}$, and secondly that the marginal distribution $\mu_{0,(w^t,w,\eta)}$ has bounded $p$th moments and that the parameters~$(w^{t,i}_0,w^i_0,\eta^i_0)\sim\mu_{0,(w^t,w,\eta)}$.
    Employing \eqref{eq:proof:lem:lazytraining:64b_prelim1} and \eqref{eq:proof:lem:lazytraining:64b_prelim2} in \eqref{eq:proof:lem:lazytraining:64b_prelim} we eventually obtain the bound
    \begin{equation}
        \label{eq:proof:lem:lazytraining:64b}
    \begin{split}
        \bbE\N{\left(T_{B(\mu^N_s)} - T_{B_0} \right) \uNhat{\theta_s}}_{L_p(D_T)}
        &\leq \frac{C}{N^{1-\beta}} + \frac{C}{N^{1/2}}
    \end{split}
    \end{equation}
    for a constant $C=C(p,\alpha,\CT,T,D,\CL,q ,\sigma,\mu_0,C_c)$ which is in particular independent of $N$.
    
    To estimate the last term of \eqref{eq:proof:lem:lazytraining:63} we can directly employ \Cref{lem:TBboundednessLinfty} to obtain
    \begin{equation}
        \label{eq:proof:lem:lazytraining:64c}
    \begin{split}
        \N{T_{B_0}\!\left(\uNhat{\theta_s} - \uhatp{s}\right)}_{L_p(D_T)}
        &= \left(\int_0^T\!\!\!\int_D  \abs{[T_{B_0}(\uNhat{\theta_s} - \uhatp{s}))](t,x)}^p dxdt\right)^{1/p} \\
        &\leq C \N{\uNhat{\theta_s} - \uhatp{s}}_{L_2(D_T)} \\
    \end{split}
    \end{equation}
    for a constant $C=C(p,T,D,C^{T_B}_\infty)$.

    Combining the estimates \eqref{eq:proof:lem:lazytraining:64a}, \eqref{eq:proof:lem:lazytraining:64b} and \eqref{eq:proof:lem:lazytraining:64c}, and plugging them into \eqref{eq:proof:lem:lazytraining:63} after taking the expectation we eventually arrive at
    \begin{equation}
        \label{eq:proof:lem:lazytraining:65}
    \begin{split}
        \bbE\N{g_{\theta_\tau}^N-g^*_\tau}_{L_p(D_T)}
        &\leq C \left(\frac{1}{N^{\beta-1/2}} + \frac{1}{N^{1-\beta}} + \frac{1}{N^{1/2}}+ \int_0^\tau \alpha_s \bbE\N{\uNhat{\theta_s} - \uhatp{s}}_{L_2(D_T)} ds\right).
    \end{split}
    \end{equation}
    Inserting this now into  \eqref{eq:proof:lem:lazytraining:62a} and \eqref{eq:proof:lem:lazytraining:62b} and, consecutively, the results into \eqref{eq:proof:lem:lazytraining:62c}, we get
    \begin{equation}
    \begin{split}
        &\bbE\left[\Nbig{\uNhat{\theta_\tau}-\uhatp{\tau}}_{L_2([0,T],H^1(D))} + \Nbig{\uNhat{\theta_\tau}-\uhatp{\tau}}_{L_\infty([0,T],L_2(D))}\right]\\
        &\qquad\,\leq C\left(\frac{1}{N^{\beta-1/2}} + \frac{1}{N^{1-\beta}} + \frac{1}{N^{1/2}}+ \int_0^\tau \alpha_s \bbE\N{\uNhat{\theta_s} - \uhatp{s}}_{L_2(D_T)} ds\right),
    \end{split}
    \end{equation}
    which yields after an application of Grönwall's inequality
    \begin{equation}
    \begin{split}
        &\bbE\left[\Nbig{\uNhat{\theta_\tau}-\uhatp{\tau}}_{L_2([0,T],H^1(D))} + \Nbig{\uNhat{\theta_\tau}-\uhatp{\tau}}_{L_\infty([0,T],L_2(D))}\right]\\
        &\qquad\,\leq C\left(\frac{1}{N^{\beta-1/2}} + \frac{1}{N^{1-\beta}} + \frac{1}{N^{1/2}}\right) \exp\left(\int_0^\tau \alpha_s\, ds\right).
    \end{split}
    \end{equation}
    Since $\beta\in(1/2,1)$ and $\CT<\infty$, \eqref{eq:lem:lazytrainingb} follows.
    Utilizing this, \eqref{eq:lem:lazytrainingc} follows from \eqref{eq:proof:lem:lazytraining:65}, and \eqref{eq:lem:lazytraininga} eventually follows therefrom with  \eqref{eq:proof:lem:lazytraining:62a} after taking the expected value on both sides of \eqref{eq:proof:lem:lazytraining:62a}.
    This concludes the proof.
\end{proof}

\subsection{Main Convergence Result for the NN-PDE}
\label{sec:convergence}

We are now ready to discuss our second main theoretical result, \Cref{thm:main}, which is about the convergence of the NN-PDE solution~$\up{\tau}$ to the target data~$h$, i.e., a global minimizer of the loss~$\J{}$ defined in \eqref{eq:parabolicJtau}, as the training time $\tau\rightarrow\infty$.

A few comments about \Cref{thm:main} are in order.
Sufficient conditions for the well-posedness (i.e., uniqueness and existence) of a solution $(\up{},\uhatp{})$ to the PDE system \eqref{eq:parabolicPDE*}--\eqref{eq:parabolicadjoint*} coupled with the integro-differential equation~\eqref{eq:parabolicgtau} are provided by \Cref{lem:parabolic_wellposedness_infinitewidth} and \Cref{rem:parabolic_wellposedness}.
While they, and in particular the additional \Cref{def:WP_assumptions}, are sufficient, they may not be necessary and the well-posedness of the system could be guaranteed under another set of assumptions, see \Cref{rem:assumptionsWP}.

\Cref{thm:main} proves the global convergence of the adjoint gradient descent optimization method~\eqref{eq:GD} in the infinite-width NN hidden layer limit as the training time $\tau\rightarrow\infty$.
While it is, to the best of our knowledge, a first-of-its-kind convergence result in the setting of semi-linear (and therefore strictly nonlinear) parabolic PDEs,
we substantially strengthen beyond that the notion of convergence compared to prior results~\cite{sirignano2023pde}, which considered the setting of linear PDEs, see \Cref{rem:comparisonconvergencesirignano2023pde} for more technical details.

Let us now provide a proof sketch of the statement, which gives an outline of the subsequent \Cref{sec:infinitewidthNN,sec:decayJ,sec:PDEconsiderations,sec:FunctionalCQtau,sec:cyclestoppingtimes,sec:convergences} comprising the central steps involved in the proof of \Cref{thm:main}.

\vspace{0.2cm}

\begin{proof}[Proof sketch of \Cref{thm:main}]
    \textbf{Properties of the infinite-width NN (\Cref{sec:infinitewidthNN}).}
    The training time derivative of the PDE right-hand side~$g^*_\tau=-\int_0^\tau \alpha_s T_{B_0}\uhatp{s} \,ds$ given in \eqref{eq:parabolicgtau} is $\fd{\tau}g^*_\tau = - \alpha_\tau T_{B_0}\uhatp{\tau}$.
    Due to the NN kernel operator $T_{B_0}$ being a Hilbert-Schmidt operator as of \Cref{rem:TBHilbertSchmidt} and \Cref{lem:parabolicTB},
    $T_{B_0}\uhatp{\tau}\in L_2(D_T)$ for every $\tau$ and $\N{T_{B_0}\uhatp{\tau}}_{L_2(D_T)}\leq C^B_2 \N{\uhatp{\tau}}_{L_2(D_T)}$.
    Leveraging that the NN kernel $B_0$ is uniformly bounded in $L_\infty$ as a consequence of Assumption~\ref{def:NN_assumptions} on the NN architecture, we further show in \Cref{lem:TBboundednessLinfty} that $T_{B_0}\uhatp{\tau}\in L_\infty(D_T)$ for every $\tau$ and $\N{T_{B_0}\uhatp{\tau}}_{L_\infty(D_T)}\leq C^{T_B}_\infty \N{\uhatp{\tau}}_{L_2(D_T)}$.
    Furthermore, $T_{B_0}$ is positive definite according to \Cref{lem:parabolicposdefTB} and its eigenfunctions form an orthonormal basis of $L_2(D_T)$ as of \Cref{lem:parabolicTB}.
    
    \textbf{Step 1: Decay of the loss $\J{}$ (\Cref{sec:decayJ}).}
    With chain rule and by leveraging the adjoint PDE~\eqref{eq:parabolicadjoint*},
    we obtain in \Cref{lem:parabolictimeevolutionJt*} for the training time derivative $\fd{\tau}\J{\tau}$ of the loss $\J{}$ defined in \eqref{eq:parabolicJtau} with partial integration that
    \begin{equation}
        \label{eq:proof:main:decayCJ}
    \begin{split}
        \fd{\tau}\J{\tau}
        &= \int_0^T\!\!\!\int_D \left(\upp{\tau}{t,x}-h(t,x)\right)\fd{\tau}\upp{\tau}{t,x}\, dxdt \\
        &= \int_0^T\!\!\!\int_D \uhatpp{\tau}{t,x}\fd{\tau}g^*_\tau(t,x)\, dxdt
        = - \alpha_\tau(\uhatp{\tau},T_{B_0}\uhatp{\tau})_{L_2}
        = - \alpha_\tau\Q{\tau},
    \end{split}
    \end{equation}
    where we used that $\fd{\tau}g^*_\tau = - \alpha_\tau T_{B_0}\uhatp{\tau}$ in the next-to-last step and the definition~\eqref{eq:CQ} of the functional $\Q{\tau}=(\uhatp{\tau},T_{B_0}\uhatp{\tau})_{L_2}$ in the last step.

    \textbf{Step 2: Cycle of stopping times analysis (\Cref{sec:cyclestoppingtimes}).}
    Following the frameworks of \cite{bertsekas2000gradient,sirignano2022online} on gradient convergence in gradient methods,
    we prove in \Cref{lem:convergenceCJCB} that
    \begin{equation}
        \label{eq:proof:main:convergenceQ}
        \lim_{\tau\rightarrow\infty} \Q{\tau}
        = 0.
    \end{equation}
    It is immediate to observe that $\liminf_{\tau\rightarrow\infty} \Q{\tau} = 0$.
    Namely, if there existed an $\varepsilon>0$ such that $\Q{\tau}\geq\varepsilon$ for all $\tau\geq \overbar{\tau}$,
    we would have had by \eqref{eq:proof:main:decayCJ} and the fundamental theorem of calculus that $\J{\tau} = \J{\overbarscript{\tau}} - \int_{\overbarscript{\tau}}^\tau \alpha_s \Q{s} \,ds \leq \J{\overbarscript{\tau}} - \varepsilon\int_{\overbarscript{\tau}}^\tau \alpha_s \,ds \rightarrow - \infty$ as $\tau\rightarrow\infty$ due to condition~\eqref{eq:learning_rate} on the learning rate $\alpha_\tau$.
    This contradicts the positivity of the loss~$\J{}$.
    Thus, the case that the functional $\Q{\tau}$ is larger than some $\varepsilon$ for all but a finite amount of time cannot occur.
    However, it remains to outrule the case that the functional $\Q{\tau}$ spikes above $\varepsilon$ forever, while being small most of the time.
    Let us therefore bring $\limsup_{\tau\rightarrow\infty} \Q{\tau} > 0$ to the contradiction.
    To this end, assume that there exists an $\varepsilon>0$ such that $\Q{\tau}<\varepsilon/2$ for infinitely many $\tau$'s as well as $\Q{\tau}>\varepsilon$ for infinitely many $\tau$'s.
    Then there exists an infinite cycle of stopping times
    \begin{equation}
    \begin{split}
        0=\sigma_0\leq\tau_1\leq\sigma_1\leq\tau_2\leq\sigma_2\leq\tau_3\leq\dots,
    \end{split}
    \end{equation}
    with $\tau_k$ and $\sigma_k$ being defined for $k=1,2,\dots$ according to
    \begin{equation}
    \begin{split}
        \tau_k
        &= \inf\left\{\tau>\sigma_{k-1}: \Q{\tau}\geq\varepsilon\right\}\\
        \sigma_k
        &= \sup\bigg\{\tau\geq\tau_k: \frac{1}{2}\Q{\tau_k} \leq \Q{s} \leq 2\Q{\tau_k} \text{ for all }s\in[\tau_k,\tau] \text{ and }\int_{\tau_k}^{\tau}\alpha_s\,ds\leq \frac{\varepsilon}{2L_\CQ}\bigg\},
    \end{split}
    \end{equation}
    where $L_\CQ>0$ will be defined in \textit{Step 3}.
    By a telescopic sum argument, we have for sufficiently large $\widetilde{n}$ and for all $n\geq \widetilde{n}$ that
    \begin{equation}
        \label{eq:proof:main:telescopicsum}
    \begin{split}
        \J{\tau_{n+1}}
        &= \J{\tau_{\widetilde{n}}} + \sum_{k=\widetilde{n}}^n \big(\J{\tau_{k+1}} - \J{\tau_k}\big)
        = \J{\tau_{\widetilde{n}}} + \sum_{k=\widetilde{n}}^n \Big[\underbrace{\big(\J{\tau_{k+1}} - \J{\sigma_k}\big)}_{\leq 0} + \underbrace{\big(\J{\sigma_k} - \J{\tau_k}\big)}_{\!\!\!\!\!\!\leq - (1-\vartheta)\varepsilon^2/(4 L_\CQ)\!\!\!\!\!\!}\Big]
        \rightarrow - \infty
    \end{split}
    \end{equation}
    as $n\rightarrow\infty$, which is again a contradiction as the loss~$\J{}$ is positive.
    Thus, \eqref{eq:proof:main:convergenceQ} holds.
    The properties under the brackets are derived as follows.
    \begin{itemize}
        \item On the intervals $I^1_{k+1}=[\sigma_{k},\tau_{k+1})$, where $\Q{\tau}\leq\varepsilon$, i.e., where $\Q{\tau}$ is negligibly small, we just show $\J{\tau_{k+1}}-\J{\sigma_{k}} \leq 0$ by using \eqref{eq:proof:main:decayCJ}, the fundamental theorem of calculus and the positive definiteness of $T_{B_0}$.
        \item On the intervals $I^2_k=[\tau_k,\sigma_k)$, on the other hand, where $\Q{\tau_k}/2\leq\Q{\tau}\leq2\Q{\tau_k}$, it holds $\J{\sigma_{k}}-\J{\tau_{k}}\leq - (1-\vartheta)\varepsilon^2/(4 L_\CQ)$ for any $\vartheta\in(0,1)$ as we prove in detail in \Cref{lem:convergenceCJCB}.
        The intuition behind this bound is that on those intervals $\Q{\tau}\geq \Q{\tau_k}/2 \geq\varepsilon/2$, i.e, $\Q{\tau}$ is non-negligibly large,
        while, at the same time, enough training progress is made in the sense that $\int_{\tau_k}^{\sigma_k} \alpha_\tau \, d\tau \geq (1-\vartheta)\varepsilon/(2L_{\CQ})$.
        The former is by definition of the stopping times.
        The latter is either ensured by the definition of the stopping time~$\sigma_k$ or guaranteed, as proven in \Cref{lem:auxcontradiction}, by a regularity bound~\eqref{eq:proof:main:RegularityCB} for the functional~$\Q{\tau}$ in terms of the learning rate~$\alpha_\tau$, which we derive in \textit{Step 3}.
        It allows to lower bound in this case $\int_{\tau_k}^{\sigma_k} \alpha_\tau \, d\tau$ by the change of the functional $\Q{\tau_k}$, which is (up to an arbitrarily small factor $\vartheta$) at least $\varepsilon/2$ on the intervals $I^2_k$.
    \end{itemize}

    \textbf{Step 3: PDE considerations and a regularity bound for the functional $\Q{\tau}$ in terms of the learning rate (\Cref{sec:PDEconsiderations,sec:FunctionalCQtau}).}
    A crucial property of the functional $\Q{\tau}$ in the preceding argument is the regularity bound
    \begin{equation}
        \label{eq:proof:main:RegularityCB}
        \absbig{\Q{\tau_2}-\Q{\tau_1}}
        \leq L_\CQ\int_{\tau_1}^{\tau_2} \alpha_\tau \,d\tau,
    \end{equation}
    which holds for all $0\leq\tau_1\leq\tau_2$.
    To prove \eqref{eq:proof:main:RegularityCB},
    we develop a novel approach in \Cref{lem:CQ_Regulartiy}.
    To this end,
    let us introduce for $\Q{\tau}$ and the coupled PDE system~\eqref{eq:parabolicPDE*}--\eqref{eq:parabolicadjoint*} the second-level adjoint system \eqref{eq:paraboliccoupledadjoint1}--\eqref{eq:paraboliccoupledadjoint2} with variables $(\vhatp{\tau},\whatp{\tau})$.
    With the fundamental theorem of calculus and by leveraging those adjoint PDEs, we can represent
    \begin{equation}
        \Q{\tau_2}-\Q{\tau_1}
        = \int_{\tau_1}^{\tau_2} \fd{\tau} \Q{\tau}\,d\tau
        = \int_{\tau_1}^{\tau_2} \left(\fd{\tau}g^*_\tau,\vhatp{\tau}\right)_{L_2(D_T)}d\tau
        = - \int_{\tau_1}^{\tau_2} \alpha_\tau \left(T_{B_0}\uhatp{\tau},\vhatp{\tau}\right)_{L_2(D_T)}d\tau
    \end{equation}
    and it remains to employ Cauchy-Schwarz inequality and to derive uniform (in the training time $\tau$) $L_2$ bounds on $\uhatp{\tau}$ and $\vhatp{\tau}$.
    In \Cref{lem:parabolic_uhatL2} we establish such bound for $\uhatp{\tau}$, which is a consequence of an energy estimate and $\J{\tau}$ being monotonically non-increasing as of  \Cref{lem:parabolictimeevolutionJt*}.
    In \Cref{lem:parabolic_vhatL2} we prove that also $\vhatp{\tau}$ is uniformly (in the training time $\tau$) bounded in $L_2$.
    While the proof is again based on an energy estimate,
    the technical difficulty arises from the structure of the source term in \eqref{eq:paraboliccoupledadjoint1} which is of the form $\whatp{\tau} + q_{uu}(\up{\tau})\uhatp{\tau}\whatp{\tau}$.
    In order to control this term in $L_2$, we establish in \Cref{lem:parabolic_whatLinfty} uniform (in the training time $\tau$) estimates on the $L_2$ and $L_\infty$ norms of the second-level adjoint $\whatp{\tau}$.
    That we are in particular able to obtain an $L_\infty$ bound is a consequence of the source term~$2T_{B_0}\uhatp{\tau}$ in \eqref{eq:paraboliccoupledadjoint2} being in $L_\infty$ by \Cref{lem:TBboundednessLinfty}.
    Cauchy-Schwarz inequality now yields
    \begin{equation}
    \begin{split}
        \abs{\Q{\tau_2}-\Q{\tau_1}}
        &\leq \int_{\tau_1}^{\tau_2} \alpha_\tau\N{T_{B_0}\uhatp{\tau}}_{L_2(D_T)}\N{\vhatp{\tau}}_{L_2(D_T)}d\tau
        \leq L_\CQ\int_{\tau_1}^{\tau_2} \alpha_\tau \,d\tau
    \end{split}
    \end{equation}
    with $L_\CQ=C^B_2C^{\widehat{u}}C^{\widehat{v}}$ after employing \Cref{lem:parabolicTB,lem:parabolic_uhatL2} and \Cref{lem:parabolic_vhatL2} in the last step.

    \textbf{Step 4: Convergence of the adjoint $\uhatp{\tau}$ and the solution $\up{\tau}$ (\Cref{sec:convergences}).}
    Since the functional $\Q{\tau} = (\uhatp{\tau},T_{B_0}\uhatp{\tau})_{L_2(D_T)}$ converges to zero as $\tau\rightarrow\infty$ according to \eqref{eq:proof:main:convergenceQ} and since the NN kernel operator~$T_{B_0}$ is positive definite as of \Cref{lem:parabolicposdefTB},
    the adjoint $\uhatp{\tau}$ in \eqref{eq:parabolicadjoint*} converges to zero weakly in $L_2$ as $\tau\rightarrow\infty$ as we prove in \Cref{lem:convergence_adjoint}.
    Leveraging the structure of the adjoint PDE \eqref{eq:parabolicadjointN_plain},
    we infer therefrom in \Cref{lem:convergence_solution} that the solution $\up{\tau}$ in \eqref{eq:parabolicPDE*} converges to the target data $h$ weakly in $L_2$,
    which concludes the proof.
\end{proof}

Before turning the focus of the manuscript to the numerical experiments and the proof details thereafter, let us draw in the following remark an analogy to the convergence analysis of gradient descent methods in the setting of finite-dimensional optimization problems, which highlights the challenges arising from our infinite-dimensional PDE-based setting.
\begin{remark}
    To minimize a finite-dimensional objective function or loss $J:\bbR^d\rightarrow\bbR$, we run gradient descent $\fd{\tau}\theta_\tau = -\alpha_\tau \nabla_\theta J(\theta_\tau)$ with learning rate~$\alpha_\tau$.
    
    We can compute with chain rule that $\fd{\tau} J(\theta_\tau) = \nabla_\theta J(\theta_\tau) \fd{\tau}\theta_\tau = -\alpha_\tau \N{\nabla_\theta J(\theta_\tau)}^2$, cf.\@ \eqref{eq:proof:main:decayCJ}, showing the decay of the loss~$J$.
    Assuming that the loss $J$ is $L_J$-smooth, i.e., has a $L_J$-Lipschitz continuous gradient, and that $\nabla_\theta J$ is bounded by $C_{\nabla J}$,
    we can easily verify that $\absbig{\N{\nabla_\theta J(\theta_{\tau_2})}^2-\N{\nabla_\theta J(\theta_{\tau_1})}^2} \leq 2C_{\nabla J}\N{\nabla_\theta J(\theta_{\tau_2})-\nabla_\theta J(\theta_{\tau_1})} \leq 2C_{\nabla J} L_J \N{\theta_{\tau_2}-\theta_{\tau_1}} \leq 2C_{\nabla J} L_J \int_{\tau_1}^{\tau_2} \alpha_\tau \N{\nabla_{\theta} J(\theta_\tau)} d\tau  \leq 2C_{\nabla J}^2 L_J \int_{\tau_1}^{\tau_2} \alpha_\tau \,d\tau$,
    which matches our regularity bound in terms of the learning rate, cf.\@ \eqref{eq:proof:main:RegularityCB}.
    Note that in that case, one arrives at the same statement albeit in a considerably more straightforward way compared to how we derived \eqref{eq:proof:main:RegularityCB} in the infinite dimensional setting studied in our paper.
    Then, following the same steps as in the cycle of stopping times analysis from above, we can  infer that $\lim_{\tau\rightarrow\infty} \N{\nabla_{\theta} J(\theta_\tau)}
    = 0$, cf.\@ \eqref{eq:proof:main:convergenceQ}, i.e., that gradient descent converges to a stationary point.
\end{remark}

\section{Numerical Experiments}
\label{sec:experiments}
Several numerical studies are presented below that illustrate and support our theoretical findings. 
We consider the following
second-order parabolic partial differential equation
\begin{alignat}{2}
\label{eq:heatequation}
\begin{aligned}
    \fpartial{t}u - 0.01 \Delta u - q(u)
    &= g_{\text{target}}
    \qquad&&\text{in }
    [0,T] \times D, \\
    u
    &= 0
    \qquad&&\text{on }
    [0,T]\times\partial D, \\
    u
    &= 0.2 \sin(4\pi x)\sin(2\pi y)
    \qquad&&\text{on }
    \{0\}\times D,
\end{aligned}
\end{alignat}
on the spatial domain $D = [0,0.5] \times [0,1]$ and with time horizon $T=1$.
Two scenarios, namely
\begin{enumerate}[label=(\roman*),labelsep=10pt,leftmargin=35pt]
    \item the (linear) heat equation, i.e., the case where $q\equiv0$, and
    \item the (nonlinear) Allen-Cahn equation with nonlinear term $q(u)=u^3-u$
\end{enumerate}
are investigated.
In either case, the target source term, which is unknown to the practitioner, is given by $g_{\text{target}}(t,x,y) = 1600 x(1-2x)y^2(0.2 + 0.6t-y)^2(1-y)^2$.
To learn it from data by running the adjoint gradient descent optimization method~\eqref{eq:GD} with the gradient being computed according to \eqref{eq:nablaJ},
we model $g_{\text{target}}$ by an NN $g_\theta^N$ with $N$ neurons of the form \eqref{eq:gN}.
As hyperparameter we choose $\beta=2/3$, as activation function~$\sigma$ we choose the $\tanh$, and the NN parameters are initialized according to $c^i_0\sim \CU([-1,1])$, $w^{t,i}_0\sim\CN(0,1)$, $w^{i}_0\sim\CN(0,\text{Id}_{2\times2})$ and $\eta^{i}_0\sim\CN(0,1)$ for all $i\in\{1,\dots,N\}$, which is in accordance with \Cref{def:NN_assumptions}.

We use the \texttt{Adam} optimizer with hyperparameters $\beta_1=0.9$ and $\beta_2=0.999$.
The learning rate is initially set to $\alpha^N_0=\frac{0.01}{N^{1-2\beta}}$ and decreased adaptively by the \texttt{ReduceLROnPlateau} scheduler with factor $0.95$ and patience $100$, which we gradually reduce during training.
For numerical stability, the gradients are clipped using \texttt{ZClip} \cite{kumar2025zclip}, an algorithm for robust gradient norm statistics estimation, which employs z-score-based anomaly detection and leverages exponential moving averages.
For its hyperparameters, we choose a smoothing factor of $\alpha=0.98$ and a z-score threshold of $0.4$. Qualitatively similar results can be obtained with other optimizers such as \texttt{RMSprop} or \texttt{SGD}.

\begin{figure}[p!]
\centering
\valign{#\cr
  \hsize=0.45\columnwidth
  \begin{subfigure}{0.7\columnwidth}
  \centering
  \includegraphics[width=.9\textwidth, height=.54\textwidth]{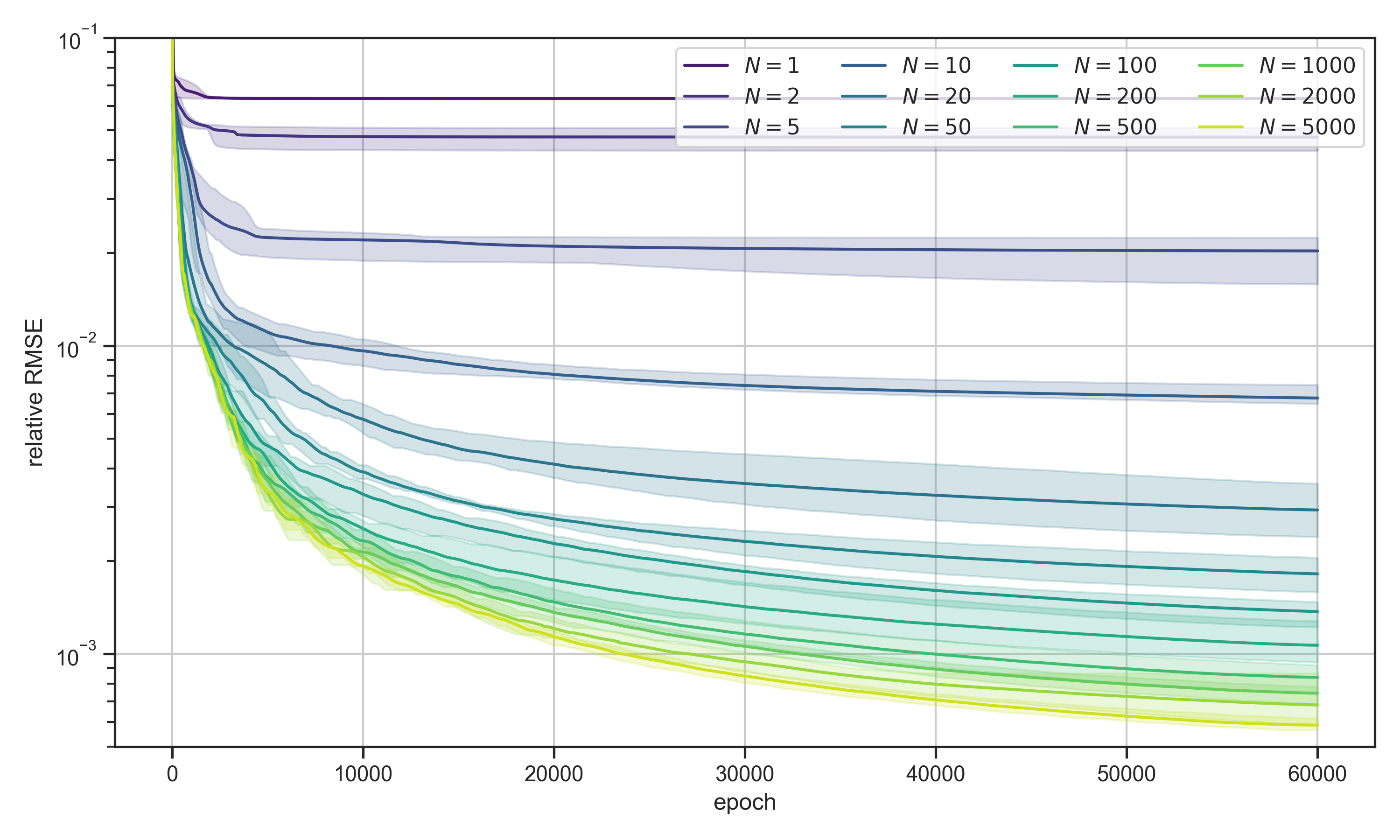}
  \caption{\footnotesize Best RMSE for $N\in\{1,2,5,10,20,50,100,200,500,1000,2000,5000\}$.}
  \label{fig:HeatEquation_a}
  \end{subfigure}\cr\noalign{\hfill}
  \hspace{0.1\columnwidth}
  \hsize=0.45\columnwidth
  \begin{subfigure}{0.3\columnwidth}
  \centering
  \includegraphics[width=.9\textwidth, height=.54\textwidth]{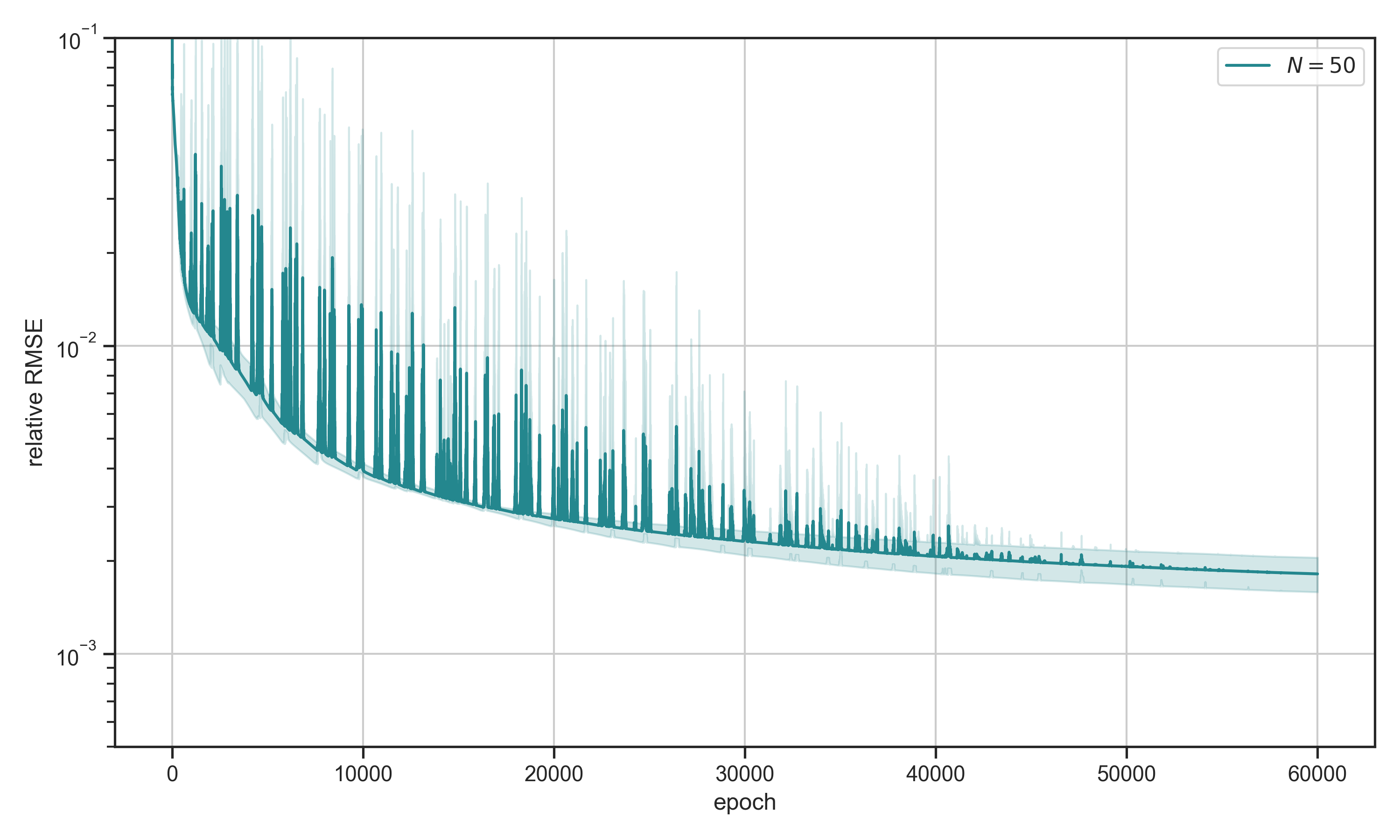}
  \caption{\footnotesize RMSE for $N=50$.}
  \label{fig:HeatEquation_b}
  \end{subfigure}\vfill
  \hspace{0.1\columnwidth}
  \begin{subfigure}{0.3\columnwidth}
  \centering
  \includegraphics[width=.9\textwidth, height=.54\textwidth]{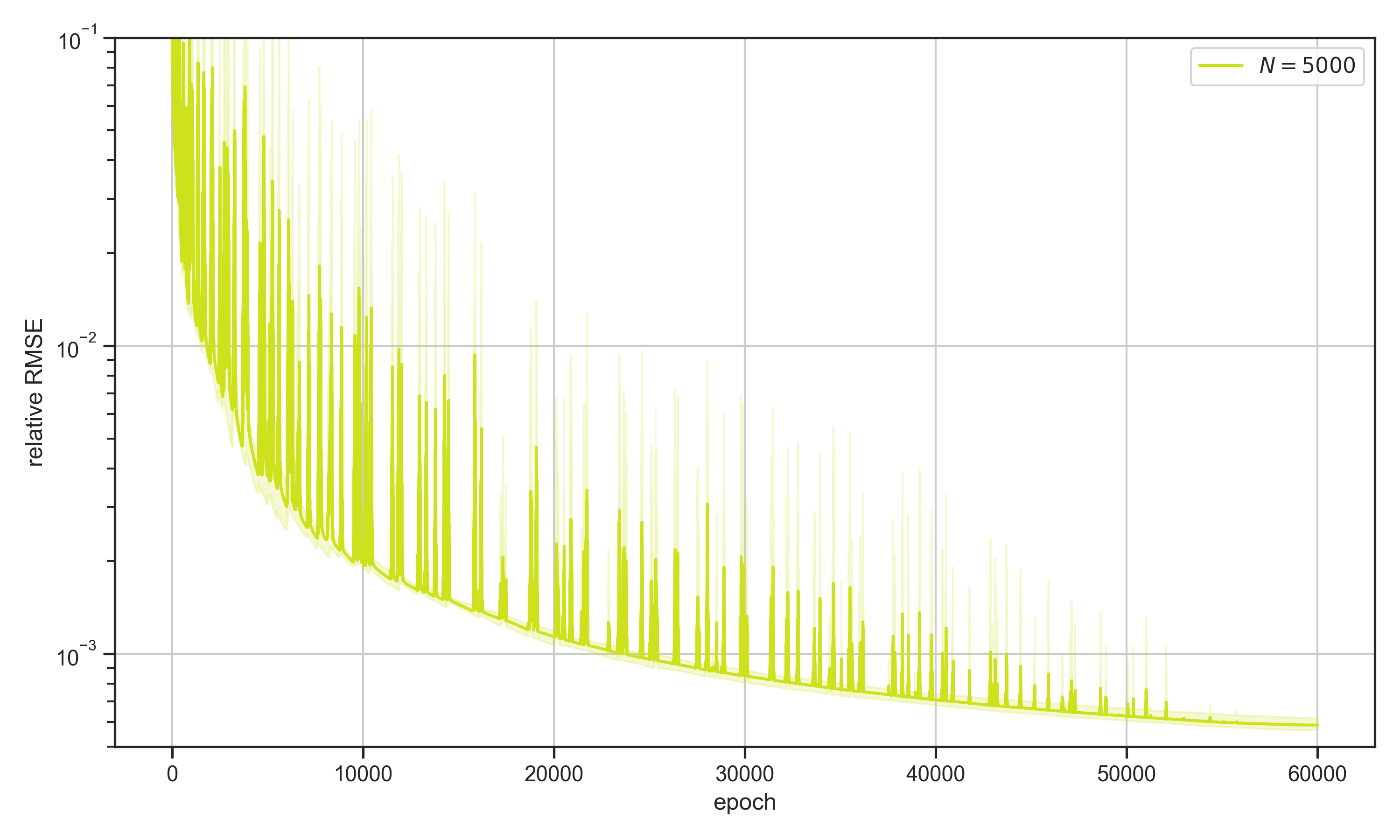}
  \caption{\footnotesize RMSE for $N=5000$.}
  \label{fig:HeatEquation_c}
  \end{subfigure}\cr
}
\caption{Decay of the relative $\text{RMSE}(\theta)$ during training of the NN~$g^N_\theta$ for different numbers of neurons~$N$ (colored in blue to green as $N$ increases) in case of the linear heat equation, i.e., scenario~(i). 
In \textbf{(a)}, we depict for a range of different numbers of neurons~$N\in\{1,...,5000\}$ the relative RMSE of the best model observed during training up until the current epoch.
That is, if the RMSE of a model~$\theta_k$ at epoch $k$ is $\text{RMSE}(\theta_k)$, then the plot displays $\min_{\ell \leq k} \text{RMSE}(\theta_{\ell})$ at each epoch $k$.
As we increase the number of neurons~$N$, we observe an improvement in the respective RMSE.
In \textbf{(b)} and \textbf{(c)}, we plot the instantaneous relative RMSE of the current model at each epoch for $N=50$ and $N=5000$ neurons. I.e., these plots display $\text{RMSE}(\theta_{k})$ at each epoch $k$.\\
In all three plots, we display the mean across five individual training runs with different initializations as a solid line together with the maximal deviation therefrom by a shaded area.}
\label{fig:HeatEquation}
\end{figure}
\begin{figure}[p!]
\centering
\valign{#\cr
  \hsize=0.45\columnwidth
  \begin{subfigure}{0.7\columnwidth}
  \centering
  \includegraphics[width=.9\textwidth, height=.54\textwidth]{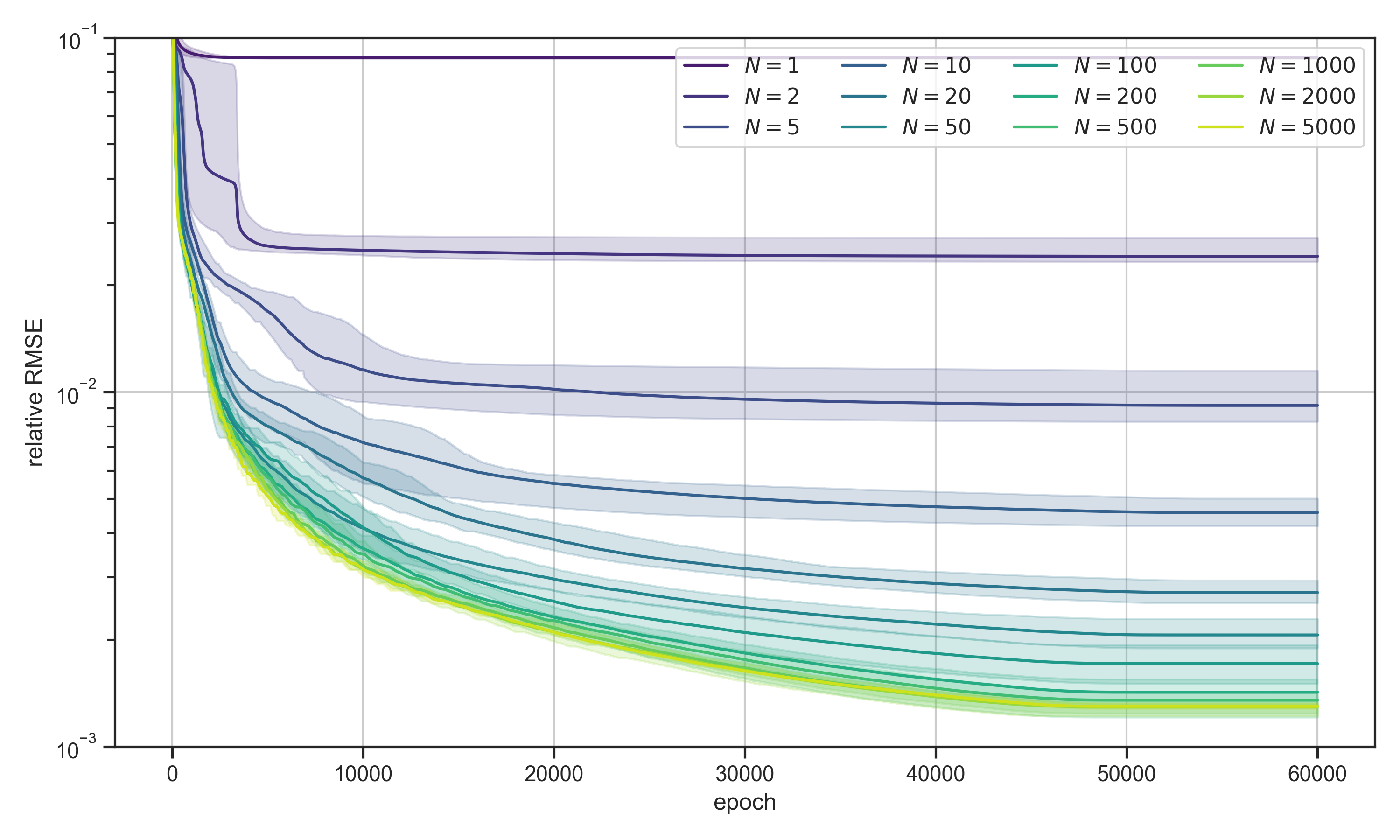}
  \caption{\footnotesize Best RMSE for $N\in\{1,2,5,10,20,50,100,200,500,1000,2000,5000\}$.}
  \label{fig:AllenCahnEquation_a}
  \end{subfigure}\cr\noalign{\hfill}
  \hspace{0.1\columnwidth}
  \hsize=0.45\columnwidth
  \begin{subfigure}{0.3\columnwidth}
  \centering
  \includegraphics[width=.9\textwidth, height=.54\textwidth]{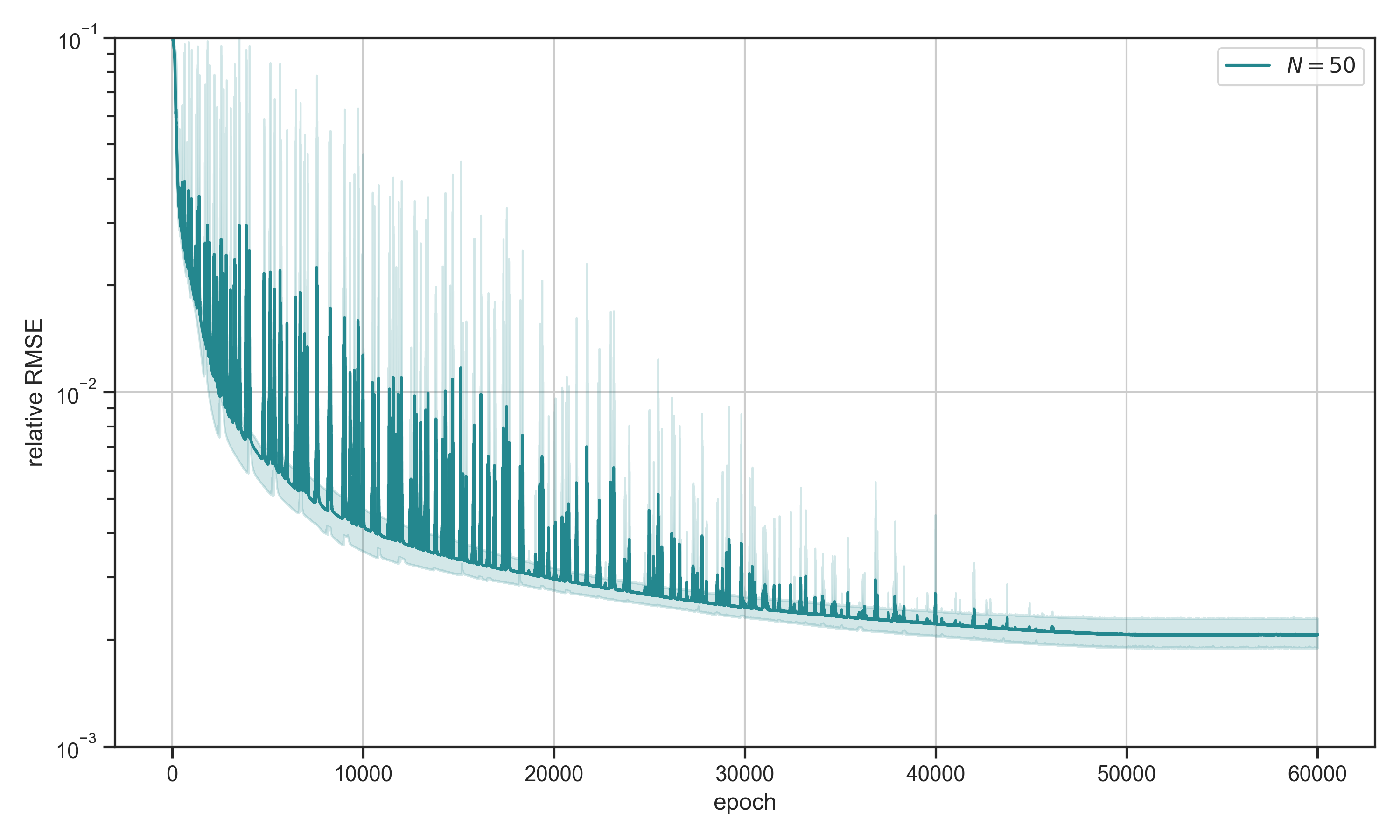}
  \caption{\footnotesize RMSE for $N=50$.}
  \label{fig:AllenCahnEquation_b}
  \end{subfigure}\vfill
  \hspace{0.1\columnwidth}
  \begin{subfigure}{0.3\columnwidth}
  \centering
  \includegraphics[width=.9\textwidth, height=.54\textwidth]{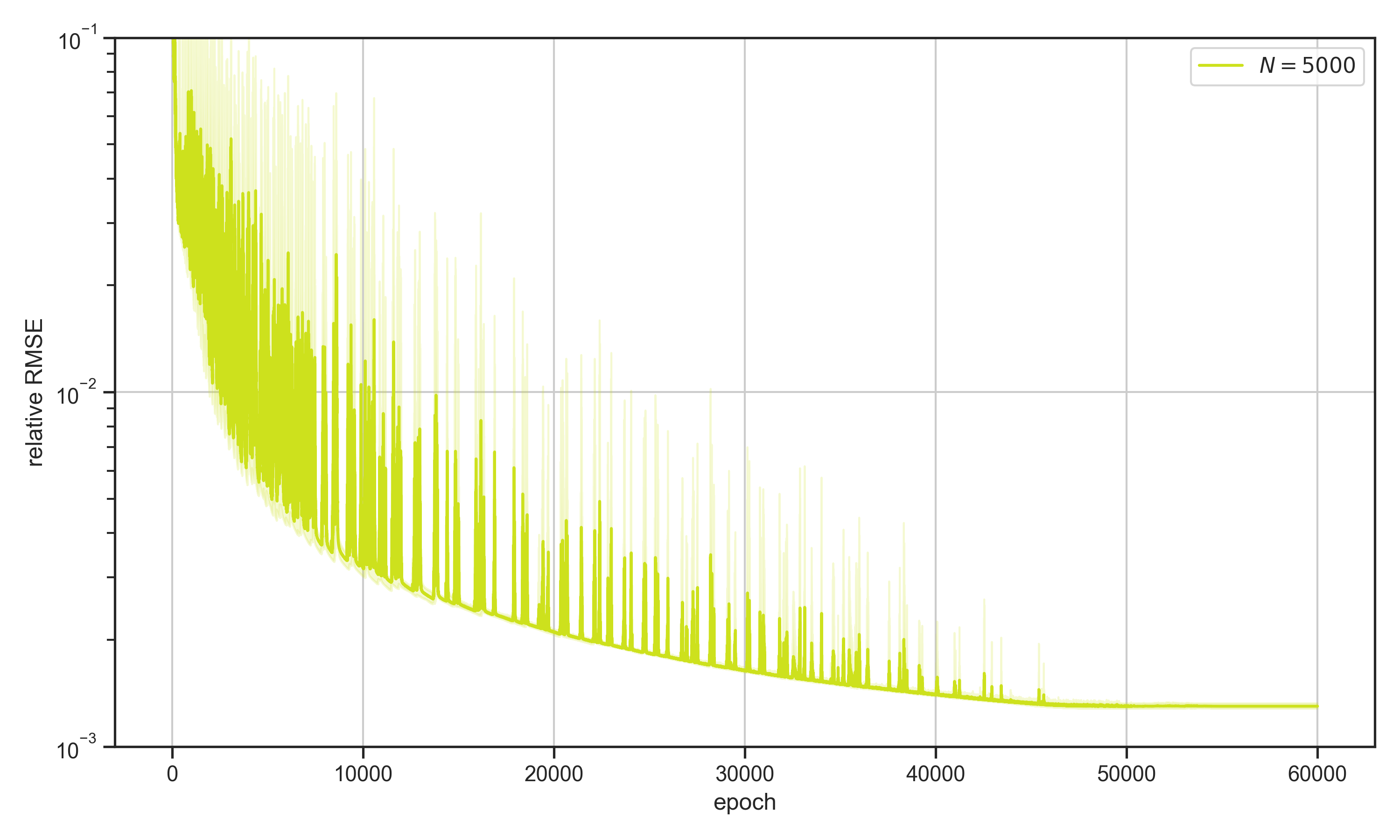}
  \caption{\footnotesize RMSE for $N=5000$.}
  \label{fig:AllenCahnEquation_c}
  \end{subfigure}\cr
}
\caption{We repeat the experiment of \Cref{fig:HeatEquation} for the nonlinear Allen-Cahn equation, i.e., scenario~(ii).}
\label{fig:AllenCahnEquation}
\end{figure}
In \Cref{fig:HeatEquation,fig:AllenCahnEquation}, respectively,
we depict for the linear heat equation, i.e. scenario~(i), and the nonlinear Allen-Cahn equation with nonlinear term~$q(u)=u^3-u$, i.e. scenario~(ii),
the relative root mean square error~(RMSE)
\begin{equation}
    \text{RMSE}(\theta)
    = \frac{1}{\N{h}_{L_\infty(D_T)}}\sqrt{2\JN{\theta}}
    = \frac{1}{\N{h}_{L_\infty(D_T)}}\sqrt{\int_0^T\!\!\!\int_D (\uNv{\theta}{t,x} -h(t,x))^2 \,dxdt}
\end{equation}
during training of the NN~$g^N_\theta$ for a range of different numbers of neurons~$N$.
Our results are averaged across five runs with different seeds and visualized as described in the captions.

We observe that the NN-PDE solution~$\uNhat{\theta}$ converges to the target data~$h$ as the number of neurons~$N$ in the NN~$g^N_\theta$ increases (see \Cref{fig:HeatEquation_a,fig:AllenCahnEquation_a}),
which confirms our theoretical expectations, cf.\@ \Cref{lem:lazytraining,thm:main}.
For a visualization of the target source term~$g_{\text{target}}$, the NN~$g^N_\theta$ as well as the target data~$h$ and the NN-PDE solution~$\uN{\theta}$, we refer the reader to the GitHub repository
\url{https://github.com/KonstantinRiedl/NNPDEs}.
In both experimental scenarios (i) and (ii), the loss decreases quickly from the beginning.
For small values of $N$, convergence saturates earlier at higher loss plateaus, while the error continues to decrease to very low loss plateaus even after $60,000$ epochs for larger $N$.
The spikes and oscillations present in the trajectory of the raw losses (see \Cref{fig:HeatEquation_b,fig:HeatEquation_c,fig:AllenCahnEquation_b,fig:AllenCahnEquation_c}) are reminiscent of the edge of stability phenomenon observed in \cite{cohen2021gradient,cohen2022adaptive} when using gradient methods to mimimize neural network
training objectives.
They observe for a wide range of machine learning tasks that the training loss does not behave monotonically over short timescales, but consistently decreases over long ones, due to exhibiting a self-stabilization property~\cite{damian2022self}.
This is in accordance with our observations.
These artifacts are intensified by the strong non-convexity of our underlying optimization problem in both the linear and nonlinear PDE scenario.
Although this non-convexity is expected to vanish in the infinite-width hidden layer limit in the case of the linear heat equation due to a convexification of the optimization problem,
for moderately sized finite-dimensional neural networks, this non-convexity appears to still have an effect on the training.
In the nonlinear case, where the problem does not convexify, this may lead to the observed more frequent oscillations in the loss.

\section{Neural Network in the PDE Source Term}
\label{sec:infinitewidthNN}

This section is about the mathematical tools related to the neural network~(NN) $g_\theta^N$ defined in \eqref{eq:gN}, which constitutes the source term of the PDE~\eqref{eq:parabolicPDEN_plain}.

\subsection{Properties of the NN Kernel \texorpdfstring{$B$}{}}

The NN kernel~$B$, a.k.a.\@ the neural tangent kernel~(NTK) is given as in \eqref{eq:parabolicB}.
In \Cref{lem:BboundednessLinfty,lem:BboundednessL2,lem:BLipschitz} we establish some properties of the NN kernel~$B_0$ at initialization that will be useful throughout the manuscript.
First, we show that the kernel~$B_0$ is uniformly bounded,
which is a direct consequence of Assumptions~\ref{asm:NN_sigma}, \ref{asm:NN_sigma'} and \ref{asm:NN_mu0}\ref{asm:NN_mu0ii}.

\begin{lemma}[$L_\infty$-boundedness of $B_0$]
    \label{lem:BboundednessLinfty}
    The kernel $B_0=B(\mu_0)$ defined in \eqref{eq:parabolicB} is uniformly bounded in $L_\infty$, i.e., it holds
    \begin{equation}
        \abs{B_0(t,x,t',x')}
        = \abs{B(t,x,t',x';\mu_0)}
        \leq C^{B}_\infty
    \end{equation}
    for all $(t,x),(t',x')\in D_T$ for a constant $C^{B}_\infty=C^{B}_\infty(T,D,\sigma,\mu_0)$.
\end{lemma}

Lemma \ref{lem:BboundednessLinfty} directly implies that the kernel~$B_0$ is bounded in $L_2$.
\begin{lemma}[$L_2$-boundedness of $B_0$]
    \label{lem:BboundednessL2}
    The kernel $B_0=B(\mu_0)$ defined in \eqref{eq:parabolicB} is bounded in $L_2$, i.e., $B_0\in L_2(D_T\times D_T)$.
    We abbreviate $C^{B}_2 = \N{B_0}_{L_2(D_T\times D_T)}$.
\end{lemma}

\begin{remark}
    \label{rem:BHilbertSchmidt}
    Since the NN kernel~$B$ is symmetric, as easily verifiable by noting that $B(t,x,t',x';\mu)=B(t',x',t,x;\mu)$, and since $B_0\in L_2(D_T\times D_T)$ as establish in \Cref{lem:BboundednessL2},
    $B_0$ is a Hilbert-Schmidt kernel.
\end{remark}

We further show that the kernel~$B_0$ is Lipschitz continuous in the time and space variables.

\begin{lemma}[Lipschitz continuity of $B_0$]
    \label{lem:BLipschitz}
    The kernel $B_0=B(\mu_0)$ defined in \eqref{eq:parabolicB} is $L_B$-Lipschitz continuous, i.e., it holds
    \begin{equation}
        \abs{B(t^1,x^1,t',x';\mu_0)-B(t^2,x^2,t',x';\mu_0)}
        \leq L_B \left(\abs{t^1-t^2} + \N{x^1-x^2}\right)
    \end{equation}
    for all $(t^1,x^2),(t^2,x^2),(t',x')\in D_T$ for a constant $L_B=L_B(T,D,\sigma,\mu_0)$.
\end{lemma}

\begin{proof}
    The Lipschitzness and boundedness of $\sigma$ and $\sigma'$ via Assumptions~\ref{asm:NN_sigma} and \ref{asm:NN_sigma'} together with Jensen's inequality give
    \begin{equation}
    \begin{split}
        &\abs{B(t^1,x^1,t',x';\mu_0)-B(t^2,x^2,t',x';\mu_0)} \\
        &\qquad\,\leq \int C(1+c^2)\left((1+\absnormal{w^t})\abs{t^1-t^2} + (1+\N{w})\N{x^1-x^2}\right)d\mu_0(w^t,w,\eta,c)
    \end{split}
    \end{equation}
    for a constant $C=C(T,D,\sigma)$.
    Since $\mu_0$ is such that the marginal distribution $\mu_{0,c}$ is compactly supported and the marginal distribution $\mu_{0,(w^t,w,\eta)}$ has bounded moments according to Assumption~\ref{asm:NN_mu0}, the statement follows.
\end{proof}

To wrap up this section, let us show that $k$ is (locally) Lipschitz continuous in the NN parameters. 

\begin{lemma}[Lipschitz continuity of $k$]
    \label{lem:kLipschitz}
    The function $k$ defined in \eqref{eq:parabolick} is $L_k$-Lipschitz continuous, i.e., it holds
    \begin{equation}
    \begin{split}
        &\abs{k(t,x,t',x';c^1,w^{t,1},w^1,\eta^1)-k(t,x,t',x';c^2,w^{t,2},w^2,\eta^2)}\\
        &\qquad\,\leq L_k(c^1,c^2) \left(\abs{c^1-c^2} + \abs{w^{t,1}-w^{t,2}} + \N{w^1-w^2} + \abs{\eta^1-\eta^2}\right)
    \end{split}
    \end{equation}
    for all $(c^1,c^2),(w^{t,1},w^{t,2}),(w^1,w^2),(\eta^1,\eta^2)$ and for all $(t,x),(t',x')\in D_T$ for a constant $L_k(c^1,c^2)=L_k(T,D,\sigma,c^1,c^2)$ that is quadratic in $c^1$ and $c^2$.
\end{lemma}

\begin{proof}
    The proof follows  directly using the Lipschitzness and boundedness of $\sigma$ and $\sigma'$ via Assumptions~\ref{asm:NN_sigma}, \ref{asm:NN_sigma'}. Details are omitted.
\end{proof}

\subsection{Properties of the NN Integral Operator \texorpdfstring{$T_{B_0}$}{}}

The NN integral operator~$T_{B_0}$ is given as in \eqref{eq:parabolicT_B}.
\begin{remark}
    \label{rem:TBHilbertSchmidt}
    With the kernel $B_0$ being, as discussed in \Cref{rem:BHilbertSchmidt}, symmetric and in $L_2$, i.e., a Hilbert-Schmidt kernel,
    the associated operator $T_{B_0}:L_2(\overbar{D_T})\rightarrow L_2(\overbar{D_T})$ is a Hilbert-Schmidt integral operator.
\end{remark}

In \Cref{lem:parabolicTB,lem:TBboundednessLinfty,lem:TBLipschitz,lem:parabolicposdefTB} we establish some properties of the NN integral operator~$T_{B_0}$ that will be useful throughout the manuscript.
First, we show that the eigenfunctions of the NN integral operator~$T_{B_0}$ come with real eigenvalues and form an orthonormal basis of $L_2(D_T)$.
\begin{lemma}[Properties of $T_{B_0}$]
    \label{lem:parabolicTB}
    The operator $T_{B_0}$ defined in \eqref{eq:parabolicT_B} with $B_0=B(\mu_0)$ is a self-adjoint compact linear operator with operator norm $\N{T_{B_0}} \leq \N{T_{B_0}}_{\mathrm{HS}} = \N{B_0}_{L_2(D_T\times D_T)}\!=C^B_2$, where $\N{\dummy}_{\mathrm{HS}}$ denotes the Hilbert-Schmidt norm.
    Furthermore, the eigenfunctions $\left\{e_k(t,x)\right\}_{k=1}^\infty$ of $T_{B_0}$ have real eigenvalues~$\left\{\lambda_k\right\}_{k=1}^\infty$ and form an orthonormal basis of $L_2(D_T)$.
\end{lemma}

\begin{proof}
    Since the kernel $B_0$ is symmetric and of Hilbert-Schmidt type as verified in \Cref{lem:BboundednessL2}, see also \Cref{rem:BHilbertSchmidt},
    the operator $T_{B_0}$ is a self-adjoint, compact linear operator.
    The Hilbert-Schmidt norm is $\N{T_{B_0}}_{\mathrm{HS}} = \N{B_0}_{L_2(D_T\times D_T)}$ and provides an upper bound to the operator norm.
    Furthermore, the spectral theorem ensures the existence of an orthonormal basis of $L_2(D_T)$ consisting of eigenvectors of $T_{B_0}$ with real eigenvalue, see \cite[Theorem~6.12]{brezis2011functional}.
\end{proof}

In fact, as we show  next,
the eigenvalues of the NN integral operator $T_{B_0}$ can be shown to be strictly positive~\cite[pages 27--28]{sirignano2023pde}.
\begin{lemma}[Positive definiteness of $T_{B_0}$]
    \label{lem:parabolicposdefTB}
    The eigenvalues $\left\{\lambda_k\right\}_{k=1}^\infty$ of the operator $T_{B_0}$ defined in \eqref{eq:parabolicT_B} with $B_0=B(\mu_0)$ are strictly positive, i.e., it holds $\lambda_k>0$ for all $k$.
    Moreover, it holds  $\lambda_k\leq \N{B_0}_{L_2(D_T\times D_T)}$ for all $k$.
\end{lemma}
\begin{proof}
    We first prove that the eigenvalues are strictly positive, i.e., $\lambda_k>0$ for all $k$.
    Using the definition of the kernel $B_0=B(\mu_0)$ in \eqref{eq:parabolicB} we can directly compute that for any function $\widehat{u}=\widehat{u}(t,x)$ it holds
    \begin{equation}       \label{eq:proof:lem:parabolicposdefTB:1}
    \begin{split}
        (\widehat{u},T_{B_0}\widehat{u})_{L_2(D_T)}
        &= \int_0^T\!\!\!\int_D \widehat{u}(t,x) \int_0^T\!\!\!\int_D B_0(t,x,t',x')\widehat{u}(t',x') \,dx'dt'\,dxdt\\
        &= \int \left(\int_0^T\!\!\!\int_D\sigma(w^tt+w^Tx+\eta)\widehat{u}(t,x)\,dxdt\right)^2\\
        &\qquad\,+ \left(\int_0^T\!\!\!\int_Dc\sigma'(w^tt+w^Tx+\eta)t\widehat{u}(t,x)\,dxdt\right)^2\\
        &\qquad\,+ \N{\int_0^T\!\!\!\int_Dc\sigma'(w^tt+w^Tx+\eta)x\widehat{u}(t,x)\,dxdt}^2\\
        &\qquad\,+ \left(\int_0^T\!\!\!\int_Dc\sigma'(w^tt+w^Tx+\eta)\widehat{u}(t,x)\,dxdt\right)^2d\mu_{0}(c,w^t,w,\eta)\\
        &\geq \int \left(\int_0^T\!\!\!\int_D\sigma(w^tt+w^Tx+\eta)\widehat{u}(t,x)\,dxdt\right)^2d\mu_{0,(w^t,w,\eta)}(w^t,w,\eta)\\
        &\geq0,
    \end{split}
    \end{equation}
    where the inequality in the next-to-last step holds due to the non-negativity of the individual summands in lines~3--5.
    Since also the summand in line 2 is non-negative, the last inequality holds, which verifies $\lambda_k=(e_k,T_{B_0}e_k)_{L_2(D_T)}\geq0$.
    
    Let us now show further that $(\widehat{u},T_{B_0}\widehat{u})_{L_2(D_T)}=0$ if and only if $\widehat{u}=0$ everywhere.
    The ``if'' direction is immediate.
    For the ``only if'' direction, we proceed by contradiction.
    Consider a function $\widehat{u}$ which is not everywhere $0$ but suppose that $(\widehat{u},T_{B_0}\widehat{u})_{L_2(D_T)}=0$.
    The latter implies due to the inequality \eqref{eq:proof:lem:parabolicposdefTB:1} that
    \begin{equation}        \label{eq:proof:lem:parabolicposdefTB:2}
        \int_0^T\!\!\!\int_D\sigma(w^tt+w^Tx+\eta)\widehat{u}(t,x)\,dxdt
        = 0
        \qquad\text{for all }
        w^t\in\bbR,
        w\in\bbR^d,
        \eta\in\bbR,
    \end{equation}
    since the marginal distribution $\mu_{0,(w^t,w,\eta)}$ assigns positive probability to every set with positive Lebesgue measure as of Assumption~\ref{asm:NN_mu0}\ref{asm:NN_mu0iv} and continuity of the integrand w.r.t.\@ the NN parameters~$w^t,w,\eta$.
    Since $\sigma$ is non-constant and bounded as of Assumption~\ref{asm:NN_sigma}, it is, according to \cite[Theorem~5]{hornik1991approximation}, discriminatory in the sense of \cite{cybenko1989approximation,hornik1991approximation}.
    This ensures (note that $\widehat{u}(t,x)\,dxdt$ is a finite signed measure since $\widehat{u}\in L_1(D_T)$ by Jensen's inequality and the fact that $\widehat{u}\in L_2(D_T)$ and $D_T$ being bounded as of Assumption~\ref{asm:Dbdd}) that \eqref{eq:proof:lem:parabolicposdefTB:2} implies that $\widehat{u}=0$ by the definition of $\sigma$ being discriminatory, see \cite{hornik1991approximation}.
    Since this is a contradiction, $(\widehat{u},T_{B_0}\widehat{u})_{L_2(D_T)}>0$ if $\widehat{u}$ is not everywhere $0$.
    In particular, for the eigenfunctions~$e_k$ it thus holds
    $\lambda_k = (e_k,T_{B_0}e_k)_{L_2(D_T)}>0$,
    which proves the first part of the statement.

    It remains to show that the eigenvalues are bounded from above, i.e., $\lambda_k<\infty$.
    For this note that by Cauchy-Schwarz inequality it holds
    \begin{equation}
    \begin{split}
        \lambda_k
        = (e_k,T_{B_0}e_k)_{L_2(D_T)}
        &\leq \N{e_k}_{L_2(D_T)}\N{T_{B_0}e_k}_{L_2(D_T)}\\
        &\leq \N{B_0}_{L_2(D_T\times D_T)}\N{e_k}_{L_2(D_T)}^2
        = \N{B_0}_{L_2(D_T\times D_T)},
    \end{split}
    \end{equation}
    where the last inequality is due to \Cref{lem:parabolicTB}.
    This concludes the statement.
\end{proof}

Next, let us show that the NN integral operator $T_{B_0}$ maps $L_2$ to $L_\infty$.
\begin{lemma}[$L_\infty$-Boundedness of $T_{B_0}\widehat{u}$]   \label{lem:TBboundednessLinfty}
    Let $\widehat{u}\in L_2(D_T)$.
    Then $T_{B_0}\widehat{u}$ defined in \eqref{eq:parabolicT_B} with $B_0=B(\mu_0)$ is uniformly bounded in $L_\infty$, i.e., it holds
    \begin{equation}        \absbig{[T_{B_0}\widehat{u}](t,x)}
        \leq C^{T_B}_\infty \N{\widehat{u}}_{L_2(D_T)}
    \end{equation}
    for all $(t,x)\in D_T$ for a constant $C^{T_B}_\infty=C^{T_B}_\infty(T,D,C^{B}_\infty)$.
\end{lemma}
\begin{proof}
    Using Cauchy-Schwarz inequality and employing \Cref{lem:BboundednessLinfty} afterwards, we bound
    \begin{equation}
    \begin{split}        \absbig{[T_{B_0}\widehat{u}](t,x)}
        &\leq\N{\widehat{u}}_{L_2(D_T)} \sqrt{\int_0^T\!\!\!\int_D \left(B_0(t,x,t',x')\right)^2 dx'dt'}\\
        &\leq \N{\widehat{u}}_{L_2(D_T)} \sqrt{T\vol{D}}C^{B}_\infty
    \end{split}
    \end{equation}    
    for each $(t,x)\in D_T$, which proves the assertion with $C^{T_B}_\infty=\sqrt{T\vol{D}}C^{B}_\infty$.
\end{proof}

To wrap up this section, we furthermore show that~$T_{B_0}\widehat{u}$ is Lipschitz continuous.
\begin{lemma}[Lipschitz continuity of $T_{B_0}\widehat{u}$]
    \label{lem:TBLipschitz}
    Let $\widehat{u}\in L_2(D_T)$.
    Then $T_{B_0}\widehat{u}$ defined in \eqref{eq:parabolicT_B} with $B_0=B(\mu_0)$ is $L_{T_B}$-Lipschitz continuous, i.e., it holds
    \begin{equation}
        \abs{[T_{B_0}\widehat{u}](t^1,x^1)-[T_{B_0}\widehat{u}](t^2,x^2)}
        \leq L_{T_B} \left(\abs{t^1-t^2} + \N{x^1-x^2}\right)
    \end{equation}
    for all $(t^1,x^2),(t^2,x^2)\in D_T$ for a constant $L_{T_B}=L_{T_B}(T,D,L_B)$.
\end{lemma}
\begin{proof}
    Using Cauchy-Schwarz inequality and employing \Cref{lem:BLipschitz} afterwards, we bound
    \begin{equation}
    \begin{split}
        &\abs{[T_{B_0}\widehat{u}](t^1,x^1)-[T_{B_0}\widehat{u}](t^2,x^2)}\\
        &\qquad\,\leq\N{\widehat{u}}_{L_2(D_T)} \sqrt{\int_0^T\!\!\!\int_D \left(B(t^1,x^1,t',x';\mu_0)-B(t^2,x^2,t',x';\mu_0)\right)^2 dx'dt'} \\
        &\qquad\,\leq \N{\widehat{u}}_{L_2(D_T)}  L_B \sqrt{T \vol{D}} \left(\abs{t^1-t^2} + \N{x^1-x^2}\right)
    \end{split}
    \end{equation}
    for $(t^1,x^2),(t^2,x^2)\in D_T$.
\end{proof}
\section{Decay of the Loss \texorpdfstring{$\J{\tau}$}{}}
\label{sec:decayJ}

\Cref{lem:parabolictimeevolutionJt*} in this section establishes that the loss $\J{\tau}$ defined in \eqref{eq:parabolicJtau} is monotonically non-increasing in the training time $\tau$.

We state the result for a training time interval $I$, which may be either $[0,\CT]$ or $[0,\infty)$.
\begin{proposition}[Decay of the loss $\J{\tau}$]    \label{lem:parabolictimeevolutionJt*}
    Let $((\up{\tau},\uhatp{\tau}))_{\tau\in I}\in \CC\left(I,\CS\times\CS\right)$ denote the unique weak solution to the PDE system \eqref{eq:parabolicPDE*}--\eqref{eq:parabolicadjoint*} coupled with the integro-differential equation~\eqref{eq:parabolicgtau} in the sense of \Cref{lem:parabolic_wellposedness_infinitewidth} and \Cref{rem:parabolic_wellposedness} on the training time interval $I$.
    Define the loss $\J{\tau}$ as in \eqref{eq:parabolicJtau}.
    Then, for the training time derivative $\fd{\tau}\J{\tau}$ it holds
    \begin{equation}
    \label{eq:parabolicdtJ*}
    \begin{split}
        \fd{\tau}\J{\tau}
        &= - \alpha_\tau(\uhatp{\tau},T_{B_0}\uhatp{\tau})_{L_2(D_T)} \\
        &= -\alpha_\tau \int_0^T\!\!\!\int_D \uhatpp{\tau}{t,x} \int_0^T\!\!\!\int_D B(t,x,t',x';\mu_0)\uhatpp{\tau}{t',x'} \,dx'dt'dxdt
    \end{split}
    \end{equation}
    for all $\tau\in I$
    with the operator $T_{B_0}$ defined in \eqref{eq:parabolicT_B} and where the kernel $B_0=B(\mu_0)$ is as in \eqref{eq:parabolicB}.
    In particular, we have
    \begin{equation}
        \fd{\tau}\J{\tau}\leq 0
    \end{equation}
    for all $\tau\in I$.
\end{proposition}
\begin{proof}
    Taking the training time derivative of our loss $\J{\tau}$, i.e., the derivative w.r.t.\@ the training time~$\tau$, 
    we obtain by chain rule and by using that $\uhatp{\tau}$ is a weak solution to the adjoint PDE~\eqref{eq:parabolicadjoint*} in the sense of \Cref{def:weak_sol_adjoint} with right-hand side $(\up{\tau}-h)$ that 
    \begin{equation}        \label{eq:proof:parabolicdtJ*1}
    \begin{split}
        \fd{\tau}\J{\tau}
        &= \fd{\tau} \frac{1}{2} \int_0^T\!\!\!\int_D (\upp{\tau}{t,x} \!-\! h(t,x))^2 \,dxdt \\
        &= \int_0^T\!\!\!\int_D (\upp{\tau}{t,x} \!-\! h(t,x)) \fd{\tau}\upp{\tau}{t,x} \, dxdt 
        = \int_0^T \!\left(\upp{\tau}{t,\dummy} \!-\! h(t,\dummy), \fd{\tau}\upp{\tau}{t,\dummy}\right)_{L_2(D)}\!dt\\
        &= \int_0^T \left\langle-\fpartial{t}\uhatpp{\tau}{t,\dummy},\fd{\tau}\upp{\tau}{t,\dummy}\right\rangle_{H^{-1}(D),H_0^1(D)} + \CB^\dagger\left[\uhatpp{\tau}{t,\dummy},\fd{\tau}\upp{\tau}{t,\dummy};t\right] \\
        & \qquad\quad\, - \left(q_u(t,\dummy,\upp{\tau}{t,\dummy})\uhatpp{\tau}{t,\dummy}, \fd{\tau}\upp{\tau}{t,\dummy}\right)_{L_2(D)} dt \\
        &= \int_0^T \left\langle\fpartial{t}\fd{\tau}\upp{\tau}{t,\dummy},\uhatpp{\tau}{t,\dummy}\right\rangle_{H^{-1}(D),H_0^1(D)} + \CB\left[\fd{\tau}\upp{\tau}{t,\dummy},\uhatpp{\tau}{t,\dummy};t\right] \\
        & \qquad\quad\, - \left(q_u(t,\dummy,\upp{\tau}{t,\dummy})\fd{\tau}\upp{\tau}{t,\dummy},\uhatpp{\tau}{t,\dummy}\right)_{L_2(D)} dt \\
        &= \int_0^T \left(\fd{\tau}g^*_\tau(t,\dummy),\uhatpp{\tau}{t,\dummy}\right)_{L_2(D)} dt
        = \int_0^T\!\!\!\int_D \left(\fd{\tau}g^*_\tau(t,x)\right)\uhatpp{\tau}{t,x}\, dxdt.
    \end{split}
    \end{equation}
    
    For the weak solution property in the third line of \eqref{eq:proof:parabolicdtJ*1}, we note that $\fd{\tau}\up{\tau}$, the weak solution to the linear parabolic PDE
    \begin{alignat}{2}
        \label{eq:parabolicpI*aux}
    \begin{aligned}
        \fpartial{t}\fd{\tau}\up{\tau} + \CL\fd{\tau}\up{\tau} - q_u(\up{\tau})\fd{\tau}\up{\tau}
        &= \fd{\tau}g^*_\tau
        \qquad&&\text{in }
        D_T, \\
        \fd{\tau}\up{\tau}
        &= 0
        \qquad&&\text{on }
        [0,T]\times\partial D, \\
        \fd{\tau}\up{\tau}
        &= \fd{\tau}f = 0
        \qquad&&\text{on }
        \{0\}\times D,
    \end{aligned}
    \end{alignat}
    which is obtained by taking in \eqref{eq:parabolicPDE*} the derivative w.r.t.\@ the training time $\tau$,
    can be used as a test function in the weak formulation of \eqref{eq:parabolicadjoint*}, see \Cref{def:weak_sol_adjoint}, since $\fd{\tau}\upp{\tau}{t,\dummy}\in H_0^1(D)$ for a.e.\@ $t\in[0,T]$.
    Existence, uniqueness and regularity of a weak solution to \eqref{eq:parabolicpI*aux} in a sense analogous to \Cref{def:weak_sol_adjoint} follow from classical results,
    see, e.g., \cite[Chapter~7.1, Theorem~3]{evans2010partial} and \cite[Chapter~7.1, Theorem~4]{evans2010partial}, as $\fd{\tau} g^*_\tau = - \alpha_\tau T_{B_0}\uhatp{\tau} \in L_2(D_T)$ by \Cref{lem:parabolicTB}.
    
    For the step in the fourth line of \eqref{eq:proof:parabolicdtJ*1}
    we first recall that 
    since $\fd{\tau}\upp{\tau}{t,\dummy}\in H_0^1(D)$ for a.e.\@ $t\in[0,T]$ and since $\fpartial{t}\uhatpp{\tau}{t,\dummy}\in L_2(D)$ for a.e.\@ $t\in[0,T]$, the dual pairing between $H^{-1}(D)$ and $H_0^1(D)$ coincides with the $L_2(D)$ scalar product \cite[Chapter~5.9, Theorem 1(iii)]{evans2010partial}.
    This allows to compute with partial integration, which applies since $\uhatpp{\tau}{t,\dummy}, \fd{\tau}\upp{\tau}{t,\dummy}\in H_0^1(D)$ for a.e.\@ $t\in[0,T]$, that
    \begin{allowdisplaybreaks}
    \label{eq:parabolicpI*aux_T51}
    \begin{align}
        &\int_0^T\left\langle\fpartial{t}\uhatpp{\tau}{t,\dummy},\fd{\tau}\upp{\tau}{t,\dummy}\right\rangle_{H^{-1}(D),H_0^1(D)} dt\notag\\
        &\qquad\,= \int_0^T\left(\fpartial{t}\uhatpp{\tau}{t,\dummy},\fd{\tau}\upp{\tau}{t,\dummy}\right)_{L_2(D)}dt
        = \int_0^T\!\!\!\int_D \left(\fpartial{t}\uhatpp{\tau}{t,x}\right)\fd{\tau}\upp{\tau}{t,x} \, dxdt\notag\\
        &\qquad\,= \int_D \underbrace{\uhatpp{\tau}{t,x} \fd{\tau}\upp{\tau}{t,x}\Big|_0^T}_{\substack{=0\\\text{since } \uhatp{\tau}=0 \text{ on }\{T\}\times D \text{ and}\\\text{since }\fd{\tau}\up{\tau}=0 \text{ on } \{0\}\times D}} \, dx -\int_0^T\!\!\!\int_D \uhatpp{\tau}{t,x}\fpartial{t}\fd{\tau}\upp{\tau}{t,x} \, dxdt \notag\\
        &\qquad\,=-\int_0^T\!\!\!\int_D  \left(\fpartial{t}\fd{\tau}\upp{\tau}{t,x}\right)\uhatpp{\tau}{t,x} \, dxdt
        = -\int_0^T\left(\fpartial{t}\fd{\tau}\upp{\tau}{t,\dummy},\uhatpp{\tau}{t,\dummy}\right)_{L_2(D)}\,dt \notag\\
        &\qquad\,= -\int_0^T\left\langle\fpartial{t}\fd{\tau}\upp{\tau}{t,\dummy},\uhatpp{\tau}{t,\dummy}\right\rangle_{H^{-1}(D),H_0^1(D)} dt,
    \end{align}
    \end{allowdisplaybreaks}
    where the last step holds again since now $\uhatpp{\tau}{t,\dummy}\in H_0^1(D)$ for a.e.\@ $t\in[0,T]$ and $\fpartial{t}\fd{\tau}\uhatp{\tau}$ is in $L_2(D)$ for a.e.\@ $t\in[0,T]$ by \cite[Chapter IV, Theorem 9.1]{ladyzhenskaia1968linear} with $p=2$.
    Those computations are analogous to the ones of \Cref{lem:parabolic_wellposedness_infinitewidth} for the PDE \eqref{eq:parabolicpI*aux} due to its with \eqref{eq:parabolicadjoint*} identical structure and since $\fd{\tau} g^*_\tau\in L_2(D_T)$.
    Secondly, by definition of the adjoint bilinear form~$\CB^\dagger$ (see \Cref{def:weak_sol_adjoint}) it holds
    \begin{equation}
        \label{eq:parabolicpI*aux_T52}
        \CB^\dagger\left[\uhatpp{\tau}{t,\dummy},\fd{\tau}\upp{\tau}{t,\dummy};t\right]
        =\CB\left[\fd{\tau}\upp{\tau}{t,\dummy},\uhatpp{\tau}{t,\dummy};t\right]
    \end{equation}
    for a.e.\@ $t\in[0,T]$ since $\uhatpp{\tau}{t,\dummy}, \fd{\tau}\upp{\tau}{t,\dummy}\in H_0^1(D)$. 

    The penultimate step of \eqref{eq:proof:parabolicdtJ*1} holds since $\fd{\tau}\uhatp{\tau}$ is a weak solution to the PDE~\eqref{eq:parabolicpI*aux} and since $\uhatpp{\tau}{t,\dummy}$ is a suitable test function as it is in $H_0^1(D)$ for a.e.\@ $t\in[0,T]$.

    Now, recalling the definition of the right-hand side~$g^*_\tau$ from \eqref{eq:parabolicgtau} and taking its training time derivative to obtain $\fd{\tau} g^*_\tau = - \alpha_\tau T_{B_0}\uhatp{\tau}$, as well as recalling the definition of the operator $T_{B_0}$ from \eqref{eq:parabolicT_B},
    we can continue \eqref{eq:proof:parabolicdtJ*1} to obtain
    \begin{equation}
    \label{eq:proof:parabolicdtJ*2}
    \begin{split}
        \fd{\tau}\J{\tau}
        &= - \alpha_\tau \int_0^T\!\!\!\int_D \uhatpp{\tau}{t,x} \int_0^T\!\!\!\int_D B(t,x,t',x';\mu_0)\uhatpp{\tau}{t',x'} \,dx'dt'dxdt \\
        &= - \alpha_\tau (\uhatp{\tau},T_{B_0}\uhatp{\tau})_{L_2(D_T)},
    \end{split}
    \end{equation}
    which concludes the first part of the proof.

    The second part now follows immediately thanks to the operator $T_{B_0}$ being positive definite as of \Cref{lem:parabolicposdefTB}.
\end{proof}

Following analogous steps we can prove \Cref{lem:grad_thetaJ}.

\begin{proof}[Proof of \Cref{lem:grad_thetaJ}]
    Taking the gradient of the loss $\JN{\theta}$ w.r.t.\@ the NN parameters~$\theta$, 
    we obtain by chain rule and by using that $\uNhat{\theta}$ is a weak solution to the adjoint PDE~\eqref{eq:parabolicadjointN_plain} in the sense of \Cref{def:weak_sol_adjoint} that
    \begin{equation}        \label{eq:proof:parabolicnablaJ*1}
    \begin{split}
        \nabla_\theta\JN{\theta}
        &= \nabla_\theta \frac{1}{2} \int_0^T\!\!\!\int_D (\uNv{\theta}{t,x} \!-\! h(t,x))^2 \,dxdt \\
        &= \int_0^T\!\!\!\int_D (\uNv{\theta}{t,x} \!-\! h(t,x))\nabla_\theta \uNv{\theta}{t,x} \,dxdt 
        = \!\int_0^T \!\!\big(\uNv{\theta}{t,\dummy} \!-\! h(t,\dummy), \nabla_\theta\uNv{\theta}{t,\dummy} \big)_{L_2(D)} \, dt \\
        &= \int_0^T \big\langle\!-\!\fpartial{t}\uNhatv{\theta}{t,\dummy},\nabla_\theta\uNv{\theta}{t,\dummy}\big\rangle_{H^{-1}(D),H_0^1(D)} + \CB^\dagger[\uNhatv{\theta}{t,\dummy},\nabla_\theta\uNv{\theta}{t,\dummy};t] \\
        & \qquad\quad\, - \big(q_u(t,\dummy,\uNv{\theta}{t,\dummy})\uNhatv{\theta}{t,\dummy}, \nabla_\theta\uNv{\theta}{t,\dummy}\big)_{L_2(D)} \, dt \\
        &= \int_0^T \big\langle\fpartial{t}\nabla_\theta\uNv{\theta}{t,\dummy},\uNhatv{\theta}{t,\dummy}\big\rangle_{H^{-1}(D),H_0^1(D)} + \CB[\nabla_\theta\uNv{\theta}{t,\dummy},\uNhatv{\theta}{t,\dummy};t] \\
        & \qquad\quad\, - \big(q_u(t,\dummy,\uNv{\theta}{t,\dummy}) \nabla_\theta\uNv{\theta}{t,\dummy},\uNhatv{\theta}{t,\dummy}\big)_{L_2(D)} \, dt \\
        &= \int_0^T\big(\nabla_\theta g_\theta^N(t,\dummy),\uNhatv{\theta}{t,\dummy}\big)_{L_2(D)}\, dt
        = \int_0^T\!\!\!\int_D \left(\nabla_\theta g_\theta^N(t,x)\right)\uNhatv{\theta}{t,x}\, dxdt.
    \end{split}
    \end{equation}
    For the weak solution property in the third line of \eqref{eq:proof:parabolicnablaJ*1}, we note that $\nabla_\theta\uN{\theta}$, the weak solution to the linear parabolic PDE
    \begin{alignat}{2}
        \label{eq:parabolicpI*aux_plain}
    \begin{aligned}
        \fpartial{t}\nabla_\theta\uN{\theta} + \CL\nabla_\theta\uN{\theta} - q_u(\uN{\theta})\nabla_\theta\uN{\theta}
        &= \nabla_\theta g_\theta^N
        \qquad&&\text{in }
        D_T, \\
        \nabla_\theta\uN{\theta}
        &= 0
        \qquad&&\text{on }
        [0,T]\times\partial D, \\
        \nabla_\theta\uN{\theta}
        &= \nabla_\theta f = 0
        \qquad&&\text{on }
        \{0\}\times D,
    \end{aligned}
    \end{alignat}
    which is obtained by taking in \eqref{eq:parabolicPDEN_plain} the gradient w.r.t.\@ the NN parameters $\theta$,
    can be used as a test function in the weak formulation of \eqref{eq:parabolicadjointN_plain}, see \Cref{def:weak_sol_adjoint}, since $\nabla_\theta\uNv{\theta}{t,\dummy}\in H_0^1(D)$ for a.e.\@ $t\in[0,T]$.
    Existence, uniqueness and regularity of a weak solution to \eqref{eq:parabolicpI*aux_plain} in a sense analogous to \Cref{def:weak_sol_adjoint} follow from classical results as $\nabla_\theta g_\theta^N\in L_2(D_T)$,
    see, e.g., \cite[Chapter~7.1, Theorem~3]{evans2010partial} and \cite[Chapter~7.1, Theorem~4]{evans2010partial}.

    For the step in the fourth line of \eqref{eq:proof:parabolicnablaJ*1} we use partial integration and the definition of the adjoint bilinear form~$\CB^\dagger$ with the same argumentation as in the proof of \Cref{lem:parabolictimeevolutionJt*}.

    The penultimate step of \eqref{eq:proof:parabolicnablaJ*1} holds since $\nabla_\theta\uN{\theta}$ is a weak solution to the PDE~\eqref{eq:parabolicpI*aux_plain} and since $\uNhatv{\theta}{t,\dummy}$ is a suitable test function as it is in $H_0^1(D)$ for a.e.\@ $t\in[0,T]$.
\end{proof}
\section{PDE Considerations}
\label{sec:PDEconsiderations}
Leveraging that the loss $\J{\tau}$ defined in \eqref{eq:parabolicJtau} is non-increasing in the training time $\tau$ as established in \Cref{lem:parabolictimeevolutionJt*},
we provide in \Cref{sec:BoundednessSolution,sec:BoundednessAdjoint} uniform (in the training time~$\tau$) estimates for the norms of the PDE solution $\up{\tau}$ to \eqref{eq:parabolicPDE*} and its adjoint $\uhatp{\tau}$ in \eqref{eq:parabolicadjoint*}.
Those bounds are in particular independent of and thus uniform in the training time $\tau$, depending only on properties of the PDE and the NN initialization at training time~$\tau=0$.

We state the results for a training time interval $I$, which may be either $[0,\CT]$ or $[0,\infty)$.

\subsection{Boundedness of the PDE Solution \texorpdfstring{$\up{\tau}$}{} Uniformly in the Training Time}
\label{sec:BoundednessSolution}

The following uniform (in the training time $\tau$) bound on the $L_2$ norm of the PDE solution $\up{\tau}$ to \eqref{eq:parabolicPDE*}
is an immediate consequence of the loss $\J{\tau}$ being monotonically non-increasing.

\begin{lemma}
    \label{lem:parabolic_uL2}
    Let $((\up{\tau},\uhatp{\tau}))_{\tau\in I}\in \CC\left(I,\CS\times\CS\right)$ denote the unique weak solution to the PDE system \eqref{eq:parabolicPDE*}--\eqref{eq:parabolicadjoint*} coupled with the integro-differential equation~\eqref{eq:parabolicgtau} in the sense of \Cref{lem:parabolic_wellposedness_infinitewidth} and \Cref{rem:parabolic_wellposedness} on the training time interval $I$.
    Then the solution~$\up{\tau}$ is uniformly (in the training time $\tau$) bounded in $L_2(D_T)$ on that interval $I$, i.e., it holds
    \begin{equation}
        \sup_{\tau\in I}\N{\up{\tau}}_{L_2(D_T)}\leq C^{u}
    \end{equation}
    for the constant $C^{u}=4\J{0} + 2\N{h}_{L_2(D_T)}^2$.
\end{lemma}

\begin{proof}
    For the solution~$\up{\tau}$ to \eqref{eq:parabolicPDE*} we can compute with Young's inequality
    \begin{equation}
    \begin{split}
        \N{\up{\tau}}_{L_2(D_T)}^2
        &= \int_0^T\!\!\!\int_D \left(\upp{\tau}{t,x}\right)^2dxdt
        = \int_0^T\!\!\!\int_D \left(\upp{\tau}{t,x}-h(t,x)+h(t,x)\right)^2dxdt\\
        &\leq \int_0^T\!\!\!\int_D 2(\upp{\tau}{t,x}-h(t,x))^2 + 2(h(t,x))^2\,dxdt
        = 4\J{\tau} + 2\N{h}_{L_2(D_T)}^2 \\
        &\leq 4\J{0} + 2\N{h}_{L_2(D_T)}^2,
    \end{split}
    \end{equation}
    where the last step is a consequence of $\J{\tau}$ being monotonically non-increasing on the training time interval $I$ according to \Cref{lem:parabolictimeevolutionJt*}.
\end{proof}

\subsection{Boundedness of the Adjoint \texorpdfstring{$\uhatp{\tau}$}{} Uniformly in the Training Time}
\label{sec:BoundednessAdjoint}

Uniform (in the training time $\tau$) bounds on the $L_2([0,T],H^1(D))$- and $L_\infty([0,T],L_2(D))$-norms of the adjoint $\uhatp{\tau}$ in \eqref{eq:parabolicadjoint*}
are obtained via an energy estimate for the linear parabolic PDE~\eqref{eq:parabolicadjoint*} leveraging that the loss $\J{\tau}$ is monotonically non-increasing in the training time~$\tau$.
\begin{lemma}
    \label{lem:parabolic_uhatL2}
    Let $((\up{\tau},\uhatp{\tau}))_{\tau\in I}\in \CC\left(I,\CS\times\CS\right)$ denote the unique weak solution to the PDE system \eqref{eq:parabolicPDE*}--\eqref{eq:parabolicadjoint*} coupled with the integro-differential equation~\eqref{eq:parabolicgtau} in the sense of \Cref{lem:parabolic_wellposedness_infinitewidth} and \Cref{rem:parabolic_wellposedness} on the training time interval $I$.
    Then the adjoint~$\uhatp{\tau}$ in \eqref{eq:parabolicadjoint*} is uniformly (in the training time $\tau$) bounded in $L_2([0,T],H^1(D))$ and $L_\infty([0,T],L_2(D))$ on that interval $I$,
    i.e., it holds
    \begin{equation}
        \label{eq:lem:parabolic_uhatL2}
        \sup_{\tau\in I} \left(\N{\uhatp{\tau}}_{L_2([0,T],H^1(D))} + \N{\uhatp{\tau}}_{L_\infty([0,T],L_2(D))} \right)
        \leq C^{\widehat{u}}
    \end{equation}
    for a constant $C^{\widehat{u}}=C^{\widehat{u}}(T,\CL,\J{0})$.
\end{lemma}
\begin{proof}
    Let us first reverse the adjoint parabolic backward PDE \eqref{eq:parabolicadjoint*} in time to obtain with a time transformation
    for $\uhatpReverted{\tau}=\uhatppReverted{\tau}{t,x}=\uhatpp{\tau}{T-t,x}$ the parabolic forward PDE
    \begin{alignat}{2}  \label{eq:parabolicadjoint*Reversed2}
    \begin{aligned}        \fpartial{t}\uhatpReverted{\tau} + \underbar{\CL}^*\uhatpReverted{\tau} - \underbar{q}_u(\upp{\tau}{T-\dummy,\dummy})\uhatpReverted{\tau}
        &= (\upp{\tau}{T-\dummy,\dummy}-h(T-\dummy,\dummy))
        \quad&&\text{in }
        D_T, \\
        \uhatpReverted{\tau}
        &= 0
        \quad&&\text{on }
        [0,T]\times\partial D, \\
        \uhatpReverted{\tau}
        &= 0
        \quad&&\text{on }
        \{0\}\times D,
    \end{aligned}
    \end{alignat}
    where $\underbar{\CL}^*=\underbar{\CL}^*(t,x)=\CL^\dagger(T-t,x)$ and $\underbar{q}=\underbar{q}(t,x,u)=\underbar{q}(T-t,x,u)$.

    Let us now start by estimating $\N{\uhatppReverted{\tau}{t,\dummy}}^2_{L_2(D)} = \int_D (\uhatppReverted{\tau}{t,x})^2 \,dx$.
    With chain rule and by using that $\uhatpReverted{\tau}$ is a weak solution to the time-reversed adjoint PDE~\eqref{eq:parabolicadjoint*Reversed2} we have
    \begin{equation}        \label{eq:proof:lem:parabolic_uhatL2:1}
    \begin{split}
        \fpartial{t}\N{\uhatppReverted{\tau}{t,\dummy}}^2_{L_2(D)}
        &= 
        2\left(\uhatppReverted{\tau}{t,\dummy} ,\fpartial{t}\uhatppReverted{\tau}{t,\dummy}\right)_{L_2(D)}
        = 
        2\left\langle\fpartial{t}\uhatppReverted{\tau}{t,\dummy},\uhatppReverted{\tau}{t,\dummy}\right\rangle_{H^{-1}(D),H_0^1(D)}\\
        &= - 2\underbar{\CB}^*[\uhatppReverted{\tau}{t,\dummy},\uhatppReverted{\tau}{t,\dummy};t] + 2(\underbar{q}_u(t,\dummy,\upp{\tau}{T-t,\dummy})\uhatppReverted{\tau}{t,\dummy},\uhatppReverted{\tau}{t,\dummy})_{L_2(D)}\\
        &\quad\,+ 2\left(\upp{\tau}{T-t,\dummy}-h(T-t,\dummy),\uhatppReverted{\tau}{t,\dummy}\right)_{L_2(D)},
    \end{split}
    \end{equation}
    where $\underbar{\CB}^*[\widehat{u},u;t]=\CB^\dagger[\widehat{u},u;T-t]$.
    For the second step recall that since $\uhatppReverted{\tau}{t,\dummy}\in H_0^1(D)$ for a.e.\@ $t\in[0,T]$ and since $\fpartial{t}\uhatppReverted{\tau}{t,\dummy}\in L_2(D)$ for a.e.\@ $t\in[0,T]$, the dual pairing between $H^{-1}(D)$ and $H_0^1(D)$ coincides with the $L_2(D)$ scalar product \cite[Chapter~5.9, Theorem 1(iii)]{evans2010partial}.
    For the third step, i.e., the weak solution property, note that $\uhatppReverted{\tau}{t,\dummy}$ is a valid test function since $\uhatppReverted{\tau}{t,\dummy}\in H_0^1(D)$ for a.e.\@ $t\in[0,T]$.
    To upper bound the right-hand side of \eqref{eq:proof:lem:parabolic_uhatL2:1}, we consider each of the three terms separately.
    For the first term, by using the definition of the bilinear form $\CB$ as well as that by Assumption~\ref{asm:PDE_L_parabolicity} the PDE operator is uniformly parabolic and that by Assumption~\ref{asm:PDE_coefficients} the coefficients are in $L_\infty$, we can estimate with Cauchy-Schwarz and Young's inequality
    \begin{allowdisplaybreaks}
    \label{eq:proof:lem:parabolic_uhatL2:21}
    \begin{align}
        &-\underbar{\CB}^*[\uhatppReverted{\tau}{t,\dummy},\uhatppReverted{\tau}{t,\dummy};t]
        =-\CB^\dagger[\uhatppReverted{\tau}{t,\dummy},\uhatppReverted{\tau}{t,\dummy};T-t]
        =-\CB[\uhatppReverted{\tau}{t,\dummy},\uhatppReverted{\tau}{t,\dummy};T-t]\notag\\
        &\qquad\qquad\,=-\int_U \sum_{i,j=1}^d  a^{ij}(T-t,x) \fpartial{x_i} \uhatppReverted{\tau}{t,x} \fpartial{x_j} \uhatppReverted{\tau}{t,x}\notag\\
        &\qquad\qquad\,\qquad\,+ \sum_{i=1}^d b^i(T-t,x) \fpartial{x_i} \uhatppReverted{\tau}{t,x} \uhatppReverted{\tau}{t,x} + c(T-t,x)\uhatppReverted{\tau}{t,x}\uhatppReverted{\tau}{t,x} \,dx\notag\\
        &\qquad\qquad\,\leq\int_U -\nu \N{\nabla_x \uhatppReverted{\tau}{t,x}}^2+\frac{\nu}{2} \N{\nabla_x \uhatppReverted{\tau}{t,x}}^2\notag\\
        &\qquad\qquad\,\qquad\,+ \frac{1}{2\nu} \sum_{i=1}^d\N{b^i}_{L_\infty(D_T)}(\uhatppReverted{\tau}{t,x})^2 + \N{c}_{L_\infty(D_T)}(\uhatppReverted{\tau}{t,x})^2 \,dx\notag\\
        &\qquad\qquad\,\leq -\frac{\nu}{2} \abs{\uhatppReverted{\tau}{t,\dummy}}^2_{H^1(D)} + \left(\frac{1}{2\nu} \sum_{i=1}^d\N{b^i}_{L_\infty(D_T)}+\N{c}_{L_\infty(D_T)}\right)\N{\uhatppReverted{\tau}{t,\dummy}}^2_{L_2(D)},
    \end{align}
    \end{allowdisplaybreaks}
    where for the middle term in the next-to-last step we note that with Young's inequality it holds
    \begin{equation}
    \begin{split}
        \sum_{i=1}^d b^i(T-t,x) \fpartial{x_i} \uhatppReverted{\tau}{t,x} \uhatppReverted{\tau}{t,x}
        &\leq
        \sum_{i=1}^d \left(\frac{\nu}{2}(\fpartial{x_i} \uhatppReverted{\tau}{t,x})^2 + \frac{1}{2\nu} (b^i(T-t,x) \uhatppReverted{\tau}{t,x})^2\right)\\
        &\leq \frac{\nu}{2} \N{\nabla_x \uhatppReverted{\tau}{t,x}}^2 + \frac{1}{2\nu} \sum_{i=1}^d\N{b^i}_{L_\infty(D_T)}(\uhatppReverted{\tau}{t,x})^2.
    \end{split}
    \end{equation}
    For the second term, by using that by Assumption~\ref{asm:PDE_nonlinearity_q_ubdd} $q_u$ is bounded, we can estimate
    \begin{equation}        \label{eq:proof:lem:parabolic_uhatL2:22}
    \begin{split}
        (\underbar{q}_u(t,\dummy,\upp{\tau}{T-t,\dummy})\uhatppReverted{\tau}{t,\dummy},\uhatppReverted{\tau}{t,\dummy})_{L_2(D)}
        &= (q_u(T-t,\dummy,\upp{\tau}{T-t,\dummy})\uhatppReverted{\tau}{t,\dummy},\uhatppReverted{\tau}{t,\dummy})_{L_2(D)}\\
        &\leq c_q\N{\uhatppReverted{\tau}{t,\dummy}}^2_{L_2(D)}.
    \end{split}
    \end{equation}
    For the third and last term, by Cauchy-Schwarz and Young's inequality we can derive the upper bound
    \begin{equation}       \label{eq:proof:lem:parabolic_uhatL2:23}
    \begin{split}
        &\left(\upp{\tau}{T-t,\dummy}-h(T-t,\dummy),\uhatppReverted{\tau}{t,\dummy}\right)_{L_2(D)}
        \leq \N{\upp{\tau}{T-t,\dummy}-h(T-t,\dummy)}_{L_2(D)}\N{\uhatppReverted{\tau}{t,\dummy}}_{L_2(D)}\\
        &\qquad\qquad\,\leq \frac{1}{2}\left(\N{\upp{\tau}{T-t,\dummy}-h(T-t,\dummy)}^2_{L_2(D)}+\N{\uhatppReverted{\tau}{t,\dummy}}^2_{L_2(D)}\right).
    \end{split}
    \end{equation}
   
    Combining the bounds established in \eqref{eq:proof:lem:parabolic_uhatL2:21}--\eqref{eq:proof:lem:parabolic_uhatL2:23} and inserting them into \eqref{eq:proof:lem:parabolic_uhatL2:1}, we can continue bounding \eqref{eq:proof:lem:parabolic_uhatL2:1} as 
    \begin{equation}
        \label{eq:proof:lem:parabolic_uhatL2:3}
    \begin{split}
        &\fpartial{t}\N{\uhatppReverted{\tau}{t,\dummy}}^2_{L_2(D)} + \frac{\nu}{2} \abs{\uhatppReverted{\tau}{t,\dummy}}^2_{H^1(D)} \\
        &\qquad\,\leq \left(\frac{1}{\nu} \sum_{i=1}^d\N{b^i}_{L_\infty(D_T)}+2\N{c}_{L_\infty(D_T)}\right) \N{\uhatppReverted{\tau}{t,\dummy}}^2_{L_2(D)} + 2c_q\N{\uhatppReverted{\tau}{t,\dummy}}^2_{L_2(D)}\\
        &\qquad\,\quad\,+ \left(\N{\upp{\tau}{T-t,\dummy}-h(T-t,\dummy)}^2_{L_2(D)}+\N{\uhatppReverted{\tau}{t,\dummy}}^2_{L_2(D)}\right)\\
        &\qquad\,\leq C\N{\uhatppReverted{\tau}{t,\dummy}}^2_{L_2(D)} + \N{\upp{\tau}{T-t,\dummy}-h(T-t,\dummy)}^2_{L_2(D)}
    \end{split}
    \end{equation}
    for a constant $C=C(\CL,q)$.
    Defining $\widehat{N}_\tau(t) = \Nnormal{\uhatppReverted{\tau}{t,\dummy}}^2_{L_2(D)} + \frac{\nu}{2} \int_0^t\absnormal{\uhatppReverted{\tau}{s,\dummy}}^2_{H^1(D)}\,ds$,    \eqref{eq:proof:lem:parabolic_uhatL2:3} translates to
    \begin{equation}           \label{eq:proof:lem:parabolic_wellposedness_infinitewidth:NORM2_6}
        \fpartial{t} \widehat{N}_\tau(t)
        \leq C\widehat{N}_\tau(t) + \N{\upp{\tau}{T-t,\dummy}-h(T-t,\dummy)}_{L_2(D)}^2.
    \end{equation}
    We can now employ Grönwall's inequality 
    to obtain
    \begin{equation}
        \label{eq:proof:lem:parabolic_wellposedness_infinitewidth:NORM2_8}
    \begin{split}
        \widehat{N}_\tau(t)
        &\leq \left(\widehat{N}_\tau(0)+\N{\up{\tau}-h}_{L_2(D_T)}^2\right)e^{Ct}
        \leq \left(\widehat{N}_\tau(0)+\N{\up{\tau}-h}_{L_2(D_T)}^2\right)e^{CT}.
    \end{split}
    \end{equation}
    Recalling that $\widehat{N}_\tau(0)=0$ by the initial condition in \eqref{eq:parabolicadjoint*Reversed2} shows
    \begin{equation}
        \label{eq:proof:adjointL2bounduniform}
    \begin{split}
        \N{\uhatp{\tau}}_{L_2([0,T],H^1(D))} + \N{\uhatp{\tau}}_{L_\infty([0,T],L_2(D))} 
        \leq 2\J{\tau}e^{Ct} \leq 2\J{0}e^{CT}
    \end{split}
    \end{equation}
    where the last step is a consequence of $\J{\tau}$ being monotonically non-increasing on the training time interval $I$ according to \Cref{lem:parabolictimeevolutionJt*}, which concludes the proof.
\end{proof}
\section{The Functional \texorpdfstring{$\Q{\tau} = (\uhatp{\tau},T_{B_0}\uhatp{\tau})_{L_2(D_T)}$}{Q}}
\label{sec:FunctionalCQtau}

This section is dedicated to proving in \Cref{lem:CQ_Regulartiy} in \Cref{sec:RegularityCQtau}
a regularity bound for the functional
\begin{equation}
    \label{eq:CQ}
\begin{split}
    \Q{\tau}
    = (\uhatp{\tau},T_{B_0}\uhatp{\tau})_{L_2(D_T)}
\end{split}
\end{equation}
of the form
\begin{equation}
    \absbig{\Q{\tau_2}-\Q{\tau_1}}
    \leq L_\CQ\int_{\tau_1}^{\tau_2} \alpha_\tau \,d\tau,
\end{equation}
which holds for all $0\leq\tau_1\leq\tau_2$,
for a constant $L_\CQ>0$ as specified after \eqref{eq:lem:CQ_Regulartiy}.
Here, the operator $T_{B_0}$ is defined in \eqref{eq:parabolicT_B} and the kernel $B_0=B(\mu_0)$ is as in \eqref{eq:parabolicB}.

In order to derive this bound, let us introduce for the functional~$\Q{\tau}$ defined in \eqref{eq:CQ} and the PDE system~\eqref{eq:parabolicPDE*}--\eqref{eq:parabolicadjoint*}
the second-level adjoint system with variables $(\vhatp{},\whatp{})$ given by
\begin{alignat}{2}
    \label{eq:paraboliccoupledadjoint1}
\begin{aligned}
    -\fpartial{t}\vhatp{\tau} + \CL^\dagger\vhatp{\tau} - q_u(\up{\tau})\vhatp{\tau}
    &= \whatp{\tau} + q_{uu}(\up{\tau})\uhatp{\tau}\whatp{\tau}
    \qquad&&\text{in }
    D_T, \\
    \vhatp{\tau}
    &= 0
    \qquad&&\text{on }
    [0,T]\times\partial D, \\
    \vhatp{\tau}
    &= 0
    \qquad&&\text{on }
    \{T\}\times D,
\end{aligned}
\end{alignat}
and
\begin{alignat}{2}
    \label{eq:paraboliccoupledadjoint2}
\begin{aligned}
    \fpartial{t}\whatp{\tau} + \CL\whatp{\tau} - q_u(\up{\tau})\whatp{\tau}
    &= 2T_{B_0}\uhatp{\tau}
    \qquad&&\text{in }
    D_T, \\
    \whatp{\tau}
    &= 0
    \qquad&&\text{on }
    [0,T]\times\partial D, \\
    \whatp{\tau}
    &= 0
    \qquad&&\text{on }
    \{0\}\times D.
\end{aligned}
\end{alignat}

Before discussing the main statement of this section, \Cref{lem:CQ_Regulartiy},
we establish in \Cref{sec:parabolic_whatLinfty,sec:parabolic_vhatL2} uniform (in the training time~$\tau$) estimates for several norms of the second-level adjoints $\whatp{\tau}$ in \eqref{eq:paraboliccoupledadjoint1} and $\vhatp{\tau}$ in \eqref{eq:paraboliccoupledadjoint2}, respectively.

\subsection{Boundedness of the Second-Level Adjoint \texorpdfstring{$\whatp{\tau}$}{of the Adjoint} Uniformly in the Training Time}
\label{sec:parabolic_whatLinfty}

We show well-posedness of the second-level adjoint $\whatp{\tau}$ in \eqref{eq:paraboliccoupledadjoint2} and derive uniform (in the training time $\tau$) bounds on its $L_\infty(D_T)$-, $L_2(D_T)$- and $L_\infty([0,T],L_2(D))$-norms.
The uniformity of the bound in time and space is a consequence of the right-hand side of \eqref{eq:paraboliccoupledadjoint2} being in $L_\infty$ as of \Cref{lem:TBboundednessLinfty},
while the uniformity in the training time $\tau$ follows from the uniformity of the bound on the adjoint~$\uhatp{\tau}$ as of \Cref{lem:parabolic_uhatL2}.
\begin{lemma}
    \label{lem:parabolic_whatLinfty}
    Let $((\up{\tau},\uhatp{\tau}))_{\tau\in I}\in \CC\left(I,\CS\times\CS\right)$ denote the unique weak solution to the PDE system \eqref{eq:parabolicPDE*}--\eqref{eq:parabolicadjoint*} coupled with the integro-differential equation~\eqref{eq:parabolicgtau} in the sense of \Cref{lem:parabolic_wellposedness_infinitewidth} and \Cref{rem:parabolic_wellposedness} on the training time interval $I$.
    Then the linear parabolic PDE~\eqref{eq:paraboliccoupledadjoint2} admits for every $\tau\in I$ a unique weak solution~$\whatp{\tau}$
    in a sense analogous to \Cref{def:weak_sol_adjoint} with right-hand side $2T_{B_0}\uhatp{\tau}$,
    which satisfies $\partial_t\whatpp{\tau}{t,\dummy} \in L_2(D)$
    for a.e.\@ $t\in[0,T]$.

    In addition, the adjoint~$\whatp{\tau}$ in \eqref{eq:paraboliccoupledadjoint2} is uniformly (in the training time $\tau$) bounded in $L_\infty(D_T)$ on that interval $I$, i.e., it holds
    \begin{equation}
        \sup_{\tau\in I}\N{\whatp{\tau}}_{L_\infty(D_T)}
        \leq C^{\widehat{w}}_\infty
    \end{equation}
    for a constant $C^{\widehat{w}}_\infty=C^{\widehat{w}}_\infty(T,\CL,q,C^B_2,C^{\widehat{u}})$.
    Furthermore, the adjoint~$\whatp{\tau}$ is uniformly (in the training time $\tau$) bounded in $L_2(D_T)$ and $L_\infty([0,T],L_2(D))$ on that interval $I$,
    i.e., it holds
    \begin{equation}
        \sup_{\tau\in I}\left(\N{\whatp{\tau}}_{L_2(D_T)} + \N{\whatp{\tau}}_{L_\infty([0,T],L_2(D))}\right)\leq C^{\widehat{w}}
    \end{equation}
    for a constant $C^{\widehat{w}}=C^{\widehat{w}}(D,C^{\widehat{w}}_\infty)$.
\end{lemma}

\begin{proof}
    \textit{Step 1a: Existence of a unique weak solution $\whatp{\tau}$.} 
    Existence and uniqueness of a weak solution to \eqref{eq:paraboliccoupledadjoint2} in a sense analogous to \Cref{def:weak_sol_adjoint} follow from classical results,
    see, e.g., \cite[Chapter~7.1, Theorem~3]{evans2010partial} and \cite[Chapter~7.1, Theorem~4]{evans2010partial}, as $2T_{B_0}\uhatp{\tau}\in L_2(D_T)$ according to \Cref{lem:parabolicTB,lem:parabolic_uhatL2}.

    The remainder of the statement follows from an application of Morrey's inequality after leveraging \cite[Chapter IV, Theorem 9.1]{ladyzhenskaia1968linear} for any $p\geq2$.
    
    \textit{Step 1b: Existence of a unique solution $\whatp{\tau}\in W^{1,2}_{p}(D_T)$ for any $p\geq 2$.}
    We first notice that, in the notation of \cite[Chapter IV, Theorem 9.1]{ladyzhenskaia1968linear}, the coefficients~$a_{ij}(t,x)=a^{ij}(t,x)$ of the linear PDE operator of the parabolic PDE~\eqref{eq:paraboliccoupledadjoint2} are bounded continuous functions in $D_T$ for all $i,j=1,\dots,d$,
    while the coefficients~$a_i(t,x)=b^i(t,x)-\sum_{j=1}^d\fpartial{x_j}a^{ji}(t,x)$ and $a(t,x)=c(t,x) - \sum_{i=1}^d \fpartial{x_i}b^i(t,x) - q_u(t,x,u(t,x))$ have finite norms $\N{a_i}_{L_r(D_T)}$ and $\N{a}_{L_s(D_T)}$ for any $r,s>0$.
    This is due to the uniform boundedness of the coefficients per Assumptions~\ref{asm:PDE_coefficients} and \ref{asm:PDE_nonlinearity_q_ubdd} combined with the boundedness of the domain per Assumption~\ref{asm:Dbdd}, see the subsequent computations with $T'=0$ and $\Delta T'=T$.
    Moreover, since it hold
    $\Nnormal{a_i}_{L_r(D_{T',T'+\Delta T'})} \leq \big(\Nnormal{b^i}_{L_\infty(D_T)}+\sum_{j=1}^d\Nnormal{\fpartial{x_j}a^{ji}}_{L_\infty(D_T)}\big)(\Delta T'\vol{D})^{1/r}$
    for all $i=1,\dots,d$ and
    $\Nnormal{a}_{L_s(D_{T',T'+\Delta T'})}\leq \big(\Nnormal{c}_{L_\infty(D_T)}+\sum_{i=1}^d\Nnormal{\fpartial{x_i}b^{i}}_{L_\infty(D_T)} + c_q\big)(\Delta T'\vol{D})^{1/s}$, 
    $\N{a_i}_{L_r(D_{T',T'+\Delta T'})}$ and $\N{a}_{L_s(D_{T',T'+\Delta T'})}$ tend to zero as $\Delta T'\rightarrow0$.
    Furthermore, $\partial D$ is sufficiently smooth as of Assumption~\ref{asm:D}.
    The right-hand side $f=2T_{B_0}\uhatp{\tau}\in L_p(D_T)$ for any $p\geq2$ due to being uniformly (in the training time $\tau$) bounded in $L_\infty$ as of \Cref{lem:TBboundednessLinfty}, \Cref{lem:parabolic_uhatL2} and the domain $D_T$ being bounded as of Assumption~\ref{asm:Dbdd},
    which ensures
    \begin{equation}
        \label{eq:proof:w_uniformbdd_TBubound}
    \begin{split}
        \N{2T_{B_0}\uhatp{\tau}}_{L_p(D_T)}
        &= \left(\int_0^T\!\!\!\int_D \absbig{[2T_{B}\uhatp{\tau}](t,x)}^p\,dxdt\right)^{1/p}\\
        &\leq 2C^{T_{B_0}}_\infty\N{\uhatp{\tau}}_{L_2(D_T)}(T\vol{D})^{1/p}
        \leq2C^{T_{B_0}}_\infty C^{\widehat{u}}(T\vol{D})^{1/p}.
    \end{split}
    \end{equation}
    Moreover, both the initial and the boundary conditions $\phi=0\in W^{2-2/p}_p(D)$ and $\Phi=0\in W^{1-1/(2p),2-1/p}_p(\partial D_T)$ trivially satisfy the compatibility condition $\phi|_{\partial D}=\Phi|_{t=0}$.
    Thus, \cite[Chapter IV, Theorem 9.1]{ladyzhenskaia1968linear} ensures  that $\whatp{\tau}\in W^{1,2}_{p}(D_T)$ obeys the bound
    \begin{equation}
        \label{eq:proof:w_uniformbdd_bound}
    \begin{split}
        \N{\whatp{\tau}}_{W^{1,2}_{p}(D_T)}
        &\leq C\N{2T_{B_0}\uhatp{\tau}}_{L_p(D_T)}
    \end{split}
    \end{equation}
    for a constant $C=C(T,\CL,q)$.
    With the uniform (in the training time $\tau$) bound \eqref{eq:proof:w_uniformbdd_TBubound} at our disposal,
    $\N{\whatp{\tau}}_{W^{1,2}_{p}(D_T)}$ can be controlled uniformly (in the training time $\tau$) as
    \begin{equation}
        \label{eq:proof:w_uniformbdd_bound_ctd}
    \begin{split}
        \N{\whatp{\tau}}_{W^{1,2}_{p}(D_T)}
        &\leq 2CC^{T_{B}}_\infty C^{\widehat{u}}(T\vol{D})^{1/p}.
    \end{split}
    \end{equation}
    This in particular proves that $\whatp{\tau}\in W^{1,2}_{2}(D_T)$ obeying \eqref{eq:proof:w_uniformbdd_bound_ctd} with $p=2$, concluding the first part of the statement since $\partial_t\whatpp{\tau}{t,\dummy} \in L_2(D)$ has to necessarily hold for a.e.\@ $t\in[0,T]$.

    \textit{Step 2a: Boundedness of the $L_\infty(D_T)$ norm of $\whatp{\tau}$.}
    With the conditions of \cite[Chapter IV, Theorem 9.1]{ladyzhenskaia1968linear} being fulfilled for any $p\geq2$ as we verified before, they are in particular fulfilled for $p>d+1$.
    Since we have for such $p$ the continuous embedding $W^{1,2}_{p}(D_{T})\hookrightarrow W^{1,1}_{p}(D_{T})\hookrightarrow L_{\infty}(\overbar{D}_T)$ by Morrey's inequality~\cite[Theorem 9.12]{brezis2011functional}, we have the first inequality in
    \begin{equation}
    \begin{split}
        \N{\whatp{\tau}}_{L_\infty(D_T)}
        &\leq c(d,p)\N{\whatp{\tau}}_{W^{1,2}_{p}(D_T)}
        \leq 2c(d,p)CC^{T_{B_0}}_\infty C^{\widehat{u}}(T\vol{D})^{1/p},
    \end{split}
    \end{equation}
    with the second one being due to \eqref{eq:proof:w_uniformbdd_bound_ctd}.
    As the right-hand side is bounded uniformly (in the training time $\tau$),
    and since $\whatp{\tau}\in W^{1,2}_{p}(D_{T})$ has a continuous version~\cite[Chapter~5.6, Theorem~5]{evans2010partial},
    this concludes the second part of the statement.

    \textit{Step 2b: Boundedness of the $L_2(D_T)$ and $L_\infty([0,T],L_2(D))$ norms of $\whatp{\tau}$.}
    The last part of the statement follows since
    $\N{\whatp{\tau}}_{L_\infty([0,T],L_2(D))}\leq \sqrt{\vol{D}}\N{\whatp{\tau}}_{L_\infty(D_T)}$ and $\N{\whatp{\tau}}_{L_2(D_T)}\leq \sqrt{T\vol{D}}\N{\whatp{\tau}}_{L_\infty(D_T)}$.
\end{proof}

\subsection{Boundedness of the Second-Level Adjoint \texorpdfstring{$\vhatp{\tau}$}{of the Solution} Uniformly in the Training Time}
\label{sec:parabolic_vhatL2}

We now show well-posedness of the second-level adjoint $\vhatp{\tau}$ in \eqref{eq:paraboliccoupledadjoint1} and derive uniform (in the training time $\tau$) bounds on its $L_2([0,T],H^1(D))$- and $L_\infty([0,T],L_2(D))$-norms.
\begin{lemma}
    \label{lem:parabolic_vhatL2}
    Let $((\up{\tau},\uhatp{\tau}))_{\tau\in I}\in \CC\left(I,\CS\times\CS\right)$ denote the unique weak solution to the PDE system \eqref{eq:parabolicPDE*}--\eqref{eq:parabolicadjoint*} coupled with the integro-differential equation~\eqref{eq:parabolicgtau} in the sense of \Cref{lem:parabolic_wellposedness_infinitewidth} and \Cref{rem:parabolic_wellposedness}
    on the training time interval $I$.
    Then the linear parabolic PDE~\eqref{eq:paraboliccoupledadjoint1} admits for every $\tau\in I$ a unique weak solution~$\vhatp{\tau}$
    in a sense analogous to \Cref{def:weak_sol_adjoint} with right-hand side $\whatp{\tau} + q_{uu}(\up{\tau})\uhatp{\tau}\whatp{\tau}$,
    which satisfies $\partial_t\vhatpp{\tau}{t,\dummy} \in L_2(D)$
    for a.e.\@ $t\in[0,T]$.

    In addition, the adjoint~$\vhatp{\tau}$ in \eqref{eq:paraboliccoupledadjoint1} is uniformly (in the training time $\tau$) bounded in $L_2([0,T],H_1(D))$ and $L_\infty([0,T],L_2(D))$ on that interval $I$,
    i.e., it holds
    \begin{equation}
        \sup_{\tau\in I}\left(\N{\vhatp{\tau}}_{L_2([0,T],H^1(D))} + \N{\vhatp{\tau}}_{L_\infty([0,T],L_2(D))} \right)\leq C^{\widehat{v}}
    \end{equation}
    for a constant $C^{\widehat{v}}=C^{\widehat{v}}(T,\CL,q,C^{\widehat{u}},C^{\widehat{w}},C^{\widehat{w}}_\infty)$.
\end{lemma}
\begin{proof}
    Let us first reverse the adjoint parabolic backward PDE \eqref{eq:paraboliccoupledadjoint1} in time to obtain with a time transformation
    for $\vhatpReverted{\tau}=\vhatppReverted{\tau}{t,x}=\vhatpp{\tau}{T-t,x}$ the parabolic forward PDE
    \begin{alignat}{2}
    \label{eq:paraboliccoupledadjoint1*Reversed}
    \begin{aligned}
        \fpartial{t}\vhatpReverted{\tau} + \underbar{\CL}^*\vhatpReverted{\tau} - \underbar{q}_u(\upp{\tau}{T-\dummy,\dummy})\vhatpReverted{\tau}
        &= \whatpp{\tau}{T-\dummy,\dummy}+ \underbar{q}_{uu}(\upp{\tau}{T-\dummy,\dummy})\cdot\\
        &\qquad\qquad\quad\,\cdot\uhatpp{\tau}{T-\dummy,\dummy}\whatpp{\tau}{T-\dummy,\dummy}
        \quad&&\text{in }
        D_T, \\
        \vhatpReverted{\tau}
        &= 0
        \quad&&\text{on }
        [0,T]\times\partial D, \\
        \vhatpReverted{\tau}
        &= 0
        \quad&&\text{on }
        \{0\}\times D,
    \end{aligned}
    \end{alignat}
    where $\underbar{\CL}^*=\underbar{\CL}^*(t,x)=\CL^\dagger(T-t,x)$ and $\underbar{q}=\underbar{q}(t,x,u)=\underbar{q}(T-t,x,u)$.

    \textit{Step 1: Existence of a unique solution $\vhatp{\tau}$.}
    Existence, uniqueness and regulartiy of a weak solution to \eqref{eq:paraboliccoupledadjoint1} in a sense analogous to \Cref{def:weak_sol_adjoint} follow analogously to \textit{Steps 1a} and \textit{b} of the proof of \Cref{lem:parabolic_whatLinfty} from classical results, namely \cite[Chapter~7.1, Theorem~3]{evans2010partial} and \cite[Chapter~7.1, Theorem~4]{evans2010partial} as well as \cite[Chapter IV, Theorem 9.1]{ladyzhenskaia1968linear} for $p=2$.
    Herefore note that $\whatp{\tau} + q_{uu}(\up{\tau})\uhatp{\tau}\whatp{\tau}\in L_2(D_T)$ by combining \Cref{lem:parabolic_whatLinfty,lem:parabolic_uhatL2} with Assumption~\ref{asm:PDE_nonlinearity_q_uubdd}. 

    \textit{Step 2: Boundedness of the $L_2([0,T],H^1(D))$ and $L_\infty([0,T],L_2(D))$ norms of $\vhatp{\tau}$.}
    Let us now estimate $\N{\vhatppReverted{\tau}{t,\dummy}}^2_{L_2(D)} = \int_D (\vhatppReverted{\tau}{t,x})^2 \,dx$.
    With chain rule and by using that $\vhatpReverted{\tau}$ is a weak solution to the time-reversed adjoint PDE~\eqref{eq:paraboliccoupledadjoint1*Reversed} we have 
    \begin{equation}
        \label{eq:proof:lem:parabolic_vhatL2:1}
    \begin{split}
        &\fpartial{t}\N{\vhatppReverted{\tau}{t,\dummy}}^2_{L_2(D)}
        =2\left(\vhatppReverted{\tau}{t,\dummy} ,\fpartial{t}\vhatppReverted{\tau}{t,\dummy}\right)_{L_2(D)}
        = 
        2\left\langle\fpartial{t}\vhatppReverted{\tau}{t,\dummy},\vhatppReverted{\tau}{t,\dummy}\right\rangle_{H^{-1}(D),H_0^1(D)}\\
        &\quad\,= - 2\underbar{\CB}^*[\vhatppReverted{\tau}{t,\dummy},\vhatppReverted{\tau}{t,\dummy};t] + 2(\underbar{q}_u(t,\dummy,\upp{\tau}{T\!-\!t,\dummy})\vhatppReverted{\tau}{t,\dummy},\vhatppReverted{\tau}{t,\dummy})_{L_2(D)}\\
        &\quad\,\quad\,+ 2\left(\whatpp{\tau}{T\!-\!t,\dummy} + \underbar{q}_{uu}(t,\dummy,\upp{\tau}{T\!-\!t,\dummy})\uhatpp{\tau}{T\!-\!t,\dummy}\whatpp{\tau}{T\!-\!t,\dummy},\vhatppReverted{\tau}{t,\dummy}\right)_{L_2(D)},
    \end{split}
    \end{equation}
    where $\underbar{\CB}^*[\widehat{u},u;t]=\CB^\dagger[\widehat{u},u;T-t]$ and 
    where we recall for the second step that since $\vhatppReverted{\tau}{t,\dummy}\in H_0^1(D)$ for a.e.\@ $t\in[0,T]$ and since $\fpartial{t}\vhatppReverted{\tau}{t,\dummy}\in L_2(D)$ for a.e.\@ $t\in[0,T]$, the dual pairing between $H^{-1}(D)$ and $H_0^1(D)$ coincides with the $L_2(D)$ scalar product \cite[Chapter~5.9, Theorem 1(iii)]{evans2010partial}.
    For the third step, i.e., the weak solution property, note that $\vhatppReverted{\tau}{t,\dummy}$ is a valid test function since $\vhatppReverted{\tau}{t,\dummy}\in H_0^1(D)$ for a.e.\@ $t\in[0,T]$.
    To upper bound the right-hand side of \eqref{eq:proof:lem:parabolic_vhatL2:1}, we consider again each of the three terms separately.
    Analogously to \eqref{eq:proof:lem:parabolic_uhatL2:21} we have for the first term
    \begin{equation}
        \label{eq:proof:lem:parabolic_vhatL2:21}
    \begin{split}
        &-\underbar{\CB}^*[\vhatppReverted{\tau}{t,\dummy},\vhatppReverted{\tau}{t,\dummy};t]\\
        &\qquad\qquad\,\leq -\frac{\nu}{2} \abs{\vhatppReverted{\tau}{t,\dummy}}^2_{H^1(D)} + \left(\frac{1}{2\nu} \sum_{i=1}^d\N{b^i}_{L_\infty(D_T)}+\N{c}_{L_\infty(D_T)}\right)\N{\vhatppReverted{\tau}{t,\dummy}}^2_{L_2(D)},
    \end{split}
    \end{equation}
    where we used the definition of the bilinear form $\CB$ as well as that by Assumption~\ref{asm:PDE_L_parabolicity} the PDE operator is uniformly parabolic and that by Assumption~\ref{asm:PDE_coefficients} the coefficients are in $L_\infty$.
    For the second term we have as in \eqref{eq:proof:lem:parabolic_uhatL2:22}  with Assumption~\ref{asm:PDE_nonlinearity_q_ubdd} that
    \begin{equation}
        \label{eq:proof:lem:parabolic_vhatL2:22}
    \begin{split}
        (\underbar{q}_u(t,\dummy,\upp{\tau}{T-t,\dummy})\vhatppReverted{\tau}{t,\dummy},\vhatppReverted{\tau}{t,\dummy})_{L_2(D)}
        &\leq c_q\N{\vhatppReverted{\tau}{t,\dummy}}^2_{L_2(D)}.
    \end{split}
    \end{equation}
    For the third and last term, using Assumption~\ref{asm:PDE_nonlinearity_q_uubdd}, by Cauchy-Schwarz, Hölder's and Young's inequality we upper bound
    \begin{align}       \label{eq:proof:lem:parabolic_vhatL2:23}
        &\left(\whatpp{\tau}{T\!-\!t,\dummy} + \underbar{q}_{uu}(T\!-\!t,\dummy,\upp{\tau}{T\!-\!t,\dummy})\uhatpp{\tau}{T\!-\!t,\dummy}\whatpp{\tau}{T\!-\!t,\dummy},\vhatppReverted{\tau}{t,\dummy}\right)_{L_2(D)}\notag\\
        &\quad\,\leq \left(\N{\whatpp{\tau}{T\!-\!t,\dummy}}_{L_2(D)} + c_q' \N{\uhatpp{\tau}{T\!-\!t,\dummy}\whatpp{\tau}{T\!-\!t,\dummy}}_{L_2(D)}\right)\N{\vhatppReverted{\tau}{t,\dummy}}_{L_2(D)} \notag\\
        &\quad\,\leq \left(\N{\whatpp{\tau}{T\!-\!t,\dummy}}_{L_2(D)} + c_q' \N{\uhatpp{\tau}{T\!-\!t,\dummy}}_{L_2(D)}\N{\whatpp{\tau}{T\!-\!t,\dummy}}_{L_\infty(D)}\right)\N{\vhatppReverted{\tau}{t,\dummy}}_{L_2(D)} \notag\\
        &\quad\,\leq \frac{1}{2}\left(\left(\N{\whatpp{\tau}{T\!-\!t,\dummy}}_{L_2(D)} \!+\! c_q' \N{\uhatpp{\tau}{T\!-\!t,\dummy}}_{L_2(D)}\N{\whatpp{\tau}{T\!-\!t,\dummy}}_{L_\infty(D)}\right)^2\!+\!\N{\vhatppReverted{\tau}{t,\dummy}}^2_{L_2(D)}\right)\notag\\
        &\quad\,\leq \frac{1}{2}\left(\left(C^{\widehat{w}} + c_q' C^{\widehat{u}}C^{\widehat{w}}_\infty\right)^2+\N{\vhatppReverted{\tau}{t,\dummy}}^2_{L_2(D)}\right),
    \end{align}
    where we employed \Cref{lem:parabolic_uhatL2} to bound the $L_2$ norm of $\uhatp{\tau}$ and \Cref{lem:parabolic_whatLinfty} to control the $L_2$ and $L_\infty$ norms of $\whatp{\tau}$.
    Combining the bounds established in \eqref{eq:proof:lem:parabolic_vhatL2:21}--\eqref{eq:proof:lem:parabolic_vhatL2:23} and inserting them into \eqref{eq:proof:lem:parabolic_vhatL2:1}, we can continue bounding \eqref{eq:proof:lem:parabolic_vhatL2:1} as
    \begin{align}
        \label{eq:proof:lem:parabolic_vhatL2:3}
        &\fpartial{t}\N{\vhatppReverted{\tau}{t,\dummy}}^2_{L_2(D)} + \frac{\nu}{2} \abs{\vhatppReverted{\tau}{t,\dummy}}^2_{H^1(D)} 
        \leq \left(\frac{1}{\nu} \sum_{i=1}^d\N{b^i}_{L_\infty(D_T)}+2\N{c}_{L_\infty(D_T)}\right) \N{\vhatppReverted{\tau}{t,\dummy}}^2_{L_2(D)}\!\!\!\!\!\!\!\!\!\!\!\!\!\!\!\!\notag\\
        &\qquad\, \quad\,+ 2c_q\N{\vhatppReverted{\tau}{t,\dummy}}^2_{L_2(D)} + \left(\left(C^{\widehat{w}} + c_q' C^{\widehat{u}}C^{\widehat{w}}_\infty\right)^2+\N{\vhatppReverted{\tau}{t,\dummy}}^2_{L_2(D)}\right)\notag\\
        &\qquad\, \leq C\N{\vhatppReverted{\tau}{t,\dummy}}^2_{L_2(D)} + \left(C^{\widehat{w}} + c_q' C^{\widehat{u}}C^{\widehat{w}}_\infty\right)^2
    \end{align}
    for a constant $C=C(\CL,q)$.
    Recalling that $\vhatppReverted{\tau}{0,\dummy}=0$ by the initial condition in \eqref{eq:paraboliccoupledadjoint1*Reversed}, an application of Grönwall's inequality shows
    \begin{equation}
    \begin{split}
        \N{\vhatp{\tau}}_{L_2([0,T],H^1(D))} + \N{\vhatp{\tau}}_{L_\infty([0,T],L_2(D))}
        &\leq \left(C^{\widehat{w}} + c_q' C^{\widehat{u}}C^{\widehat{w}}_\infty\right)^2Te^{CT},
    \end{split}
    \end{equation}
    which concludes the proof.
\end{proof}
\subsection{Regularity Bound for the Functional \texorpdfstring{$\Q{\tau}$}{Q} in Terms of the Learning Rate}
\label{sec:RegularityCQtau}

We now have all technical tools at hand to 
derive a regularity bound for the functional~$\Q{\tau}$ in terms of the learning rate~$\alpha_\tau$,
which is the main result of this section.

\begin{proposition}
    \label{lem:CQ_Regulartiy}
    Let $((\up{\tau},\uhatp{\tau}))_{\tau\in I}\in \CC\left(I,\CS\times\CS\right)$ denote the unique weak solution to the PDE system \eqref{eq:parabolicPDE*}--\eqref{eq:parabolicadjoint*} coupled with the integro-differential equation~\eqref{eq:parabolicgtau} in the sense of \Cref{lem:parabolic_wellposedness_infinitewidth} and \Cref{rem:parabolic_wellposedness}
    on the training time interval $I$.
    Then the functional~$\Q{\tau}$ as defined in \eqref{eq:CQ} obeys the regularity bound
    \begin{equation}
        \label{eq:lem:CQ_Regulartiy}
        \absbig{\Q{\tau_2}-\Q{\tau_1}}
        \leq L_\CQ\int_{\tau_1}^{\tau_2} \alpha_\tau \,d\tau
    \end{equation}
    for all $\tau_1,\tau_2\in I$ with $0\leq\tau_1\leq\tau_2$ for a constant $L_\CQ=L_\CQ(C^{\widehat{u}},C^{\widehat{v}},C^B_2)$.
\end{proposition}

\begin{proof}
    By the fundamental theorem of calculus, it holds for all $0\leq\tau_1\leq\tau_2$ that
    \begin{equation}
        \label{eq:proof:Regularity:B1}
    \begin{split}
        \Q{\tau_2}-\Q{\tau_1}
        &= \int_{\tau_1}^{\tau_2} \fd{\tau} \Q{\tau}\,d\tau
    \end{split}
    \end{equation}
    and it thus remains to compute and estimate $\fd{\tau} \Q{\tau}$.
    Recalling that $\Q{\tau} = (\uhatp{\tau},T_{B_0}\uhatp{\tau})_{L_2(D_T)}$ as defined in \eqref{eq:CQ}, we obtain for its training time derivative by chain rule and by using that $\whatp{\tau}$ and $\vhatp{\tau}$ are weak solutions (in a sense analogous to \Cref{def:weak_sol_adjoint}) to the second-level adjoint system consisting of the PDEs \eqref{eq:paraboliccoupledadjoint2} and \eqref{eq:paraboliccoupledadjoint1} that
    \begin{equation}
    \begin{split}
        \label{eq:proof:Regularity:B2}
        \fd{\tau} \Q{\tau}
        &= \fd{\tau} \int_0^T\!\!\!\int_D \uhatpp{\tau}{t,x}[T_{B_0}\uhatp{\tau}](t,x)\,dxdt \\
        &= \int_0^T\!\!\!\int_D 2[T_{B_0}\uhatp{\tau}](t,x)\fd{\tau}\uhatpp{\tau}{t,x}\,dxdt
        = \int_0^T\left(2[T_{B_0}\uhatp{\tau}](t,\dummy),\fd{\tau}\uhatpp{\tau}{t,\dummy}\right)_{L_2(D)}dt\\
        &=\int_0^T\left\langle\fpartial{t}\whatpp{\tau}{t,\dummy},\fd{\tau}\uhatpp{\tau}{t,\dummy}\right\rangle_{H^{-1}(D),H_0^1(D)} + \CB\left[\whatpp{\tau}{t,\dummy},\fd{\tau}\uhatpp{\tau}{t,\dummy};t\right] \\
        & \qquad\quad\, -\left(q_u(t,\dummy,\upp{\tau}{t,\dummy})\whatpp{\tau}{t,\dummy},\fd{\tau}\uhatpp{\tau}{t,\dummy}\right)_{L_2(D)}dt  \\
        &=\int_0^T\left\langle\fpartial{t}\whatpp{\tau}{t,\dummy},\fd{\tau}\uhatpp{\tau}{t,\dummy}\right\rangle_{H^{-1}(D),H_0^1(D)} + \CB\left[\whatpp{\tau}{t,\dummy},\fd{\tau}\uhatpp{\tau}{t,\dummy};t\right] \\
        & \qquad\quad\, -\left(q_u(t,\dummy,\upp{\tau}{t,\dummy})\whatpp{\tau}{t,\dummy},\fd{\tau}\uhatpp{\tau}{t,\dummy}\right)_{L_2(D)}dt
        \\
        &\quad\,+\int_0^T\left\langle-\fpartial{t}\vhatpp{\tau}{t,\dummy},\fd{\tau}\upp{\tau}{t,\dummy}\right\rangle_{H^{-1}(D),H_0^1(D)} + \CB^\dagger\left[\vhatpp{\tau}{t,\dummy},\fd{\tau}\upp{\tau}{t,\dummy};t\right] \\
        & \qquad\quad\, -\left(q_u(t,\dummy,\upp{\tau}{t,\dummy})\vhatpp{\tau}{t,\dummy},\fd{\tau}\upp{\tau}{t,\dummy}\right)_{L_2(D)}\\
        & \qquad\quad\, -\left(\whatpp{\tau}{t,\dummy} + q_{uu}(t,\dummy,\upp{\tau}{t,\dummy})\uhatpp{\tau}{t,\dummy}\whatpp{\tau}{t,\dummy},\fd{\tau}\upp{\tau}{t,\dummy}\right)_{L_2(D)} dt.
    \end{split}
    \end{equation}
    For the weak solution property in the third line of \eqref{eq:proof:Regularity:B2}, we note that $\fd{\tau}\uhatp{\tau}$, the weak solution to the linear parabolic PDE
    \begin{alignat}{2}
        \label{eq:parabolicpII*aux}
    \begin{aligned}
        -\fpartial{t}\fd{\tau}\uhatp{\tau} + \CL^\dagger\fd{\tau}\uhatp{\tau} - q_u(\up{\tau})\fd{\tau}\uhatp{\tau}
        &= \fd{\tau}\up{\tau} + q_{uu}(\up{\tau})\uhatp{\tau}\fd{\tau}\up{\tau}
        \qquad&&\text{in }
        D_T, \\
        \fd{\tau}\uhatp{\tau}
        &= 0
        \qquad&&\text{on }
        [0,T]\times\partial D, \\
        \fd{\tau}\uhatp{\tau}
        &= 0
        \qquad&&\text{on }
        \{T\}\times D,
    \end{aligned}
    \end{alignat}
    which is obtained by taking in \eqref{eq:parabolicPDE*} the derivative w.r.t.\@ the training time $\tau$,
    can be used as a test function in the weak formulation of \eqref{eq:parabolicadjoint*}, see \Cref{def:weak_sol_adjoint}, since $\fd{\tau}\uhatpp{\tau}{t,\dummy}\in H_0^1(D)$ for a.e.\@ $t\in[0,T]$.
    Existence and uniqueness of a weak solution to \eqref{eq:parabolicpII*aux} in a sense analogous to \Cref{def:weak_sol_adjoint} follow from classical results,
    see, e.g., \cite[Chapter~7.1, Theorem~3]{evans2010partial} and \cite[Chapter~7.1, Theorem~4]{evans2010partial}, as $\fd{\tau}\up{\tau} + q_{uu}(\up{\tau})\uhatp{\tau}\fd{\tau}\up{\tau}\in L_2(D_T)$.
    That the right-hand side is indeed in $L_2$ follows directly after noting that the PDE~\eqref{eq:parabolicpI*aux} for $\fd{\tau}\up{\tau}$ has a structure identical to \eqref{eq:paraboliccoupledadjoint2} with right-hand side $\fd{\tau}g^*_\tau = - \alpha_\tau T_{B_0}\uhatp{\tau}$, i.e., the same up to a constant factor.
    Following the lines of the proof of \Cref{lem:parabolic_whatLinfty} this ensures that $\fd{\tau}\up{\tau}\in L_\infty(D_T)$.
    For the weak solution property in the fifth step of \eqref{eq:proof:Regularity:B2}, we note that $\fd{\tau}\up{\tau}$, the weak solution to the linear parabolic PDE \eqref{eq:parabolicpI*aux} can be used as a test function in the weak formulation of \eqref{eq:paraboliccoupledadjoint1}, since $\fd{\tau}\upp{\tau}{t,\dummy}\in H_0^1(D)$ for a.e.\@ $t\in[0,T]$, see the discussion after \eqref{eq:parabolicpI*aux}.

    We now perform partial integration.
    For this purpose, first recall that 
    since $\fd{\tau}\uhatpp{\tau}{t,\dummy}\in H_0^1(D)$ for a.e.\@ $t\in[0,T]$ and since $\fpartial{t}\whatpp{\tau}{t,\dummy}\in L_2(D)$ for a.e.\@ $t\in[0,T]$ according to \Cref{lem:parabolic_whatLinfty}, the dual pairing between $H^{-1}(D)$ and $H_0^1(D)$ coincides with the $L_2(D)$ scalar product \cite[Chapter~5.9, Theorem 1(iii)]{evans2010partial}.
    This allows to compute with partial integration, which applies since $\whatpp{\tau}{t,\dummy}, \fd{\tau}\uhatpp{\tau}{t,\dummy}\in H_0^1(D)$ for a.e.\@ $t\in[0,T]$, that
    \begin{equation}       \label{eq:proof:Regularity:B2_a1}
    \begin{split}        &\int_0^T\left\langle\fpartial{t}\whatpp{\tau}{t,\dummy},\fd{\tau}\uhatpp{\tau}{t,\dummy}\right\rangle_{H^{-1}(D),H_0^1(D)}dt\\
        &\qquad\,= \int_0^T\left(\fpartial{t}\whatpp{\tau}{t,\dummy},\fd{\tau}\uhatpp{\tau}{t,\dummy}\right)_{L_2(D)}dt
        = \int_0^T\!\!\!\int_D \left(\fpartial{t}\whatpp{\tau}{t,x}\right)\fd{\tau}\uhatpp{\tau}{t,x} \, dxdt\\
        &\qquad\,= \int_D \underbrace{\whatpp{\tau}{t,x} \fd{\tau}\uhatpp{\tau}{t,x}\Big|_0^T}_{\substack{=0\\\text{since } \whatp{\tau}=0 \text{ on }\{0\}\times D \text{ and}\\\text{since }\fd{\tau}\uhatp{\tau}=0 \text{ on } \{T\}\times D}} \, dx -\int_0^T\!\!\!\int_D \whatpp{\tau}{t,x}\fpartial{t}\fd{\tau}\uhatpp{\tau}{t,x} \, dxdt \\
        &\qquad\,=-\int_0^T\!\!\!\int_D  \left(\fpartial{t}\fd{\tau}\uhatpp{\tau}{t,x}\right)\whatpp{\tau}{t,x} \, dxdt
        = -\int_0^T\left(\fpartial{t}\fd{\tau}\uhatpp{\tau}{t,\dummy},\whatpp{\tau}{t,\dummy}\right)_{L_2(D)}\,dt \\
        &\qquad\,= -\int_0^T\left\langle\fpartial{t}\fd{\tau}\uhatpp{\tau}{t,\dummy},\whatpp{\tau}{t,\dummy}\right\rangle_{H^{-1}(D),H_0^1(D)}dt,
    \end{split}
    \end{equation}
    where the last step holds again since now $\whatpp{\tau}{t,\dummy}\in H_0^1(D)$ for a.e.\@ $t\in[0,T]$ and $\fpartial{t}\fd{\tau}\uhatp{\tau}$ is in $L_2(D)$ for a.e.\@ $t\in[0,T]$, which follows again analogously to \Cref{lem:parabolic_vhatL2} for the PDE \eqref{eq:parabolicpII*aux} due to its with \eqref{eq:paraboliccoupledadjoint1} identical structure and $\fd{\tau}\up{\tau} + q_{uu}(\up{\tau})\uhatp{\tau}\fd{\tau}\up{\tau}\in L_2(D_T)$.
    Similarly, since
    $\fd{\tau}\upp{\tau}{t,\dummy}\in H_0^1(D)$ for a.e.\@ $t\in[0,T]$ (see the discussion after \eqref{eq:parabolicpI*aux}) and since $\fpartial{t}\vhatpp{\tau}{t,\dummy}\in L_2(D)$ for a.e.\@ $t\in[0,T]$ according to \Cref{lem:parabolic_vhatL2}, we may compute analogously to \eqref{eq:parabolicpI*aux_T51} that
    \begin{equation}
        \label{eq:proof:Regularity:B2_a2}
    \begin{split}
        &\int_0^T\left\langle\fpartial{t}\vhatpp{\tau}{t,\dummy},\fd{\tau}\upp{\tau}{t,\dummy}\right\rangle_{H^{-1}(D),H_0^1(D)}dt
        = -\int_0^T\left\langle\fpartial{t}\fd{\tau}\upp{\tau}{t,\dummy},\vhatpp{\tau}{t,\dummy}\right\rangle_{H^{-1}(D),H_0^1(D)}dt
    \end{split}
    \end{equation}
    since $\vhatpp{\tau}{t,\dummy}, \fd{\tau}\upp{\tau}{t,\dummy}\in H_0^1(D)$ for a.e.\@ $t\in[0,T]$.
    Secondly, by definition of the adjoint bilinear form~$\CB^\dagger$ (see \Cref{def:weak_sol_adjoint}) it hold
    $\CB\left[\whatpp{\tau}{t,\dummy},\fd{\tau}\uhatpp{\tau}{t,\dummy};t\right]= \CB^\dagger\left[\fd{\tau}\uhatpp{\tau}{t,\dummy},\whatpp{\tau}{t,\dummy};t\right]$ and $\CB^\dagger\left[\vhatpp{\tau}{t,\dummy},\fd{\tau}\upp{\tau}{t,\dummy};t\right] = \CB\left[\fd{\tau}\upp{\tau}{t,\dummy},\vhatpp{\tau}{t,\dummy};t\right] $
    for a.e.\@ $t\in[0,T]$ since $\whatpp{\tau}{t,\dummy}$, $\fd{\tau}\uhatpp{\tau}{t,\dummy}$, $\vhatpp{\tau}{t,\dummy}$, $\fd{\tau}\upp{\tau}{t,\dummy} \in H_0^1(D)$. 
    With
    \eqref{eq:proof:Regularity:B2_a1}, \eqref{eq:proof:Regularity:B2_a2} and the former,
    we can continue \eqref{eq:proof:Regularity:B2} as
    \begin{allowdisplaybreaks}
    \begin{align}
 \label{eq:proof:Regularity:B3}
        \fd{\tau} \Q{\tau}
        &=\int_0^T\left\langle-\fpartial{t}\fd{\tau}\uhatpp{\tau}{t,\dummy},\whatpp{\tau}{t,\dummy}\right\rangle_{H^{-1}(D),H_0^1(D)} + \CB^\dagger\left[\fd{\tau}\uhatpp{\tau}{t,\dummy},\whatpp{\tau}{t,\dummy};t\right] \nonumber\\
        & \qquad\quad\, -\left(q_u(t,\dummy,\upp{\tau}{t,\dummy})\fd{\tau}\uhatpp{\tau}{t,\dummy},\whatpp{\tau}{t,\dummy}\right)_{L_2(D)}dt
        \nonumber\\
        &\quad\,+\int_0^T\left\langle\fpartial{t}\fd{\tau}\upp{\tau}{t,\dummy},\vhatpp{\tau}{t,\dummy}\right\rangle_{H^{-1}(D),H_0^1(D)} + \CB\left[\fd{\tau}\upp{\tau}{t,\dummy},\vhatpp{\tau}{t,\dummy};t\right] \\
        & \qquad\quad\, -\left(q_u(t,\dummy,\upp{\tau}{t,\dummy})\fd{\tau}\upp{\tau}{t,\dummy},\vhatpp{\tau}{t,\dummy}\right)_{L_2(D)}\nonumber\\
        & \qquad\quad\, -\left(\fd{\tau}\upp{\tau}{t,\dummy} + q_{uu}(t,\dummy,\upp{\tau}{t,\dummy})\uhatpp{\tau}{t,\dummy}\fd{\tau}\upp{\tau}{t,\dummy},\whatpp{\tau}{t,\dummy}\right)_{L_2(D)} dt.\nonumber
    \end{align}
    \end{allowdisplaybreaks}%
    A simple reordering of the terms for later convenience gives
    \begin{equation}
        \label{eq:proof:Regularity:B4}
    \begin{split}
        \fd{\tau} \Q{\tau}
        &=\int_0^T\left\langle-\fpartial{t}\fd{\tau}\uhatpp{\tau}{t,\dummy},\whatpp{\tau}{t,\dummy}\right\rangle_{H^{-1}(D),H_0^1(D)} + \CB^\dagger\left[\fd{\tau}\uhatpp{\tau}{t,\dummy},\whatpp{\tau}{t,\dummy};t\right] \\
        & \qquad\, -\bigg(q_u(t,\dummy,\upp{\tau}{t,\dummy})\fd{\tau}\uhatpp{\tau}{t,\dummy} \\
        &\qquad\, \qquad\qquad+ q_{uu}(t,\dummy,\upp{\tau}{t,\dummy})\uhatpp{\tau}{t,\dummy}\fd{\tau}\upp{\tau}{t,\dummy},\whatpp{\tau}{t,\dummy}\bigg)_{L_2(D)}dt
        \\
        &\quad\,+\int_0^T\left\langle\fpartial{t}\fd{\tau}\upp{\tau}{t,\dummy},\vhatpp{\tau}{t,\dummy}\right\rangle_{H^{-1}(D),H_0^1(D)} + \CB\left[\fd{\tau}\upp{\tau}{t,\dummy},\vhatpp{\tau}{t,\dummy};t\right] \\
        & \qquad\, -\left(q_u(t,\dummy,\upp{\tau}{t,\dummy})\fd{\tau}\upp{\tau}{t,\dummy},\vhatpp{\tau}{t,\dummy}\right)_{L_2(D)} dt\\
        &\quad\,- \int_0^T\left(\fd{\tau}\upp{\tau}{t,\dummy},\whatpp{\tau}{t,\dummy}\right)_{L_2(D)} dt.
    \end{split}
    \end{equation}
    Leveraging now in the first and second line of \eqref{eq:proof:Regularity:B4} the weak formulation of \eqref{eq:parabolicpII*aux}, with test function $\whatpp{\tau}{t,\dummy}$ (suitable due to \Cref{lem:parabolic_whatLinfty}),
    and in the third and fourth line the weak formulation of \eqref{eq:parabolicpI*aux}, with test function $\vhatpp{\tau}{t,\dummy}$ (suitable due to \Cref{lem:parabolic_vhatL2}),
    we arrive at
    \begin{equation}
        \label{eq:proof:Regularity:B5}
    \begin{split}
        \fd{\tau} \Q{\tau}
        &=\int_0^T\left(\fd{\tau}\upp{\tau}{t,\dummy},\whatpp{\tau}{t,\dummy}\right)_{L_2(D)}dt
        + \int_0^T\left(\fd{\tau}g^*_\tau(t,\dummy),\vhatpp{\tau}{t,\dummy}\right)_{L_2(D)} dt\\
        &\quad\,- \int_0^T\left(\fd{\tau}\upp{\tau}{t,\dummy},\whatpp{\tau}{t,\dummy}\right)_{L_2(D)} dt\\
        &=\int_0^T\left(\fd{\tau}g^*_\tau(t,\dummy),\vhatpp{\tau}{t,\dummy}\right)_{L_2(D)} dt
        = \left(\fd{\tau}g^*_\tau,\vhatp{\tau}\right)_{L_2(D_T)}.
    \end{split}
    \end{equation}
    With the expression derived in \eqref{eq:proof:Regularity:B5} for $\fd{\tau}\Q{\tau}$, we can now obtain a bound on \eqref{eq:proof:Regularity:B1}. 
    Recalling that $\fd{\tau}g^*_\tau = - \alpha_\tau T_{B_0}\uhatp{\tau}$ by taking the training time derivative of $g_{\tau}$ as defined in \eqref{eq:parabolicgtau}
    and employing Cauchy-Schwarz inequality yields
    \begin{equation}
        \label{eq:proof:Regularity:B22}
    \begin{split}
        \absbig{\Q{\tau_2}\!-\!\Q{\tau_1}}
        = \abs{\int_{\tau_1}^{\tau_2} \!\fd{\tau} \Q{\tau}\,d\tau}
        &= \abs{\int_{\tau_1}^{\tau_2} \left(\fd{\tau}g^*_\tau,\vhatp{\tau}\right)_{L_2(D_T)}d\tau}
        = \abs{\int_{\tau_1}^{\tau_2} \left(\alpha_\tau T_{B_0}\uhatp{\tau},\vhatp{\tau}\right)_{L_2(D_T)}d\tau}\\
        &\leq \int_{\tau_1}^{\tau_2} \!\!\alpha_\tau\N{T_{B_0}}\N{\uhatp{\tau}}_{L_2(D_T)}\!\N{\vhatp{\tau}}_{L_2(D_T)}d\tau
        \leq C^B_2C^{\widehat{u}}C^{\widehat{v}} \!\int_{\tau_1}^{\tau_2} \!\!\alpha_\tau \,d\tau,
    \end{split}
    \end{equation}
    where the last inequality is due to
    the operator norm of $T_{B_0}$ being bounded by the $L_2$-norm of the kernel $B$ as of \Cref{lem:parabolicTB},
    $\sup_{\tau\in I}\N{\uhatp{\tau}}_{L_2(D_T)}\leq C^{\widehat{u}}$ according to \Cref{lem:parabolic_uhatL2},
    and $\sup_{\tau\in I}\N{\vhatp{\tau}}_{L_2(D_T)}\leq C^{\widehat{v}}$ according to \Cref{lem:parabolic_vhatL2}.
\end{proof}
\section{Cycle of Stopping Times Analysis}
\label{sec:cyclestoppingtimes}

Exploiting the regularity bound for the functional $\Q{\tau}$ in terms of the learning rate~$\alpha_\tau$ established in \Cref{lem:CQ_Regulartiy}
together with the fact that $\fd{\tau}\J{\tau} = - \alpha_\tau\Q{\tau}$ as shown in \Cref{lem:parabolictimeevolutionJt*},
we prove in \Cref{lem:convergenceCJCB} of this section by using a cycle of stopping times analysis as conducted in \cite{bertsekas2000gradient,sirignano2022online} that this entails $\Q{\tau}\rightarrow0$ as $\tau\rightarrow\infty$ provided that the learning rate $(\alpha_\tau)_{\tau\geq0}$ is decreasing and such that $\int_{0}^{\infty}\alpha_\tau\,d\tau=\infty$.
\begin{proposition}
    \label{lem:convergenceCJCB}
    Let $((\up{\tau},\uhatp{\tau}))_{\tau\in [0,\infty)}\in \CC\left([0,\infty),\CS\times\CS\right)$ denote the unique weak solution to the PDE system \eqref{eq:parabolicPDE*}--\eqref{eq:parabolicadjoint*} coupled with the integro-differential equation~\eqref{eq:parabolicgtau} in the sense of \Cref{rem:parabolic_wellposedness}
    on the training time interval $[0,\infty)$.
    Then,
    \begin{equation}        \lim_{\tau\rightarrow\infty} \Q{\tau}
        = 0
    \end{equation}
    and thus also $\lim_{\tau\rightarrow\infty}\fd{\tau}\J{\tau}=0$.
\end{proposition}

\begin{proof}
    The proof borrows the cycle of stopping times argument from \cite[Proposition~1]{bertsekas2000gradient} and \cite[Theorem~3.1]{sirignano2022online}, which crucially depends on the regularity bound~\eqref{eq:lem:CQ_Regulartiy} for the functional $\Q{\tau}$ in terms of the learning rate as apparent in the proof of \Cref{lem:auxcontradiction}.
    
    \textbf{Setup.}
    Let $\varepsilon>0$ and set $A=\varepsilon/(2L_\CQ)>0$.
    We define the cycle of stopping times
    \begin{equation}
    \begin{split}
        0=\sigma_0\leq\tau_1\leq\sigma_1\leq\tau_2\leq\sigma_2\leq\tau_3\leq\dots,
    \end{split}
    \end{equation}
    where $\tau_k$ and $\sigma_k$ are defined for $k=1,2,\dots$ according to
    \begin{equation}
    \begin{split}
        \tau_k
        &= \inf\left\{\tau>\sigma_{k-1}: \Q{\tau}\geq\varepsilon\right\}\\
        \sigma_k
        &= \sup\bigg\{\tau\geq\tau_k: \frac{1}{2}\Q{\tau_k} \leq \Q{s} \leq 2\Q{\tau_k} \text{ for all }s\in[\tau_k,\tau] \text{ and }\int_{\tau_k}^{\tau}\alpha_s\,ds\leq A\bigg\}.
    \end{split}
    \end{equation}
    We further introduce the intervals $I^1_k=[\sigma_{k-1},\tau_k)$ and $I^2_k=[\tau_k,\sigma_k)$.
    It is easy to convince ourselves that by continuity (in the training time $\tau$) it holds $\Q{\tau}<\varepsilon$ for $\tau\in I^1_k$ as well as $\Q{\tau_k}/2\leq\Q{\tau}\leq2\Q{\tau_k}$ for $\tau\in I^2_k$ according to the definitions of the stopping times.

    \textbf{Main Proof.}
    We wish to show that there exists a finite time $\CT^*$ such that it holds $\Q{\tau}\leq\varepsilon$ for all $\tau>\CT^*$.
    Since $\varepsilon$ was arbitrary, the statement then follows.

    \textit{Case 1a: Finitely many $\tau_k$'s, $\tau_K=\infty$.}
    In this case, since there are only finitely many $\tau_k$'s with $\tau_K=\infty$, there indeed exists $\CT^*$ such that $\Q{\tau}\leq\varepsilon$ for all $\tau>\CT^*$.

    \textit{Case 1b: Finitely many $\tau_k$'s, $\sigma_K=\infty$.}
    This case cannot occur, since it would necessitate $\int_{\tau_k}^{\infty}\alpha_\tau\,d\tau\leq A$, which contradicts that by assumption on the learning rate $\int_{0}^{\infty}\alpha_\tau\,d\tau=\infty$.

    It thus remains to show that the case of infinitely many $\tau_k$'s cannot occur either.

    \textit{Case 2: Infinitely many $\tau_k$'s.}
    In this case, we have for sufficiently large $\widetilde{n}$ and for all $n\geq \widetilde{n}$ by a telescopic sum argument that
    \begin{equation}       \label{eq:proof:telescopicsum_aux}
    \begin{split}
        \J{\tau_{n+1}} - \J{\tau_{\widetilde{n}}}
        &= \sum_{k=\widetilde{n}}^n \big(\J{\tau_{k+1}} - \J{\tau_k}\big)
        = \sum_{k=\widetilde{n}}^n \left[\big(\J{\tau_{k+1}} - \J{\sigma_k}\big) + \big(\J{\sigma_k} - \J{\tau_k}\big)\right],
    \end{split}
    \end{equation}
    where, in the last line, the respective first term captures the behavior on the intervals $I^1_{k+1}=[\sigma_{k},\tau_{k+1})$, while the second term captures the behavior on the intervals $I^2_k=[\tau_{k},\sigma_k)$.
    
    On the intervals~$I^1_{k+1}=[\sigma_{k},\tau_{k+1})$ we have $\Q{\tau}\leq\varepsilon$ for $\tau\in I^1_{k+1}$.
    By the fundamental theorem of calculus it holds
    \begin{equation}        \label{eq:proof:telescopicsum_aux1}
    \begin{split}
        \J{\tau_{k+1}}-\J{\sigma_{k}}
        = \int_{\sigma_{k}}^{\tau_{k+1}} \fd{\tau}\J{\tau} \,d\tau
        =  - \int_{\sigma_{k}}^{\tau_{k+1}} \alpha_\tau\Q{\tau} \,d\tau
        \leq 0,
    \end{split}
    \end{equation}
    where we used \Cref{lem:parabolictimeevolutionJt*} to obtain the second equality and the positivity of $\Q{\tau}$, a consequence of the positive definiteness of $T_{B_0}$ from \Cref{lem:parabolicposdefTB}, for the last inequality.
    
    On the other hand, on the intervals~$I^2_k=[\tau_k,\sigma_k)$ we have
    \begin{equation}
        \frac{1}{2}\Q{\tau_k}\leq\Q{\tau}\leq2\Q{\tau_k}
        \qquad\text{and}\qquad
        \int_{\tau_k}^{\tau}\alpha_s\,ds\leq A
    \end{equation}
    for $\tau\in I^2_k$.
    Thus, again by the fundamental theorem of calculus and using \Cref{lem:parabolictimeevolutionJt*} in the second equality,
    it holds
    \begin{equation}        \label{eq:proof:telescopicsum_aux2}
    \begin{split}
        \J{\sigma_{k}}-\J{\tau_k}
        =\int_{\tau_k}^{\sigma_k} \fd{\tau}\J{\tau} \,d\tau
        = - \int_{\tau_k}^{\sigma_k} \alpha_\tau\Q{\tau} \,d\tau
        \leq - \frac{1}{2}\Q{\tau_k}\int_{\tau_k}^{\sigma_k} \alpha_\tau \,d\tau 
        \leq -\frac{(1-\vartheta)}{2}\varepsilon A
    \end{split}
    \end{equation}
    for any $\vartheta\in(0,1)$,
    where the third inequality is due to the property of the interval $I^2_k$, while the fourth inequality is firstly since by continuity and by definition of the stopping time~$\tau_k$ it holds $\Q{\tau_k}\geq\varepsilon$ and secondly since as of \Cref{lem:auxcontradiction} it holds $(1-\vartheta) A\leq\int_{\tau_k}^{\sigma_k} \alpha_\tau\,d\tau$.

    Inserting \eqref{eq:proof:telescopicsum_aux1} and \eqref{eq:proof:telescopicsum_aux2} into \eqref{eq:proof:telescopicsum_aux}
    yields 
    \begin{equation}
        \J{\tau_{n+1}}
        \leq \J{\tau_{\widetilde{n}}} -\sum_{k=\widetilde{n}}^n \frac{(1-\vartheta)}{2}\varepsilon A
        = \J{\tau_{\widetilde{n}}} -\sum_{k=\widetilde{n}}^n \frac{(1-\vartheta)\varepsilon^2}{4L_\CQ}.
    \end{equation}
    Letting $n\rightarrow\infty$, we would obtain that $\J{\tau_{n+1}}\rightarrow-\infty$, which contradicts the fact that $\J{\tau}\geq0$ by definition.
    By excluding that this case can occur, the proof is concluded.
\end{proof}

In the proof of \Cref{lem:convergenceCJCB} we made use of the following auxiliary result.
\begin{lemma}
    \label{lem:auxcontradiction}
    Let $((\up{\tau},\uhatp{\tau}))_{\tau\in [0,\infty)}\in \CC\left([0,\infty),\CS\times\CS\right)$ denote the unique weak solution to the PDE system \eqref{eq:parabolicPDE*}--\eqref{eq:parabolicadjoint*} coupled with the integro-differential equation~\eqref{eq:parabolicgtau} in the sense of \Cref{rem:parabolic_wellposedness}
    on the training time interval $[0,\infty)$.
    For given $\varepsilon>0$, let $A=\varepsilon/(2L_\CQ)$.
    Then, for $k$ large enough and for $\eta>0$ small enough (potentially depending on $k$), one has $\int_{\tau_k}^{\sigma_k+\eta} \alpha_\tau\,d\tau>A$.
    Moreover, we also have
    $(1-\vartheta)A\leq\int_{\tau_k}^{\sigma_k} \alpha_\tau\,d\tau \leq A$ for any $\vartheta\in(0,1)$.
\end{lemma}
\begin{proof}
    The proof of the first part of the statement proceeds by contradiction.
    Let us therefore assume that $\int_{\tau_k}^{\sigma_k+\eta} \alpha_\tau\,d\tau \leq A$.
    Leveraging the regularity bound for the functional \texorpdfstring{$\Q{\tau}$}{Q} in terms of the learning rate established in \Cref{lem:CQ_Regulartiy} with $\tau_1=\tau_k$ and $\tau_2=\sigma_k+\eta$,
    we have
    \begin{equation}        \label{eq:proof:lemma:contradiction_aux1}
    \begin{split}
        \Q{\sigma_k+\eta}-\Q{\tau_k}
        \leq \absbig{\Q{\sigma_k+\eta}-\Q{\tau_k}}
        \leq L_\CQ\int_{\tau_k}^{\sigma_k+\eta}\alpha_\tau\, d\tau
        \leq L_\CQ A
        = \frac{1}{2}\varepsilon
        \leq \frac{1}{2}\Q{\tau_k},
    \end{split}
    \end{equation}
    where we used the contradiction assumption in the third step,
    the definition of $A=\varepsilon/(2L_\CQ)$ in the fourth step
    and that by definition of the stopping time~$\tau_k$ it holds $\Q{\tau_k}\geq\varepsilon$ in the last.
    The computation \eqref{eq:proof:lemma:contradiction_aux1} implies
    $\Q{\sigma_k+\eta} \leq \Q{\tau_k} + \Q{\tau_k}/2 \leq 2\Q{\tau_k}$ by simple reordering 
    as well as 
    $\Q{\tau_k} - \Q{\sigma_k+\eta} \leq \abs{\Q{\sigma_k+\eta}-\Q{\tau_k}} \leq \Q{\tau_k}/2$, or rearranged $\Q{\tau_k}/2 \leq \Q{\sigma_k+\eta}$.
    In summary, $\frac{1}{2}\Q{\tau_k} \leq \Q{\sigma_k+\eta} \leq 2\Q{\tau_k}$.
    Since the same reasoning holds for any $0<\widetilde\eta\leq\eta$, this yields a contradiction, as this would imply that $\sigma_k=\sigma_k+\eta$, contradicting $\eta>0$.
    Thus, $\int_{\tau_k}^{\sigma_k+\eta} \alpha_\tau\,d\tau > A$ holds proving the first part of the statement.
    
    What concerns the second part, since the learning rate $\alpha_\tau$ is decreasing in $\tau$, for large enough $k$ and small enough $\eta$ we can ensure $\int_{\sigma_k}^{\sigma_k+\eta} \alpha_\tau\,d\tau \leq \vartheta A$.
    Thus,
    \begin{equation}
        \int_{\tau_k}^{\sigma_k} \alpha_\tau\,d\tau
        = \int_{\tau_k}^{\sigma_k+\eta} \alpha_\tau\,d\tau - \int_{\sigma_k}^{\sigma_k+\eta} \alpha_\tau\,d\tau
        \geq A - \vartheta A
        = (1-\vartheta) A.
    \end{equation}
    Since by definition $\int_{\tau_k}^{\sigma_k} \alpha_\tau\,d\tau \leq A$, this concludes the proof.
\end{proof}
\section{Convergence of the Adjoint \texorpdfstring{$\uhatp{\tau}$}{} and the Solution \texorpdfstring{$\up{\tau}$}{}}
\label{sec:convergences}

Since the functional $\Q{\tau} = (\uhatp{\tau},T_{B_0}\uhatp{\tau})_{L_2(D_T)}$ 
converges to zero as $\tau\rightarrow\infty$ according to \Cref{lem:convergenceCJCB} and since the NN kernel operator~$T_{B_0}$ is positive definite as of \Cref{lem:parabolicposdefTB},
we can derive in \Cref{lem:convergence_adjoint} in \Cref{sec:ConvergenceAdjoint} the weak $L_2$ convergence of the adjoint $\uhatp{\tau}$ in \eqref{eq:parabolicadjoint*} to zero as $\tau\rightarrow\infty$.
Noticing that this entails that the left-hand side of the adjoint PDE \eqref{eq:parabolicadjointN_plain} converges to zero when evaluated against any test function,
we infer therefrom in \Cref{lem:convergence_solution} in \Cref{sec:ConvergenceSolution} the weak $L_2$ convergence of the solution $\up{\tau}$ in \eqref{eq:parabolicPDE*} to the target data $h$  by definition of the adjoint PDE.
In \Cref{sec:characterization_fixedpoint}, we provide a result of independent interest showing that (strong) limit points of the trained NN-PDE solution are global minimizers of the loss~$\J{}$ for an even more general class of second-order parabolic NN-PDEs.

\subsection{Convergence of the Adjoint \texorpdfstring{$\uhatp{\tau}$}{} as \texorpdfstring{$\tau\rightarrow\infty$}{training time goes to infinity}}
\label{sec:ConvergenceAdjoint}

Let us first infer the weak $L_2$ convergence of the adjoint $\uhatp{\tau}$ in \eqref{eq:parabolicadjoint*} to zero.

\begin{proposition}
    \label{lem:convergence_adjoint}
    Let $((\up{\tau},\uhatp{\tau}))_{\tau\in [0,\infty)}\in \CC\left([0,\infty),\CS\times\CS\right)$ denote the unique weak solution to the PDE system \eqref{eq:parabolicPDE*}--\eqref{eq:parabolicadjoint*} coupled with the integro-differential equation~\eqref{eq:parabolicgtau} in the sense of \Cref{rem:parabolic_wellposedness}
    on the training time interval $[0,\infty)$.
    Then,
    \begin{equation}
        \uhatp{\tau}\rightharpoonup0
        \text{ in }
        L_2
        \quad
        \text{as }
        \tau\rightarrow\infty,
    \end{equation}
    i.e., for each test function~$\phi\in L_2(D_T)$ it holds $\lim_{\tau\rightarrow\infty} (\uhatp{\tau},\phi)_{L_2(D_T)}=0$.
\end{proposition}

\begin{proof}
    Since the eigenfunctions~$\{e_k(t,x)\}_{k=1}^\infty$ of $T_{B_0}$ form an orthonormal basis of $L_2(D_T)$ according to \Cref{lem:parabolicTB},
    we have for $\uhatp{\tau}$ the expansion $\uhatpp{\tau}{t,x} = \sum_{k=1}^\infty c_k(\tau)e_k(t,x)$.
    Using this, we can express $\Q{\tau}=\left(\uhatp{\tau},T_{B_0}\uhatp{\tau}\right)_{L_2(D_T)} = \sum_{k=1}^\infty \lambda_kc^2_k(\tau)$,
 where the last equality holds as $T_{B_0}$ is a continuous operator.
    Taking the limit $\tau\rightarrow\infty$ and leveraging \Cref{lem:convergenceCJCB} in the last step of the following display,
    this shows
    \begin{equation}
        \lim_{\tau\rightarrow\infty} \sum_{k=1}^\infty \lambda_kc^2_k(\tau)
        =\lim_{\tau\rightarrow\infty} \left(\uhatp{\tau},T_{B_0}\uhatp{\tau}\right)_{L_2(D_T)}
        =\lim_{\tau\rightarrow\infty} \Q{\tau}
        =0.
    \end{equation}
    Consequently, for $k$ fixed, it holds $\lim_{\tau\rightarrow\infty}\lambda_kc_k^2(\tau)=0$.
    Furthermore, with $\lambda_k>0$ according to \Cref{lem:parabolicposdefTB}, for $k$ fixed, it also holds $\lim_{\tau\rightarrow\infty}c_k(\tau)=0$.
    
    Let $\varepsilon>0$ and let $\phi\in L_2(D_T)$ denote a test function, which we can represent as $\phi(t,x) = \sum_{k=1}^\infty \varphi_k e_k(t,x)$ with $\sum_{k=1}^\infty\varphi_k^2<\infty$.
    Thus, there exists $K>0$ such that $\sum_{k=K+1}^\infty\varphi_k^2 \leq \varepsilon^2/(2C^{\widehat{u}})^2$.
   
   If $\varphi_k=0$ for all $k=1,\dots,K$, then it holds $\absbig{\sum_{k=1}^K \varphi_k c_k(\tau)}=0$. Otherwise, recalling that $\lim_{\tau\rightarrow\infty}c_k(\tau)=0$ for any fixed $k$, there exists $\overbar{\tau}>0$ such that we have $\abs{c_k(\tau)}\leq \varepsilon/(2K\max_{\tilde{k}=1,\dots,K}\absnormal{\varphi_{\tilde{k}}})$ (uniformly for $k=1,\dots,K$) for all $\tau\geq\overbar{\tau}$. This shows in particular that $\absbig{\sum_{k=1}^K \varphi_k c_k(\tau)} \leq \sum_{k=1}^K \abs{\varphi_k} \abs{c_k(\tau)} \leq \varepsilon/2$.
    We can now estimate with triangle inequality for all such $\tau\geq\overbar{\tau}$ that
    \begin{equation}
    \begin{split}
        \abs{(\phi,\uhatp{\tau})_{L_2(D_T)}}
        = \abs{\sum_{k=1}^\infty \varphi_k c_k(\tau)}
        &\leq \abs{\sum_{k=1}^K \varphi_k c_k(\tau)} + \abs{\sum_{k=K+1}^\infty \varphi_k c_k(\tau)}\\
        &\leq \frac{\varepsilon}{2} + \frac{\varepsilon}{2 C^{\widehat{u}}}\N{\uhatp{\tau}}_{L_2(D_T)}
        \leq \frac{\varepsilon}{2} + \frac{\varepsilon}{2 C^{\widehat{u}}}C^{\widehat{u}}
        \leq \varepsilon,
    \end{split}
    \end{equation}
    where we used Cauchy-Schwarz inequality and the former estimates together with \Cref{lem:parabolic_uhatL2} to obtain the bound on the tail of the series in the inequalities in the second line.
    Thus $\abs{(\phi,\uhatp{\tau})_{L_2(D_T)}}\leq\varepsilon$ for all $\tau\geq\overbar{\tau}$.
    Since $\varepsilon>0$ was arbitrary,
    this shows that it holds $\lim_{\tau \rightarrow \infty} (\phi,\uhatp{\tau})_{L_2(D_T)} = 0$ for all test functions $\phi\in L_2(D_T)$, proving the weak convergence of $\uhatp{\tau}$ to zero in $L_2$ as $\tau\rightarrow\infty$.
\end{proof}
\subsection{Convergence of the Solution \texorpdfstring{$\up{\tau}$}{} as \texorpdfstring{$\tau\rightarrow\infty$}{training time goes to infinity}}
\label{sec:ConvergenceSolution}

It remains to infer the weak $L_2$ convergence of the solution $\up{\tau}$ to \eqref{eq:parabolicPDE*} to the target data~$h$.
\begin{proposition}  \label{lem:convergence_solution}
    Let $((\up{\tau},\uhatp{\tau}))_{\tau\in [0,\infty)}\in \CC\left([0,\infty),\CS\times\CS\right)$ denote the unique weak solution to the PDE system \eqref{eq:parabolicPDE*}--\eqref{eq:parabolicadjoint*} coupled with the integro-differential equation~\eqref{eq:parabolicgtau} in the sense of \Cref{rem:parabolic_wellposedness}
    on the training time interval $[0,\infty)$.
    Then,
    \begin{equation}
        \up{\tau}\rightharpoonup h
        \text{ in }
        L_2
        \quad
        \text{as }
        \tau\rightarrow\infty,
    \end{equation}
    i.e., for each test function~$\phi\in L_2(D_T)$ it holds $\lim_{\tau\rightarrow\infty} (\up{\tau}-h,\phi)_{L_2(D_T)}=0$.
\end{proposition}
\begin{proof}
    Let us first show that $\lim_{\tau\rightarrow\infty} (\up{\tau}-h,\widetilde\phi)_{L_2(D_T)}=0$ for each test function $\widetilde\phi\in C^\infty_c(D_T)$ that vanishes on the boundary.
    By using that $\uhatp{\tau}$ is a weak solution to the adjoint PDE~\eqref{eq:parabolicadjoint*} in the sense of \Cref{def:weak_sol_adjoint} with right-hand side $(\up{\tau}-h)$ we compute for the test function $\widetilde\phi\in C^\infty_c(D_T)$ that
    \begin{equation}
    \begin{split}
        \label{eq:proof:lem:convergence_solution1}
        (\up{\tau}-h,\widetilde\phi)_{L_2(D_T)}
        &= \int_0^T \left(\upp{\tau}{t,\dummy}-h(t,\dummy), \widetilde\phi(t,\dummy)\right)_{L_2(D)}dt\\
        &= \int_0^T \left\langle-\fpartial{t}\uhatpp{\tau}{t,\dummy},\widetilde\phi(t,\dummy)\right\rangle_{H^{-1}(D),H_0^1(D)} + \CB^\dagger\left[\uhatpp{\tau}{t,\dummy},\widetilde\phi(t,\dummy);t\right] \\
        &\qquad\quad\, - \left(q_u(t,\dummy,\upp{\tau}{t,\dummy})\uhatpp{\tau}{t,\dummy}, \widetilde\phi(t,\dummy)\right)_{L_2(D)} dt \\
        &= \int_0^T \left\langle\fpartial{t}\widetilde\phi(t,\dummy),\uhatpp{\tau}{t,\dummy}\right\rangle_{H^{-1}(D),H_0^1(D)} + \CB\left[\widetilde\phi(t,\dummy),\uhatpp{\tau}{t,\dummy};t\right] \\
        &\qquad\quad\, - \left(q_u(t,\dummy,\upp{\tau}{t,\dummy})\widetilde\phi(t,\dummy),\uhatpp{\tau}{t,\dummy}\right)_{L_2(D)} dt
    \end{split}
    \end{equation}        
    with the last step following analogously to \eqref{eq:proof:parabolicdtJ*1}, where we justified the individual steps in detail, see \eqref{eq:parabolicpI*aux_T51}--\eqref{eq:parabolicpI*aux_T52}.
    Herefore, note that in the case here, even $\widetilde\phi\in C^\infty_c(D_T)$.

    As a consequence of the convergence $\uhatp{\tau}\rightharpoonup0$ in $L_2$ as $\tau\rightarrow\infty$, which we established in \Cref{lem:convergence_adjoint}, the right-hand side of \eqref{eq:proof:lem:convergence_solution1} converges to zero as $\tau\rightarrow\infty$.
    To be precise, let us discuss each of the three terms.
    Firstly, since $\fpartial{t}\widetilde\phi\in C^\infty_c(D_T)\subset L_2(D_T)$ and $\uhatpp{\tau}{t,\dummy}\in H_0^1(D)$ for a.e.\@ $t\in[0,T]$, the dual pairing between $H^{-1}(D)$ and $H_0^1(D)$ coincides with the $L_2(D)$ scalar product \cite[Chapter~5.9, Theorem 1(iii)]{evans2010partial} and thus
    \begin{equation}
        \label{eq:proof:lem:convergence_solution2}
    \begin{split}
        \int_0^T \left\langle\fpartial{t}\widetilde\phi(t,\dummy),\uhatpp{\tau}{t,\dummy}\right\rangle_{H^{-1}(D),H_0^1(D)}dt
        &=\int_0^T \left(\uhatpp{\tau}{t,\dummy},\fpartial{t}\widetilde\phi(t,\dummy)\right)_{L_2(D)}dt \\
        &=\big(\uhatp{\tau},\fpartial{t}\widetilde\phi\big)_{L_2(D_T)},
    \end{split}
    \end{equation}
    which converges to zero as $\tau\rightarrow\infty$ since $\uhatp{\tau}\rightharpoonup0$ in $L_2$ according to \Cref{lem:convergence_adjoint} with test function $\fpartial{t}\widetilde\phi\in C^\infty_c(D_T)\subset L_2(D_T)$.
    Secondly, by definition of the bilinear form~$\CB$ in \eqref{eq:CB} we have
    \begin{allowdisplaybreaks}
    \label{eq:proof:lem:convergence_solution3}  
    \begin{align}  
        &\int_0^T\CB\left[\widetilde\phi(t,\dummy),\uhatpp{\tau}{t,\dummy};t\right]dt 
        =\int_0^T\!\!\!\int_U \sum_{i,j=1}^d  a^{ij}(t,x) \fpartial{x_i} \widetilde\phi(t,x) \fpartial{x_j} \uhatpp{\tau}{t,x}\notag\\
        &\qquad\qquad\,\qquad\,
        + \sum_{i=1}^d b^i(t,x) \fpartial{x_i} \widetilde\phi(t,x) \uhatpp{\tau}{t,x}
        + c(t,x)\widetilde\phi(t,x)\uhatpp{\tau}{t,x} \,dxdt\notag \\
        &\qquad\qquad\,=\int_0^T\!\!\!\int_U -\sum_{i,j=1}^d  \fpartial{x_j}\left(a^{ij}(t,x) \fpartial{x_i} \widetilde\phi(t,x)\right) \uhatpp{\tau}{t,x}\notag\\
        &\qquad\qquad\,\qquad\,
        + \sum_{i=1}^d b^i(t,x) \fpartial{x_i} \widetilde\phi(t,x) \uhatpp{\tau}{t,x}
        + c(t,x)\widetilde\phi(t,x)\uhatpp{\tau}{t,x} \,dxdt\notag\\
        &\qquad\qquad\,=\int_0^T\!\!\!\int_U -\sum_{i,j=1}^d  a^{ij}(t,x) \fpartiall{x_ix_j} \widetilde\phi(t,x) \uhatpp{\tau}{t,x}\notag\\
        &\qquad\qquad\,\qquad\,
        - \sum_{i,j=1}^d  \fpartial{x_j}a^{ij}(t,x) \fpartial{x_i} \widetilde\phi(t,x) \uhatpp{\tau}{t,x}\notag\\
        &\qquad\qquad\,\qquad\,
        + \sum_{i=1}^d b^i(t,x) \fpartial{x_i} \widetilde\phi(t,x) \uhatpp{\tau}{t,x}
        + c(t,x)\widetilde\phi(t,x)\uhatpp{\tau}{t,x} \,dxdt\notag\\
        &\qquad\qquad\,=\left(\uhatp{\tau},-\sum_{i,j=1}^da^{ij}\fpartiall{x_ix_j} \widetilde\phi-\sum_{i,j=1}^d\fpartial{x_j}a^{ij}\fpartial{x_i} \widetilde\phi+\sum_{i=1}^db^i\fpartial{x_i}\widetilde\phi + c\widetilde\phi\right)_{L_2(D_T)},
    \end{align}
    \end{allowdisplaybreaks}%
    where the second step is just partial integration with all boundary terms vanishing since also $\fpartial{x_i}\widetilde\phi\in C^\infty_c(D_T)$.
    Since the coefficients $a^{ij},\fpartial{x_j}a^{ij},b^i,c\in L_\infty(D_T)$ as of Assumption~\ref{asm:PDE_coefficients}, the test function in the scalar-product in the last line of \eqref{eq:proof:lem:convergence_solution3} is in $L_2(D_T)$ and thus the right-hand side of \eqref{eq:proof:lem:convergence_solution3} converges to zero as $\tau\rightarrow\infty$ since $\uhatp{\tau}\rightharpoonup0$ in $L_2$ according to \Cref{lem:convergence_adjoint}.
    Thirdly and lastly, since with $\widetilde\phi\in C^\infty_c(D_T) \subset L_2(D_T)$ and $q_u$ being uniformly bounded as of Assumption~\ref{asm:PDE_nonlinearity_q_ubdd}, also
    $q_u(\dummy,\dummy,\up{\tau})\widetilde\phi\in L_2(D_T)$,
    \begin{equation}
        \label{eq:proof:lem:convergence_solution4}
        \int_0^T \left(q_u(t,\dummy,\upp{\tau}{t,\dummy})\widetilde\phi(t,\dummy),\uhatpp{\tau}{t,\dummy}\right)_{L_2(D)} dt
        = \big(\uhatp{\tau},q_u(\up{\tau})\widetilde\phi\big)_{L_2(D_T)}
    \end{equation}
    converges to zero as $\tau\rightarrow\infty$ since $\uhatp{\tau}\rightharpoonup0$ in $L_2$ according to \Cref{lem:convergence_adjoint}.    
    With this we have shown that $(\up{\tau}-h,\widetilde\phi)_{L_2(D_T)}\rightarrow0$ for all $\widetilde\phi\in C^\infty_c(D_T)$.

    Let now $\phi\in L_2(D_T)$ and $\varepsilon>0$.
    Since $C^\infty_c(D_T)$ is dense in $L_2(D_T)$~\cite[Corollary~4.23]{brezis2011functional},
    there exists $\widetilde\phi\in C^\infty_c(D_T)$ such that $\Nbig{\phi-\widetilde\phi}_{L_2(D_T)}\leq \varepsilon/\sqrt{2\J{0}}$.
    We can thus estimate
    \begin{equation}
    \begin{split}
        \abs{(\up{\tau}-h,\phi)_{L_2(D_T)}}
        &\leq \absbig{(\up{\tau}-h,\phi-\widetilde\phi)_{L_2(D_T)}} + \absbig{(\up{\tau}-h,\widetilde\phi)_{L_2(D_T)}} \\
        &\leq \N{\up{\tau}-h}_{L_2(D_T)}\Nbig{\phi-\widetilde\phi}_{L_2(D_T)} + \absbig{(\up{\tau}-h,\widetilde\phi)_{L_2(D_T)}} \\
        &= \sqrt{2\J{\tau}}\Nbig{\phi-\widetilde\phi}_{L_2(D_T)} + \absbig{(\up{\tau}-h,\widetilde\phi)_{L_2(D_T)}} \\
        &\leq \sqrt{2\J{0}}\Nbig{\phi-\widetilde\phi}_{L_2(D_T)} + \absbig{(\up{\tau}-h,\widetilde\phi)_{L_2(D_T)}} \\
        &\leq \frac{\varepsilon}{2} + \frac{\varepsilon}{2} = \varepsilon
    \end{split}
    \end{equation}
    for sufficiently large $\tau$.
    In the next-to-last step we used that by \Cref{lem:parabolictimeevolutionJt*} the loss $\J{\tau}$ is non-increasing.
    The last step holds since $\absnormal{(\up{\tau}-h,\widetilde\phi)_{L_2(D_T)}}\rightarrow0$ for $\widetilde\phi\in C^\infty_c(D_T)$, thus $\absnormal{(\up{\tau}-h,\widetilde\phi)_{L_2(D_T)}}\leq\varepsilon/2$ for sufficiently large $\tau$.
    Consequently, $\absnormal{(\up{\tau}-h,\phi)_{L_2(D_T)}}\rightarrow0$ for all $\phi\in L_2(D_T)$, which concludes the proof.
\end{proof}

Before closing this section, let us compare \Cref{lem:convergence_solution} to prior work to indicate that we substantially strengthen the notion of convergence for a significantly wider class of PDEs and a more general loss.
\begin{remark}
\label{rem:comparisonconvergencesirignano2023pde}
    The weak convergences $\uhatp{\tau}\rightharpoonup 0$ in $L_2$ and $\up{\tau}\rightharpoonup h$ in $L_2$ as $\tau\rightarrow\infty$ established in \Cref{lem:convergence_adjoint,lem:convergence_solution}, respectively, significantly improve prior work~\cite{sirignano2023pde},
    where only convergence of the time averages has been established, cf.\@ \cite[Theorem~9.3]{sirignano2023pde}.
    In the elliptic linear PDE setting,
    the authors of \cite{sirignano2023pde} prove $\lim_{\tau\rightarrow\infty}\frac{1}{\tau}\int_0^\tau (\phi,\uhatp{s})_{L_2}^2\,ds=0$ for all $\phi\in L_2$ and $\lim_{\tau\rightarrow\infty}\frac{1}{\tau}\int_0^\tau (\psi,\uhatp{s}-h)_{L_2}^2\,ds=0$ for all $\psi\in\CA:=\{\psi\in H_0^1:\CL\psi\in L_2\}\subset L_2$.

    To see that \Cref{lem:convergence_adjoint,lem:convergence_solution} are stronger, simply observe that the time average $\frac{1}{\tau}\int_0^\tau
    f_s^2\,ds\rightarrow0$ might converge while $f_\tau\not\rightarrow0$.
    ($f_\tau$ corresponds here to either $(\phi,\uhatp{\tau})_{L_2}^2$ or $(\psi,\uhatp{\tau}-h)_{L_2}^2$.)
    A straightforward smooth example is given by 
    \begin{equation}
    f_\tau = 
    \begin{cases}
        \exp\left(1-\frac{1}{1-(\tau-2^\ell)^2}\right), & \text{for } \tau\in[2^\ell-1,2^\ell+1] \text{ for } \ell=1,2,\dots,\\
        0, & \text{else.}
    \end{cases}
    \end{equation}
    The function $\tau\mapsto f_\tau$ concatenates infinitely many bump functions centered around $2^\ell$, $\ell=1,2,\dots$, with width $2$ and maximal height $1$.
    Therefore, clearly, $u_\tau\not\rightarrow0$.
    However, since there are $\lfloor\log_2(\tau)\rfloor$ such bumps before time $\tau$,
    \begin{equation}
    \begin{split}
        \frac{1}{\tau}\int_0^\tau f_s^2\,ds
        \leq \frac{1}{\tau}\sum_{\ell=1}^{\lfloor\log_2(\tau)\rfloor}2 \leq \frac{2}{\tau}\log_2(\tau) \rightarrow0
        \quad\text{as } \tau\rightarrow\infty.
    \end{split}
    \end{equation}
    Conversely, it is immediate to see that $f_\tau\rightarrow0$ implies $\frac{1}{\tau}\int_0^\tau
    f_s^2\,ds\rightarrow0$.

    Secondly, unlike \cite{sirignano2023pde}, where the considered loss is given by 
    \begin{equation}
        \widetilde\CJ^*_\tau
        = \frac{1}{2} \sum_{\ell=1}^L  \left(\up{\tau}-h,m_\ell\right)_{L_2}^2
    \end{equation}
    for given functions $\{m_\ell\}_{\ell=1}^L$,
    we consider the stronger loss $\J{\tau}=\N{\up{\tau}-h}_{L_2}^2$ as in \eqref{eq:parabolicJtau}.

    We therefore generalize in this paper not just the class of considered PDEs substantially by allowing for nonlinear PDEs,
    but significantly improve the notion of convergence.
\end{remark}

\subsection{Limit Points of the Trained NN-PDE Solution are Global Minimizers of the Loss \texorpdfstring{$\J{}$}{}}
\label{sec:characterization_fixedpoint}

To conclude the theoretical contributions of this work,
let us provide a result about the 
limit points of the trained NN-PDE solution~$\up{\tau}$, which holds for
the even more general class of fully nonlinear second-order parabolic NN-PDEs
\begin{alignat}{2}
\label{eq:parabolicPDE*_general}
\begin{aligned}
    \fpartial{t}\up{\tau} + \CL\up{\tau} - q(\up{\tau}, \nabla_x \up{\tau}, \mathbf{H}_{xx} \up{\tau})
    &= g^*_\tau
    \qquad&&\text{in }
    D_T, \\
    \up{\tau}
    &= 0
    \qquad&&\text{on }
    [0,T]\times\partial D, \\
    \up{\tau}
    &= f
    \qquad&&\text{on }
    \{0\}\times D,
\end{aligned}
\end{alignat}
with associated adjoint PDE
\begin{alignat}{2}
\label{eq:parabolicadjoint*_general}
\begin{aligned}
    -\fpartial{t}\uhatp{\tau} + \CL^\dagger\uhatp{\tau} - q_u(\up{\tau}, \nabla_x \up{\tau}, \mathbf{H}_{xx} \up{\tau})\uhatp{\tau}&&\\
    + \; \textstyle\sum_{i=1}^d \fpartial{x_i}\big(q_{p_i}(\up{\tau}, \nabla_x \up{\tau}, \mathbf{H}_{xx} \up{\tau}) \uhatp{\tau}\big)&&\\
    - \; \textstyle\sum_{i,j=1}^d \fpartiall{x_ix_j}\big(q_{\mathbf{H}_{ij}}(\up{\tau}, \nabla_x \up{\tau}, \mathbf{H}_{xx} \up{\tau}) \uhatp{\tau}\big)
    &= (\up{\tau}-h)
    \qquad&&\text{in }
    D_T, \\
    \uhatp{\tau}
    &= 0
    \qquad&&\text{on }
    [0,T]\times\partial D, \\
    \uhatp{\tau}
    &= 0
    \qquad&&\text{on }
    \{T\}\times D,
\end{aligned}
\end{alignat}
and coupled with the integro-differential equation~\eqref{eq:parabolicgtau} for $g^*_\tau$.

We show that any (strong) limit point of the solution of the trained NN-PDE, when using the adjoint gradient descent optimization method~\eqref{eq:GD} with the gradient being computed according to \eqref{eq:nablaJ}, is a global minimizer of the loss~$\J{}$.
\begin{theorem}
    \label{lem:stationary_point}
    Let $((\up{\tau},\uhatp{\tau}))_{\tau\in [0,\infty)}\in \CC\left([0,\infty),\CS\times\CS\right)$ denote the unique weak solution to the more general PDE system \eqref{eq:parabolicPDE*_general}--\eqref{eq:parabolicadjoint*_general} coupled with the integro-differential equation~\eqref{eq:parabolicgtau} in a sense analogous to \Cref{lem:parabolic_wellposedness_infinitewidth} and \Cref{rem:parabolic_wellposedness}
    on the training time interval $[0,\infty)$.
    Assume that $(\up{\tau},\uhatp{\tau})$ converges to some $(\up{\infty},\uhatp{\infty})$ in $L_2(D_T)$ as $\tau\rightarrow\infty$.
    Then $\uhatp{\infty}\equiv0$ a.e.\@ in $L_2(D_T)$ and
    \begin{equation}
        \up{\infty}\equiv h
        \text{ a.e.\@ in } L_2(D_T),
    \end{equation}
    i.e., $\up{\infty}$ is a global minimizer of $\J{}$.
\end{theorem}
\begin{proof}
    Leveraging the adjoint PDE~\eqref{eq:parabolicadjoint*_general}, we can derive analogously to \Cref{lem:parabolictimeevolutionJt*} that $\fd{\tau}\J{\tau} = - \alpha_\tau(\uhatp{\tau},T_{B_0}\uhatp{\tau})_{L_2} = - \alpha_\tau\Q{\tau}$ for all $\tau\in[0,\infty)$.
    Moreover, by following the computations of \Cref{lem:parabolic_uhatL2}, we can derive a uniform (in the training time~$\tau$) estimate of the form \eqref{eq:lem:parabolic_uhatL2} for the adjoint~$\uhatp{\tau}$.
    
    \textbf{Step 1: $\uhatp{\infty}\equiv0$ a.e.\@ in $L_2(D_T)$.}
    Since the eigenfunctions~$\{e_k(t,x)\}_{k=1}^\infty$ of $T_{B_0}$ form an orthonormal basis of $L_2(D_T)$ according to \Cref{lem:parabolicTB},
    $\uhatp{\infty}$ has the expansion $\uhatpp{\infty}{t,x} = \sum_{k=1}^\infty c_ke_k(t,x)$.
    We now proceed by contradiction and suppose that $\uhatp{\infty}$ is not $0$ a.e.\@ in $L_2(D_T)$.
    Then there exists at least one $\tilde{k}\in\bbN$ with $c_{\tilde{k}}\neq0$.
    Using this, we can lower bound $\Q{\infty}=\left(\uhatp{\infty},T_{B_0}\uhatp{\infty}\right)_{L_2(D_T)} = \sum_{k=1}^\infty \lambda_kc^2_k \geq \lambda_{\tilde{k}}c_{\tilde{k}}^2>0$ after recalling that $\lambda_{\tilde{k}}>0$ according to \Cref{lem:parabolicposdefTB}.
    We furthermore have
    \begin{equation}
    \begin{split}
        \Q{\tau} &= \left(\uhatp{\tau},T_{B_0}\uhatp{\tau}\right)_{L_2(D_T)} \\
        &= \left(\uhatp{\tau}-\uhatp{\infty},T_{B_0}\uhatp{\tau}\right)_{L_2(D_T)} + \left(\uhatp{\infty},T_{B_0}(\uhatp{\tau}-\uhatp{\infty})\right)_{L_2(D_T)} + \left(\uhatp{\infty},T_{B_0}\uhatp{\infty}\right)_{L_2(D_T)} \\
        &\geq - 2C^B_2C^{\widehat{u}}\N{\uhatp{\tau}-\uhatp{\infty}}_{L_2(D_T)} + \left(\uhatp{\infty},T_{B_0}\uhatp{\infty}\right)_{L_2(D_T)},
    \end{split}
    \end{equation}
    where we used Cauchy-Schwarz inequality together with \Cref{lem:parabolicTB} in the last step.
    Since $\uhatp{\tau}$ converges to $\uhatp{\infty}$ in $L_2(D_T)$ by assumption as $\tau\rightarrow\infty$, there exists $\overbar{\tau}>0$ such that $\N{\uhatp{\tau}-\uhatp{\infty}}_{L_2(D_T)} \leq \lambda_{\tilde{k}}c_{\tilde{k}}^2/(4C^B_2C^{\widehat{u}})$ for all $\tau>\overbar{\tau}$.
    Thus, $\Q{\tau}\geq \lambda_{\tilde{k}}c_{\tilde{k}}^2/2$ for all $\tau>\overbar{\tau}$.
    With the fundamental theorem of calculus it then holds
    \begin{equation}
        \J{\tau} = \J{\overbarscript{\tau}} - \int_{\overbarscript{\tau}}^\tau \alpha_s \Q{s} \,ds \leq \J{\overbarscript{\tau}} - \frac{\lambda_{\tilde{k}}c_{\tilde{k}}^2}{2}\int_{\overbarscript{\tau}}^\tau \alpha_s \,ds \rightarrow - \infty
    \end{equation}
    as $\tau\rightarrow\infty$ due to condition~\eqref{eq:learning_rate} on the learning rate $\alpha_\tau$.
    This contradicts the positivity of the loss~$\J{}$.
    Therefore, $\uhatp{\infty}\equiv0$ a.e.\@ in $L_2(D_T)$.

    \textbf{Step 2: $\up{\infty}\equiv h$ a.e.\@ in $L_2(D_T)$.}
    By using that $\uhatp{\infty}$ is a weak solution to the adjoint PDE~\eqref{eq:parabolicadjoint*_general} in a sense analogous to \Cref{def:weak_sol_adjoint} with right-hand side $(\up{\infty}-h)$,
    we infer that the left-hand side vanishes for all test functions $\phi\in L_2(D_T)$ as in \Cref{lem:convergence_solution}.
    Thus, $\up{\infty}\equiv h$ a.e.\@ in $L_2(D_T)$.
\end{proof}


\acks{KR would like to profusely thank Tom Hickling for many insightful discussions about practical perspectives on the topic.

This research project was supported by ``DMS-EPSRC: Asymptotic Analysis of Online Training Algorithms in Machine Learning: Recurrent, Graphical, and Deep Neural Networks'' (NSF DMS-2311500).
The authors would like to acknowledge the use of the University of Oxford Advanced Research Computing (ARC) facility in carrying out this work. 
For the purpose of Open Access, the authors have applied a CC BY public copyright license to any Author Accepted Manuscript (AAM) version arising from this submission.}

\appendix

\section{Well-Posedness of the NN-PDE Training Dynamics}
\label{sec:WellPosedness_Proof}

In this appendix, we show the well-posedness of the NN-PDE training dynamics
in both the finite-width hidden layer regime and the infinite-width hidden layer limit.
In \Cref{sec:WellPosedness_Proof_infintiewidth} we prove \Cref{lem:parabolic_wellposedness_infinitewidth}, which is concerned with the latter, i.e., the well-posedness of PDE system \eqref{eq:parabolicPDE*}--\eqref{eq:parabolicadjoint*} coupled with the integro-differential equation~\eqref{eq:parabolicgtau} for $g^*_\tau$,
while \Cref{sec:WellPosedness_Proof_N} is concerned with \Cref{lem:parabolic_wellposedness_N}, i.e., the well-posedness of PDE system \eqref{eq:parabolicPDEN_plain}\,\&\,\eqref{eq:parabolicadjointN_plain} coupled with the gradient descent update~\eqref{eq:GD} for the NN parameters of the NN function $g_{\theta_\tau}^N$.

Recall that $\CS = L_2([0,T],H^1(D)) \cap L_\infty([0,T],L_2(D))$.

\subsection{Well-Posedness Proof of the NN-PDE Training Dynamics in the Infinite-Width Hidden Layer Limit}
\label{sec:WellPosedness_Proof_infintiewidth}

\begin{proof}[Proof of \Cref{lem:parabolic_wellposedness_infinitewidth}]
    \textbf{Existence.}
    The existence proof is based on a fixed point argument employing the Banach fixed point theorem.
    For a given training time horizon $\CT>0$, let us denote by $\CV_\CT=\CC\left([0,\CT],\CS\right)$ the Banach space consisting of elements with finite norm
    \begin{equation}
        \label{eq:CV_CT_norm}
        \N{u}_{\CV_\CT}
        = \sup_{\tau\in[0,\CT]} \left(\N{u_\tau}_{L_2([0,T],H^1(D))} + \N{u_\tau}_{L_\infty([0,T],L_2(D))}\right).
    \end{equation}
    A solution $((\up{\tau},\uhatp{\tau}))_{\tau\in[0,\CT]}$ to the PDE system \eqref{eq:parabolicPDE*}--\eqref{eq:parabolicadjoint*} is shown in what follows to be an element of the space $\CC\left([0,\CT],\CS\times\CS\right)$ (which we identify with the space $\CV_\CT\times\CV_\CT$) with additional regularity.

    \textit{Step~1: Existence and regularity for given right-hand side $\widetilde{g}_\tau= - \int_0^\tau \alpha_s b_s\,ds$.}
    For given $\CT>0$, let $b:[0,\CT] \rightarrow L_2(D_T)$ be a given function with $b_\tau$ being Lipschitz continuous on $\overbar{D_T}$ for each $\tau\in[0,\CT]$ and such that $\sup_{\tau\in[0,\CT]} \N{b_\tau}_{L_\infty(D_T)}\leq C_b$, where $C_b$ may depend in particular on $\CT$.
    Consider the auxiliary PDE system
    \begin{alignat}{2}
   \label{eq:parabolicPDE*_tilde}
    \begin{aligned}
        \fpartial{t}\upt{\tau} + \CL\upt{\tau} - q(\upt{\tau})
        &= \widetilde{g}_\tau = - \int_0^\tau \alpha_s b_s\,ds
        \qquad&&\text{in }
        D_T, \\
        \upt{\tau}
        &= 0
        \qquad&&\text{on }
        [0,T]\times\partial D, \\
        \upt{\tau}
        &= f
        \qquad&&\text{on }
        \{0\}\times D
    \end{aligned}
    \end{alignat}
    and
    \begin{alignat}{2}
    \label{eq:parabolicadjoint*_tilde}
    \begin{aligned}
        -\fpartial{t}\uhatpt{\tau} + \CL^\dagger\uhatpt{\tau} - q_u(\upt{\tau})\uhatpt{\tau}
        &= (\upt{\tau}-h)
        \qquad&&\text{in }
        D_T, \\
        \uhatpt{\tau}
        &= 0
        \qquad&&\text{on }
        [0,T]\times\partial D, \\
        \uhatpt{\tau}
        &= 0
        \qquad&&\text{on }
        \{T\}\times D.
    \end{aligned}
    \end{alignat}
    We first prove that there exists a solution~$(\upt{\tau},\uhatpt{\tau})\in \CS\times\CS$ to the system~\eqref{eq:parabolicPDE*_tilde}--\eqref{eq:parabolicadjoint*_tilde} for all $\tau\in[0,\CT]$ using classical existence results from \cite{ladyzhenskaia1968linear}.
    Such solution, as we show, enjoys the property that for all $\tau\in[0,\CT]$ it holds $(\partial_t\uppt{\tau}{t,\dummy},\partial_t\uhatppt{\tau}{t,\dummy})\in L_2(D)\times L_2(D)$ for a.e.\@ $t\in[0,T]$.

    \textit{Step~1a: Existence of solution to PDE \eqref{eq:parabolicPDE*_tilde}.}
    For the existence of a solution to the nonlinear PDE \eqref{eq:parabolicPDE*_tilde}, we invoke \cite[Chapter V, Theorem 6.2]{ladyzhenskaia1968linear}.
    To begin with,
    we notice that, in the notation of \cite[Chapter V, Theorem 6.2]{ladyzhenskaia1968linear}, the coefficients of the nonlinear PDE operator of the parabolic PDE~\eqref{eq:parabolicPDE*_tilde} are $a_{i}(t,x,u,p)=\sum_{j=1}^d a^{ji}(t,x)p_j$ and $a(t,x,u,p)=\sum_{i=1}^d b^i(t,x) p_i + c(t,x)u - q(t,x,u) + \int_0^\tau \alpha_s b_s(t,x)\,ds$, and thus also $A(t,x,u,p) = \sum_{i=1}^d b^i(t,x) p_i + c(t,x)u - q(t,x,u) + \int_0^\tau \alpha_s b_s(t,x)\,ds - \sum_{j=1}^d \fpartial{x_i}a^{ji}(t,x)p_j$.
    Clearly, for $(t,x)\in\overbar{D_T}$ and arbitrary $u$ it holds $\sum_{i,j=1}^d\fpartial{p_j} a_i(t,x,u,p)\xi_i\xi_j\big|_{p=0} = \sum_{i,j=1}^d a^{ji}(t,x)\xi_i\xi_j \geq \nu \N{\xi}^2\geq 0$
    by uniform parabolicity of $\fpartial{t}+\CL$, i.e., Assumption~\ref{asm:PDE_L_parabolicity},
    and it holds with Young's inequality 
    \begin{equation}
    \begin{split}
        &A(t,x,u,0)u
        = \left(c(t,x)u - q(t,x,u) + \int_0^\tau \alpha_s b_s(t,x)\,ds\right)u\\
        &\qquad\,\geq  -\!\N{c}_{L_\infty(D_T)}u^2 - C_q(1+\abs{u})\abs{u} + \frac{1}{2}\left(\int_0^\tau\! \alpha_s b_s(t,x)\,ds\right)^2 \!-\frac{1}{2}u^2
        \geq -b_1u^2-b_2\!\!
    \end{split}
    \end{equation}
    by Assumptions~\ref{asm:PDE_coefficients} and \ref{asm:WP_coefficients} for the first term, by Assumption~\ref{asm:WP_nonlinearity_q_growth} for the second term, and, for the last term, due to $\alpha_\tau$ being bounded from above together with $\sup_{\tau\in[0,\CT]} \N{b_\tau}_{L_\infty(D_T)}\leq C_b$ by assumption on $b$.
    Moreover, by Assumptions~\ref{asm:WP_coefficients} and \ref{asm:WP_nonlinearity_q_growth} the functions $a_{i}$ and $a$ are continuous w.r.t.\@ $t,x,u,p$ since again $b_\tau$ is continuous for every $\tau\in[0,\CT]$.
    Interchanging limits in the term $\int_0^\tau \alpha_s b_s(t,x)\,ds$ is warranted by the dominated convergence theorem since $\alpha_\tau$ is bounded from above and $\sup_{\tau\in[0,\CT]} \N{b_\tau}_{L_\infty(D_T)}\leq C_b$.
    In addition, the functions $a_{i}$ are differentiable w.r.t.\@ $x,u,p$ by Assumption~\ref{asm:WP_coefficients}.
    For $(t,x)\in\overbar{D_T}$, $\abs{u}\leq M$ and arbitrary $p$ we furthermore have
    \begin{equation}
    \begin{split}
        &\sum_{i=1}^d \left(\abs{a_i} + \abs{\fpartial{u}a_i}\right)(1+\N{p}) + \sum_{i,j=1}^d \abs{\fpartial{x_j} a_i} + \abs{a} \\
        &\qquad\,= \sum_{i=1}^d \abs{\sum_{j=1}^d a^{ji}(t,x)p_j} (1+\N{p}) + \sum_{i,j=1}^d \abs{\fpartial{x_j}\sum_{k=1}^d a^{ki}(t,x)p_k}\\
        &\qquad\,\quad\, + \abs{\sum_{i=1}^d b^i(t,x) p_i + c(t,x)u - q(t,x,u) + \int_0^\tau \alpha_s b_s(t,x)\,ds}
        \leq \mu (1+\N{p})^2,
    \end{split}
    \end{equation}
    where the last inequality holds due to Assumptions~\ref{asm:PDE_coefficients} and  \ref{asm:WP_nonlinearity_q_growth}, and due to the last term being uniformly bounded with the same arguments as above.
    Furthermore, for $(t,x)\in\overbar{D_T}$, $\abs{u}\leq M$ and $\N{p}\leq \widetilde{M}$, we have the following Hölder continuity properties in $(t,x,u,p)$ (we denote by $\star$ the exponent if the respective function does not depend on the variable, thus being Hölder continuous with any exponent):
    the functions $a_i$ are $(\gamma_1/2,\gamma_1,\star,1)$-Hölder continuous, 
    the functions $\partial_{p_j}a_i$ are $(\gamma_1/2,\gamma_1,\star,\star)$-Hölder continuous,
    the functions $\partial_{u}a_i$ are $(\star,\star,\star,\star)$-Hölder continuous,
    the functions $\partial_{x_i}a_i$ are $(\gamma_1/2,\gamma_1,\star,\star)$-Hölder continuous,
    and the function $a$ is $(\min\{\gamma_1/2,1\},\min\{\gamma_1,1\},1,1)$-Hölder continuous.
    The Hölder properties of all those functions are due to Assumption~\ref{asm:WP_coefficients}, except
    for the last function, where we further used that firstly $q$ is $(\gamma_1/2,\gamma_1,1)$-Hölder continuous in $(t,x,u)$ by Assumption~\ref{asm:WP_nonlinearity_q_growth} for $t,x$ and the mean-value theorem together with Assumption~\ref{asm:PDE_nonlinearity_q_ubdd} for $u$, and secondly that $b_\tau$ is $(1,1)$-Hölder continuous for every $\tau\in[0,\CT]$ by assumption together with $\alpha_\tau$ being bounded from above.
    Lastly, the boundary $\partial D$ and the initial condition~$f$ and boundary condition satisfy the assumptions due to Assumptions~\ref{asm:D} and \ref{asm:WP_initialcondition}, respectively. 
    Thus, \cite[Chapter V, Theorem 6.2]{ladyzhenskaia1968linear} ensures the existence of a solution~$\upt{\tau}\in H^{\gamma'/2,\gamma'}(\overbar{D_T})$ to \eqref{eq:parabolicPDE*_tilde} with $\partial_{x_i}\upt{\tau}$ being bounded in $\overbar{D_T}$.
    Since we are on a compact domain as of Assumptions~\ref{asm:WP_coefficients}, where Hölder continuity implies uniform boundedness, we proved $\upt{\tau}\in \CS$.
    \cite[Chapter V, Theorem 6.2]{ladyzhenskaia1968linear} further ensures that $\partial_t \upt{\tau} \in H^{\gamma'/2,\gamma'}(D_T)$, and thus also $\partial_t\uppt{\tau}{t,\dummy}\in L_2(D)$ for a.e.\@ $t\in[0,T]$ is proven.

    \textit{Step~1b: Existence of solution to adjoint PDE \eqref{eq:parabolicadjoint*_tilde}.}
    For the existence of a solution to the linear adjoint PDE \eqref{eq:parabolicadjoint*_tilde}, 
    we invoke the classical results \cite[Chapter~7.1, Theorem~3]{evans2010partial} and \cite[Chapter~7.1, Theorem~4]{evans2010partial} as well as \cite[Chapter IV, Theorem 9.1]{ladyzhenskaia1968linear} with $p=2$.
    To this end, let us first reverse the adjoint parabolic backward PDE \eqref{eq:parabolicadjoint*_tilde} in time to obtain with a time transformation
    for $\uhatptReverted{\tau}=\uhatpptReverted{\tau}{t,x}=\uhatppt{\tau}{T-t,x}$ the parabolic forward PDE
    \begin{alignat}{2}
\label{eq:parabolicadjoint*_tildereverted}
    \begin{aligned}
     \fpartial{t}\uhatptReverted{\tau} + \underbar{\CL}^*\uhatptReverted{\tau} \!-\! \underbar{q}_u(\uppt{\tau}{T\!-\!\dummy,\dummy})\uhatptReverted{\tau}
        &= (\uppt{\tau}{T\!-\!\dummy,\dummy} \!-\! h(T\!-\!\dummy,\dummy))
        \qquad&&\text{in }
        D_T, \\
        \uhatptReverted{\tau}
        &= 0
        \qquad&&\text{on }
        [0,T]\times\partial D, \\
        \uhatptReverted{\tau}
        &= 0
        \qquad&&\text{on }
        \{0\}\times D, 
    \end{aligned}
    \end{alignat}
    where $\underbar{\CL}^*=\underbar{\CL}^*(t,x)=\CL^\dagger(T-t,x)$ (analogously for the individual coefficients of the operator $\underbar{\CL}^*$) and $\underbar{q}=\underbar{q}(t,x,\widehat{u})=q(T-t,x,\widehat{u})$.
    Since the parabolic PDE \eqref{eq:parabolicadjoint*_tildereverted} is linear,
    existence and uniqueness of a weak solution of \eqref{eq:parabolicadjoint*_tildereverted} in the sense of \Cref{def:weak_sol_adjoint} follow from classical results,
    see, e.g., \cite[Chapter~7.1, Theorem~3]{evans2010partial} and \cite[Chapter~7.1, Theorem~4]{evans2010partial} for existence and uniqueness, respectively.
    To apply those results, note that the term $\underbar{q}_u(\uppt{\tau}{T-\dummy,\dummy})\uhatptReverted{\tau}$ can be absorbed into a PDE operator ${\widetilde{\underbar{\CL}}}^*$ with $\widetilde{\underbar{c}} = \underbar{c} - \sum_{i=1}^d \fpartial{x_i}\underbar{b}^i-\underbar{q}_u(\uppt{\tau}{T-\dummy,\dummy})\in L_\infty(D_T)$ due to Assumptions~\ref{asm:PDE_coefficients} and \ref{asm:PDE_nonlinearity_q_ubdd}.
    Moreover, since $\upt{\tau}\in L_2(D_T)$ by the former statement and since $h\in L_2(D_T)$ by assumption, the right-hand side~$(\uppt{\tau}{T-\dummy,\dummy}-h(T-\dummy,\dummy))\in L_2(D_T)$.
    With this, we proved $\uhatpt{\tau}\in L_2([0,T],H^1(D)) \cap L_\infty([0,T],L_2(D))$.
    To prove additional regularity, we invoke \cite[Chapter IV, Theorem 9.1]{ladyzhenskaia1968linear} with $p=2$.
    We now notice that, in the notation of \cite[Chapter IV, Theorem 9.1]{ladyzhenskaia1968linear}, the coefficients~$a_{ij}=\underbar{a}^{ij}$ of the linear PDE operator of the parabolic PDE~\eqref{eq:parabolicadjoint*_tildereverted} are bounded continuous functions in $D_T$ for all $i,j=1,\dots,d$ due to Assumptions~\ref{asm:PDE_coefficients} and \ref{asm:WP_coefficients},
    while the coefficients~$a_i=\underbar{b}^i-\sum_{j=1}^d\fpartial{x_j}\underbar{a}^{ji}$ and $a=\underbar{c} - \sum_{i=1}^d \fpartial{x_i}\underbar{b}^i - \underbar{q}_u(\uppt{\tau}{T-\dummy,\dummy})$ have finite norms $\N{a_i}_{L_r(D_T)}$ and $\N{a}_{L_s(D_T)}$ for any $r,s>0$.
    This is due to the uniform boundedness of the coefficients per Assumptions~\ref{asm:PDE_coefficients} and \ref{asm:PDE_nonlinearity_q_ubdd} combined with the boundedness of the domain per Assumption~\ref{asm:Dbdd}, see the subsequent computations with $T'=0$ and $\Delta T'=T$.
    Moreover, since it hold $\N{a_i}_{L_r(D_{T',T'+\Delta T'})} \leq \big(\Nnormal{b^i}_{L_\infty(D_T)}+\sum_{j=1}^d\Nnormal{\fpartial{x_j}a^{ji}}_{L_\infty(D_T)}\big)(\Delta T'\vol{D})^{1/r}$
    for all $i=1,\dots,d$ and
    $\Nnormal{a}_{L_s(D_{T',T'+\Delta T'})}\leq \big(\Nnormal{c}_{L_\infty(D_T)}+\sum_{i=1}^d\Nnormal{\fpartial{x_i}b^{i}}_{L_\infty(D_T)} + c_q\big)(\Delta T'\vol{D})^{1/s}$,
    $\N{a_i}_{L_r(D_{T',T'+\Delta T'})}$ and $\N{a}_{L_s(D_{T',T'+\Delta T'})}$ tend to zero as $\Delta T'\rightarrow0$.
    Furthermore, $\partial D$ is sufficiently smooth as of Assumption~\ref{asm:D}.
    The right-hand side $f=(\uppt{\tau}{T-\dummy,\dummy}-h(T-\dummy,\dummy))\in L_2(D_T)$ as argued before.
    Lastly, the initial and boundary conditions $\phi=0\in W^{1}_2(D)$ and $\Phi=0\in W^{3/4,3/2}_2(\partial D_T)$ satisfy the compatibility condition $\phi|_{\partial D}=\Phi|_{t=0}$.
    Thus, \cite[Chapter IV, Theorem 9.1]{ladyzhenskaia1968linear} ensures the existence of a unique solution~$\uhatptReverted{\tau}\in W^{1,2}_{2}(D_T)$ to \eqref{eq:parabolicadjoint*_tildereverted} and thus also a unique solution~$\uhatpt{\tau}\in W^{1,2}_{2}(D_T)$ to the parabolic backward PDE \eqref{eq:parabolicadjoint*_tilde}.
    We moreover have the bound
    \begin{equation}
        \Nbig{\uhatpt{\tau}}_{W^{1,2}_{2}(D_T)}
        = \Nbig{\uhatptReverted{\tau}}_{W^{1,2}_{2}(D_T)}
        \lesssim \N{\upt{\tau}}_{L_2(D_T)} + \N{h}_{L_2(D_T)}.
    \end{equation}
    In particular, since $\uhatpt{\tau}\in W^{1,2}_{2}(D_T)$, also $\partial_t\uhatppt{\tau}{t,\dummy}\in L_2(D)$ for a.e.\@ $t\in[0,T]$ is proven.
    
    \textit{Step~1c: Explicit norm bound for the solution to PDE system \eqref{eq:parabolicPDE*_tilde}--\eqref{eq:parabolicadjoint*_tilde}.}
    In this step, we compute explicit bounds on the norms $\N{\upt{\tau}}_{L_2([0,T],H^1(D))} + \N{\upt{\tau}}_{L_\infty([0,T],L_2(D))}$ as well as $\Nnormal{\uhatpt{\tau}}_{L_2([0,T],H^1(D))} + \Nnormal{\uhatpt{\tau}}_{L_\infty([0,T],L_2(D))}$, respectively.

    \textit{Step~1c(i): Energy estimate for solution to \eqref{eq:parabolicPDE*_tilde}.}
    For the norm of a solution to the nonlinear PDE \eqref{eq:parabolicPDE*_tilde} we conduct the following computations.
    We obtain by chain rule and by using that $\upt{\tau}$ is a weak solution to \eqref{eq:parabolicPDE*_tilde} in the sense of \Cref{def:weak_sol} that
    \begin{equation}
        \label{eq:proof:lem:parabolic_wellposedness_infinitewidth:NORM_0}
    \begin{split}
        \fpartial{t}\Nbig{\uppt{\tau}{t,\dummy}}^2_{L_2(D)}
        &=2\big(\uppt{\tau}{t,\dummy},\fpartial{t}\uppt{\tau}{t,\dummy}\big)_{L_2(D)}
        = 
        2\big\langle\fpartial{t}\uppt{\tau}{t,\dummy},\uppt{\tau}{t,\dummy}\big\rangle_{H^{-1}(D),H_0^1(D)}\\
        &= - 2\CB\big[\uppt{\tau}{t,\dummy},\uppt{\tau}{t,\dummy};t\big]
        + 2\big(q(t,\dummy,\uppt{\tau}{t,\dummy}),\uppt{\tau}{t,\dummy}\big)_{L_2(D)}\\
        &\quad\,+ 2\big(\widetilde{g}_\tau(t,\dummy),\uppt{\tau}{t,\dummy}\big)_{L_2(D)},
    \end{split}
    \end{equation}
    where the second step is due to the dual pairing between $H^{-1}(D)$ and $H_0^1(D)$ coinciding with the $L_2(D)$ scalar product \cite[Chapter~5.9, Theorem 1(iii)]{evans2010partial} since $\uppt{\tau}{t,\dummy}\in H_0^1(D)$ for a.e.\@ $t\in[0,T]$ and since $\fpartial{t}\uppt{\tau}{t,\dummy}\in L_2(D)$ for a.e.\@ $t\in[0,T]$.
    For the weak solution property in the third step of \eqref{eq:proof:lem:parabolic_wellposedness_infinitewidth:NORM_0} to hold, we note that $\upt{\tau}$ can be used as a test function in the weak formulation of~\eqref{eq:parabolicPDE*_tilde}, see \Cref{def:weak_sol}, since $\uppt{\tau}{t,\dummy}\in H_0^1(D)$ for a.e.\@ $t\in[0,T]$.
    To estimate the right-hand side of \eqref{eq:proof:lem:parabolic_wellposedness_infinitewidth:NORM_0} from above,
    we consider each of the three terms separately.
    For the first term of \eqref{eq:proof:lem:parabolic_wellposedness_infinitewidth:NORM_0}, by using the definition of the bilinear form $\CB$ as well as that by Assumption~\ref{asm:PDE_L_parabolicity} the PDE operator is uniformly parabolic and that by Assumption~\ref{asm:PDE_coefficients} the coefficients are in $L_\infty$, we can estimate with Cauchy-Schwarz and Young's inequality analogously to \eqref{eq:proof:lem:parabolic_uhatL2:21}
    \begin{equation}        \label{eq:proof:lem:parabolic_wellposedness_infinitewidth:NORM_2}
    \begin{split}
        -\CB[\uppt{\tau}{t,\dummy},\uppt{\tau}{t,\dummy};t]
        &\!\leq\!
        -\frac{\nu}{2}\! \abs{\uppt{\tau}{t,\dummy}}^2_{H^1(D)}\!+\! \left(\frac{1}{2\nu} \sum_{i=1}^d\N{b^i}_{L_\infty(D_T)}\!\!+\!\N{c}_{L_\infty(D_T)}\!\right)\!\N{\uppt{\tau}{t,\dummy}}^2_{L_2(D)}.
    \end{split}
    \end{equation}
    For the second term of \eqref{eq:proof:lem:parabolic_wellposedness_infinitewidth:NORM_0}
    we can estimate with Assumption~\ref{asm:WP_nonlinearity_q_growth} that
    \begin{equation}        \label{eq:proof:lem:parabolic_wellposedness_infinitewidth:NORM_3}
    \begin{split}
        \left(q(t,\dummy,\uppt{\tau}{t,\dummy}),\uppt{\tau}{t,\dummy}\right)_{L_2(D)}
        &\leq\N{q(t,\dummy,\uppt{\tau}{t,\dummy})}_{L_2(D)}\N{\uppt{\tau}{t,\dummy}}_{L_2(D)}\\
        &\leq C_q(1+\N{\uppt{\tau}{t,\dummy}}_{L_2(D)})\N{\uppt{\tau}{t,\dummy}}_{L_2(D)}\\
        & = C_q\left(\frac{1}{2} + \frac{3}{2}\N{\uppt{\tau}{t,\dummy}}_{L_2(D)}^2\right).
    \end{split}
    \end{equation}
    For the third and last term of \eqref{eq:proof:lem:parabolic_wellposedness_infinitewidth:NORM_0}, by Cauchy-Schwarz and Young’s inequality we upper bound
    \begin{equation}        \label{eq:proof:lem:parabolic_wellposedness_infinitewidth:NORM_4}
    \begin{split}
        \left(\widetilde{g}_\tau(t,\dummy),\uppt{\tau}{t,\dummy}\right)_{L_2(D)}
        &\leq \N{\widetilde{g}_\tau(t,\dummy)}_{L_2(D)}\N{\uppt{\tau}{t,\dummy}}_{L_2(D)}\\
        &\leq \frac{1}{2}\left(\N{\widetilde{g}_\tau(t,\dummy)}_{L_2(D)}^2 + \N{\uppt{\tau}{t,\dummy}}_{L_2(D)}^2\right).
    \end{split}
    \end{equation}
    Combining the bounds established in \eqref{eq:proof:lem:parabolic_wellposedness_infinitewidth:NORM_2}--\eqref{eq:proof:lem:parabolic_wellposedness_infinitewidth:NORM_4} and inserting them into \eqref{eq:proof:lem:parabolic_wellposedness_infinitewidth:NORM_0}, we arrive after reordering at
    \begin{equation}        \label{eq:proof:lem:parabolic_wellposedness_infinitewidth:NORM_5}
    \begin{split}        &\fpartial{t}\N{\uppt{\tau}{t,\dummy}}^2_{L_2(D)} + \frac{\nu}{2} \abs{\uppt{\tau}{t,\dummy}}^2_{H^1(D)}
        \leq C\N{\uppt{\tau}{t,\dummy}}^2_{L_2(D)} + \N{\widetilde{g}_\tau(t,\dummy)}_{L_2(D)}^2 + C,
    \end{split}
    \end{equation}
    for a constant $C=C(\CL,q)$.
    An application of Grönwall's inequality shows
    \begin{equation}        \label{eq:proof:lem:parabolic_wellposedness_infinitewidth:NORM_utilde}
        \N{\upt{\tau}}_{L_2([0,T],H^1(D))} + \N{\upt{\tau}}_{L_\infty([0,T],L_2(D))}
        \leq C\left(\N{f}_{L_2(D)} + \N{\widetilde{g}_\tau}_{L_2(D_T)} + 1\right)
    \end{equation}
    for some other, potentially larger, constant $C=C(T,\CL,q)$.
    Thus, in particular,
    \begin{equation}        \label{eq:proof:lem:parabolic_wellposedness_infinitewidth:NORM_utilde_UNIFORM}
        \N{\upt{}}_{\CV_\CT}
        \leq C\left(\N{f}_{L_2(D)} + \sup_{\tau\in[0,\CT]}\N{\widetilde{g}_\tau}_{L_2(D_T)} + 1\right).
    \end{equation}

    \textit{Step~1c(ii): Energy estimate for solution to \eqref{eq:parabolicadjoint*_tilde}.}
    For the norm of a solution to the PDE \eqref{eq:parabolicadjoint*_tilde} we proceed as follows using the time-reversed formulation \eqref{eq:parabolicadjoint*_tildereverted}.
    We obtain again by chain rule and by using that $\uhatptReverted{\tau}$ is a weak solution to \eqref{eq:parabolicPDE*_tilde} in the sense of \Cref{def:weak_sol_adjoint} that
    \begin{equation}        \label{eq:proof:lem:parabolic_wellposedness_infinitewidth:NORM2_1}
    \begin{split}        \fpartial{t}\Nnormal{\uhatpptReverted{\tau}{t,\dummy}}^2_{L_2(D)}        &=2\big(\uhatpptReverted{\tau}{t,\dummy},\fpartial{t}\uhatpptReverted{\tau}{t,\dummy}\big)_{L_2(D)}
        =         2\big\langle\fpartial{t}\uhatpptReverted{\tau}{t,\dummy},\uhatpptReverted{\tau}{t,\dummy}\big\rangle_{H^{-1}(D),H_0^1(D)}\\
        &= - 2\underbar{\CB}^*\big[\uhatpptReverted{\tau}{t,\dummy},\uhatpptReverted{\tau}{t,\dummy};t\big]
        + 2\big(\underbar{q}_u(t,\dummy,\uppt{\tau}{T-t,\dummy})\uhatpptReverted{\tau}{t,\dummy},\uhatpptReverted{\tau}{t,\dummy}\big)_{L_2(D)}\\
        &\quad\,+ 2\big(\uppt{\tau}{T-t,\dummy}-h(T-t,\dummy),\uhatpptReverted{\tau}{t,\dummy}\big)_{L_2(D)},
    \end{split}
    \end{equation}
    where the individual steps hold as before since $\uhatpptReverted{\tau}{t,\dummy}\in H_0^1(D)$ for a.e.\@ $t\in[0,T]$ and since $\fpartial{t}\uhatpptReverted{\tau}{t,\dummy}\in L_2(D)$ for a.e.\@ $t\in[0,T]$.
    To estimate the right-hand side of \eqref{eq:proof:lem:parabolic_wellposedness_infinitewidth:NORM2_1} from above,
    we consider each of the three terms separately.
    For the first term of \eqref{eq:proof:lem:parabolic_wellposedness_infinitewidth:NORM2_1}, by using the definition of the bilinear form $\CB^\dagger$ as well as that by Assumption~\ref{asm:PDE_L_parabolicity} the PDE operator is uniformly parabolic and that by Assumption~\ref{asm:PDE_coefficients} the coefficients are in $L_\infty$, we can estimate as in \eqref{eq:proof:lem:parabolic_wellposedness_infinitewidth:NORM_2} that 
    \begin{equation}        \label{eq:proof:lem:parabolic_wellposedness_infinitewidth:NORM2_2}
    \begin{split}
        &-\underbar{\CB}^*\big[\uhatpptReverted{\tau}{t,\dummy},\uhatpptReverted{\tau}{t,\dummy};t\big]
        =-\CB^\dagger\big[\uhatpptReverted{\tau}{t,\dummy},\uhatpptReverted{\tau}{t,\dummy};T-t\big]
        =-\CB\big[\uhatpptReverted{\tau}{t,\dummy},\uhatpptReverted{\tau}{t,\dummy};T-t\big]\\
        &\qquad\qquad\,\leq -\frac{\nu}{2} \absnormal{\uhatpptReverted{\tau}{t,\dummy}}^2_{H^1(D)}+ \left(\frac{1}{2\nu} \sum_{i=1}^d\N{b^i}_{L_\infty(D_T)}+\N{c}_{L_\infty(D_T)}\right)\Nnormal{\uhatpptReverted{\tau}{t,\dummy}}^2_{L_2(D)}.
    \end{split}
    \end{equation}
    For the second term of \eqref{eq:proof:lem:parabolic_wellposedness_infinitewidth:NORM2_1}
    we can estimate with Assumption~\ref{asm:PDE_nonlinearity_q_ubdd} that
    \begin{equation}        \label{eq:proof:lem:parabolic_wellposedness_infinitewidth:NORM2_3}
    \begin{split}
        \big(\underbar{q}_u(t,\dummy,\uppt{\tau}{T-t,\dummy})\uhatpptReverted{\tau}{t,\dummy},\uhatpptReverted{\tau}{t,\dummy}\big)_{L_2(D)}
        &\leq c_q\Nnormal{\uhatpptReverted{\tau}{t,\dummy}}_{L_2(D)}^2.
    \end{split}
    \end{equation}
    For the third and last term of \eqref{eq:proof:lem:parabolic_wellposedness_infinitewidth:NORM2_1}, by Cauchy-Schwarz and Young’s inequality we upper bound
    \begin{equation}        \label{eq:proof:lem:parabolic_wellposedness_infinitewidth:NORM2_4}
    \begin{split}
        &\big(\uppt{\tau}{T-t,\dummy}-h(T-t,\dummy),\uhatpptReverted{\tau}{t,\dummy}\big)_{L_2(D)}
        \leq \N{\uppt{\tau}{T-t,\dummy}-h(T-t,\dummy)}_{L_2(D)}\Nnormal{\uhatpptReverted{\tau}{t,\dummy}}_{L_2(D)}\\
        &\qquad\qquad\,\leq \frac{1}{2}\left(\N{\uppt{\tau}{T-t,\dummy}-h(T-t,\dummy)}_{L_2(D)}^2 + \Nnormal{\uhatpptReverted{\tau}{t,\dummy}}_{L_2(D)}^2\right).
    \end{split}
    \end{equation}
    Combining the bounds established in \eqref{eq:proof:lem:parabolic_wellposedness_infinitewidth:NORM2_2}--\eqref{eq:proof:lem:parabolic_wellposedness_infinitewidth:NORM2_4} and inserting them into \eqref{eq:proof:lem:parabolic_wellposedness_infinitewidth:NORM2_1}, we arrive after reordering at
    \begin{equation}        \label{eq:proof:lem:parabolic_wellposedness_infinitewidth:NORM2_5}
    \begin{split}        &\fpartial{t}\Nnormal{\uhatpptReverted{\tau}{t,\dummy}}^2_{L_2(D)} + \frac{\nu}{2} \absnormal{\uhatpptReverted{\tau}{t,\dummy}}^2_{H^1(D)}
        \leq C\Nnormal{\uhatpptReverted{\tau}{t,\dummy}}^2_{L_2(D)} + \N{\uppt{\tau}{T-t,\dummy}-h(T-t,\dummy)}_{L_2(D)}^2,
    \end{split}
    \end{equation}
    for a constant $C=C(\CL,q)$.
    Recalling that $\Nnormal{\uhatpptReverted{\tau}{0,\dummy}}^2_{L_2(D)} =0$, an application of Grönwall's inequality shows
    \begin{equation}        \label{eq:proof:lem:parabolic_wellposedness_infinitewidth:NORM_uhattilde}
    \begin{split}
        \Nnormal{\uhatpt{\tau}}_{L_2([0,T],H^1(D))} + \Nnormal{\uhatpt{\tau}}_{L_\infty([0,T],L_2(D))}
        &= \Nnormal{\uhatptReverted{\tau}}_{L_2([0,T],H^1(D))} + \Nnormal{\uhatptReverted{\tau}}_{L_\infty([0,T],L_2(D))} \\
        &\leq C\N{\upt{\tau}-h}_{L_2(D_T)}
        \leq C\left(\N{\upt{\tau}}_{L_2(D_T)} + \N{h}_{L_2(D_T)}\right)
    \end{split}
    \end{equation}
    for some other, potentially larger, constant $C=C(T,\CL,q)$.
    Thus, in particular,
    \begin{equation}       \label{eq:proof:lem:parabolic_wellposedness_infinitewidth:NORM_uhattilde_UNIFORM}
        \Nnormal{\uhatpt{}}_{\CV_\CT}
        \leq C\left(\N{\upt{}}_{\CV_\CT} + \N{h}_{L_2(D_T)}\right).
    \end{equation}

    \textit{Step~1d: Existence of solution to PDE system \eqref{eq:parabolicPDE*_tilde}--\eqref{eq:parabolicadjoint*_tilde}.}
    Summarizing the former results from \textit{Steps~1a, 1b} and \textit{1c}
    we thus proved that for each $\tau\in[0,\CT]$ there exists a solution $(\upt{\tau},\uhatpt{\tau})\in \CS\times\CS$ to \eqref{eq:parabolicPDE*_tilde}--\eqref{eq:parabolicadjoint*_tilde}.
    As $\widetilde{g}_{\tau}$ is Lipschitz continuous in $\tau$ by the dominated convergence theorem, which can be seen since
    \begin{equation}
        \N{\widetilde{g}_{\tau_2} - \widetilde{g}_{\tau_1}}_{L_2(D_T)}
        = \N{\int_{\tau_1}^{\tau_2} \alpha_s b_s\,ds}_{L_2(D_T)}
        \leq C \int_{\tau_1}^{\tau_2} \N{b_s}_{L_\infty(D_T)} ds
        \leq C\abs{\tau_2-\tau_1}
    \end{equation}
    for a constant $C=C(\alpha,T,D,C_b)$,
    the solution $(\upt{},\uhatpt{})$ is in particular continuous in the training time $\tau$, i.e., $(\upt{},\uhatpt{})\in \CV_\CT\times\CV_\CT$.
    As we further showed, for each $\tau\in[0,\CT]$ such solution satisfies $(\partial_t\uppt{\tau}{t,\dummy},\partial_t\uhatppt{\tau}{t,\dummy})\in L_2(D) \times L_2(D)$
    for a.e.\@ $t\in[0,T]$.

    \textit{Step~2: Existence for specific right-hand side $\widetilde{g}_\tau = g^*_\tau= - \int_0^\tau \alpha_s T_{B_0}\uhatp{s}\,ds$.}
    We now make a specific choice for the functions $b_\tau$.

    \textit{Step~2a: Choice of NN update function $b_\tau = T_{B_0}\uhatp{\tau}$.}
    For an arbitrarily given $\uhatp{\tau}\in \CS$, $\tau\in[0,\CT]$, with $\sup_{\tau\in[0,\CT]} \N{\uhatp{\tau}}_{L_2(D_T)}\leq M$ ($M$ may depend on $\CT$), 
    we set
    \begin{equation}
        b_\tau = T_{B_0}\uhatp{\tau}
    \end{equation}
    for all $\tau\in[0,\CT]$.
    It holds with \Cref{lem:parabolicTB} that
    \begin{equation}
    \begin{split}
        \N{b_\tau}_{L_2(D_T)}
        &= \N{T_{B_0}\uhatp{\tau}}_{L_2(D_T)}\leq C^B_2 \N{\uhatp{\tau}}_{L_2(D_T)} 
        \leq C^B_2 M,
    \end{split}
    \end{equation}
    which is a uniform bound in $\tau$.
    In fact, a more careful estimate employing \Cref{lem:TBboundednessLinfty} shows
    \begin{equation}
    \begin{split}
        \N{b_\tau}_{L_\infty(D_T)}
        &= \N{T_{B_0}\uhatp{\tau}}_{L_\infty(D_T)}\\
        &=\sup_{(t,x)\in D_T}\absbig{[T_{B_0}\uhatp{\tau}](t,x)}
        \leq C^{T_B}_\infty \N{\uhatp{\tau}}_{L_2(D_T)}
        \leq C^{T_B}_\infty M.
    \end{split}
    \end{equation}
    Since the right-hand side is uniform in $\tau$, $\sup_{\tau\in[0,\CT]} \N{b_\tau}_{L_\infty(D_T)} \leq C$, where $C$ may depend on $\CT$.
    Furthermore, it is immediate to see, that using the definition of $T_{B_0}$ in \eqref{eq:parabolicT_B} and that the kernel $B$ as given in \eqref{eq:parabolicB} is continuous in $t,x$ on $\overbar{D_T}$, the function $b_\tau=T_{B_0}\uhatp{\tau}$ is continuous on $\overbar{D_T}$ for each $\tau\in[0,\CT]$ by the dominated convergence theorem.
    In fact, the function $b_\tau=T_{B_0}\uhatp{\tau}$ is Lipschitz continuous on $\overbar{D_T}$ since it holds by \Cref{lem:TBLipschitz} that
    \begin{equation}
    \begin{split}
        \abs{b_\tau(t^1,x^1)\!-\!b_\tau(t^2,x^2)}
        &= \abs{[T_{B_0}\uhatp{\tau}](t^1,x^1)\!-\![T_{B_0}\uhatp{\tau}](t^2,x^2)} 
        \leq L_{T_B}\! \left(\abs{t^1\!-\!t^2} \!+\! \N{x^1\!-\!x^2}\right)\!\!
    \end{split}
    \end{equation}
    for all $(t^1,x^1),(t^2,x^2)\in\overbar{D_T}$.

    \textit{Step~2b: Definition of fixed point mapping.}
    Let us consider the fixed point map
    \begin{equation}        \label{eq:proof:lem:parabolic_wellposedness_infinitewidth:F}
        F: \CV_\CT\times\CV_\CT \rightarrow \CV_\CT\times\CV_\CT,\quad
        (\up{},\uhatp{})\mapsto(\upt{},\uhatpt{})
    \end{equation}
    and define for given $M<\infty$ and $\CT<\infty$ the function space $\CV_\CT(M)=\{u\in\CV_\CT\!:\!\N{u}_{\CV_\CT}\leq M\}$.

    We will first show in \textit{Step~2d} existence locally in the training time by proving that there exist $M_0>0$ and $\CT_0>0$ such that $F$ is a fixed point mapping on $\CV_{\CT_0}(M_0)\times\CV_{\CT_0}(M_0)$, which allows to apply the Banach fixed point theorem.
    In \textit{Step~2e} we will then extend the proof by a bootstrapping argument to any given (arbitrarily large) time horizon~$\CT$.
    
    \textit{Step~2c: Preliminary computations.}
    Let us start by conducting some preliminary computations on a generic space $\CV_{\widetilde{\CT}}(\widetilde{M})\times \CV_{\widetilde{\CT}}(\widetilde{M})$.
    
    \textit{Step~2c(i): Preliminary computations for self-mapping property of $F$.}
    Consider $(\upt{},\uhatpt{})$ together with its corresponding $(\up{},\uhatp{}) \in \CV_{\widetilde{\CT}}(\widetilde{M})\times \CV_{\widetilde{\CT}}(\widetilde{M})$.
    Using \eqref{eq:proof:lem:parabolic_wellposedness_infinitewidth:NORM_utilde_UNIFORM} and \eqref{eq:proof:lem:parabolic_wellposedness_infinitewidth:NORM_uhattilde_UNIFORM} in the first inequality and \Cref{lem:parabolicTB} in the last step, we establish
    \begin{equation}        \label{eq:proof:lem:parabolic_wellposedness_infinitewidth:SELFMAPPING}
    \begin{split}
        &\N{\upt{}}_{\CV_{\widetilde{\CT}}} + \Nnormal{\uhatpt{}}_{\CV_{\widetilde{\CT}}} \\
        &\qquad\,\leq C\sup_{\tau\in[0,\widetilde{\CT}]} \N{\widetilde{g}_\tau}_{L_2(D_T)} + C\left(\N{h}_{L_2(D_T)} + \N{f}_{L_2(D)} + 1\right)\\ 
        &\qquad\,\leq C\sup_{\tau\in[0,\widetilde{\CT}]} \N{\int_0^\tau \alpha_s T_{B_0}\uhatp{s}\,ds}_{L_2(D_T)} + C\left(\N{h}_{L_2(D_T)} + \N{f}_{L_2(D)} + 1\right)\\
        &\qquad\,\leq C \int_0^{\widetilde{\CT}} \N{\alpha_s T_{B_0}\uhatp{s}}_{L_2(D_T)}ds + C\left(\N{h}_{L_2(D_T)} + \N{f}_{L_2(D)} + 1\right)\\
        &\qquad\,\leq C_1 \int_0^{\widetilde{\CT}} \N{\uhatp{s}}_{L_2(D_T)}ds + C_1\left(\N{h}_{L_2(D_T)} + \N{f}_{L_2(D)} + 1\right)
    \end{split}
    \end{equation}
    for a constant $C_1=C_1(\alpha,\CL,q,C^B_2)$ (to be precise, $C_1=C\max\{\alpha_0C^B_2,1\}$).

    \textit{Step~2c(ii): Preliminary computations for contractivity of $F$.}
    Consider two pairs $(\uptu{}{1},\uhatptu{}{1})$, $(\uptu{}{2},\uhatptu{}{2})$ with their corresponding $(\upu{}{1},\uhatpu{}{1}), (\upu{}{2},\uhatpu{}{2}) \in \CV_{\widetilde{\CT}}(\widetilde{M})\times \CV_{\widetilde{\CT}}(\widetilde{M})$.

    \textit{A bound for $\Nnormal{\uptu{\tau}{1}-\uptu{\tau}{2}}_{L_2([0,T],H^1(D))} + \Nnormal{\uptu{\tau}{1}-\uptu{\tau}{2}}_{L_\infty([0,T],L_2(D))}$.}
    Since both $\uptu{\tau}{1}$ and $\uptu{\tau}{2}$ weakly satisfy \eqref{eq:parabolicPDE*_tilde} in the sense of \Cref{def:weak_sol},
    it weakly holds
    \begin{equation}        \label{eq:proof:lem:parabolic_wellposedness_infinitewidth:CONTRACTIVITYu1}        \fpartial{t}\left(\uptu{\tau}{1}-\uptu{\tau}{2}\right) + \CL\left(\uptu{\tau}{1}-\uptu{\tau}{2}\right) - \left(q(\uptu{\tau}{1})-q(\uptu{\tau}{2})\right) = \widetilde{g}_\tau^1-\widetilde{g}_\tau^2
    \end{equation}
    with zero initial and zero boundary conditions.
    We obtain by chain rule and by using that $\uptu{\tau}{1}-\uptu{\tau}{2}$ is a weak solution to \eqref{eq:proof:lem:parabolic_wellposedness_infinitewidth:CONTRACTIVITYu1} in the sense of \Cref{def:weak_sol} that
    \begin{equation}        \label{eq:proof:lem:parabolic_wellposedness_infinitewidth:CONTRACTIVITYu2}
    \begin{split}        &\fpartial{t}\N{\upptu{\tau}{1}{t,\dummy}-\upptu{\tau}{2}{t,\dummy}}^2_{L_2(D)}\\        &\qquad\,=2\left(\upptu{\tau}{1}{t,\dummy}-\upptu{\tau}{2}{t,\dummy},\fpartial{t}\left(\upptu{\tau}{1}{t,\dummy}-\upptu{\tau}{2}{t,\dummy}\right)\right)_{L_2(D)}\\
        &\qquad\,= 
        2\left\langle\fpartial{t}\left(\upptu{\tau}{1}{t,\dummy}-\upptu{\tau}{2}{t,\dummy}\right),\upptu{\tau}{1}{t,\dummy}-\upptu{\tau}{2}{t,\dummy}\right\rangle_{H^{-1}(D),H_0^1(D)}\\
        &\qquad\,= - 2\CB[\upptu{\tau}{1}{t,\dummy}-\upptu{\tau}{2}{t,\dummy},\upptu{\tau}{1}{t,\dummy}-\upptu{\tau}{2}{t,\dummy};t]\\
        &\qquad\,\quad\,+ 2\left(q(t,\dummy,\upptu{\tau}{1}{t,\dummy})-q(t,\dummy,\upptu{\tau}{2}{t,\dummy}),\upptu{\tau}{1}{t,\dummy}-\upptu{\tau}{2}{t,\dummy}\right)_{L_2(D)}\\
        &\qquad\,\quad\,+ 2\left(\widetilde{g}_\tau^1(t,\dummy)-\widetilde{g}_\tau^2(t,\dummy),\upptu{\tau}{1}{t,\dummy}-\upptu{\tau}{2}{t,\dummy}\right)_{L_2(D)},
    \end{split}
    \end{equation}
    where the individual steps hold as previously described.
    To estimate the right-hand side of \eqref{eq:proof:lem:parabolic_wellposedness_infinitewidth:CONTRACTIVITYu2} from above,
    we consider each of the three terms separately.
    For the first term, by using the definition of the bilinear form $\CB$ as well as that by Assumption~\ref{asm:PDE_L_parabolicity} the PDE operator is uniformly parabolic and that by Assumption~\ref{asm:PDE_coefficients} the coefficients are in $L_\infty$, we can estimate as in \eqref{eq:proof:lem:parabolic_wellposedness_infinitewidth:NORM_2} that
    \begin{equation}        \label{eq:proof:lem:parabolic_wellposedness_infinitewidth:CONTRACTIVITYu3}
    \begin{split}
        &-\CB[\upptu{\tau}{1}{t,\dummy}-\upptu{\tau}{2}{t,\dummy},\upptu{\tau}{1}{t,\dummy}-\upptu{\tau}{2}{t,\dummy};t]\\
        &\qquad\,\leq -\frac{\nu}{2} \abs{\upptu{\tau}{1}{t,\dummy}-\upptu{\tau}{2}{t,\dummy}}^2_{H^1(D)}\\
        &\qquad\,\quad\,+ \left(\frac{1}{2\nu} \sum_{i=1}^d\N{b^i}_{L_\infty(D_T)}+\N{c}_{L_\infty(D_T)}\right)\N{\upptu{\tau}{1}{t,\dummy}-\upptu{\tau}{2}{t,\dummy}}^2_{L_2(D)}.
    \end{split}
    \end{equation}
    For the second term, we first note that by the mean-value theorem, for any $(t,x)\in D_T$ there exists a $\xi(t,x)$ such that
    \begin{equation}        \label{eq:proof:lem:parabolic_wellposedness_infinitewidth:CONTRACTIVITYu4aux}
    \begin{split}
        q(t,x,\upptu{\tau}{1}{t,x})-q(t,x,\upptu{\tau}{2}{t,x})
        = q_u(t,x,\xi(t,x))\left(\upptu{\tau}{1}{t,x}-\upptu{\tau}{2}{t,x}\right).
    \end{split}
    \end{equation}
    Leveraging this while using that by Assumption~\ref{asm:PDE_nonlinearity_q_ubdd}  $q_u$ is bounded,
    we can estimate
    \begin{equation}        \label{eq:proof:lem:parabolic_wellposedness_infinitewidth:CONTRACTIVITYu4}
    \begin{split}
        &\left(q(t,\dummy,\upptu{\tau}{1}{t,\dummy})-q(t,\dummy,\upptu{\tau}{2}{t,\dummy}),\upptu{\tau}{1}{t,\dummy}-\upptu{\tau}{2}{t,\dummy}\right)_{L_2(D)}\\
        &\qquad\,= \int_D \left(q_u(t,x,\xi(t,x))\left(\upptu{\tau}{1}{t,x}-\upptu{\tau}{2}{t,x}\right)\right)\left(\upptu{\tau}{1}{t,x}-\uppu{\tau}{2}{t,x}\right) dx\\
        &\qquad\,\leq c_q\int_D \left(\upptu{\tau}{1}{t,x}-\uppu{\tau}{2}{t,x}\right)^2 dx
        = c_q \N{\upptu{\tau}{1}{t,\dummy}-\upptu{\tau}{2}{t,\dummy}}_{L_2(D)}^2.
    \end{split}
    \end{equation}
    For the third and last term, by Cauchy-Schwarz and Young’s inequality we upper bound
    \begin{equation}        \label{eq:proof:lem:parabolic_wellposedness_infinitewidth:CONTRACTIVITYu5}
    \begin{split}
        &\left(\widetilde{g}_\tau^1(t,\dummy)-\widetilde{g}_\tau^2(t,\dummy),\upptu{\tau}{1}{t,\dummy}-\upptu{\tau}{2}{t,\dummy}\right)_{L_2(D)} \\
        &\qquad\,\leq \N{\widetilde{g}_\tau^1(t,\dummy)-\widetilde{g}_\tau^2(t,\dummy)}_{L_2(D)}\N{\upptu{\tau}{1}{t,\dummy}-\upptu{\tau}{2}{t,\dummy}}_{L_2(D)}\\
        &\qquad\,\leq \frac{1}{2}\left(\N{\widetilde{g}_\tau^1(t,\dummy)-\widetilde{g}_\tau^2(t,\dummy)}_{L_2(D)}^2 + \N{\upptu{\tau}{1}{t,\dummy}-\upptu{\tau}{2}{t,\dummy}}_{L_2(D)}^2\right).
    \end{split}
    \end{equation}
    Combining the bounds established in \eqref{eq:proof:lem:parabolic_wellposedness_infinitewidth:CONTRACTIVITYu3}--\eqref{eq:proof:lem:parabolic_wellposedness_infinitewidth:CONTRACTIVITYu5} and inserting them into \eqref{eq:proof:lem:parabolic_wellposedness_infinitewidth:CONTRACTIVITYu2},
    we arrive after reordering at
    \begin{equation}        \label{eq:proof:lem:parabolic_wellposedness_infinitewidth:CONTRACTIVITYu6}
    \begin{split}        &\fpartial{t}\N{\upptu{\tau}{1}{t,\dummy}-\upptu{\tau}{2}{t,\dummy}}^2_{L_2(D)} + \frac{\nu}{2} \abs{\upptu{\tau}{1}{t,\dummy}-\upptu{\tau}{2}{t,\dummy}}^2_{H^1(D)}\\
        &\qquad\,\leq C\N{\upptu{\tau}{1}{t,\dummy}-\upptu{\tau}{2}{t,\dummy}}^2_{L_2(D)} + \N{\widetilde{g}_\tau^1(t,\dummy)-\widetilde{g}_\tau^2(t,\dummy)}_{L_2(D)}^2,
    \end{split}
    \end{equation}
    for a constant $C=C(\CL,q)$.
    Recalling that $\upptu{\tau}{1}{0,\dummy}=\upptu{\tau}{2}{0,\dummy}$, an application of Grönwall's inequality shows
    \begin{equation}        \label{eq:proof:lem:parabolic_wellposedness_infinitewidth:CONTRACTIVITY_utilde}
        \N{\uptu{\tau}{1}-\uptu{\tau}{2}}_{L_2([0,T],H^1(D))} + \N{\uptu{\tau}{1}-\uptu{\tau}{2}}_{L_\infty([0,T],L_2(D))}
        \leq C\N{\widetilde{g}_\tau^1-\widetilde{g}_\tau^2}_{L_2(D_T)}
    \end{equation}
    for some other, potentially larger, constant $C=C(T,\CL,q)$.
    Thus, in particular,
    \begin{equation}        \label{eq:proof:lem:parabolic_wellposedness_infinitewidth:CONTRACTIVITY_utilde_UNIFORM}
        \N{\uptu{\tau}{1}-\uptu{\tau}{2}}_{\CV_\CT}
        \leq C\sup_{\tau\in[0,\CT]}\N{\widetilde{g}_\tau^1-\widetilde{g}_\tau^2}_{L_2(D_T)}.
    \end{equation}

    \textit{A bound for $\Nnormal{\uptu{\tau}{1}-\uptu{\tau}{2}}_{L_\infty(D_T)}$.}
    For later use, let us further provide an $L_\infty$ bound for the weak solution to \eqref{eq:proof:lem:parabolic_wellposedness_infinitewidth:CONTRACTIVITYu1} by employing Morrey's inequality after leveraging \cite[Chapter IV, Theorem 9.1]{ladyzhenskaia1968linear} for any $p>1$.
    Therefore notice, that by the mean-value theorem, for any $(t,x)\in D_T$ there exists a $\xi(t,x)$ such that, in place of \eqref{eq:proof:lem:parabolic_wellposedness_infinitewidth:CONTRACTIVITYu1}, it weakly holds
    \begin{equation}        \label{eq:proof:lem:parabolic_wellposedness_infinitewidth:CONTRACTIVITYu1_aux}        \fpartial{t}\left(\uptu{\tau}{1}-\uptu{\tau}{2}\right) + \CL\left(\uptu{\tau}{1}-\uptu{\tau}{2}\right) - q_u(\xi)\left(\uptu{\tau}{1}-\uptu{\tau}{2}\right) = \widetilde{g}_\tau^1-\widetilde{g}_\tau^2
    \end{equation}
    with zero initial and zero boundary conditions.
    We now notice that, in the notation of \cite[Chapter IV, Theorem 9.1]{ladyzhenskaia1968linear}, the coefficients~$a_{ij}(t,x)=a^{ij}(t,x)$ of the linear PDE operator of the parabolic PDE~\eqref{eq:proof:lem:parabolic_wellposedness_infinitewidth:CONTRACTIVITYu1_aux} are bounded continuous functions in $D_T$ for all $i,j=1,\dots,d$,
    while the coefficients~$a_i(t,x)=b^i(t,x)-\sum_{j=1}^d\fpartial{x_j}a^{ji}(t,x)$ and $a(t,x)=c(t,x) - \sum_{i=1}^d \fpartial{x_i}b^i(t,x) - q_u(t,x,\xi(t,x))$ have finite norms $\N{a_i}_{L_r(D_T)}$ and $\N{a}_{L_s(D_T)}$ for any $r,s>0$.
    This is due to the uniform boundedness of the coefficients per Assumptions~\ref{asm:PDE_coefficients} and \ref{asm:PDE_nonlinearity_q_ubdd} combined with the boundedness of the domain per Assumption~\ref{asm:Dbdd}, see the subsequent computations with $T'=0$ and $\Delta T'=T$.
    Moreover, since it hold
    $\Nnormal{a_i}_{L_r(D_{T',T'+\Delta T'})}\leq \big(\Nnormal{b^i}_{L_\infty(D_T)}+\sum_{j=1}^d\Nnormal{\fpartial{x_j}a^{ji}}_{L_\infty(D_T)}\big)(\Delta T'\vol{D})^{1/r}$
    for all $i=1,\dots,d$ and
    $\Nnormal{a}_{L_s(D_{T',T'+\Delta T'})}\leq \big(\Nnormal{c}_{L_\infty(D_T)}+\sum_{i=1}^d\Nnormal{\fpartial{x_i}b^{i}}_{L_\infty(D_T)} + c_q\big)(\Delta T'\vol{D})^{1/s}$,
    $\N{a_i}_{L_r(D_{T',T'+\Delta T'})}$ and $\N{a}_{L_s(D_{T',T'+\Delta T'})}$ tend to zero as $\Delta T'\rightarrow0$.
    Furthermore, $\partial D$ is sufficiently smooth as of Assumption~\ref{asm:D}.
    The right-hand side $f=\widetilde{g}_\tau^1-\widetilde{g}_\tau^2\in L_p(D_T)$ for any $p\geq2$ since
    \begin{equation}        \label{eq:proof:lem:parabolic_wellposedness_infinitewidth:CONTRACTIVITYuinfty3}
    \begin{split}
        &\N{\widetilde{g}_\tau^1-\widetilde{g}_\tau^2}_{L_p(D_T)}
        = \N{\int_0^\tau \alpha_s T_{B_0}(\uhatpu{s}{1}-\uhatpu{s}{2})\,ds}_{L_p(D_T)} \\
        &\quad\,\leq \alpha_0 \int_0^\tau \!\N{T_{B_0}(\uhatpu{s}{1}\!-\!\uhatpu{s}{2})}_{L_p(D_T)}ds 
        = \alpha_0 \int_0^\tau\! \left(\int_0^T\!\!\!\int_D \absbig{[T_{B_0}(\uhatpu{s}{1}\!-\!\uhatpu{s}{2})](t,x)}^p\,dxdt\right)^{1/p} \!ds \\
        &\quad\,\leq \alpha_0C_\infty^{T_B} \int_0^\tau \left(\int_0^T\!\!\!\int_D \N{\uhatpu{s}{1}-\uhatpu{s}{2}}_{L_2(D_T)}^p\,dxdt\right)^{1/p} ds \\
        &\quad\,= \alpha_0C_\infty^{T_B}(T\vol{D})^{1/p} \int_0^\tau \N{\uhatpu{s}{1}-\uhatpu{s}{2}}_{L_2(D_T)} ds\\
        &\quad\,\leq \alpha_0\CT C_\infty^{T_B}(T\vol{D})^{1/p} \sup_{\tau\in[0,\CT]}\N{\uhatpu{\tau}{1}-\uhatpu{\tau}{2}}_{L_2(D_T)} ds,
    \end{split}
    \end{equation}
    where we used \Cref{lem:TBboundednessLinfty} in the third line.
    Moreover, the initial and boundary conditions $\phi=0\in W^{2-2/p}_p(D)$ and $\Phi=0\in W^{1-1/(2p),2-1/p}_p(\partial D_T)$ satisfy the compatibility condition $\phi|_{\partial D}=\Phi|_{t=0}$.
    Thus, \cite[Chapter IV, Theorem 9.1]{ladyzhenskaia1968linear} ensures that the unique solution~$(\uptu{\tau}{1}-\uptu{\tau}{2})\in W^{1,2}_{p}(D_T)$ to    \eqref{eq:proof:lem:parabolic_wellposedness_infinitewidth:CONTRACTIVITYu1_aux} and thus \eqref{eq:proof:lem:parabolic_wellposedness_infinitewidth:CONTRACTIVITYu1} obeys the bound
    \begin{equation}        \label{eq:proof:lem:parabolic_wellposedness_infinitewidth:CONTRACTIVITYuinfty4}
    \begin{split}
        \N{\uptu{\tau}{1}-\uptu{\tau}{2}}_{W^{1,2}_{p}(D_T)}
        &\leq C\N{\widetilde{g}_\tau^1-\widetilde{g}_\tau^2}_{L_p(D_T)} 
    \end{split}
    \end{equation}
    for a constant $C=C(T,\CL,q)$.
    
    With the conditions of \cite[Chapter IV, Theorem 9.1]{ladyzhenskaia1968linear} being fulfilled for any $p\geq2$ as we verified before, they are in particular fulfilled for $p>d+1$.
    Since we have for such $p$ the continuous embedding $W^{1,2}_{p}(D_{T})\hookrightarrow W^{1,1}_{p}(D_{T})\hookrightarrow L_{\infty}(\overbar{D}_T)$ by Morrey's inequality~\cite[Theorem 9.12]{brezis2011functional}, we have the first inequality in
    \begin{equation}        \label{eq:proof:lem:parabolic_wellposedness_infinitewidth:CONTRACTIVITYuinfty5}
    \begin{split}
        \N{\uptu{\tau}{1}-\uptu{\tau}{2}}_{L_\infty(D_T)}
        &\leq c(d,p)\N{\uptu{\tau}{1}-\uptu{\tau}{2}}_{W^{1,2}_{p}(D_T)}\\
        &\leq c(d,p)C\N{\widetilde{g}_\tau^1-\widetilde{g}_\tau^2}_{L_p(D_T)},
    \end{split}
    \end{equation}
    with the second one being due to \eqref{eq:proof:lem:parabolic_wellposedness_infinitewidth:CONTRACTIVITYuinfty4}.
    Hence, using the last two lines of \eqref{eq:proof:lem:parabolic_wellposedness_infinitewidth:CONTRACTIVITYuinfty3}, we have 
    \begin{equation}        \label{eq:proof:lem:parabolic_wellposedness_infinitewidth:CONTRACTIVITYuinfty6}
    \begin{split}
        \N{\uptu{\tau}{1}-\uptu{\tau}{2}}_{L_\infty(D_T)}
        &\leq C \int_0^\tau \N{\uhatpu{s}{1}-\uhatpu{s}{2}}_{L_2(D_T)} ds \\
        &\leq C \CT \sup_{\tau\in[0,\CT]}\N{\uhatpu{\tau}{1}-\uhatpu{\tau}{2}}_{L_2(D_T)} ds,
    \end{split}
    \end{equation}
    for a constant $C=C(\alpha,T,D,\CL,q,C^{T_B}_\infty)$.

    \textit{A bound for $\Nnormal{\uhatptu{\tau}{1}-\uhatptu{\tau}{2}}_{L_2([0,T],H^1(D))} + \Nnormal{\uhatptu{\tau}{1}-\uhatptu{\tau}{2}}_{L_\infty([0,T],L_2(D))}$.}
    Since both $\uhatptu{\tau}{1}$ and $\uhatptu{\tau}{2}$ weakly satisfy \eqref{eq:parabolicadjoint*_tilde} in the sense of \Cref{def:weak_sol_adjoint},
    it weakly holds
    \begin{equation}        \label{eq:proof:lem:parabolic_wellposedness_infinitewidth:CONTRACTIVITYuhat1'}
        -\fpartial{t}\big(\uhatptu{\tau}{1}-\uhatptu{\tau}{2}\big) + \CL^\dagger\big(\uhatptu{\tau}{1}-\uhatptu{\tau}{2}\big) - \big(q_u(\uptu{\tau}{1})\uhatptu{\tau}{1}-q_u(\uptu{\tau}{2})\uhatptu{\tau}{2}\big)
        = \uptu{\tau}{1}-\uptu{\tau}{2}
    \end{equation}
    with zero terminal and zero boundary conditions,
    or equivalently
    \begin{equation}        \label{eq:proof:lem:parabolic_wellposedness_infinitewidth:CONTRACTIVITYuhat1}
    \begin{split}
        -&\fpartial{t}\big(\uhatptu{\tau}{1}-\uhatptu{\tau}{2}\big) + \CL^\dagger\big(\uhatptu{\tau}{1}-\uhatptu{\tau}{2}\big) - \big(q_u(\uptu{\tau}{1})\big(\uhatptu{\tau}{1}-\uhatptu{\tau}{2}\big)\big)\\
        &\qquad\qquad\qquad\qquad\qquad\qquad\qquad\;\;= (\uptu{\tau}{1}-\uptu{\tau}{2}) + \big(q_u(\uptu{\tau}{1})-q_u(\uptu{\tau}{2})\big)\uhatptu{\tau}{2}.
    \end{split}
    \end{equation}
    Let us now first reverse \eqref{eq:proof:lem:parabolic_wellposedness_infinitewidth:CONTRACTIVITYuhat1} in time to obtain with a time transformation
    for $\uhatptuReverted{\tau}{k}=\uhatpptuReverted{\tau}{k}{t,x}=\uhatpptu{\tau}{k}{T-t,x}$ for $k=1,2$ the parabolic forward PDE
    \begin{equation}        \label{eq:proof:lem:parabolic_wellposedness_infinitewidth:CONTRACTIVITYuhat1reverted}
    \begin{split}
        &\fpartial{t}\big(\uhatptuReverted{\tau}{1}-\uhatptuReverted{\tau}{2}\big) + \underbar{\CL}^*\big(\uhatptuReverted{\tau}{1}-\uhatptuReverted{\tau}{2}\big) - \big(\underbar{q}_u(\upptu{\tau}{1}{T-\dummy,\dummy})\big(\uhatptuReverted{\tau}{1}-\uhatptuReverted{\tau}{2}\big)\big)\\
        &\qquad\,= (\upptu{\tau}{1}{T-\dummy,\dummy}-\upptu{\tau}{2}{T-\dummy,\dummy}) + \big(\underbar{q}_u(\upptu{\tau}{1}{T-\dummy,\dummy})-\underbar{q}_u(\upptu{\tau}{2}{T-\dummy,\dummy})\big)\uhatptuReverted{\tau}{2}
    \end{split}
    \end{equation}
    with zero initial and zero boundary conditions.
    We obtain by chain rule and by using that $\uhatptuReverted{\tau}{1}-\uhatptuReverted{\tau}{2}$ is a weak solution to \eqref{eq:proof:lem:parabolic_wellposedness_infinitewidth:CONTRACTIVITYuhat1reverted} in the sense of \Cref{def:weak_sol_adjoint} that
    \begin{equation}        \label{eq:proof:lem:parabolic_wellposedness_infinitewidth:CONTRACTIVITYuhat2}
    \begin{split}        &\fpartial{t}\Nbig{\uhatpptuReverted{\tau}{1}{t,\dummy}-\uhatpptuReverted{\tau}{2}{t,\dummy}}^2_{L_2(D)}\\       &\quad\,=2\big(\uhatpptuReverted{\tau}{1}{t,\dummy}-\uhatpptuReverted{\tau}{2}{t,\dummy},\fpartial{t}\big(\uhatpptuReverted{\tau}{1}{t,\dummy}-\uhatpptuReverted{\tau}{2}{t,\dummy}\big)\big)_{L_2(D)}\\
        &\quad\,=        2\big\langle\fpartial{t}\big(\uhatpptuReverted{\tau}{1}{t,\dummy}-\uhatpptuReverted{\tau}{2}{t,\dummy}\big),\uhatpptuReverted{\tau}{1}{t,\dummy}-\uhatpptuReverted{\tau}{2}{t,\dummy}\big\rangle_{H^{-1}(D),H_0^1(D)}\\
        &\quad\,= - 2\underbar{\CB}^*\big[\uhatpptuReverted{\tau}{1}{t,\dummy}-\uhatpptuReverted{\tau}{2}{t,\dummy},\uhatpptuReverted{\tau}{1}{t,\dummy}-\uhatpptuReverted{\tau}{2}{t,\dummy};t\big]\\
        &\quad\,\quad\,+ 2\big(\underbar{q}_u(t,\dummy,\upptu{\tau}{1}{T\!-\!t,\dummy})\big(\uhatpptuReverted{\tau}{1}{t,\dummy}-\uhatpptuReverted{\tau}{2}{t,\dummy}\big),\uhatpptuReverted{\tau}{1}{t,\dummy}-\uhatpptuReverted{\tau}{2}{t,\dummy}\big)_{L_2(D)}\\
        &\quad\,\quad\,+ 2\big(\upptu{\tau}{1}{T\!-\!t,\dummy}-\upptu{\tau}{2}{T\!-\!t,\dummy},\uhatpptuReverted{\tau}{1}{t,\dummy}-\uhatpptuReverted{\tau}{2}{t,\dummy}\big)_{L_2(D)}\\
        &\quad\,\quad\,+ 2\big(\big(\underbar{q}_u(t,\dummy,\upptu{\tau}{1}{T\!-\!t,\dummy})\!-\!\underbar{q}_u(t,\dummy,\upptu{\tau}{2}{T\!-\!t,\dummy})\big)\uhatpptuReverted{\tau}{2}{t,\dummy},\uhatpptuReverted{\tau}{1}{t,\dummy}\!-\!\uhatpptuReverted{\tau}{2}{t,\dummy}\big)_{L_2(D)},
    \end{split}
    \end{equation}
    where the individual steps hold as previously described.
    To estimate the right-hand side of \eqref{eq:proof:lem:parabolic_wellposedness_infinitewidth:CONTRACTIVITYuhat2} from above,
    we consider each of the four terms separately.
    For the first term, by using the definition of the bilinear form $\CB$ as well as that by Assumption~\ref{asm:PDE_L_parabolicity} the PDE operator is uniformly parabolic and that by Assumption~\ref{asm:PDE_coefficients} the coefficients are in $L_\infty$, we can estimate as in \eqref{eq:proof:lem:parabolic_wellposedness_infinitewidth:NORM2_2} that
    \begin{equation}        \label{eq:proof:lem:parabolic_wellposedness_infinitewidth:CONTRACTIVITYuhat3}
    \begin{split}
        &-\underbar{\CB}^*\big[\uhatpptuReverted{\tau}{1}{t,\dummy}-\uhatpptuReverted{\tau}{2}{t,\dummy},\uhatpptuReverted{\tau}{1}{t,\dummy}-\uhatpptuReverted{\tau}{2}{t,\dummy};t\big]\\
        &\qquad\,\leq -\frac{\nu}{2} \absbig{\uhatpptuReverted{\tau}{1}{t,\dummy}-\uhatpptuReverted{\tau}{2}{t,\dummy}}^2_{H^1(D)}\\
        &\qquad\,\quad\,+ \left(\frac{1}{2\nu} \sum_{i=1}^d\N{b^i}_{L_\infty(D_T)}+\N{c}_{L_\infty(D_T)}\right)\Nbig{\uhatpptuReverted{\tau}{1}{t,\dummy}-\uhatpptuReverted{\tau}{2}{t,\dummy}}^2_{L_2(D)}.
    \end{split}
    \end{equation}
    For the second term, using that by Assumption~\ref{asm:PDE_nonlinearity_q_ubdd}  $q_u$ is bounded,
    we can estimate directly
    \begin{equation}        \label{eq:proof:lem:parabolic_wellposedness_infinitewidth:CONTRACTIVITYuhat4}
    \begin{split}
        &\big(\underbar{q}_u(t,\dummy,\upptu{\tau}{1}{T-t,\dummy})\big(\uhatpptuReverted{\tau}{1}{t,\dummy}-\uhatpptuReverted{\tau}{2}{t,\dummy}\big),\uhatpptuReverted{\tau}{1}{t,\dummy}-\uhatpptuReverted{\tau}{2}{t,\dummy}\big)_{L_2(D)}\\
        &\qquad\, \leq c_q \Nbig{\uhatpptuReverted{\tau}{1}{t,\dummy}-\uhatpptuReverted{\tau}{2}{t,\dummy}}_{L_2(D)}^2.
    \end{split}
    \end{equation}
    For the third term, by Cauchy-Schwarz and Young’s inequality we upper bound
    \begin{equation}        \label{eq:proof:lem:parabolic_wellposedness_infinitewidth:CONTRACTIVITYuhat5}
    \begin{split}
        &\big(\upptu{\tau}{1}{T-t,\dummy}-\upptu{\tau}{2}{T-t,\dummy},\uhatpptuReverted{\tau}{1}{t,\dummy}-\uhatpptuReverted{\tau}{2}{t,\dummy}\big)_{L_2(D)} \\
        &\qquad\,\leq \N{\upptu{\tau}{1}{T-t,\dummy}-\upptu{\tau}{2}{T-t,\dummy}}_{L_2(D)}\Nbig{\uhatpptuReverted{\tau}{1}{t,\dummy}-\uhatpptuReverted{\tau}{2}{t,\dummy}}_{L_2(D)}\\
        &\qquad\,\leq \frac{1}{2}\left(\N{\upptu{\tau}{1}{T-t,\dummy}-\upptu{\tau}{2}{T-t,\dummy}}_{L_2(D)}^2 + \Nbig{\uhatpptuReverted{\tau}{1}{t,\dummy}-\uhatpptuReverted{\tau}{2}{t,\dummy}}_{L_2(D)}^2\right).
    \end{split}
    \end{equation}
    For the fourth and final term, we first note that by the mean-value theorem, for any $(t,x)\in D_T$ there exists a $\xi(t,x)$ such that
    \begin{equation}        \label{eq:proof:lem:parabolic_wellposedness_infinitewidth:CONTRACTIVITYuhat6aux}
    \begin{split}
        &\underbar{q}_u(t,x,\upptu{\tau}{1}{T\!-\!t,x})\!-\!\underbar{q}_u(t,x,\upptu{\tau}{2}{T\!-\!t,x})\\
        &\qquad\,= \underbar{q}_{uu}(t,x,\xi(t,x))\left(\upptu{\tau}{1}{T\!-\!t,x}\!-\!\upptu{\tau}{2}{T\!-\!t,x}\right).
    \end{split}
    \end{equation}
    Leveraging this while using that by Assumption~\ref{asm:PDE_nonlinearity_q_uubdd}  $q_{uu}$ is bounded,
    we can estimate by Cauchy-Schwarz and Young’s inequality
    \begin{equation}       \label{eq:proof:lem:parabolic_wellposedness_infinitewidth:CONTRACTIVITYuhat6}
    \begin{split}
        &\big(\big(\underbar{q}_u(t,\dummy,\upptu{\tau}{1}{T\!-\!t,\dummy})\!-\!\underbar{q}_u(t,\dummy,\upptu{\tau}{2}{T\!-\!t,\dummy})\big)\uhatpptuReverted{\tau}{2}{t,\dummy},\uhatpptuReverted{\tau}{1}{t,\dummy}\!-\!\uhatpptuReverted{\tau}{2}{t,\dummy}\big)_{L_2(D)} \\
        &\quad\,\leq \Nbig{\big(\underbar{q}_u(t,\dummy,\upptu{\tau}{1}{T\!-\!t,\dummy})\!-\!\underbar{q}_u(t,\dummy,\upptu{\tau}{2}{T\!-\!t,\dummy})\big)\uhatpptuReverted{\tau}{2}{t,\dummy}}_{L_2(D)}\\
        &\quad\,\qquad\cdot\Nbig{\uhatpptuReverted{\tau}{1}{t,\dummy}\!-\!\uhatpptuReverted{\tau}{2}{t,\dummy}}_{L_2(D)}\\
        &\quad\,\leq \Nbig{\underbar{q}_u(t,\dummy,\upptu{\tau}{1}{T\!-\!t,\dummy})\!-\!\underbar{q}_u(t,\dummy,\upptu{\tau}{2}{T\!-\!t,\dummy})}_{L_\infty(D)}\Nbig{\uhatpptuReverted{\tau}{2}{t,\dummy}}_{L_2(D)}\\
        &\quad\,\qquad\cdot\Nbig{\uhatpptuReverted{\tau}{1}{t,\dummy}\!-\!\uhatpptuReverted{\tau}{2}{t,\dummy}}_{L_2(D)}\\
        &\quad\,\leq c_q'\Nbig{\upptu{\tau}{1}{T\!-\!t,\dummy}\!-\!\upptu{\tau}{2}{T\!-\!t,\dummy}}_{L_\infty(D)}\Nbig{\uhatpptuReverted{\tau}{2}{t,\dummy}}_{L_2(D)}\Nbig{\uhatpptuReverted{\tau}{1}{t,\dummy}\!-\!\uhatpptuReverted{\tau}{2}{t,\dummy}}_{L_2(D)}\\
        &\quad\,\leq \frac{1}{2} \Big(\Nbig{\uhatpptuReverted{\tau}{2}{t,\dummy}}_{L_2(D)}^2\Nbig{\upptu{\tau}{1}{T\!-\!t,\dummy}\!-\!\upptu{\tau}{2}{T\!-\!t,\dummy}}_{L_\infty(D)}^2 \\
        &\quad\,\qquad\qquad+(c_q')^2\Nbig{\uhatpptuReverted{\tau}{1}{t,\dummy}\!-\!\uhatpptuReverted{\tau}{2}{t,\dummy}}_{L_2(D)}^2\Big).
    \end{split}
    \end{equation}
    Combining the bounds established in \eqref{eq:proof:lem:parabolic_wellposedness_infinitewidth:CONTRACTIVITYuhat3}--\eqref{eq:proof:lem:parabolic_wellposedness_infinitewidth:CONTRACTIVITYuhat6} and inserting them into \eqref{eq:proof:lem:parabolic_wellposedness_infinitewidth:CONTRACTIVITYuhat2},
    we arrive after reordering at
    \begin{equation}       \label{eq:proof:lem:parabolic_wellposedness_infinitewidth:CONTRACTIVITYuhat7}
    \begin{split}        &\fpartial{t}\Nbig{\uhatpptuReverted{\tau}{1}{t,\dummy}-\uhatpptuReverted{\tau}{2}{t,\dummy}}^2_{L_2(D)} + \frac{\nu}{2} \absbig{\uhatpptuReverted{\tau}{1}{t,\dummy}-\uhatpptuReverted{\tau}{2}{t,\dummy}}^2_{H^1(D)}\\
        &\qquad\,\leq C\Nbig{\uhatpptuReverted{\tau}{1}{t,\dummy}-\uhatpptuReverted{\tau}{2}{t,\dummy}}_{L_2(D)}^2 + \N{\upptu{\tau}{1}{T-t,\dummy}-\upptu{\tau}{2}{T-t,\dummy}}_{L_2(D)}^2 \\
        &\qquad\,\qquad\, + \Nbig{\uhatpptuReverted{\tau}{2}{t,\dummy}}_{L_2(D)}^2 \N{\upptu{\tau}{1}{T-t,\dummy}-\upptu{\tau}{2}{T-t,\dummy}}_{L_\infty(D)}^2\\
        &\qquad\,\leq C\Nbig{\uhatpptuReverted{\tau}{1}{t,\dummy}-\uhatpptuReverted{\tau}{2}{t,\dummy}}_{L_2(D)}^2 + \N{\upptu{\tau}{1}{T-t,\dummy}-\upptu{\tau}{2}{T-t,\dummy}}_{L_2(D)}^2 \\
        &\qquad\,\qquad\, + \Nbig{\uhatptu{\tau}{2}}_{L_\infty([0,T], L_2(D))}^2 \N{\upptu{\tau}{1}{T-t,\dummy}-\upptu{\tau}{2}{T-t,\dummy}}_{L_\infty(D)}^2
    \end{split}
    \end{equation}
    for a constant $C=C(\CL,q)$.
    Recalling that $\uhatpptu{\tau}{1}{0,\dummy}=\uhatpptu{\tau}{2}{0,\dummy}$, an application of Grönwall's inequality shows
    \begin{equation}        \label{eq:proof:lem:parabolic_wellposedness_infinitewidth:CONTRACTIVITY_uhattilde}
    \begin{split}
        &\Nbig{\uhatptu{\tau}{1}-\uhatptu{\tau}{2}}_{L_2([0,T],H^1(D))} + \Nbig{\uhatptu{\tau}{1}-\uhatptu{\tau}{2}}_{L_\infty([0,T],L_2(D))} \\
        &\qquad\,\leq C\left(\N{\uptu{\tau}{1}-\uptu{\tau}{2}}_{L_2(D_T)} + \Nbig{\uhatptu{\tau}{2}}_{L_\infty([0,T], L_2(D))}\N{\uptu{\tau}{1}-\uptu{\tau}{2}}_{L_\infty(D_T)}\right)
    \end{split}
    \end{equation}
    for some other, potentially larger, constant $C=C(T,\CL,q)$.
    Thus, in particular,
    \begin{equation}        \label{eq:proof:lem:parabolic_wellposedness_infinitewidth:CONTRACTIVITY_uhattilde_UNIFORM}
    \begin{split}
        &\Nbig{\uhatptu{}{1}\!-\!\uhatptu{}{2}}_{\CV_\CT}\\
        &\qquad\,\leq C\left(\sup_{\tau\in[0,\CT]}\!\N{\uptu{\tau}{1}\!-\!\uptu{\tau}{2}}_{L_2(D_T)} \!+\!\!\! \sup_{\tau\in[0,\CT]}\Nbig{\uhatptu{\tau}{2}}_{L_\infty([0,T], L_2(D))}\!\sup_{\tau\in[0,\CT]}\!\N{\uptu{\tau}{1}\!-\!\uptu{\tau}{2}}_{L_\infty(D_T)}\right).
    \end{split}
    \end{equation}

    \textit{Combination of the bounds.}
    Using \eqref{eq:proof:lem:parabolic_wellposedness_infinitewidth:CONTRACTIVITY_utilde_UNIFORM} and \eqref{eq:proof:lem:parabolic_wellposedness_infinitewidth:CONTRACTIVITY_uhattilde_UNIFORM} in the first inequality together with the first step of \eqref{eq:proof:lem:parabolic_wellposedness_infinitewidth:CONTRACTIVITYuinfty6} in the second inequality and \Cref{lem:parabolicTB} in the next-to-last step, we have
    \begin{allowdisplaybreaks}
    \label{eq:proof:lem:parabolic_wellposedness_infinitewidth:CONTRACTION}
    \begin{align}    
        &\N{\uptu{}{1}-\uptu{}{2}}_{\CV_{\widetilde{\CT}}} + \Nbig{\uhatptu{}{1}-\uhatptu{}{2}}_{\CV_{\widetilde{\CT}}}\notag\\
        &\qquad\,\leq C\sup_{\tau\in[0,\widetilde{\CT}]} \N{\widetilde{g}_\tau^1-\widetilde{g}_\tau^2}_{L_2(D_T)}\notag\\
        &\qquad\,\quad\,+ C\sup_{\tau\in[0,\widetilde{\CT}]}\Nbig{\uhatptu{\tau}{2}}_{L_\infty([0,T], L_2(D))} \sup_{\tau\in[0,\widetilde{\CT}]}\N{\uptu{\tau}{1}-\uptu{\tau}{2}}_{L_\infty(D_T)}\notag\\ 
        &\qquad\,\leq C\sup_{\tau\in[0,\widetilde{\CT}]} \N{\int_0^\tau \alpha_s T_{B_0}[\uhatpu{s}{1}-\uhatpu{s}{2}]\,ds}_{L_2(D_T)} \notag\\
        &\qquad\,\quad\,+ C\sup_{\tau\in[0,\widetilde{\CT}]}\Nbig{\uhatptu{\tau}{2}}_{L_\infty([0,T], L_2(D))} \sup_{\tau\in[0,\widetilde{\CT}]}\int_0^\tau \N{\uhatpu{s}{1}-\uhatpu{s}{2}}_{L_2(D_T)} ds \notag\\
        &\qquad\,\leq C \int_0^{\widetilde{\CT}} \N{\alpha_s T_{B_0}[\uhatpu{s}{1}-\uhatpu{s}{2}]}_{L_2(D_T)}ds \notag\\
        &\qquad\,\quad\,+  C\sup_{\tau\in[0,\widetilde{\CT}]}\Nbig{\uhatptu{\tau}{2}}_{L_\infty([0,T], L_2(D))} \int_0^{\widetilde{\CT}} \N{\uhatpu{s}{1}-\uhatpu{s}{2}}_{L_2(D_T)} ds \notag\\
        &\qquad\,\leq C_2 \left(1+\sup_{\tau\in[0,\widetilde{\CT}]}\Nbig{\uhatptu{\tau}{2}}_{L_\infty([0,T], L_2(D))}\right) \int_0^{\widetilde{\CT}} \N{\uhatpu{s}{1}-\uhatpu{s}{2}}_{L_2(D_T)} ds
    \end{align}
    \end{allowdisplaybreaks}
    for a constant $C_2=C_2(\alpha,\CL,q,C^B_2)$ (to be precise, $C_2=C\max\{\alpha_0C_2^B,1\}$).
    
    \textit{Step~2d: Existence locally in training time.}
    Let us choose
    \begin{equation}        \label{eq:proof:lem:parabolic_wellposedness_infinitewidth:M0CT0}
        M_0
        = 2C_1\!\left(\N{h}_{L_2(D_T)} \!+\! \N{f}_{L_2(D)} \!+\! 1\right)
        \quad\text{and}\quad
        \CT_0
        = \min\left\{\frac{1}{4C_1},\frac{1}{2C_2(1+M_0)}\right\}
    \end{equation}
    where the constants $C_1$ and $C_2$ are as given implicitly in \eqref{eq:proof:lem:parabolic_wellposedness_infinitewidth:SELFMAPPING} and \eqref{eq:proof:lem:parabolic_wellposedness_infinitewidth:CONTRACTION}, respectively.
    We show in what follows that there exists a unique solution $(\up{},\uhatp{})\in \CV_{\CT_0}(M_0)\times\CV_{\CT_0}(M_0)$.

    \textit{Step~2d(i): Self-mapping property of $F$.}
    Consider $(\upt{},\uhatpt{})$ together with its corresponding $(\up{},\uhatp{}) \in \CV_{\CT_0}(M_0)\times \CV_{\CT_0}(M_0)$.
    Using the definitions of $M_0$ and $\CT_0$ in \eqref{eq:proof:lem:parabolic_wellposedness_infinitewidth:M0CT0}, respectively, we can derive from \eqref{eq:proof:lem:parabolic_wellposedness_infinitewidth:SELFMAPPING} that
    \begin{equation}
    \begin{split}
        &\N{\upt{}}_{\CV_{\CT_0}} \!+\! \Nnormal{\uhatpt{}}_{\CV_{\CT_0}}
        \leq C_1 \int_0^{\CT_0} \N{\uhatp{s}}_{L_2(D_T)}ds \!+\! C_1\left(\N{h}_{L_2(D_T)} \!+\! \N{f}_{L_2(D)} \!+\! 1\right)\\
        &\qquad\,\leq C_1 \CT_0 \N{\uhatp{}}_{\CV_{\CT_0}} + \frac{M_0}{2}
        \leq C_1 \CT_0 M_0 + \frac{M_0}{2}
        \leq \frac{M_0}{4} + \frac{M_0}{2} \leq M_0.
    \end{split}
    \end{equation}
    Thus, $(\upt{},\uhatpt{}) \in \CV_{\CT_0}(M_0)\times \CV_{\CT_0}(M_0)$.

    \textit{Step~2d(ii): Contractivity of $F$.}
    Consider two pairs $(\uptu{}{1},\uhatptu{}{1})$, $(\uptu{}{2},\uhatptu{}{2})$ together with their corresponding $(\upu{}{1},\uhatpu{}{1}), (\upu{}{2},\uhatpu{}{2}) \in \CV_{\CT_0}(M_0)\times \CV_{\CT_0}(M_0)$.
    According to \textit{Step~2d(i)}, $(\uptu{}{1},\uhatptu{}{1}),(\uptu{}{2},\uhatptu{}{2}) \in \CV_{\CT_0}(M_0)\times \CV_{\CT_0}(M_0)$.
    Using the definitions of $M_0$ and $\CT_0$ in \eqref{eq:proof:lem:parabolic_wellposedness_infinitewidth:M0CT0}, we can derive from \eqref{eq:proof:lem:parabolic_wellposedness_infinitewidth:CONTRACTION} that
    \begin{equation}
    \begin{split}
        &\N{\uptu{}{1}-\uptu{}{2}}_{\CV_{\CT_0}} + \Nbig{\uhatptu{}{1}-\uhatptu{}{2}}_{\CV_{\CT_0}}\\
        &\quad\,\leq C_2 \left(1+\sup_{\tau\in[0,\CT_0]}\Nbig{\uhatptu{\tau}{2}}_{L_\infty([0,T], L_2(D))}\right) \int_0^{\CT_0} \N{\uhatpu{s}{1}-\uhatpu{s}{2}}_{L_2(D_T)} ds \\
        &\quad\,\leq C_2 \left(1+M_0\right) \int_0^{\CT_0} \N{\uhatpu{s}{1}-\uhatpu{s}{2}}_{L_2(D_T)} ds
        \leq C_2 \left(1+M_0\right)\CT_0 \Nbig{\uhatpu{}{1}-\uhatpu{}{2}}_{\CV_{\CT_0}} \\
        &\quad\,\leq \frac{1}{2} \Nbig{\uhatpu{}{1}-\uhatpu{}{2}}_{\CV_{\CT_0}}
        \leq \frac{1}{2} \left(\N{\upu{}{1}-\upu{}{2}}_{\CV_{\CT_0}} + \Nbig{\uhatpu{}{1}-\uhatpu{}{2}}_{\CV_{\CT_0}}\right),
    \end{split}
    \end{equation}
    showing that the map $F$ is a contraction.

    \textit{Step~2d(iii): Banach fixed point theorem.}
    Hence, the Banach fixed point theorem guarantees that there exists a unique solution $(\up{},\uhatp{})\in \CV_{\CT_0}(M_0)\times\CV_{\CT_0}(M_0)$,
    which satisfies $(\up{},\uhatp{}) = F(\up{},\uhatp{})$.
    We have thus established the existence of a unique local-in-training-time solution to the PDE system~\eqref{eq:parabolicPDE*}--\eqref{eq:parabolicadjoint*} coupled with the integro-differential equation~\eqref{eq:parabolicgtau} for $\widetilde{g}_\tau = g^*_\tau$ on the training time domain $[0,\CT_0]$.
    In particular, $(\up{\tau},\uhatp{\tau})\in \CS\times \CS$ for every $\tau\in[0,\CT_0]$.
    Reapplying the classical existence and regularity results from \textit{Steps 1a}, \textit{1b} and \textit{1c} further shows that for each $\tau\in[0,\CT]$ such solution satisfies $(\partial_t\uppt{\tau}{t,\dummy},\partial_t\uhatppt{\tau}{t,\dummy})\in L_2(D) \times L_2(D)$
    for a.e.\@ $t\in[0,T]$.
    
    \textit{Step~2e: Existence globally in training time.}
    Leveraging a bootstrapping argument, we now extend this argument to obtain a solution on a training time domain $[0,\CT]$ for an arbitrary $\CT<\infty$.
    To do so, we proceed inductively.
    Suppose we have a solution $(\up{},\uhatp{})\in \CV_{\CT_{k-1}}(M_{k-1})\times\CV_{\CT_{k-1}}(M_{k-1})$ which is such that for each $\tau\in[0,\CT_{k-1}]$ it satisfies $(\partial_t\upp{\tau}{t,\dummy},\partial_t\uhatpp{\tau}{t,\dummy})\in L_2(D) \times L_2(D)$
    for a.e.\@ $t\in[0,T]$. (We showed in \textit{Step~2d} before that this is the case for the induction start $k=1$.)

    On the training time interval $I = [0,\CT_{k-1}]$ we can now employ \Cref{lem:parabolictimeevolutionJt*} which ensures that $\fd{\tau}\J{\tau}\leq0$ for all $\tau\in I=[0,\CT_{k-1}]$.
    Thanks to this, \Cref{lem:parabolic_uhatL2} (applied in the setting $I=[0,\CT_{k-1}]$) provides a uniform (in the training time $\tau$ and on the training time interval $[0,\CT_{k-1}]$) bound $\sup_{\tau\in[0,\CT_{k-1}]}\N{\uhatp{\tau}}_{L_2(D_T)}\leq C^{\widehat{u}}$, where $C^{\widehat{u}}$ does not depend on $\CT_{k-1}$ but only on $\J{0}$.  
    Let us now choose 
    \begin{equation}
\label{eq:proof:lem:parabolic_wellposedness_infinitewidth:MkCTk}
        M_k
        = k C^{\widehat{u}} + 2C_1\!\left(\N{h}_{L_2(D_T)} \!+\! \N{f}_{L_2(D)} \!+\! 1\right)
        \quad\text{and}\quad
        \CT_k
        = \CT_{k-1} + \min\left\{\frac{1}{4C_1},\frac{1}{2C_2(1+M_k)}\right\}.
    \end{equation}
    We show in what follows that there exists a unique solution $(\up{},\uhatp{})\in \CV_{\CT_{k}}(M_{k})\times\CV_{\CT_{k}}(M_{k})$.

    \textit{Step~2e(i): Self-mapping property of $F$.}
    Consider $(\upt{},\uhatpt{})$ together with its corresponding $(\up{},\uhatp{}) \in \CV_{\CT_k}(M_k)\times \CV_{\CT_k}(M_k)$.
    Using the definitions of $M_k$ and $\CT_k$ in \eqref{eq:proof:lem:parabolic_wellposedness_infinitewidth:MkCTk}, we can derive from \eqref{eq:proof:lem:parabolic_wellposedness_infinitewidth:SELFMAPPING} that
    \begin{allowdisplaybreaks}
    \begin{align}
        &\N{\upt{}}_{\CV_{\CT_k}} + \Nnormal{\uhatpt{}}_{\CV_{\CT_k}} \leq C_1 \int_0^{\CT_k} \N{\uhatp{s}}_{L_2(D_T)}ds + C_1\left(\N{h}_{L_2(D_T)} + \N{f}_{L_2(D)} + 1\right)\nonumber\\
        &\qquad\,\leq C_1 \int_0^{\CT_{k-1}} \N{\uhatp{s}}_{L_2(D_T)}ds + C_1 \int_{\CT_{k-1}}^{\CT_k} \N{\uhatp{s}}_{L_2(D_T)}ds + C_1\left(\N{h}_{L_2(D_T)} + \N{f}_{L_2(D)} + 1\right)\nonumber\\
        &\qquad\,\leq C_1 \CT_{k-1} C^{\widehat{u}} + C_1 (\CT_k-\CT_{k-1}) \N{\uhatp{}}_{\CV_{\CT_k}} + \frac{M_k}{2}\nonumber\\
        &\qquad\,\leq C_1 \CT_{k-1} C^{\widehat{u}} + C_1 (\CT_k-\CT_{k-1}) M_k + \frac{M_k}{2}
        \leq \frac{M_k}{4} + \frac{M_k}{4} + \frac{M_k}{2} \leq M_k,\nonumber
    \end{align}
    \end{allowdisplaybreaks}%
    where we used for the first step in the last line that (with $\CT_{-1}:=0$) according to the definition of $\CT_k$ in \eqref{eq:proof:lem:parabolic_wellposedness_infinitewidth:MkCTk} it holds
    $C_1 \CT_{k-1} C^{\widehat{u}}= C_1 \sum_{\ell=0}^{k-1} (\CT_{\ell}-\CT_{\ell-1}) C^{\widehat{u}}\leq \frac{1}{4} k C^{\widehat{u}}\leq \frac{M_k}{4}.$
    Thus, $(\upt{},\uhatpt{}) \in \CV_{\CT_k}(M_k)\times \CV_{\CT_k}(M_k)$.

    \textit{Step~2e(ii): Contractivity of $F$.}
    Consider two pairs $(\uptu{}{1},\uhatptu{}{1})$, $(\uptu{}{2},\uhatptu{}{2})$ together with their corresponding $(\upu{}{1},\uhatpu{}{1}), (\upu{}{2},\uhatpu{}{2}) \in \CV_{\CT_k}(M_k)\times \CV_{\CT_k}(M_k)$.
    Using the definitions of $M_k$ and $\CT_k$ in \eqref{eq:proof:lem:parabolic_wellposedness_infinitewidth:MkCTk}, we can derive from \eqref{eq:proof:lem:parabolic_wellposedness_infinitewidth:CONTRACTION} that
    \begin{equation}
    \begin{split}
        &\N{\uptu{}{1}-\uptu{}{2}}_{\CV_{\CT_k}} + \Nbig{\uhatptu{}{1}-\uhatptu{}{2}}_{\CV_{\CT_k}}\\
        &\qquad\,\leq C_2 \left(1+\sup_{\tau\in[0,\CT_k]}\Nbig{\uhatptu{\tau}{2}}_{L_\infty([0,T],L_2(D))}\right) \int_0^{\CT_k} \N{\uhatpu{s}{1}-\uhatpu{s}{2}}_{L_2(D_T)} ds \\
        &\qquad\,\leq C_2 \left(1\!+\!M_k\right) \int_0^{\CT_k} \N{\uhatpu{s}{1}-\uhatpu{s}{2}}_{L_2(D_T)} ds \\
        &\qquad\,= C_2 \left(1\!+\!M_k\right) \int_{\CT_{k-1}}^{\CT_k} \N{\uhatpu{s}{1}\!-\!\uhatpu{s}{2}}_{L_2(D_T)} ds 
        \leq C_2 \left(1\!+\!M_k\right)(\CT_k\!-\!\CT_{k-1}) \Nbig{\uhatpu{}{1}\!-\!\uhatpu{}{2}}_{\CV_{\CT_k}} \\
        &\qquad\,\leq \frac{1}{2} \Nbig{\uhatpu{}{1}-\uhatpu{}{2}}_{\CV_{\CT_k}}
        \leq \frac{1}{2} \left(\N{\upu{}{1}-\upu{}{2}}_{\CV_{\CT_k}} + \Nbig{\uhatpu{}{1}-\uhatpu{}{2}}_{\CV_{\CT_k}}\right),
    \end{split}
    \end{equation}
    where the third step is due to the uniqueness of the solution on the training time interval $[0,\CT_{k-1}]$.
    Thus, the map $F$ is a contraction.

    \textit{Step~2e(iii): Banach fixed point theorem.}
    Hence, the Banach fixed point theorem guarantees that there exists a unique solution $(\up{},\uhatp{})\in \CV_{\CT_k}(M_k)\times\CV_{\CT_k}(M_k)$,
    which satisfies $(\up{},\uhatp{}) = F(\up{},\uhatp{})$.
    We have thus established the existence of a unique solution to the PDE system~\eqref{eq:parabolicPDE*}--\eqref{eq:parabolicadjoint*} coupled with the integro-differential equation~\eqref{eq:parabolicgtau} for $\widetilde{g}_\tau = g^*_\tau$ on the training time domain $[0,\CT_k]$.
    In particular, $(\up{\tau},\uhatp{\tau})\in \CS\times \CS$ for every $\tau\in[0,\CT_k]$.
    Reapplying the classical existence and regularity results from \textit{Steps 1a}, \textit{1b} and \textit{1c} further shows that for each $\tau\in[0,\CT_k]$ such solution satisfies $(\partial_t\uppt{\tau}{t,\dummy},\partial_t\uhatppt{\tau}{t,\dummy})\in L_2(D) \times L_2(D)$
    for a.e.\@ $t\in[0,T]$.

    \textit{Step~2e(iv): Globality of the construction in training time.}
    It remains to notice that, due to the definition of the times~$\CT_k$ in \eqref{eq:proof:lem:parabolic_wellposedness_infinitewidth:MkCTk}, the telescopic sum
    \begin{equation}
        \label{eq:proof:lem:parabolic_wellposedness_infinitewidth:harmonicseries}
    \begin{split}
        \sum_{k=1}^\infty (\CT_k-\CT_{k-1})
        &= \sum_{k=1}^\infty \min\left\{\frac{1}{4C_1},\frac{1}{2C_2(1+M_k)}\right\}
        \geq \sum_{k=K}^\infty \frac{1}{2C_2(1+M_k)}\\
        &= \sum_{k=K}^\infty \frac{1}{2C_2(1+k C^{\widehat{u}} + 2C_1(\N{h}_{L_2(D_T)} + \N{f}_{L_2(D)} + 1))}
    \end{split}
    \end{equation}
    diverges.
    To see this, simply note that the definition of the bounds~$M_k$ in \eqref{eq:proof:lem:parabolic_wellposedness_infinitewidth:MkCTk} grows linearly in $k$, thus the last term in \eqref{eq:proof:lem:parabolic_wellposedness_infinitewidth:harmonicseries} being a harmonic series for some sufficiently large integer $K$.
    This ensures that the above construction in \textit{Step~2d} and \textit{2e} gives a solution in $\CV_{\CT}(M)\times\CV_{\CT}(M)$ for any given $\CT<\infty$ and suitable associated $M>0$.

    \textbf{Uniqueness.}
    It remains to prove the uniqueness of a solution $(\up{\tau},\uhatp{\tau})$ to the PDE system~\eqref{eq:parabolicPDE*}--\eqref{eq:parabolicadjoint*} coupled with the integro-differential equation~\eqref{eq:parabolicgtau} for $g^*_\tau$.
    For this purpose, suppose that there are two weak solutions $(\upu{}{1},\uhatpu{}{1}),(\upu{}{2},\uhatpu{}{2})\in \CV_\CT(M)\times \CV_\CT(M)$.
    This means we have $(\upu{\tau}{1},\uhatpu{\tau}{1}),(\upu{\tau}{2},\uhatpu{\tau}{2})\in \CS\times \CS$ with satisfy $(\fpartial{t}\upu{\tau}{1},\fpartial{t}\uhatpu{\tau}{1}),(\fpartial{t}\upu{\tau}{2},\fpartial{t}\uhatpu{\tau}{2})\in L_2(D)\times L_2(D)$ for a.e.\@ $t\in[0,T]$ and where $g^{*,1}_\tau,g^{*,2}_\tau\in L_2(D_T)$ denote the corresponding integral terms \eqref{eq:parabolicgtau} for each $\tau\in[0,\CT]$.
    By repeating the computations of the existence proof in \textit{Step 2c(ii)} we obtain analogously to \eqref{eq:proof:lem:parabolic_wellposedness_infinitewidth:CONTRACTIVITY_utilde} and \eqref{eq:proof:lem:parabolic_wellposedness_infinitewidth:CONTRACTIVITYuinfty5} that
    \begin{equation}       \label{eq:proof:lem:parabolic_wellposedness_infinitewidth:UNIQUENESS1}
        \N{\upu{\tau}{1}-\upu{\tau}{2}}_{L_\infty([0,T],H^1(D))} + \N{\upu{\tau}{1}-\upu{\tau}{2}}_{L_2([0,T],L_2(D))}
        \leq C\N{g^{*,1}_\tau-g^{*,2}_\tau}_{L_2(D_T)}
    \end{equation}
    and for some $p>d+1$ that
    \begin{equation}        \label{eq:proof:lem:parabolic_wellposedness_infinitewidth:UNIQUENESS2}
    \begin{split}
        \N{\upu{\tau}{1}-\upu{\tau}{2}}_{L_\infty(D_T)}
        &\leq C\N{g^{*,1}_\tau-g^{*,2}_\tau}_{L_p(D_T)},
    \end{split}
    \end{equation}
    as well as analogously to \eqref{eq:proof:lem:parabolic_wellposedness_infinitewidth:CONTRACTIVITY_uhattilde} that 
    \begin{equation}        \label{eq:proof:lem:parabolic_wellposedness_infinitewidth:UNIQUENESS3}
    \begin{split}
        &\Nbig{\uhatpu{\tau}{1}-\uhatpu{\tau}{2}}_{L_2([0,T],H^1(D))} + \Nbig{\uhatpu{\tau}{1}-\uhatpu{\tau}{2}}_{L_\infty([0,T],L_2(D))}\\
        &\qquad\,\leq C\left(\N{\upu{\tau}{1}-\upu{\tau}{2}}_{L_2(D_T)} + M\N{\upu{\tau}{1}-\upu{\tau}{2}}_{L_\infty(D_T)}\right),
    \end{split}
    \end{equation}
    where we used directly that $\uhatpu{\tau}{2}\in\CV_\CT(M)$.
    Since it holds
    \begin{equation}        \label{eq:proof:lem:parabolic_wellposedness_infinitewidth:UNIQUENESS4}
    \begin{split}
        \N{g^{*,1}_\tau-g^{*,2}_\tau}_{L_p(D_T)}
        &\leq C \int_0^\tau \N{\uhatpu{s}{1}-\uhatpu{s}{2}}_{L_2(D_T)} ds,
    \end{split}
    \end{equation}
    according to the next-to-last step in \eqref{eq:proof:lem:parabolic_wellposedness_infinitewidth:CONTRACTIVITYuinfty3},
    we get
    \begin{equation}        \label{eq:proof:lem:parabolic_wellposedness_infinitewidth:UNIQUENESS5}
    \begin{split}
        &\N{g^{*,1}_\tau-g^{*,2}_\tau}_{L_2(D_T)} + \N{g^{*,1}_s-g^{*,2}_s}_{L_p(D_T)}
        \leq C \int_0^\tau \N{\uhatpu{s}{1}-\uhatpu{s}{2}}_{L_2(D_T)} ds \\
        &\qquad\,\leq C \int_0^\tau \N{\upu{s}{1}-\upu{s}{2}}_{L_2(D_T)} + \N{\upu{s}{1}-\upu{s}{2}}_{L_\infty(D_T)} ds \\
        &\qquad\,\leq C \int_0^\tau \N{g^{*,1}_s-g^{*,2}_s}_{L_2(D_T)} + \N{g^{*,1}_s-g^{*,2}_s}_{L_p(D_T)} ds,
    \end{split}
    \end{equation}
    where $C$ may depend in particular on $M$.
    
    Recalling that $g_0^1=g_0^2=0$, we can now employ Grönwall's inequality in its integral form to obtain
    \begin{equation}       \label{eq:proof:lem:parabolic_wellposedness_infinitewidth:UNIQUENESS6}
        \N{g^{*,1}_\tau-g^{*,2}_\tau}_{L_2(D_T)} + \N{g^{*,1}_\tau-g^{*,2}_\tau}_{L_p(D_T)}
        = 0
    \end{equation}
    for every $\tau\in[0,\CT]$.
    With \eqref{eq:proof:lem:parabolic_wellposedness_infinitewidth:UNIQUENESS1} and \eqref{eq:proof:lem:parabolic_wellposedness_infinitewidth:UNIQUENESS3}
    we hence conclude that for every $\tau\in[0,\CT]$ it hold
    \begin{equation}
        \N{\upu{\tau}{1}-\upu{\tau}{2}}_{L_2([0,T],H^1(D))} + \N{\upu{\tau}{1}-\upu{\tau}{2}}_{L_\infty([0,T],L_2(D))}
        =0
    \end{equation}
    and
    \begin{equation}
        \N{\uhatpu{\tau}{1}-\uhatpu{\tau}{2}}_{L_2([0,T],H^1(D))} + \N{\uhatpu{\tau}{1}-\uhatpu{\tau}{2}}_{L_\infty([0,T],L_2(D))}
        =0.
    \end{equation}
    Thus $\Nbig{\upu{}{1}-\upu{}{2}}_{\CV_\CT}=0$ and $\Nbig{\uhatpu{}{1}-\uhatpu{}{2}}_{\CV_\CT}=0$, proving uniqueness in $\CV_\CT(M)$.
\end{proof}

\subsection{Well-Posedness Proof of the NN-PDE Training Dynamics in the Finite-Width Hidden Layer Regime}
\label{sec:WellPosedness_Proof_N}

\begin{proof}[Proof of \Cref{lem:parabolic_wellposedness_N}]
    \textbf{Existence.}
    As in the proof of \Cref{lem:parabolic_wellposedness_infinitewidth},
    existence is shown using a fixed point argument.
    We denote for a given training time horizon $\CT>0$ by $\Theta_\CT=\CC\left([0,\CT],\Theta:=(\bbR^N\times (L_2(D_T))^N\times (L_2(D_T))^N\right)$ the Banach space consisting of elements with finite norm
    \begin{equation}
        \label{eq:Theta_CT_norm}
        \N{(c,n,m)}_{\Theta_\CT}
        = \sup_{\tau\in[0,\CT]} \sum_{i=1}^N \left(\abs{c^i_\tau} + \N{n_\tau^{i}}_{L_2(D_T)} + \N{m_\tau^{i}}_{L_2(D_T)}\right),
    \end{equation}
    where $n_\tau^{i}(t,x)=\sigma\big(w^{t,i}_\tau t + (w^i_\tau)^T x + \eta_\tau^i\big)$ and $m_\tau^{i}(t,x)=\sigma'\big(w^{t,i}_\tau t + (w^i_\tau)^T x + \eta_\tau^i\big)$,
    and by $\CV_\CT=\CC\left([0,\CT],\CS\right)$ again the Banach space consisting of elements with finite norm $\N{\dummy}_{\CV_\CT}$ as defined in \eqref{eq:CV_CT_norm}.
    
    A solution $(((c_\tau,n_\tau,m_\tau),\uN{\theta_\tau},\uNhat{\theta_\tau}))_{\tau\in[0,\CT]}$ to the PDE system \eqref{eq:parabolicPDEN_plain}\,\&\,\eqref{eq:parabolicadjointN_plain} coupled with the gradient descent update~\eqref{eq:GD} within the above definitions is shown in what follows to be an element of the space $\CC\left([0,\CT],\Theta\times\CS\times\CS\right)$ (which we identify with the space $\Theta_\CT\times\CV_\CT\times\CV_\CT$) with additional regularity.

    \textit{Step~1: Existence and regularity for given NN parameter updates $\widetilde{c}_\tau^i= c_0^i - \int_0^\tau \alpha_s b^{c^i}_{\theta_s}\,ds$ etc.}
    For given $\CT>0$, let $b^{c^i}_{\theta},b^{w^{t,i}}_{\theta},b^{\eta^i}_{\theta}:[0,\CT] \rightarrow \bbR$ and $b^{w^i}_{\theta}:[0,\CT] \rightarrow \bbR^d$ be given functions with $b^{c^i}_{\theta},b^{w^{t,i}}_{\theta},b^{w^i}_{\theta},b^{\eta^i}_{\theta}$ being such that $\sup_{\tau\in[0,\CT]} \sum_{i=1}^N\big(\absnormal{b^{c^i}_{\theta_\tau}}+\absnormal{b^{w^{t,i}}_{\theta_\tau}}+\Nnormal{b^{w^i}_{\theta_\tau}}+\absnormal{b^{\eta^i}_{\theta_\tau}}\big)\leq C_{b}$, where $C_{b}$ may depend in particular on $\CT$.
    Consider the auxiliary NN parameter update
    \begin{subequations}
        \label{eq:proof:lem:parabolic_wellposedness_N:setup:1}
    \begin{align}
        \widetilde{c}_\tau^i
        &= c_0^i - \int_0^\tau \alpha_s b^{c^i}_{\theta_s}\,ds, \label{eq:proof:lem:parabolic_wellposedness_N:setup:1a} \\
        \widetilde{w}_\tau^{t,i}
        &= w^{t,i}_0 - \int_0^\tau \alpha_s b^{w^{t,i}}_{\theta_s}\, ds, \label{eq:proof:lem:parabolic_wellposedness_N:setup:1b}\\
        \widetilde{w}_\tau^i
        &= w_0^i - \int_0^\tau \alpha_s b^{w^i}_{\theta_s}\, ds, \label{eq:proof:lem:parabolic_wellposedness_N:setup:1c}\\
        \widetilde{\eta}_\tau^i
        &= \eta_0^i - \int_0^\tau \alpha_s b^{\eta^i}_{\theta_s}\, ds, \label{eq:proof:lem:parabolic_wellposedness_N:setup:1d}
    \end{align}
    \end{subequations}
    as well as $\widetilde{n}_\tau^i(t,x) = \sigma\big(\widetilde{w}^{t,i}_\tau t + (\widetilde{w}^i_\tau)^T x + \widetilde{\eta}_\tau^i\big)$ and $\widetilde{m}_\tau^i(t,x) = \sigma'\big(\widetilde{w}^{t,i}_\tau t + (\widetilde{w}^i_\tau)^T x + \widetilde{\eta}_\tau^i\big)$
    together with the auxiliary PDE system
    \begin{alignat}{2}
        \label{eq:proof:lem:parabolic_wellposedness_N:setup:2}
    \begin{aligned}
        \fpartial{t}\uNt{\widetilde{\theta}_\tau} + \CL\uNt{\widetilde{\theta}_\tau} - q(\uNt{\widetilde{\theta}_\tau})
        &= \widetilde{g}_{\widetilde{\theta}_\tau}^N
        \qquad&&\text{in }
        D_T, \\
        \uNt{\widetilde{\theta}_\tau}
        &= 0
        \qquad&&\text{on }
        [0,T]\times\partial D, \\
        \uNt{\widetilde{\theta}_\tau}
        &= f
        \qquad&&\text{on }
        \{0\}\times D,
    \end{aligned}
    \end{alignat}
    where $\widetilde{g}_{\widetilde{\theta}_\tau}^N(t,x) = \frac{1}{N^\beta}\sum_{i=1}^N \widetilde{c}^i_\tau\widetilde{n}_\tau^i(t,x)= \frac{1}{N^\beta}\sum_{i=1}^N \widetilde{c}^i_\tau\sigma\big(\widetilde{w}^{t,i}_\tau t + (\widetilde{w}_\tau^i)^Tx + \widetilde{\eta}_\tau^i\big)$,
    and
    \begin{alignat}{2}
        \label{eq:proof:lem:parabolic_wellposedness_N:setup:3}
    \begin{aligned}
        -\fpartial{t}\uNhatt{\widetilde{\theta}_\tau} + \CL^\dagger\uNhatt{\widetilde{\theta}_\tau} - q_u(\upt{\widetilde{\theta}_\tau})\uNhatt{\widetilde{\theta}_\tau}
        &= (\uNt{\widetilde{\theta}_\tau}-h)
        \qquad&&\text{in }
        D_T, \\
        \uNhatt{\widetilde{\theta}_\tau}
        &= 0
        \qquad&&\text{on }
        [0,T]\times\partial D, \\
        \uNhatt{\widetilde{\theta}_\tau}
        &= 0
        \qquad&&\text{on }
        \{T\}\times D.
    \end{aligned}
    \end{alignat}
    We first prove that there exists a solution~$(\uNt{\widetilde{\theta}_\tau},\uNhatt{\widetilde{\theta}_\tau})\in \CS\times\CS$ to the system~\eqref{eq:proof:lem:parabolic_wellposedness_N:setup:2}--\eqref{eq:proof:lem:parabolic_wellposedness_N:setup:3} for all $\tau\in[0,\CT]$ using classical existence results from \cite{ladyzhenskaia1968linear}.
    Such solution, as we show, enjoys the property that for all $\tau\in[0,\CT]$ it holds $(\partial_t\uNpt{\widetilde{\theta}_\tau}{t,\dummy},\partial_t\uNhatpt{\widetilde{\theta}_\tau}{t,\dummy})\in L_2(D)\times L_2(D)$ for a.e.\@ $t\in[0,T]$.

    Computing $\fd{\tau}\widetilde{g}^N_{\widetilde{\theta}_\tau}$ by taking the training time derivative in $\widetilde{g}^N_{\widetilde{\theta}_\tau}$ as defined above and combining it with the expressions \eqref{eq:proof:lem:parabolic_wellposedness_N:setup:1},
    we obtain by the fundamental theorem of calculus that
    \begin{equation}
        \label{eq:proof:lem:parabolic_wellposedness_N:setup:4}
        \widetilde{g}^N_{\widetilde{\theta}_\tau}
        = g_{\theta_0}^N - \int_0^\tau \alpha_s b^N_{s} \,ds
        = g_{\theta_0}^N - \int_0^\tau\frac{\alpha_s}{N^\beta} \sum_{i=1}^N  b^{c^i}_{\theta_s}\sigma(\widetilde{\star}) + \widetilde{c}_s^i\sigma'(\widetilde{\star})\left(b^{w^{t,i}}_{\theta_s} t + \big(b^{w^{i}}_{\theta_s}\big)^T x + b^{\eta^{i}}_{\theta_s}\right)ds,
    \end{equation}
    where we abbreviated $\widetilde{\star}=\widetilde{w}^{t,i}_s t + (\widetilde{w}^i_s)^T x + \widetilde{\eta}^i_s$ and defined $b^N$ implicitly.
    We now notice that $b^N:[0,\CT] \rightarrow L_2(D_T)$
    is such that $b^N_{\tau}$ is Lipschitz continuous on $\overbar{D_T}$ for each $\tau\in[0,\CT]$ and such that $\sup_{\tau\in[0,\CT]} \N{b^N_{\tau}}_{L_\infty(D_T)}\leq C_{b^N}$, where $C_{b^N}$ may depend in particular on $\CT$.
    For the latter, uniform boundedness, we note that $b^{c^i}_{\theta},b^{w^{t,i}}_{\theta},b^{w^i}_{\theta},b^{\eta^i}_{\theta}$ are bounded by $C_b$, that $\sigma$ and $\sigma'$ are bounded as of Assumptions~\ref{asm:NN_sigma} and \ref{asm:NN_sigma'}, that the domain $D$ is bounded as of Assumption~\ref{asm:Dbdd}, and that by Cauchy-Schwarz inequality it holds with \eqref{eq:proof:lem:parabolic_wellposedness_N:setup:1a} that $\absnormal{\widetilde{c}_\tau^i}\leq \absnormal{c_0^i} + \absbig{\int_0^\tau \alpha_s b^{c^i}_{\theta_s}\,ds} \leq \absnormal{c_0^i} + \int_0^\tau \alpha_s^2\,ds\, \int_0^\tau\absnormal{b^{c^i}_{\theta_s}}^2\,ds$ is bounded as of Assumption~\ref{asm:NN_mu0}\ref{asm:NN_mu0ii} and due to condition~\eqref{eq:learning_rate} on the learning rate $\alpha_\tau$.
    Since $\sigma$ and $\sigma'$ are further Lipschitz continuous as of Assumptions~\ref{asm:NN_sigma} and \ref{asm:NN_sigma'}, and since also $\widetilde{w}_\tau^{t,i}$, $\widetilde{w}_\tau^i$ and $\widetilde{\eta}_\tau^i$ are bounded with the above argument and for given initial conditions $w_0^{t,i}$, $w_0^i$ and $\eta_0^i$, it is straightforward to check that $b^N_{\tau}$ is Lipschitz continuous on $\overbar{D_T}$ for each $\tau\in[0,\CT]$.
    We can thus follow \textit{Steps~1a} and \textit{b} in the proof of \Cref{lem:parabolic_wellposedness_infinitewidth}
    to show that there exists a solution~$(\uNt{\theta_\tau},\uNhatt{\theta_\tau})\in \CS\times\CS$ for all $\tau\in[0,\CT]$, which enjoys the property that for all $\tau\in[0,\CT]$ it holds $(\partial_t\uNpt{\theta_\tau}{t,\dummy},\partial_t\uNhatpt{\theta_\tau}{t,\dummy})\in L_2(D)\times L_2(D)$ for a.e.\@ $t\in[0,T]$.
    For this, one only needs to additionally notice that, by definition~\eqref{eq:gN},
    $g_{\theta_0}^N$ is continuous w.r.t.\@ $t,x$,
    uniformly bounded due to Assumptions~\ref{asm:NN_sigma} and \ref{asm:NN_mu0}\ref{asm:NN_mu0ii},
    i.e., $\N{g_{\theta_0}^N}_{L_\infty(D_T)}\leq C$ for a constant $C=C(N,\sigma,\mu_0)$, which may depend on $N$ (since it is fixed throughout the proof),
    and $(1,1)$-Hölder continuous due to Assumptions~\ref{asm:NN_sigma}.

    Repeating further the energy estimates of \textit{Step~1c} in the proof of \Cref{lem:parabolic_wellposedness_infinitewidth} in the above setting,
    we can derive the estimates
    \begin{equation}        \label{eq:proof:lem:parabolic_wellposedness_N:NORM_utilde}
        \N{\uNt{\theta_\tau}}_{L_2([0,T],H^1(D))} + \N{\uNt{\theta_\tau}}_{L_\infty([0,T],L_2(D))}
        \leq C\left(\N{f}_{L_2(D)} + \N{\widetilde{g}_{\theta_\tau}^N}_{L_2(D_T)} + 1\right)
    \end{equation}
    and, in particular,
    \begin{equation}        \label{eq:proof:lem:parabolic_wellposedness_N:NORM_utilde_UNIFORM}
        \N{\uNt{\theta}}_{\CV_\CT}
        \leq C\left(\N{f}_{L_2(D)} + \sup_{\tau\in[0,\CT]}\N{\widetilde{g}_{\theta_\tau}^N}_{L_2(D_T)} + 1\right)
    \end{equation}
    for some constant $C=C(T,\CL,q)$,
    as well as 
    \begin{equation}        \label{eq:proof:lem:parabolic_wellposedness_N:NORM_uhattilde}
    \begin{split}
        \Nbig{\uNhatt{\tau}}_{L_2([0,T],H^1(D))} + \Nbig{\uNhatt{\tau}}_{L_\infty([0,T],L_2(D))}
        &\leq C\N{\uNt{\theta_\tau}-h}_{L_2(D_T)} \\
        &\leq C\left(\N{\uNt{\theta_\tau}}_{L_2(D_T)} + \N{h}_{L_2(D_T)}\right)
    \end{split}
    \end{equation}
    and, in particular,
    \begin{equation}       \label{eq:proof:lem:parabolic_wellposedness_N:NORM_uhattilde_UNIFORM}
        \Nbig{\uNhatt{\theta}}_{\CV_\CT}
        \leq C\left(\N{\uNt{\theta}}_{\CV_\CT} + \N{h}_{L_2(D_T)}\right)
    \end{equation}
    for some other, potentially larger, constant $C=C(T,\CL,q)$.

    Since $\widetilde{g}_{\widetilde{\theta}_\tau}^N$ is Lipschitz continuous in $\tau$ with the argument from \textit{Step~1d} in the proof of \Cref{lem:parabolic_wellposedness_infinitewidth} together with the formerly established uniform boundedness of $b^N$ in \eqref{eq:proof:lem:parabolic_wellposedness_N:setup:4},
    we proved that there exists a continuous in the training time $\tau$ solution $(\upt{},\uhatpt{})\in \CV_\CT\times\CV_\CT$ to \eqref{eq:proof:lem:parabolic_wellposedness_N:setup:2}--\eqref{eq:proof:lem:parabolic_wellposedness_N:setup:3}.
    As we further ensured, for each $\tau\in[0,\CT]$ such solution satisfies $(\partial_t\uppt{\tau}{t,\dummy},\partial_t\uhatppt{\tau}{t,\dummy})\in L_2(D) \times L_2(D)$
    for a.e.\@ $t\in[0,T]$.

    \textit{Step~2: Existence for specific NN parameter updates $\widetilde{c}_\tau^i= c_0^i - \frac{1}{N^{1-\beta}} \int_0^\tau \alpha_s \int_0^T\!\!\!\int_D n_s^i(t,x)$ $\uNhatv{\theta_s}{t,x}\,dxdtds$ etc.}
    We now make specific choices for the functions $b^{c^i}_{\theta_\tau},b^{w^{t,i}}_{\theta_\tau},b^{w^i}_{\theta_\tau},b^{\eta^i}_{\theta_\tau}$.
    
    \textit{Step~2a: Choice of NN parameter update functions $b^{c^i}_{\theta_\tau}=\frac{1}{N^{1-\beta}} \int_0^T\!\!\!\int_D n_{\tau}^i(t,x)\uNhatv{\theta_{\tau}}{t,x}\,dxdt$ etc.}
    For arbitrarily given $(c_\tau,n_\tau,m_\tau)\in \Theta$ and $\uNhat{\theta_\tau}\in \CS$, $\tau\in[0,\CT]$, with $\sup_{\tau\in[0,\CT]} \sum_{i=1}^N (\absnormal{c^i_\tau} + \Nnormal{n_\tau^{i}}_{L_2(D_T)} + \Nnormal{m_\tau^{i}}_{L_2(D_T)})\leq M$ and $\sup_{\tau\in[0,\CT]} \N{\uNhat{\theta_\tau}}_{L_2(D_T)}\leq M$ ($M$ may depend on $\CT$), 
    we set    \begin{subequations}
    \begin{align}
        b^{c^i}_{\theta_\tau}
        &=\frac{1}{N^{1-\beta}} \int_0^T\!\!\!\int_D n_{\tau}^i(t,x)\uNhatv{\theta_{\tau}}{t,x}\,dxdt,  \\
        b^{w^{t,i}}_{\theta_\tau}
        &=\frac{1}{N^{1-\beta}} \int_0^T\!\!\!\int_D c^i_{\tau}m_{\tau}^i(t,x)t\uNhatv{\theta_{\tau}}{t,x}\,dxdt \\
        b^{w^i}_{\theta_\tau}
        &=\frac{1}{N^{1-\beta}} \int_0^T\!\!\!\int_D c^i_{\tau}m_{\tau}^i(t,x)x\uNhatv{\theta_{\tau}}{t,x}\,dxdt, \\
        b^{\eta^i}_{\theta_\tau}
        &=\frac{1}{N^{1-\beta}} \int_0^T\!\!\!\int_D c^i_{\tau}m_{\tau}^i(t,x)\uNhatv{\theta_{\tau}}{t,x}\,dxdt, 
    \end{align}
    \end{subequations}
    for all $\tau\in[0,\CT]$.
    By Cauchy-Schwarz inequality
    it holds   \begin{equation}
    \begin{split}
        &\sum_{i=1}^N \left(\absbig{b^{c^i}_{\theta_\tau}} + \absbig{b^{w^{t,i}}_{\theta_\tau}} + \Nbig{b^{w^i}_{\theta_\tau}} + \absbig{b^{\eta^i}_{\theta_\tau}}\right) \\
        &\qquad\,\leq C\sum_{i=1}^N\left(\N{n_{\tau}^i}_{L_2(D_T)}\N{\uNhat{\theta_\tau}}_{L_2(D_T)}+\absnormal{c^i_{\tau}}\N{m_{\tau}^i}_{L_2(D_T)}\N{\uNhat{\theta_\tau}}_{L_2(D_T)}\right) \\
        &\qquad\,\leq C \left(M^2+M^3\right)
    \end{split}
    \end{equation}
    for a constant $C=C(N,T,D)$, which is a uniform bound in $\tau$.
    Since the right-hand side is uniform in $\tau$, $\sup_{\tau\in[0,\CT]} \sum_{i=1}^N\big(\absnormal{b^{c^i}_{\theta_\tau}}+\absnormal{b^{w^{t,i}}_{\theta_\tau}}+\Nnormal{b^{w^i}_{\theta_\tau}}+\absnormal{b^{\eta^i}_{\theta_\tau}}\big)\leq C_{b^N}$, where $C_{b^N}$ may depend on $\CT$.

    \textit{Step~2b: Definition of fixed point mapping.}
    Let us consider the fixed point map
\begin{equation}        \label{eq:proof:lem:parabolic_wellposedness_N:F}
        F: \Theta_\CT\times\CV_\CT\times\CV_\CT \rightarrow \Theta_\CT\times\CV_\CT\times\CV_\CT,\quad
        ((c,n,m),\uN{},\uNhat{})\mapsto((\widetilde{c},\widetilde{n},\widetilde{m}),\uNt{},\uNhatt{})
    \end{equation}
    and define for given $M<\infty$ and $\CT<\infty$ the function spaces $\Theta_\CT(M)=\{(c,n,m)\in\Theta_\CT\!:\!\N{(c,n,m)}_{\Theta_\CT}\!\leq\! M\}$ and $\CV_\CT(M)=\{u\in\CV_\CT\!:\!\N{u}_{\CV_\CT}\!\leq\! M\}$.

    We will first show in \textit{Step~2d} existence locally in the training time by proving that there exist $M_0>0$ and $\CT_0>0$ such that $F$ is a fixed point mapping on $\Theta_{\CT_0}(M_0)\times\CV_{\CT_0}(M_0)\times\CV_{\CT_0}(M_0)$, which allows to apply the Banach fixed point theorem.
    In \textit{Step~2e} we will then extend the proof by a bootstrapping argument to any given (arbitrarily large) time horizon~$\CT$.
    
    \textit{Step~2c: Preliminary computations.}
    Let us start by conducting some preliminary computations on a generic space $\Theta_{\widetilde{\CT}}(\widetilde{M})\times\CV_{\widetilde{\CT}}(\widetilde{M})\times \CV_{\widetilde{\CT}}(\widetilde{M})$.
    
    \textit{Step~2c(i): Preliminary computations for self-mapping property of $F$.}
    Consider the triple $((\widetilde{c},\widetilde{n},\widetilde{m}),\uNt{},\uNhatt{})$ together with its corresponding $((c,n,m),\uN{},\uNhat{}) \in \Theta_{\widetilde{\CT}}(\widetilde{M})\times \CV_{\widetilde{\CT}}(\widetilde{M})\times \CV_{\widetilde{\CT}}(\widetilde{M})$.

    \textit{A bound for $\absnormal{\widetilde{c}_\tau^i}$.}
    Recalling \eqref{eq:proof:lem:parabolic_wellposedness_N:setup:1a},
    we can estimate with triangle inequality and two applications of Cauchy-Schwarz inequality, while using the boundedness of $\sigma$ as of Assumptions \ref{asm:NN_sigma} and that the domain $D$ has bounded volume as of Assumption~\ref{asm:Dbdd} together with
    the additional (and with \eqref{eq:learning_rate} compatible) assumption $\int_{0}^\infty\alpha_\tau^{4/3}\,d\tau<\infty$ on the learning rate,
    that
\begin{equation}
        \label{eq:proof:lem:parabolic_wellposedness_N:SELFMAPPING_prelim10_1}
    \begin{split}
        \abs{\widetilde{c}_\tau^i}
        &\leq \abs{c_0^i} + C \abs{\int_0^\tau \alpha_s \int_0^T\!\!\!\int_D   \sigma\big(w^{t,i}_s t + (w^i_s)^T x + \eta_s^i\big)\uNhatv{\theta_s}{t,x}\, dxdtds}\\
        &\leq \abs{c_0^i} + C \abs{\int_0^\tau \alpha_s \sqrt{\int_0^T\!\!\!\int_D  \big(\sigma\big(w^{t,i}_s t + (w^i_s)^T x + \eta_s^i\big)\big)^2\, dxdt} \N{\uNhat{\theta_s}}_{L_2(D_T)} ds}\\
        &\leq \abs{c_0^i} + C \abs{\int_0^\tau \alpha_s \N{\uNhat{\theta_s}}_{L_2(D_T)} ds}\\
        &\leq \abs{c_0^i} + C \left(\int^{\tau}_{0}\alpha_s^{4/3} \,ds\right)^{3/4} \left(\int_0^\tau \N{\uNhat{\theta_s}}_{L_2(D_T)}^4 ds\right)^{1/4}\\
        &\leq \abs{c_0^i} + C \left(\int_0^\tau \N{\uNhat{\theta_s}}_{L_2(D_T)}^4 ds\right)^{1/4}
    \end{split}
    \end{equation}
    for a constant $C=C(\alpha,N,T,D,\CL,q,\sigma)$.

    \textit{A bound for $\Nnormal{\widetilde{n}_\tau^i}_{L_2(D_T)}$.}
    Using the boundedness of $\sigma$ as of Assumptions \ref{asm:NN_sigma} and that the domain $D$ has bounded volume as of Assumption~\ref{asm:Dbdd}, clearly,
    \begin{equation}
        \label{eq:proof:lem:parabolic_wellposedness_N:SELFMAPPING_prelim10_4}
        \N{\widetilde{n}_\tau^i}_{L_2(D_T)} \leq C
    \end{equation}
    for a constant $C=C(T,D,\sigma)$.
    
    \textit{A bound for $\Nnormal{\widetilde{m}_\tau^i}_{L_2(D_T)}$.}
    Since also $\sigma'$ is bounded as of Assumptions \ref{asm:NN_sigma'},
    \begin{equation}
        \label{eq:proof:lem:parabolic_wellposedness_N:SELFMAPPING_prelim10_5}
        \N{\widetilde{m}_\tau^i}_{L_2(D_T)} \leq C.
    \end{equation}

    \textit{A bound for $\Nnormal{\widetilde{g}_{\widetilde{\theta}_\tau}^N}_{L_2(D_T)}$.}
    Using that $\sigma$ is bounded as of Assumption \ref{asm:NN_sigma} to obtain the first inequality, that the domain $D$ has bounded volume as of Assumption~\ref{asm:Dbdd} to get the equality in the second line, and \eqref{eq:proof:lem:parabolic_wellposedness_N:SELFMAPPING_prelim10_1} in the last step, we can upper bound
    \begin{equation}
        \label{eq:proof:lem:parabolic_wellposedness_N:SELFMAPPING_prelim11}
    \begin{split}
        \Nbig{\widetilde{g}_{\widetilde{\theta}_\tau}^N}_{L_2(D_T)}
        &= \N{\frac{1}{N^\beta}\sum_{i=1}^N \widetilde{c}^i_\tau\sigma\big(\widetilde{w}^{t,i}_\tau t + (\widetilde{w}_\tau^i)^Tx + \widetilde{\eta}_\tau^i\big)}_{L_2(D_T)} \\
        &\leq C \N{\sum_{i=1}^N \absnormal{\widetilde{c}^i_\tau}}_{L_2(D_T)} = C \sum_{i=1}^N \absnormal{\widetilde{c}^i_\tau}\\
        &\leq \N{\theta_0} + C \left(\int_0^\tau \N{\uNhat{\theta_s}}_{L_2(D_T)}^4 ds\right)^{1/4}
    \end{split}
    \end{equation}
    for a constant $C=C(N,T,D,\sigma)$.
    
    \textit{Combination of the bounds.}
    Using \eqref{eq:proof:lem:parabolic_wellposedness_N:NORM_utilde_UNIFORM} and \eqref{eq:proof:lem:parabolic_wellposedness_N:NORM_uhattilde_UNIFORM} in the first inequality,
    together with \eqref{eq:proof:lem:parabolic_wellposedness_N:SELFMAPPING_prelim10_1}, \eqref{eq:proof:lem:parabolic_wellposedness_N:SELFMAPPING_prelim10_4} and \eqref{eq:proof:lem:parabolic_wellposedness_N:SELFMAPPING_prelim11} in the second step,
    we establish (under the additional assumption $\int_{0}^\infty\alpha_\tau^{4/3}\,d\tau<\infty$ on the learning rate)
    \begin{equation}        \label{eq:proof:lem:parabolic_wellposedness_N:SELFMAPPING}
    \begin{split}
        &\Nnormal{(\widetilde{c},\widetilde{n},\widetilde{m})}_{\Theta_{\widetilde{\CT}}} + \N{\uNt{\theta}}_{\CV_{\widetilde{\CT}}} + \Nbig{\uNhatt{\theta}}_{\CV_{\widetilde{\CT}}} \\ 
        &\qquad\,\leq \Nnormal{(\widetilde{c},\widetilde{n},\widetilde{m})}_{\Theta_{\widetilde{\CT}}} + C\sup_{\tau\in[0,\CT]}\Nbig{\widetilde{g}_{\widetilde{\theta}_\tau}^N}_{L_2(D_T)} + C \left(\N{h}_{L_2(D_T)} + \N{f}_{L_2(D)} + 1\right)\\
        &\qquad\,\leq C_1 \sup_{\tau\in[0,\widetilde{\CT}]} \left(\int_0^\tau \N{\uNhat{\theta_s}}_{L_2(D_T)}^4 ds\right)^{1/4} + C_1\left(\N{h}_{L_2(D_T)} + \N{f}_{L_2(D)} + \N{\theta_0} + 1\right)\\
        &\qquad\,\leq C_1 \left(\int_0^{\widetilde{\CT}} \N{\uNhat{\theta_s}}_{L_2(D_T)}^4 ds\right)^{1/4} + C_1\left(\N{h}_{L_2(D_T)} + \N{f}_{L_2(D)} + \N{\theta_0} + 1\right)
    \end{split}
    \end{equation}
    for a constant $C_1=C_1(\alpha,N,T,D,\CL,q,\sigma)$.
    (Note that the bound on the right-hand side only grows like $\widetilde{\CT}^{1/2}$ instead of $\widetilde{\CT}^{1/4}$ due to the additional assumption on the learning rate.)

    \textit{Step~2c(ii): Preliminary computations for contractivity of $F$.}
    Consider two pairs of triples $((\widetilde{c}^1,\widetilde{n}^1,\widetilde{m}^1),\uNt{\widetilde{\theta}^1},\uNhatt{\widetilde{\theta}^1})$, $((\widetilde{c}^2,\widetilde{n}^2,\widetilde{m}^2),\uNt{\widetilde{\theta}^2},\uNhatt{\widetilde{\theta}^2})$ with their corresponding two pairs of triples $((c^1,n^1,m^1),\uN{\theta^1},\uNhat{\theta^1}), ((c^2,n^2,m^2),\uN{\theta^2},\uNhat{\theta^2}) \in \Theta_{\widetilde{\CT}}(\widetilde{M}) \times \CV_{\widetilde{\CT}}(\widetilde{M})\times \CV_{\widetilde{\CT}}(\widetilde{M})$.
    
    \textit{A bound for $\absnormal{\widetilde{c}^{i,1}_\tau-\widetilde{c}^{i,2}_\tau}$.}
    Recalling \eqref{eq:proof:lem:parabolic_wellposedness_N:setup:1a},
    we can estimate after inserting mixed terms with two applications of Cauchy-Schwarz inequality, while using the boundedness and Lipschitz continuity of $\sigma$ as of Assumption \ref{asm:NN_sigma} and that the domain $D$ has bounded volume as of Assumption~\ref{asm:Dbdd} together with the fact that the learning rate is decreasing, that
    \begin{equation}
        \label{eq:proof:lem:parabolic_wellposedness_N:CONTRACTIVITY_ctilde}
    \begin{split}
        \abs{\widetilde{c}_\tau^{i,1}-\widetilde{c}_\tau^{i,2}}
        &= C\;\! \bigg|\!\int_0^\tau \!\alpha_s\! \int_0^T\!\!\!\int_D \sigma\big(w^{t,i,1}_s t + (w^{i,1}_s)^T x + \eta_s^{i,1}\big)\uNhatv{\theta^1_s}{t,x} \\
        &\qquad\qquad\qquad\qquad\,- \sigma\big(w^{t,i,2}_s t + (w^{i,2}_s)^T x + \eta_s^{i,2}\big)\uNhatv{\theta^2_s}{t,x}\, dxdtds\bigg|\\
        &= C\;\! \bigg|\!\int_0^\tau \!\alpha_s \!\int_0^T\!\!\!\int_D \Big(\sigma\big(w^{t,i,1}_s t \!+\! (w^{i,1}_s)^T x \!+\! \eta_s^{i,1}\big)\!-\!\sigma\big(w^{t,i,2}_s t \!+\! (w^{i,2}_s)^T x \!+\! \eta_s^{i,2}\big)\Big)\uNhatv{\theta^1_s}{t,x} \\
        &\qquad\qquad\qquad\qquad\,+ \sigma\big(w^{t,i,2}_s t + (w^{i,2}_s)^T x + \eta_s^{i,2}\big)\left(\uNhatv{\theta^1_s}{t,x}-\uNhatv{\theta^2_s}{t,x}\right) dxdtds\bigg|\\
        &\leq C\int_0^\tau \!\alpha_s  \Bigg(\sqrt{\int_0^T\!\!\!\int_D \Big(\sigma\big(w^{t,i,1}_s t \!+\! (w^{i,1}_s)^T x \!+\! \eta_s^{i,1}\big)\!-\!\sigma\big(w^{t,i,2}_s t \!+\! (w^{i,2}_s)^T x \!+\! \eta_s^{i,2}\big)\Big)^2\, dxdt} \\
        &\qquad\qquad\qquad\qquad\,\cdot\Nbig{\uNhat{\theta^1_s}}_{L_2(D_T)}\\
        &\qquad\qquad\qquad\,+ \sqrt{\int_0^T\!\!\!\int_D  \big(\sigma\big(w^{t,i,2}_s t + (w^{i,2}_s)^T x + \eta_s^{i,2}\big)\big)^2\, dxdt}\Nbig{\uNhat{\theta^1_s}-\uNhat{\theta^2_s}}_{L_2(D_T)} \Bigg)\, ds\\
        &\leq C \int_0^\tau \!\alpha_s \left(\Nbig{n^{i,1}_s-n^{i,2}_s}_{L_2(D_T)}\Nbig{\uNhat{\theta^1_s}}_{L_2(D_T)} + \Nbig{\uNhat{\theta^1_s}-\uNhat{\theta^2_s}}_{L_2(D_T)}\right) ds\\
        &\leq C \int_0^\tau \Nbig{n^{i,1}_s-n^{i,2}_s}_{L_2(D_T)}\Nbig{\uNhat{\theta^1_s}}_{L_2(D_T)} + \Nbig{\uNhat{\theta^1_s}-\uNhat{\theta^2_s}}_{L_2(D_T)} \, ds.
    \end{split}
    \end{equation}
    Proceeding analogously for \eqref{eq:proof:lem:parabolic_wellposedness_N:setup:1b}--\eqref{eq:proof:lem:parabolic_wellposedness_N:setup:1d} while using that also $\sigma'$ is bounded and Lipschitz continuous as of Assumption \ref{asm:NN_sigma'}, we obtain
    \begin{equation}
        \label{eq:proof:lem:parabolic_wellposedness_N:CONTRACTIVITY_wwetatilde}
    \begin{split}
        &\abs{\widetilde{w}_\tau^{t,i,1}-\widetilde{w}_\tau^{t,i,2}} + \Nbig{\widetilde{w}_\tau^{i,1}-\widetilde{w}_\tau^{i,2}} + \absbig{\widetilde{\eta}_\tau^{i,1}-\widetilde{\eta}_\tau^{i,2}}\\
        &\qquad\,\leq C \int_0^\tau \absbig{c_s^{i,1}-c_s^{i,2}} \Nbig{\uNhat{\theta^1_s}}_{L_2(D_T)} + \absbig{c_s^{i,2}} \Nbig{m^{i,1}_s-m^{i,2}_s}_{L_2(D_T)} \Nbig{\uNhat{\theta^1_s}}_{L_2(D_T)}\\
        &\qquad\,\qquad\qquad+ \absbig{c_s^{i,2}}\Nbig{\uNhat{\theta^1_s}-\uNhat{\theta^2_s}}_{L_2(D_T)} \, ds
    \end{split}
    \end{equation}

    \textit{A bound for $\Nnormal{\widetilde{n}^{i,1}_\tau-\widetilde{n}^{i,2}_\tau}_{L_2(D_T)}$.}
    Using that $\sigma$ is Lipschitz as of Assumption \ref{asm:NN_sigma} and that the domain $D$ is bounded and has bounded volume as of Assumption~\ref{asm:Dbdd}, we can estimate
    \begin{equation}
        \label{eq:proof:lem:parabolic_wellposedness_N:CONTRACTIVITY_ntilde}
    \begin{split}
        \N{\widetilde{n}^{i,1}_\tau-\widetilde{n}^{i,2}_\tau}_{L_2(D_T)}
        &= \N{\sigma\big(\widetilde{w}^{t,i,1}_\tau t + (\widetilde{w}_\tau^{i,1})^Tx + \widetilde{\eta}_\tau^{i,1}\big)-\sigma\big(\widetilde{w}^{t,i,2}_\tau t + (\widetilde{w}_\tau^{i,2})^Tx + \widetilde{\eta}_\tau^{i,2}\big)}_{L_2(D_T)} \\
        &\leq C \left(\abs{\widetilde{w}_\tau^{t,i,1}-\widetilde{w}_\tau^{t,i,2}} + \Nbig{\widetilde{w}_\tau^{i,1}-\widetilde{w}_\tau^{i,2}} + \absbig{\widetilde{\eta}_\tau^{i,1}-\widetilde{\eta}_\tau^{i,2}}\right),
    \end{split}
    \end{equation}
    i.e., resulting in the same bound as in \eqref{eq:proof:lem:parabolic_wellposedness_N:CONTRACTIVITY_wwetatilde}.

    \textit{A bound for $\Nnormal{\widetilde{m}^{i,1}_\tau-\widetilde{m}^{i,2}_\tau}_{L_2(D_T)}$.}
    Since also $\sigma'$ is Lipschitz as of Assumption~\ref{asm:NN_sigma'},
    \begin{equation}
        \label{eq:proof:lem:parabolic_wellposedness_N:CONTRACTIVITY_mtilde}
    \begin{split}
        \N{\widetilde{m}^{i,1}_\tau-\widetilde{m}^{i,2}_\tau}_{L_2(D_T)}
        &\leq C \left(\abs{\widetilde{w}_\tau^{t,i,1}-\widetilde{w}_\tau^{t,i,2}} + \Nbig{\widetilde{w}_\tau^{i,1}-\widetilde{w}_\tau^{i,2}} + \absbig{\widetilde{\eta}_\tau^{i,1}-\widetilde{\eta}_\tau^{i,2}}\right),
    \end{split}
    \end{equation}
    i.e., resulting in the same bound as in \eqref{eq:proof:lem:parabolic_wellposedness_N:CONTRACTIVITY_wwetatilde}.

    \textit{A bound for $\Nnormal{\widetilde{g}_{\widetilde{\theta}^1_\tau}^N-\widetilde{g}_{\widetilde{\theta}^2_\tau}^N}_{L_p(D_T)}$.}
    Using that $\sigma$ is bounded and Lipschitz continuous as of Assumption \ref{asm:NN_sigma} and that the domain $D$ has bounded volume as of Assumption~\ref{asm:Dbdd} to obtain the second inequality,
    the latter again to get the equality thereafter,
    \eqref{eq:proof:lem:parabolic_wellposedness_N:SELFMAPPING_prelim10_1} in the subsequent step, and \eqref{eq:proof:lem:parabolic_wellposedness_N:CONTRACTIVITY_ctilde}--\eqref{eq:proof:lem:parabolic_wellposedness_N:CONTRACTIVITY_wwetatilde} in the next-to-last step,
    we can upper bound for $p\geq2$ with triangle inequality to obtain the first inequality that 
    \begin{allowdisplaybreaks}
    \begin{align}        \label{eq:proof:lem:parabolic_wellposedness_N:CONTRACTIVITYuinfty3}
        &\Nbig{\widetilde{g}_{\widetilde{\theta}^1_\tau}^N-\widetilde{g}_{\widetilde{\theta}^2_\tau}^N}_{L_p(D_T)}\nonumber\\
        &\qquad\, = \N{\frac{1}{N^\beta}\sum_{i=1}^N \widetilde{c}^{i,1}_\tau\sigma\big(\widetilde{w}^{t,i,1}_\tau t + (\widetilde{w}_\tau^{i,1})^Tx + \widetilde{\eta}_\tau^{i,1}\big)-\widetilde{c}^{i,2}_\tau\sigma\big(\widetilde{w}^{t,i,2}_\tau t + (\widetilde{w}_\tau^{i,2})^Tx + \widetilde{\eta}_\tau^{i,2}\big)}_{L_p(D_T)} \nonumber\\
        &\qquad\, \leq \N{\frac{1}{N^\beta}\sum_{i=1}^N (\widetilde{c}^{i,1}_\tau-\widetilde{c}^{i,2}_\tau)\sigma\big(\widetilde{w}^{t,i,1}_\tau t + (\widetilde{w}_\tau^{i,1})^Tx + \widetilde{\eta}_\tau^{i,1}\big)}_{L_p(D_T)}\nonumber\\
        &\qquad\,\quad\, +\N{\frac{1}{N^\beta}\sum_{i=1}^N \widetilde{c}^{i,2}_\tau\left(\sigma\big(\widetilde{w}^{t,i,1}_\tau t + (\widetilde{w}_\tau^{i,1})^Tx + \widetilde{\eta}_\tau^{i,1}\big)-\sigma\big(\widetilde{w}^{t,i,2}_\tau t + (\widetilde{w}_\tau^{i,2})^Tx + \widetilde{\eta}_\tau^{i,2}\big)\right)}_{L_p(D_T)} \nonumber\\
        &\qquad\, \leq C\N{\sum_{i=1}^N \abs{\widetilde{c}^{i,1}_\tau-\widetilde{c}^{i,2}_\tau}}_{L_p(D_T)}\nonumber\\
        &\qquad\,\quad\, + C\N{\sum_{i=1}^N \abs{\widetilde{c}^{i,2}_\tau}\left(\abs{\widetilde{w}^{t,i,1}_s-\widetilde{w}^{t,i,2}_s} + \N{\widetilde{w}^{i,1}_s-\widetilde{w}^{i,2}_s} + \abs{\widetilde{\eta}_s^{i,1}-\widetilde{\eta}_s^{i,2}}\right)}_{L_p(D_T)} \nonumber\\
        &\qquad\,= C \sum_{i=1}^N \abs{\widetilde{c}^{i,1}_\tau-\widetilde{c}^{i,2}_\tau} + \abs{\widetilde{c}^{i,2}_\tau}\left(\abs{\widetilde{w}^{t,i,1}_s-\widetilde{w}^{t,i,2}_s} + \N{\widetilde{w}^{i,1}_s-\widetilde{w}^{i,2}_s} + \abs{\widetilde{\eta}_s^{i,1}-\widetilde{\eta}_s^{i,2}}\right) \nonumber\\
        &\qquad\,\leq C\sum_{i=1}^N \abs{\widetilde{c}^{i,1}_\tau-\widetilde{c}^{i,2}_\tau} + \left(\abs{\widetilde{w}^{t,i,1}_s-\widetilde{w}^{t,i,2}_s} + \N{\widetilde{w}^{i,1}_s-\widetilde{w}^{i,2}_s} + \abs{\widetilde{\eta}_s^{i,1}-\widetilde{\eta}_s^{i,2}}\right)\nonumber\\
        &\qquad\,\qquad\qquad\qquad\qquad\qquad\qquad\qquad\qquad\quad\cdot\left(\N{\theta_0} + \left(\int_0^\tau \Nbig{\uNhat{\theta^2_s}}_{L_2(D_T)}^4 ds\right)^{1/4}\right)\nonumber\\
        &\qquad\,\leq C \int_0^\tau \sum_{i=1}^N \Nbig{n^{i,1}_s-n^{i,2}_s}_{L_2(D_T)}\Nbig{\uNhat{\theta^1_s}}_{L_2(D_T)} + \Nbig{\uNhat{\theta^1_s}-\uNhat{\theta^2_s}}_{L_2(D_T)}\nonumber \\
        &\qquad\,\quad\,\qquad\quad + \Big( \absbig{c_s^{i,1}-c_s^{i,2}} \Nbig{\uNhat{\theta^1_s}}_{L_2(D_T)} + \absbig{c_s^{i,2}} \Nbig{m^{i,1}_s-m^{i,2}_s}_{L_2(D_T)} \Nbig{\uNhat{\theta^1_s}}_{L_2(D_T)}  \nonumber\\
        &\qquad\,\qquad\qquad\qquad+ \absbig{c_s^{i,2}}\Nbig{\uNhat{\theta^1_s}-\uNhat{\theta^2_s}}_{L_2(D_T)}\Big)\left(\N{\theta_0} + \left(\int_0^\tau \Nbig{\uNhat{\theta_s^2}}_{L_2(D_T)}^4 ds\right)^{1/4}\right)  ds\nonumber\\
        &\qquad\,\leq C \int_0^\tau \sum_{i=1}^N \left(\absbig{c_s^{i,1}\!-\!c_s^{i,2}} \!+\!\Nbig{n^{i,1}_s\!-\!n^{i,2}_s}_{L_2(D_T)}\!+\!\Nbig{m^{i,1}_s\!-\!m^{i,2}_s}_{L_2(D_T)}\!+\!\Nbig{\uNhat{\theta^1_s}\!-\!\uNhat{\theta^2_s}}_{L_2(D_T)}\right)\nonumber\\
        &\qquad\,\qquad\qquad\qquad\cdot\left(1+\absbig{c_s^{i,2}}+\Nbig{\uNhat{\theta^1_s}}_{L_2(D_T)}+\absbig{c_s^{i,2}}\Nbig{\uNhat{\theta^1_s}}_{L_2(D_T)} \right)\nonumber\\
        &\qquad\,\qquad\qquad\qquad\cdot\left(1+\N{\theta_0} + \left(\int_0^\tau \Nbig{\uNhat{\theta_s^2}}_{L_2(D_T)}^4 ds\right)^{1/4}\right) ds
    \end{align}
    \end{allowdisplaybreaks}
    for a constant $C=C(p,N,T,D,\sigma)$.
    The last inequality is just a rough upper bound.

    \textit{A bound for $\Nnormal{\uNt{\widetilde{\theta}^1_\tau}-\uNt{\widetilde{\theta}^2_\tau}}_{L_2([0,T],H^1(D))} + \Nnormal{\uNt{\widetilde{\theta}^1_\tau}-\uNt{\widetilde{\theta}^2_\tau}}_{L_\infty([0,T],L_2(D))}$.}
    By following the computations of \textit{Step 2c(ii)} in the proof of \Cref{lem:parabolic_wellposedness_infinitewidth} in \Cref{sec:WellPosedness_Proof} that lead to \eqref{eq:proof:lem:parabolic_wellposedness_infinitewidth:CONTRACTIVITY_utilde},
    it holds
    \begin{equation}        \label{eq:proof:lem:parabolic_wellposedness_N:CONTRACTIVITY_utilde}
        \Nbig{\uNt{\widetilde{\theta}^1_\tau}-\uNt{\widetilde{\theta}^2_\tau}}_{L_2([0,T],H^1(D))} + \Nbig{\uNt{\widetilde{\theta}^1_\tau}-\uNt{\widetilde{\theta}^2_\tau}}_{L_\infty([0,T],L_2(D))}
        \leq C\Nbig{\widetilde{g}_{\widetilde{\theta}^1_\tau}^N-\widetilde{g}_{\widetilde{\theta}^2_\tau}^N}_{L_2(D_T)}
    \end{equation}
    for a constant $C=C(T,\CL,q)$.
    Thus, in particular,
    \begin{equation}        \label{eq:proof:lem:parabolic_wellposedness_N:CONTRACTIVITY_utilde_UNIFORM}
        \Nbig{\uNt{\widetilde{\theta}^1}-\uNt{\widetilde{\theta}^2}}_{\CV_\CT}
        \leq C\sup_{\tau\in[0,\CT]}\Nbig{\widetilde{g}_{\widetilde{\theta}^1_\tau}^N-\widetilde{g}_{\widetilde{\theta}^2_\tau}^N}_{L_2(D_T)}.
    \end{equation}

    \textit{A bound for $\Nnormal{\uNt{\widetilde{\theta}^1_\tau}-\uNt{\widetilde{\theta}^2_\tau}}_{L_\infty(D_T)}$.}
    By following the computations of \textit{Step 2c(ii)} in the proof of \Cref{lem:parabolic_wellposedness_infinitewidth} in \Cref{sec:WellPosedness_Proof} that lead to \eqref{eq:proof:lem:parabolic_wellposedness_infinitewidth:CONTRACTIVITYuinfty6},
    it holds for $p>d+1$ that
    \begin{equation}        \label{eq:proof:lem:parabolic_wellposedness_N:CONTRACTIVITYuinfty5}
    \begin{split}
        \Nbig{\uNt{\widetilde{\theta}^1_\tau}-\uNt{\widetilde{\theta}^2_\tau}}_{L_\infty(D_T)}
        &\leq c(d,p)C\Nbig{\widetilde{g}_{\widetilde{\theta}^1_\tau}^N-\widetilde{g}_{\widetilde{\theta}^2_\tau}^N}_{L_p(D_T)} 
    \end{split}
    \end{equation}
    for a constant $C=C(d,p,N,T,D,\CL,q,\sigma)$.
    
    \textit{A bound for $\Nnormal{\uNhatt{\widetilde{\theta}^1}-\uNhatt{\widetilde{\theta}^2}}_{L_2([0,T],H^1(D))} + \Nnormal{\uNhatt{\widetilde{\theta}^1}-\uNhatt{\widetilde{\theta}^2}}_{L_\infty([0,T],L_2(D))}$.}
    By following the computations of \textit{Step 2c(ii)} in the proof of \Cref{lem:parabolic_wellposedness_infinitewidth} in \Cref{sec:WellPosedness_Proof} that lead to \eqref{eq:proof:lem:parabolic_wellposedness_infinitewidth:CONTRACTIVITY_uhattilde},
    it holds
    \begin{equation}        \label{eq:proof:lem:parabolic_wellposedness_N:CONTRACTIVITY_uhattilde}
    \begin{split}
        &\Nbig{\uNhatt{\widetilde{\theta}^1_\tau}-\uNhatt{\widetilde{\theta}^2_\tau}}_{L_2([0,T],H^1(D))} + \Nbig{\uNhatt{\widetilde{\theta}^1_\tau}-\uNhatt{\widetilde{\theta}^2_\tau}}_{L_\infty([0,T],L_2(D))} \\
        &\qquad\,\leq C\left(\Nbig{\uNt{\widetilde{\theta}^1_\tau}-\uNt{\widetilde{\theta}^2_\tau}}_{L_2(D_T)} + \Nbig{\uNhatt{\widetilde{\theta}^2_\tau}}_{L_\infty([0,T], L_2(D))}\Nbig{\uNt{\widetilde{\theta}^1_\tau}-\uNt{\widetilde{\theta}^2_\tau}}_{L_\infty(D_T)}\right)
    \end{split}
    \end{equation}
    for a constant $C=C(T,\CL,q)$.
    Thus, in particular,
    \begin{equation}        \label{eq:proof:lem:parabolic_wellposedness_N:CONTRACTIVITY_uhattilde_UNIFORM}
    \begin{split}
        &\Nbig{\uNhatt{\widetilde{\theta}^1}-\uNhatt{\widetilde{\theta}^2}}_{\CV_\CT}\\
        &\qquad\,\leq C\left(\sup_{\tau\in[0,\CT]}\!\Nbig{\uNt{\widetilde{\theta}^1_\tau}-\uNt{\widetilde{\theta}^2_\tau}}_{L_2(D_T)} \!+\!\!\! \sup_{\tau\in[0,\CT]}\Nbig{\uNhatt{\widetilde{\theta}^2_\tau}}_{L_\infty([0,T], L_2(D))}\!\sup_{\tau\in[0,\CT]}\!\Nbig{\uNt{\widetilde{\theta}^1_\tau}-\uNt{\widetilde{\theta}^2_\tau}}_{L_\infty(D_T)}\right).
    \end{split}
    \end{equation}

    \textit{Combination of the bounds.}
    Using \eqref{eq:proof:lem:parabolic_wellposedness_N:CONTRACTIVITY_utilde_UNIFORM} and \eqref{eq:proof:lem:parabolic_wellposedness_N:CONTRACTIVITY_uhattilde_UNIFORM} in the first inequality together with \eqref{eq:proof:lem:parabolic_wellposedness_N:CONTRACTIVITYuinfty5} in the second inequality as well as \eqref{eq:proof:lem:parabolic_wellposedness_N:CONTRACTIVITY_ctilde}--\eqref{eq:proof:lem:parabolic_wellposedness_N:CONTRACTIVITYuinfty3} in the third, we have
    \begin{allowdisplaybreaks}
    \begin{align}    \label{eq:proof:lem:parabolic_wellposedness_N:CONTRACTION}
        &\Nbig{(\widetilde{c}^1,\widetilde{n}^1,\widetilde{m}^1)-(\widetilde{c}^2,\widetilde{n}^2,\widetilde{m}^2)}_{\Theta_{\widetilde{\CT}}} + \Nbig{\uNt{\widetilde{\theta}^1}-\uNt{\widetilde{\theta}^2}}_{\CV_{\widetilde{\CT}}} + \Nbig{\uNhatt{\widetilde{\theta}^1}-\uNhatt{\widetilde{\theta}^2}}_{\CV_{\widetilde{\CT}}}\nonumber\\
        &\qquad\,\leq \Nbig{(\widetilde{c}^1,\widetilde{n}^1,\widetilde{m}^1)-(\widetilde{c}^2,\widetilde{n}^2,\widetilde{m}^2)}_{\Theta_{\widetilde{\CT}}} + C\sup_{\tau\in[0,\widetilde{\CT}]}\Nbig{\widetilde{g}_{\widetilde{\theta}^1_\tau}^N-\widetilde{g}_{\widetilde{\theta}^2_\tau}^N}_{L_2(D_T)}\nonumber\\
        &\qquad\,\quad\,+ C\sup_{\tau\in[0,\widetilde{\CT}]}\Nbig{\uNhatt{\widetilde{\theta}^2_\tau}}_{L_\infty([0,T], L_2(D))}\!\sup_{\tau\in[0,\widetilde{\CT}]}\!\Nbig{\uNt{\widetilde{\theta}^1_\tau}-\uNt{\widetilde{\theta}^2_\tau}}_{L_\infty(D_T)} \nonumber\\ 
        &\qquad\,\leq \Nbig{(\widetilde{c}^1,\widetilde{n}^1,\widetilde{m}^1)-(\widetilde{c}^2,\widetilde{n}^2,\widetilde{m}^2)}_{\Theta_{\widetilde{\CT}}} + \nonumber\\
        &\qquad\,\quad\,+ C\left(1+\sup_{\tau\in[0,\widetilde{\CT}]}\Nbig{\uNhatt{\widetilde{\theta}^2_\tau}}_{L_\infty([0,T], L_2(D))}\right) \sup_{\tau\in[0,\widetilde{\CT}]}\Nbig{\widetilde{g}_{\widetilde{\theta}^1_\tau}^N-\widetilde{g}_{\widetilde{\theta}^2_\tau}^N}_{L_2(D_T)} \nonumber\\ 
        &\qquad\,\leq C_2 \left(1+\sup_{\tau\in[0,\widetilde{\CT}]}\Nbig{\uNhatt{\widetilde{\theta}^2_\tau}}_{L_\infty([0,T], L_2(D))}\right) \left(1+\N{\theta_0} + \left(\int_0^{\widetilde{\CT}} \N{\uNhat{\theta_s}}_{L_2(D_T)}^4 ds\right)^{1/4}\right)\nonumber \\
        &\qquad\, \qquad\cdot\int_0^{\widetilde{\CT}} \sum_{i=1}^N \left(\absbig{c_s^{i,1}\!-\!c_s^{i,2}} \!+\!\Nbig{n^{i,1}_s\!-\!n^{i,2}_s}_{L_2(D_T)}\!+\!\Nbig{m^{i,1}_s\!-\!m^{i,2}_s}_{L_2(D_T)}\!+\!\Nbig{\uNhat{\theta^1_s}\!-\!\uNhat{\theta^2_s}}_{L_2(D_T)}\right)\nonumber\\
        &\qquad\,\qquad\qquad\qquad\cdot\left(1+\absbig{c_s^{i,2}}+\Nbig{\uNhat{\theta^1_s}}_{L_2(D_T)}+\absbig{c_s^{i,2}}\Nbig{\uNhat{\theta^1_s}}_{L_2(D_T)} \right) ds\nonumber\\
    \end{align}
    \end{allowdisplaybreaks}%
    for a constant $C_2=C_2(\alpha,N,T,D,\CL,q,\sigma)$.
    (Note that multiple higher-order (up to fourth-order) product terms appear on the right-hand side due to the NTK varying during training.)
    
    \textit{Step~2d: Existence locally in training time.}
    Let us choose
\begin{equation}  \label{eq:proof:lem:parabolic_wellposedness_N:M0}
    \begin{split}
        M_0
        &= 2C_1\!\left(\N{h}_{L_2(D_T)} \!+\! \N{f}_{L_2(D)} \!+\! \N{\theta_0} \!+\! 1\right)
    \end{split}
    \end{equation}
    and
    \begin{equation}        \label{eq:proof:lem:parabolic_wellposedness_N:CT0}
    \begin{split}
        \CT_0
        &= \min\left\{\frac{1}{2^5C_1^4},\frac{1}{2C_2c_2N(1+M_0)^4}\right\}
    \end{split}
    \end{equation}
    where the constants $C_1$ and $C_2$ are as given implicitly in \eqref{eq:proof:lem:parabolic_wellposedness_N:SELFMAPPING} and \eqref{eq:proof:lem:parabolic_wellposedness_N:CONTRACTION}, respectively, and where $c_2=\max\{1+\N{\theta_0},1/(2C_1)\}$.
    We show in what follows that there exists a unique solution $((c,n,m), \uN{\theta},\uNhat{\theta})\in \Theta_{\CT_0}(M_0) \times \CV_{\CT_0}(M_0)\times\CV_{\CT_0}(M_0)$.

    \textit{Step~2d(i): Self-mapping property of $F$.}
    Consider $((\widetilde{c},\widetilde{n},\widetilde{m}),\uNt{\theta},\uNhatt{\theta})$ together with its corresponding $((c,n,m),\uN{\theta},\uNhat{\theta}) \in \Theta_{\CT_0}(M_0) \times \CV_{\CT_0}(M_0)\times \CV_{\CT_0}(M_0)$.
    Using the definitions of $M_0$ and $\CT_0$ in \eqref{eq:proof:lem:parabolic_wellposedness_N:M0} and \eqref{eq:proof:lem:parabolic_wellposedness_N:CT0}, respectively, we can derive from \eqref{eq:proof:lem:parabolic_wellposedness_N:SELFMAPPING} that    \begin{equation}
    \begin{split}
        &\Nnormal{(\widetilde{c},\widetilde{n},\widetilde{m})}_{\Theta_{\CT_0}} + \N{\uNt{\theta}}_{\CV_{\CT_0}} + \Nbig{\uNhatt{\theta}}_{\CV_{\CT_0}} \\ 
        &\qquad\,\leq C_1 \left(\int_0^{\CT_0} \N{\uNhat{\theta_s}}_{L_2(D_T)}^4 ds\right)^{1/4} + C_1\left(\N{h}_{L_2(D_T)} + \N{f}_{L_2(D)} + \N{\theta_0} + 1\right) \\
        &\qquad\,\leq C_1 \CT_0^{1/4} \Nbig{\uNhat{\theta}}_{\CV_{\CT_0}} + \frac{M_0}{2}
        \leq C_1 \CT_0^{1/4} M_0 + \frac{M_0}{2}
        \leq \frac{M_0}{2} + \frac{M_0}{2} \leq M_0.
    \end{split}
    \end{equation}
    Thus, $((\widetilde{c},\widetilde{n},\widetilde{m}),\uNt{\theta},\uNhatt{\theta}) \in \Theta_{\CT_0}(M_0) \times \CV_{\CT_0}(M_0)\times \CV_{\CT_0}(M_0)$.

    \textit{Step~2d(ii): Contractivity of $F$.}
    Consider triples $((\widetilde{c}^1,\widetilde{n}^1,\widetilde{m}^1),\uNt{\widetilde{\theta}^1},\uNhatt{\widetilde{\theta}^1})$, $((\widetilde{c}^2,\widetilde{n}^2,\widetilde{m}^2),\uNt{\widetilde{\theta}^2},\uNhatt{\widetilde{\theta}^2})$ with their corresponding $((c^1,n^1,m^1),\uN{\theta^1},\uNhat{\theta^1}), ((c^2,n^2,m^2),\uN{\theta^2},\uNhat{\theta^2}) \in \Theta_{\CT_0}(M_0) \times \CV_{\CT_0}(M_0)\times \CV_{\CT_0}(M_0)$.
    According to \textit{Step~2d(i)}, we have  that $((\widetilde{c}^1,\widetilde{n}^1,\widetilde{m}^1),\uNt{\widetilde{\theta}^1},\uNhatt{\widetilde{\theta}^1}),$ $((\widetilde{c}^2,\widetilde{n}^2,\widetilde{m}^2),\uNt{\widetilde{\theta}^2},\uNhatt{\widetilde{\theta}^2}) \in \Theta_{\CT_0}(M_0) \times \CV_{\CT_0}(M_0)\times \CV_{\CT_0}(M_0)$.
    Using the definitions of $M_0$ and $\CT_0$ in \eqref{eq:proof:lem:parabolic_wellposedness_N:M0} and \eqref{eq:proof:lem:parabolic_wellposedness_N:CT0}, we can derive from \eqref{eq:proof:lem:parabolic_wellposedness_N:CONTRACTION} that    \begin{allowdisplaybreaks}
    \begin{align}
        &\Nbig{(\widetilde{c}^1,\widetilde{n}^1,\widetilde{m}^1)-(\widetilde{c}^2,\widetilde{n}^2,\widetilde{m}^2)}_{\Theta_{\CT_0}} + \Nbig{\uNt{\widetilde{\theta}^1}-\uNt{\widetilde{\theta}^2}}_{\CV_{\CT_0}} + \Nbig{\uNhatt{\widetilde{\theta}^1}-\uNhatt{\widetilde{\theta}^2}}_{\CV_{\CT_0}} \notag \\
        &\qquad\,\leq C_2 \left(1+\sup_{\tau\in[0,\CT_0]}\Nbig{\uNhatt{\widetilde{\theta}^2_\tau}}_{L_\infty([0,T], L_2(D))}\right) \left(1+\N{\theta_0} + \left(\int_0^{\CT_0} \N{\uNhat{\theta_s}}_{L_2(D_T)}^4 ds\right)^{1/4}\right) \notag \\
        &\qquad\, \qquad\cdot\int_0^{\CT_0} \sum_{i=1}^N \left(\absbig{c_s^{i,1}\!-\!c_s^{i,2}} \!+\!\Nbig{n^{i,1}_s\!-\!n^{i,2}_s}_{L_2(D_T)}\!+\!\Nbig{m^{i,1}_s\!-\!m^{i,2}_s}_{L_2(D_T)}\!+\!\Nbig{\uNhat{\theta^1_s}\!-\!\uNhat{\theta^2_s}}_{L_2(D_T)}\right) \notag \\
        &\qquad\,\qquad\qquad\qquad\cdot\left(1+\absbig{c_s^{i,2}}+\Nbig{\uNhat{\theta^1_s}}_{L_2(D_T)}+\absbig{c_s^{i,2}}\Nbig{\uNhat{\theta^1_s}}_{L_2(D_T)} \right) ds \notag \\
        &\qquad\,\leq C_2 \left(1+M_0\right) \left(1+\N{\theta_0} + \CT_0^{1/4}M_0\right) \notag \\
        &\qquad\, \qquad\cdot\int_0^{\CT_0} \sum_{i=1}^N \left(\absbig{c_s^{i,1}\!-\!c_s^{i,2}} \!+\!\Nbig{n^{i,1}_s\!-\!n^{i,2}_s}_{L_2(D_T)}\!+\!\Nbig{m^{i,1}_s\!-\!m^{i,2}_s}_{L_2(D_T)}\!+\!\Nbig{\uNhat{\theta^1_s}\!-\!\uNhat{\theta^2_s}}_{L_2(D_T)}\right) \notag \\
        &\qquad\,\qquad\qquad\qquad\cdot\left(1+2M_0+M_0^2\right) ds \notag\\
        &\qquad\,\leq C_2c_2 \left(1+M_0\right)^4 \notag \\
        &\qquad\, \qquad\cdot\int_0^{\CT_0} \sum_{i=1}^N \Big(\absbig{c_s^{i,1}\!-\!c_s^{i,2}} \!+\!\Nbig{n^{i,1}_s\!-\!n^{i,2}_s}_{L_2(D_T)}\!+\!\Nbig{m^{i,1}_s\!-\!m^{i,2}_s}_{L_2(D_T)}\notag\\
        &\qquad\, \qquad\qquad\qquad\qquad+\Nbig{\uNhat{\theta^1_s}\!-\!\uNhat{\theta^2_s}}_{L_2(D_T)}\Big) \,ds \notag\\
        &\qquad\,\leq C_2c_2 \left(1+M_0\right)^4 \CT_0 \left(\N{(c^1,n^1,m^1)-(c^2,n^2,m^2)}_{\Theta_{\CT_0}}+N\Nbig{\uNhat{\theta^1}-\uNhat{\theta^2}}_{\CV_{\CT_0}}\right) ds \notag\\
        &\qquad\,
        \leq \frac{1}{2} \left(\N{(c^1,n^1,m^1)-(c^2,n^2,m^2)}_{\Theta_{\CT_0}} + \Nbig{\uN{\theta^1}-\uN{\theta^2}}_{\CV_{\CT_0}} + \Nbig{\uNhat{\theta^1}-\uNhat{\theta^2}}_{\CV_{\CT_0}}\right),
    \end{align}
    \end{allowdisplaybreaks}
    showing that the map $F$ is a contraction.

    \textit{Step~2d(iii): Banach fixed point theorem.}
    Hence, the Banach fixed point theorem guarantees that there exists a unique solution $((c,n,m), \uN{\theta},\uNhat{\theta})\in \Theta_{\CT_0}(M_0) \times\CV_{\CT_0}(M_0)\times\CV_{\CT_0}(M_0)$,
    which satisfies $((c,n,m), \uN{\theta},\uNhat{\theta}) = F((c,n,m), \uN{\theta},\uNhat{\theta})$.
    We have thus established the existence of a unique local-in-training-time solution to the PDE system \eqref{eq:parabolicPDEN_plain}\,\&\,\eqref{eq:parabolicadjointN_plain} coupled with the gradient descent update~\eqref{eq:GD} on the training time domain $[0,\CT_0]$.
    In particular, $(\uN{\theta_\tau},\uNhat{\theta_\tau})\in \CS\times \CS$ for every $\tau\in[0,\CT_0]$.
    Reapplying the classical existence and regularity results from \textit{Steps 1a}, \textit{1b} and \textit{1c} further shows that for each $\tau\in[0,\CT_0]$ such solution satisfies $(\partial_t\uNv{\theta_\tau}{t,\dummy},\partial_t\uNhatv{\theta_\tau}{t,\dummy})\in L_2(D) \times L_2(D)$
    for a.e.\@ $t\in[0,T]$.
    
    \textit{Step~2e: Existence globally in training time.}
    Leveraging a bootstrapping argument, we now extend this argument to obtain a solution on a training time domain $[0,\CT]$ for an arbitrary $\CT<\infty$.
    To do so, we proceed inductively.
    Suppose we have a solution $(\uN{},\uNhat{})\in \CV_{\CT_{k-1}}(M_{k-1})\times\CV_{\CT_{k-1}}(M_{k-1})$ which is such that for each $\tau\in[0,\CT_{k-1}]$ it satisfies $(\partial_t\uNv{\tau}{t,\dummy},\partial_t\uNhatv{\tau}{t,\dummy})\in L_2(D) \times L_2(D)$
    for a.e.\@ $t\in[0,T]$. (We showed in \textit{Step~2d} before that this is the case for the induction start $k=1$.)

    On the training time interval $I = [0,\CT_{k-1}]$ we can now employ \Cref{lem:parabolictimeevolutionJtN} which ensures that $\fd{\tau}\JN{\theta_\tau}\leq0$ for all $\tau\in I=[0,\CT_{k-1}]$.
    Thanks to this, \Cref{lem:parabolic_uNhatL2} (applied in the setting $I=[0,\CT_{k-1}]$) provides a uniform (in the training time $\tau$ and on the training time interval $[0,\CT_{k-1}]$) bound $\sup_{\tau\in[0,\CT_{k-1}]}\N{\uNhat{\theta_\tau}}_{L_2(D_T)}\leq C^{\widehat{u}^N}$, where $C^{\widehat{u}^N}$ does not depend on $\CT_{k-1}$ but only on $\JN{\theta_0}$.
    Let us now choose
    \begin{equation}
    \label{eq:proof:lem:parabolic_wellposedness_N:Mk}
        M_k
        = k^{1/4} C^{\widehat{u}^N} + 2C_1\!\left(\N{h}_{L_2(D_T)} \!+\! \N{f}_{L_2(D)} \!+\! \N{\theta_0} \!+\! 1\right)
    \end{equation}
    (note that here we can leverage the slower growth of the worst-case bound and impose it in the proposed bound~$M_k$) and
    \begin{equation}
    \label{eq:proof:lem:parabolic_wellposedness_N:CTk}
        \CT_k
        = \CT_{k-1} + \min\left\{\frac{1}{2^5C_1^4},\frac{1}{2C_2c_2N(1+M_k)^4}\right\}.
    \end{equation}
    We show in what follows that there exists a unique solution $((c,n,m),\uN{\theta},\uNhat{\theta}) \in \Theta_{\CT_k}(M_k) \times \CV_{\CT_k}(M_k)\times \CV_{\CT_k}(M_k)$.

    \textit{Step~2e(i): Self-mapping property of $F$.}
    Consider $((\widetilde{c},\widetilde{n},\widetilde{m}),\uNt{\theta},\uNhatt{\theta})$ together with its corresponding $((c,n,m),\uN{\theta},\uNhat{\theta}) \in \Theta_{\CT_k}(M_k) \times \CV_{\CT_k}(M_k)\times \CV_{\CT_k}(M_k)$.
    Using the definitions of $M_k$ and $\CT_k$ in \eqref{eq:proof:lem:parabolic_wellposedness_N:Mk} and \eqref{eq:proof:lem:parabolic_wellposedness_N:CTk}, we can derive from \eqref{eq:proof:lem:parabolic_wellposedness_N:SELFMAPPING} that
    \begin{equation}       \label{eq:proof:lem:parabolic_wellposedness_N:SELFMAPPINGk}
    \begin{split}
        &\Nnormal{(\widetilde{c},\widetilde{n},\widetilde{m})}_{\Theta_{\CT_k}} + \N{\uNt{\theta}}_{\CV_{\CT_k}} + \Nbig{\uNhatt{\theta}}_{\CV_{\CT_k}} \\ 
        &\qquad\,\leq C_1 \left(\int_0^{\CT_{k}} \N{\uNhat{\theta_s}}_{L_2(D_T)}^4 ds\right)^{1/4} + C_1\left(\N{h}_{L_2(D_T)} + \N{f}_{L_2(D)} + \N{\theta_0} + 1\right) \\
        &\qquad\,\leq C_1 \left(\int_0^{\CT_{k-1}} \N{\uNhat{\theta_s}}_{L_2(D_T)}^4 ds + \int_{\CT_{k-1}}^{\CT_k} \N{\uNhat{\theta_s}}_{L_2(D_T)}^4 ds\right)^{1/4} \\
        &\qquad\,\quad\,+ C_1\left(\N{h}_{L_2(D_T)} + \N{f}_{L_2(D)} + \N{\theta_0} + 1\right) \\
        &\qquad\,\leq C_1 \left(\CT_{k-1}(C^{\widehat{u}^N})^4 + (\CT_k-\CT_{k-1})M_k^4\right)^{1/4} + \frac{M_k}{2} \\
        &\qquad\,\leq \left(\frac{1}{2^5} k (C^{\widehat{u}^N})^4 + \frac{1}{2^5}M_k^4\right)^{1/4} + \frac{M_k}{2} \leq \left(\frac{1}{2^5} M_k^4 + \frac{1}{2^5}M_k^4\right)^{1/4} + \frac{M_k}{2}
        \leq \frac{M_k}{2} + \frac{M_k}{2} \leq M_k,
    \end{split}
    \end{equation}
    where we used for the first step in the last line that (with $\CT_{-1}:=0$) according to the definition of $\CT_k$ in \eqref{eq:proof:lem:parabolic_wellposedness_N:CTk} it holds
    $\CT_{k-1} (C^{\widehat{u}^N})^4 = \sum_{\ell=0}^{k-1} (\CT_{\ell}-\CT_{\ell-1}) (C^{\widehat{u}^N})^4 \leq \frac{1}{2^5C_1^4} k (C^{\widehat{u}^N})^4.$
    Thus, $((\widetilde{c},\widetilde{n},\widetilde{m}),\uNt{\theta},\uNhatt{\theta}) \in \Theta_{\CT_k}(M_k)\times\CV_{\CT_k}(M_k)\times \CV_{\CT_k}(M_k)$.

    \textit{Step~2e(ii): Contractivity of $F$.}
    Consider triples $((\widetilde{c}^1,\widetilde{n}^1,\widetilde{m}^1),\uNt{\widetilde{\theta}^1},\uNhatt{\widetilde{\theta}^1})$, $((\widetilde{c}^2,\widetilde{n}^2,\widetilde{m}^2),\uNt{\widetilde{\theta}^2},\uNhatt{\widetilde{\theta}^2})$ with their corresponding $((c^1,n^1,m^1),\uN{\theta^1},\uNhat{\theta^1}), ((c^2,n^2,m^2),\uN{\theta^2},\uNhat{\theta^2}) \in \Theta_{\CT_k}(M_k) \times \CV_{\CT_k}(M_k)\times \CV_{\CT_k}(M_k)$.
    Using the definitions of $M_k$ and $\CT_k$ in \eqref{eq:proof:lem:parabolic_wellposedness_N:Mk} and \eqref{eq:proof:lem:parabolic_wellposedness_N:CTk}, we can derive from \eqref{eq:proof:lem:parabolic_wellposedness_N:CONTRACTION} that
    \begin{allowdisplaybreaks}
    \begin{align}\label{Eq:ContractivityFfiniteSystem}
        &\Nbig{(\widetilde{c}^1,\widetilde{n}^1,\widetilde{m}^1)-(\widetilde{c}^2,\widetilde{n}^2,\widetilde{m}^2)}_{\Theta_{\CT_k}} + \Nbig{\uNt{\widetilde{\theta}^1}-\uNt{\widetilde{\theta}^2}}_{\CV_{\CT_k}} + \Nbig{\uNhatt{\widetilde{\theta}^1}-\uNhatt{\widetilde{\theta}^2}}_{\CV_{\CT_k}} \nonumber\\
        &\qquad\,\leq C_2 \left(1+\sup_{\tau\in[0,\CT_k]}\Nbig{\uNhatt{\widetilde{\theta}^2_\tau}}_{L_\infty([0,T], L_2(D))}\right) \left(1+\N{\theta_0} + \left(\int_0^{\CT_k} \N{\uNhat{\theta_s}}_{L_2(D_T)}^4 ds\right)^{1/4}\right)\nonumber \\
        &\qquad\, \qquad\cdot\int_0^{\CT_k} \sum_{i=1}^N \left(\absbig{c_s^{i,1}\!-\!c_s^{i,2}} \!+\!\Nbig{n^{i,1}_s\!-\!n^{i,2}_s}_{L_2(D_T)}\!+\!\Nbig{m^{i,1}_s\!-\!m^{i,2}_s}_{L_2(D_T)}\!+\!\Nbig{\uNhat{\theta^1_s}\!-\!\uNhat{\theta^2_s}}_{L_2(D_T)}\right)\nonumber\\
        &\qquad\,\qquad\qquad\qquad\cdot\left(1+\absbig{c_s^{i,2}}+\Nbig{\uNhat{\theta^1_s}}_{L_2(D_T)}+\absbig{c_s^{i,2}}\Nbig{\uNhat{\theta^1_s}}_{L_2(D_T)} \right) ds \nonumber \\
        &\qquad\,\leq C_2c_2 \left(1+M_k\right)^4 \nonumber \\
        &\qquad\, \qquad\cdot\int_0^{\CT_k} \sum_{i=1}^N \Big(\absbig{c_s^{i,1}\!-\!c_s^{i,2}} \!+\!\Nbig{n^{i,1}_s\!-\!n^{i,2}_s}_{L_2(D_T)}\!+\!\Nbig{m^{i,1}_s\!-\!m^{i,2}_s}_{L_2(D_T)}\nonumber \\
        &\qquad\, \qquad\qquad\qquad\qquad\Nbig{\uNhat{\theta^1_s}\!-\!\uNhat{\theta^2_s}}_{L_2(D_T)}\Big) \,ds\nonumber \\
        &\qquad\,= C_2c_2 \left(1+M_k\right)^4 \nonumber \\
        &\qquad\, \qquad\cdot\int_{\CT_{k-1}}^{\CT_k} \sum_{i=1}^N \Big(\absbig{c_s^{i,1}\!-\!c_s^{i,2}} \!+\!\Nbig{n^{i,1}_s\!-\!n^{i,2}_s}_{L_2(D_T)}\!+\!\Nbig{m^{i,1}_s\!-\!m^{i,2}_s}_{L_2(D_T)}\nonumber \\
        &\qquad\, \qquad\qquad\qquad\qquad+\Nbig{\uNhat{\theta^1_s}\!-\!\uNhat{\theta^2_s}}_{L_2(D_T)}\Big)\, ds\nonumber \\&\qquad\,\leq C_2c_2 \left(1\!+\!M_k\right)^4 (\CT_k\!-\!\CT_{k-1}) \left(\N{(c^1,n^1,m^1)\!-\!(c^2,n^2,m^2)}_{\Theta_{\CT_k}}\!+\!N\Nbig{\uNhat{\theta^1}\!-\!\uNhat{\theta^2}}_{\CV_{\CT_k}}\right) ds\nonumber \\
        &\qquad\,
        \leq \frac{1}{2} \left(\N{(c^1,n^1,m^1)-(c^2,n^2,m^2)}_{\Theta_{\CT_k}} + \Nbig{\uN{\theta^1}-\uN{\theta^2}}_{\CV_{\CT_k}} + \Nbig{\uNhat{\theta^1}-\uNhat{\theta^2}}_{\CV_{\CT_k}}\right),
    \end{align}
    \end{allowdisplaybreaks}
    where the second step reuses from \eqref{eq:proof:lem:parabolic_wellposedness_N:SELFMAPPINGk} that $\big(\int_0^{\CT_{k}} \Nnormal{\uNhat{\theta_s}}_{L_2(D_T)}^4 ds\big)^{1/4} \leq M_k/(2C_1)$, while the third step is due to the uniqueness of the solution on the training time interval $[0,\CT_{k-1}]$.
    Thus, the map $F$ is a contraction.

    \textit{Step~2e(iii): Banach fixed point theorem.}
    Hence, the Banach fixed point theorem guarantees that there exists a unique solution $((c,n,m), \uN{\theta},\uNhat{\theta})\in \Theta_{\CT_k}(M_k) \times\CV_{\CT_k}(M_k)\times\CV_{\CT_k}(M_k)$,
    which satisfies $((c,n,m), \uN{\theta},\uNhat{\theta}) = F((c,n,m), \uN{\theta},\uNhat{\theta})$.
    We have thus established the existence of a unique local-in-training-time solution to the PDE system \eqref{eq:parabolicPDEN_plain}\,\&\,\eqref{eq:parabolicadjointN_plain} coupled with the gradient descent update~\eqref{eq:GD} on the training time domain $[0,\CT_k]$.
    In particular, $(\uN{\theta_\tau},\uNhat{\theta_\tau})\in \CS\times \CS$ for every $\tau\in[0,\CT_k]$.
    Reapplying the classical existence and regularity results from \textit{Steps 1a}, \textit{1b} and \textit{1c} further shows that for each $\tau\in[0,\CT_k]$ such solution satisfies $(\partial_t\uNv{\theta_\tau}{t,\dummy},\partial_t\uNhatv{\theta_\tau}{t,\dummy})\in L_2(D) \times L_2(D)$
    for a.e.\@ $t\in[0,T]$.

    \textit{Step~2e(iv): Globality of the construction in training time.}
    It remains to notice that, due to the definition of the times~$\CT_k$ in \eqref{eq:proof:lem:parabolic_wellposedness_N:CTk}, the telescopic sum
    \begin{equation}
        \label{eq:proof:lem:parabolic_wellposedness_N:harmonicseries}
    \begin{split}
        \sum_{k=1}^\infty (\CT_k-\CT_{k-1})
        &= \sum_{k=1}^\infty \min\left\{\frac{1}{2^5C_1^4},\frac{1}{2C_2c_2N(1+M_k)^4}\right\}
        \geq \sum_{k=K}^\infty \frac{1}{2C_2c_2N(1+M_k)^4}\\
        &= \sum_{k=K}^\infty \frac{1}{2C_2c_2N(1+k^{1/4} C^{\widehat{u}^N} + 2C_1(\N{h}_{L_2(D_T)} + \N{f}_{L_2(D)} +  \N{\theta_0} + 1))^4}
    \end{split}
    \end{equation}
    diverges.
    To see this, simply note that the last term in \eqref{eq:proof:lem:parabolic_wellposedness_N:harmonicseries} is a harmonic series for some sufficiently large integer $K$.
    This ensures that the above construction in \textit{Step~2d} and \textit{2e} gives a solution in $\CV_{\CT}(M)\times\CV_{\CT}(M)$ for any given $\CT<\infty$ and suitable associated $M>0$.
    (Note that here we exploit that we correctly balanced the appearing higher-order product terms with the slower worst-case growth, enabling to get a diverging series, and thus a global existence result.
    With only \eqref{eq:learning_rate} and without the additional assumption $\int_{0}^\infty\alpha_\tau^{4/3}\,d\tau<\infty$ on the learning rate, one would have only been able to get $M_k\propto k^{1/2}$, which would have lead to a geometric series of the form $\sum_{k=K}^\infty \frac{c}{1+k^2}$, which does not diverge, thus leading to no global existence result.)

    \textbf{Uniqueness.}
    It remains to prove the uniqueness of a solution $(\uN{\theta_\tau},\uNhat{\theta_\tau})$ to the PDE system \eqref{eq:parabolicPDEN_plain}\,\&\,\eqref{eq:parabolicadjointN_plain} coupled with the gradient descent update~\eqref{eq:GD}.
    For this purpose, suppose that there are two weak solutions $((c^1,n^1,m^1),\uN{\theta^1},\uNhat{\theta^1}), ((c^2,n^2,m^2),\uN{\theta^2},\uNhat{\theta^2}) \in \Theta_{\CT}(M) \times \CV_{\CT}(M)\times \CV_{\CT}(M)$.
    This means we have two $(\uN{\theta^1_\tau},\uNhat{\theta^1_\tau})$, $(\uN{\theta^2_\tau},\uNhat{\theta^2_\tau})\in \CS\times \CS$ which satisfy $(\partial_t\uNv{\theta_\tau^1}{t,\dummy},\partial_t\uNhatv{\theta_\tau^1}{t,\dummy}),(\partial_t\uNv{\theta_\tau^2}{t,\dummy},\partial_t\uNhatv{\theta_\tau^2}{t,\dummy})\in L_2(D) \times L_2(D)$ for a.e.\@ $t\in[0,T]$ and where $g_{\theta^1_\tau}^N,g_{\theta^2_\tau}^N\in L_2(D_T)$ denote the corresponding NN functions \eqref{eq:gN} for each $\tau\in[0,\CT]$.
    By repeating the computations of the existence proof in \textit{Step 2c(ii)} we obtain analogously to \eqref{eq:proof:lem:parabolic_wellposedness_N:CONTRACTIVITY_utilde} and \eqref{eq:proof:lem:parabolic_wellposedness_N:CONTRACTIVITYuinfty5} that
    \begin{equation}       \label{eq:proof:lem:parabolic_wellposedness_N:UNIQUENESS1}
        \Nbig{\uN{\theta^1_\tau}-\uN{\theta^2_\tau}}_{L_2([0,T],H^1(D))} + \Nbig{\uN{\theta^1_\tau}-\uN{\theta^2_\tau}}_{L_\infty([0,T],L_2(D))}
        \leq C\Nbig{g_{\theta^1_\tau}^N-g_{\theta^2_\tau}^N}_{L_2(D_T)}
    \end{equation}
    and for some $p>d+1$ that
    \begin{equation}        \label{eq:proof:lem:parabolic_wellposedness_N:UNIQUENESS2}
    \begin{split}
        \Nbig{\uN{\theta^1_\tau}-\uN{\theta^2_\tau}}_{L_\infty(D_T)}
        &\leq c(d,p)C\Nbig{g_{\theta^1_\tau}^N-g_{\theta^2_\tau}^N}_{L_p(D_T)},
    \end{split}
    \end{equation}
    as well as analogously to \eqref{eq:proof:lem:parabolic_wellposedness_N:CONTRACTIVITY_uhattilde} that 
    \begin{equation}        \label{eq:proof:lem:parabolic_wellposedness_N:UNIQUENESS3}
    \begin{split}
        &\Nbig{\uNhat{\theta^1_\tau}-\uNhat{\theta^2_\tau}}_{L_2([0,T],H^1(D))} + \Nbig{\uNhat{\theta^1_\tau}-\uNhat{\theta^2_\tau}}_{L_\infty([0,T],L_2(D))} \\
        &\qquad\,\leq C\left(\Nbig{\uN{\theta^1_\tau}-\uN{\theta^2_\tau}}_{L_2(D_T)} + M\Nbig{\uN{\theta^1_\tau}-\uN{\theta^2_\tau}}_{L_\infty(D_T)}\right),
    \end{split}
    \end{equation}
    where we used directly that $\uNhat{\theta^2_\tau}\in\CV_\CT(M)$.
    Since it holds    \begin{equation}
    \begin{split}
        &\Nbig{g_{\theta^1_\tau}^N-g_{\theta^2_\tau}^N}_{L_p(D_T)}\\
        &\qquad\,\leq C \int_0^\tau \sum_{i=1}^N \left(\absbig{c_s^{i,1}\!-\!c_s^{i,2}} \!+\!\Nbig{n^{i,1}_s\!-\!n^{i,2}_s}_{L_2(D_T)}\!+\!\Nbig{m^{i,1}_s\!-\!m^{i,2}_s}_{L_2(D_T)}\!+\!\Nbig{\uNhat{\theta^1_s}\!-\!\uNhat{\theta^2_s}}_{L_2(D_T)}\right)\\
        &\qquad\,\qquad\qquad\qquad\cdot(1+M)^2(1+\N{\theta_0} + \CT^{1/4}M) \, ds
    \end{split}
    \end{equation}
    by repeating the computations \eqref{eq:proof:lem:parabolic_wellposedness_N:CONTRACTIVITYuinfty3} of the existence proof in \textit{Step 2c(ii)} and using directly that $(c^i,n^i,m^i)\in\Theta_{\CT}(M)$ as well as  $\uNhat{\theta^i_\tau}\in\CV_\CT(M)$,
    and since it furthermore holds    \begin{equation}
    \begin{split}
        &\sum_{i=1}^N \left(\abs{c_\tau^{i,1}-c_\tau^{i,2}} + \N{n^{i,1}_\tau-n^{i,2}_\tau}_{L_2(D_T)} + \N{m^{i,1}_\tau-m^{i,2}_\tau}_{L_2(D_T)}\right)\\
        &\qquad\,\leq C \int_0^\tau \sum_{i=1}^N \left(\absbig{c_s^{i,1}\!-\!c_s^{i,2}} \!+\!\Nbig{n^{i,1}_s\!-\!n^{i,2}_s}_{L_2(D_T)}\!+\!\Nbig{m^{i,1}_s\!-\!m^{i,2}_s}_{L_2(D_T)}\!+\!\Nbig{\uNhat{\theta^1_s}\!-\!\uNhat{\theta^2_s}}_{L_2(D_T)}\right)\\
        &\qquad\,\qquad\qquad\qquad\cdot(1+M)^2 \, ds
    \end{split}
    \end{equation}
    by repeating and combining the computations \eqref{eq:proof:lem:parabolic_wellposedness_N:CONTRACTIVITY_ctilde}--\eqref{eq:proof:lem:parabolic_wellposedness_N:CONTRACTIVITY_mtilde} of the existence proof in \textit{Step 2c(ii)} and using directly that $(c^i,n^i,m^i)\in\Theta_{\CT}(M)$ as well as  $\uNhat{\theta^i_\tau}\in\CV_\CT(M)$,
    we get    \begin{equation}        \label{eq:proof:lem:parabolic_wellposedness_N:UNIQUENESS5}
    \begin{split}
        &\Nbig{g_{\theta^1_\tau}^N\!-\!g_{\theta^2_\tau}^N}_{L_2(D_T)} \!\!+\! \Nbig{g_{\theta^1_\tau}^N\!-\!g_{\theta^2_\tau}^N}_{L_p(D_T)} \!\!+\! \sum_{i=1}^N\! \left(\abs{c_\tau^{i,1}\!-\!c_\tau^{i,2}} \!+\! \N{n^{i,1}_\tau\!-\!n^{i,2}_\tau}_{L_2(D_T)} \!\!+\! \N{m^{i,1}_\tau\!-\!m^{i,2}_\tau}_{L_2(D_T)}\right)\\
        &\qquad\,\leq  C \int_0^\tau \sum_{i=1}^N \left(\absbig{c_s^{i,1}\!-\!c_s^{i,2}} \!+\!\Nbig{n^{i,1}_s\!-\!n^{i,2}_s}_{L_2(D_T)}\!\!+\!\Nbig{m^{i,1}_s\!-\!m^{i,2}_s}_{L_2(D_T)}\!\!+\!\Nbig{\uNhat{\theta^1_s}\!-\!\uNhat{\theta^2_s}}_{L_2(D_T)}\right) ds \\
        &\qquad\,\leq  C \int_0^\tau \sum_{i=1}^N \Big(\absbig{c_s^{i,1}\!-\!c_s^{i,2}} \!+\!\Nbig{n^{i,1}_s\!-\!n^{i,2}_s}_{L_2(D_T)}\!\!+\!\Nbig{m^{i,1}_s\!-\!m^{i,2}_s}_{L_2(D_T)}\\
        &\qquad\,\qquad\qquad\qquad+\Nbig{\uN{\theta^1_\tau}-\uN{\theta^2_\tau}}_{L_2(D_T)} + \Nbig{\uN{\theta^1_\tau}-\uN{\theta^2_\tau}}_{L_\infty(D_T)}\Big) \,ds \\
        &\qquad\,\leq  C \int_0^\tau \sum_{i=1}^N \Big(\absbig{c_s^{i,1}\!-\!c_s^{i,2}} \!+\!\Nbig{n^{i,1}_s\!-\!n^{i,2}_s}_{L_2(D_T)}\!\!+\!\Nbig{m^{i,1}_s\!-\!m^{i,2}_s}_{L_2(D_T)} \\
        &\qquad\,\qquad\qquad\qquad+\Nbig{g_{\theta^1_\tau}^N-g_{\theta^2_\tau}^N}_{L_2(D_T)} + \Nbig{g_{\theta^1_\tau}^N-g_{\theta^2_\tau}^N}_{L_p(D_T)}\Big) \,ds,
    \end{split}
    \end{equation}
    where $C$ may depend in particular on $M$ and $\CT$.
    
    Recalling that $\theta^1_0=\theta^2_0$ and consequently  $c^{i,1}_0=c^{i,2}_0$, $n^{i,1}_0=n^{i,2}_0$, $m^{i,1}_0=m^{i,2}_0$ as well as $g_{\theta^1_0}^N=g_{\theta^2_0}^N=0$,
    we can now employ Grönwall's inequality in its integral form to obtain
    \begin{equation}       \label{eq:proof:lem:parabolic_wellposedness_N:UNIQUENESS6}
        \Nbig{g_{\theta^1_\tau}^N-g_{\theta^2_\tau}^N}_{L_2(D_T)} + \Nbig{g_{\theta^1_\tau}^N-g_{\theta^2_\tau}^N}_{L_p(D_T)}
        = 0
    \end{equation}
    for every $\tau\in[0,\CT]$.
    With \eqref{eq:proof:lem:parabolic_wellposedness_N:UNIQUENESS1} and \eqref{eq:proof:lem:parabolic_wellposedness_N:UNIQUENESS3}
    we hence conclude that for every $\tau\in[0,\CT]$ it hold
    \begin{equation}
        \Nbig{\uN{\theta^1_\tau}-\uN{\theta^2_\tau}}_{L_2([0,T],H^1(D))} + \Nbig{\uN{\theta^1_\tau}-\uN{\theta^2_\tau}}_{L_\infty([0,T],L_2(D))}
        =0
    \end{equation}
    and
    \begin{equation}
        \Nbig{\uNhat{\theta^1_\tau}-\uNhat{\theta^2_\tau}}_{L_2([0,T],H^1(D))} + \Nbig{\uNhat{\theta^1_\tau}-\uNhat{\theta^2_\tau}}_{L_\infty([0,T],L_2(D))}
        =0.
    \end{equation}
    Thus $\Nbig{\uN{\theta^1_\tau}-\uN{\theta^2_\tau}}_{\CV_\CT}=0$ and $\Nbig{\uNhat{\theta^1_\tau}-\uNhat{\theta^2_\tau}}_{\CV_\CT}=0$, proving uniqueness in $\CV_\CT(M)$.
\end{proof}

We now provide auxiliary results that were used in the proof of \Cref{lem:parabolic_wellposedness_N}.
First, we establish that the loss $\JN{\theta_\tau}$ defined in \eqref{eq:parabolicJ} is monotonically non-increasing in the training time $\tau$.

\begin{lemma}[Decay of the loss $\JN{\theta_\tau}$]    \label{lem:parabolictimeevolutionJtN}
    Assume that the learning rate satisfies additionally $\int_{0}^\infty\alpha_\tau^{4/3}\,d\tau<\infty$.
    Let $((\uN{\theta_\tau},\uNhat{\theta_\tau}))_{\tau\in I}\in \CC\left(I,\CS\times\CS\right)$ denote the unique weak solution to the PDE system \eqref{eq:parabolicPDEN_plain}\,\&\,\eqref{eq:parabolicadjointN_plain} coupled with the gradient descent update~\eqref{eq:GD} in the sense of \Cref{lem:parabolic_wellposedness_N} on the training time interval $I$.
    Define the loss $\JN{\theta_\tau}$ as in \eqref{eq:parabolicJ}.
    Then, for the training time derivative $\fd{\tau}\JN{\theta_\tau}$ it holds
    \begin{equation}
    \begin{split}
        \fd{\tau}\JN{\theta_\tau}
        &= - \alpha_\tau (\uNhat{\theta_\tau},T_{B(\mu^N_\tau)}\uNhat{\theta_\tau})_{L_2(D_T)} \\
        &= -\alpha_\tau \int_0^T\!\!\!\int_D \uNhatv{\theta_\tau}{t,x} \int_0^T\!\!\!\int_D B(t,x,t',x';\mu^N_\tau)\uNhatv{\theta_\tau}{t',x'} \,dx'dt'dxdt
    \end{split}
    \end{equation}
    for all $\tau\in I$
    with the operator $T_{B(\mu^N_\tau)}$ defined in \eqref{eq:parabolicT_B} and where the kernel $B(\mu^N_\tau)$ is as in \eqref{eq:parabolicB}.
    In particular, we have $\fd{\tau}\JN{\theta_\tau}\leq 0$ for all $\tau\in I$.
\end{lemma}

\begin{proof}
    Taking the training time derivative of our loss $\JN{\theta_\tau}$, i.e., the derivative w.r.t.\@ the training time~$\tau$, 
    we obtain by chain rule and by using that $\uNhat{\theta_\tau}$ is a weak solution to the adjoint PDE~\eqref{eq:parabolicadjointN_plain} in the sense of \Cref{def:weak_sol_adjoint} with right-hand side $(\uN{\theta_\tau}-h)$ that
    \begin{equation}        \label{eq:proof:parabolicdtJN1}
    \begin{split}
        &\fd{\tau}\JN{\theta_\tau}
        = \fd{\tau} \frac{1}{2} \int_0^T\!\!\!\int_D (\uNv{\theta_\tau}{t,x} \!-\! h(t,x))^2 \,dxdt \\
        &\,\,\,\,\,= \int_0^T\!\!\!\int_D (\uNv{\theta_\tau}{t,x} \!-\! h(t,x)) \fd{\tau}\uNv{\theta_\tau}{t,\dummy} \, dxdt 
        = \int_0^T \!\left(\uNv{\theta_\tau}{t,\dummy} \!-\! h(t,\dummy), \fd{\tau}\uNv{\theta_\tau}{t,\dummy}\right)_{L_2(D)}\!dt\\
        &\,\,\,\,\,= \int_0^T \left\langle-\fpartial{t}\uNhatv{\theta_\tau}{t,\dummy},\fd{\tau}\uNv{\theta_\tau}{t,\dummy}\right\rangle_{H^{-1}(D),H_0^1(D)} + \CB^\dagger\left[\uNhatv{\theta_\tau}{t,\dummy},\fd{\tau}\uNv{\theta_\tau}{t,\dummy};t\right] \\
        &\,\,\,\,\, \qquad\quad\, - \left(q_u(t,\dummy,\uNv{\theta_\tau}{t,\dummy})\uNhatv{\theta_\tau}{t,\dummy}, \fd{\tau}\uNv{\theta_\tau}{t,\dummy}\right)_{L_2(D)} dt \\
        &\,\,\,\,\,= \int_0^T \left\langle\fpartial{t}\fd{\tau}\uNv{\theta_\tau}{t,\dummy},\uNhatv{\theta_\tau}{t,\dummy}\right\rangle_{H^{-1}(D),H_0^1(D)} + \CB\left[\fd{\tau}\uNv{\theta_\tau}{t,\dummy},\uNhatv{\theta_\tau}{t,\dummy};t\right] \\
        &\,\,\,\,\, \qquad\quad\, - \left(q_u(t,\dummy,\uNv{\theta_\tau}{t,\dummy})\fd{\tau}\uNv{\theta_\tau}{t,\dummy},\uNhatv{\theta_\tau}{t,\dummy}\right)_{L_2(D)} dt \\
        &\,\,\,\,\,= \int_0^T \left(\fd{\tau}g_{\theta_\tau}^N(t,\dummy),\uNhatv{\theta_\tau}{t,\dummy}\right)_{L_2(D)} dt
        = \int_0^T\!\!\!\int_D \left(\fd{\tau}g_{\theta_\tau}^N(t,x)\right)\uNhatv{\theta_\tau}{t,x}\, dxdt,
    \end{split}
    \end{equation}
    where the individual steps are analogous to the ones taken in \eqref{eq:proof:parabolicdtJ*1} in the proof of \Cref{lem:parabolictimeevolutionJt*}.
    
    Now, recalling the representation of the right-hand side~$g_{\theta_\tau}^N$ from \eqref{eq:parabolicgNtau} and taking its training time derivative to obtain $\fd{\tau} g_{\theta_\tau}^N = - \alpha_\tau T_{B(\mu^N_\tau)}\uNhat{\theta_\tau}$, as well as recalling the definition of the operator $T_{B(\mu^N_\tau)}$ from \eqref{eq:parabolicT_B},
    we can continue \eqref{eq:proof:parabolicdtJN1} to obtain
    \begin{equation}
    \begin{split}
        \fd{\tau}\JN{\theta_\tau}
        &= - \alpha_\tau \int_0^T\!\!\!\int_D \uNhatv{\theta_\tau}{t,x} \int_0^T\!\!\!\int_D B(t,x,t',x';\mu^N_\tau)\uNhatv{\theta_\tau}{t',x'} \,dx'dt'dxdt \\
        &= - \alpha_\tau (\uNhat{\theta_\tau},T_{B(\mu^N_\tau)}\uNhat{\theta_\tau})_{L_2(D_T)},
    \end{split}
    \end{equation}
    which concludes the first part of the proof.

    The second part now follows immediately thanks to the operator $T_{B(\mu^N_\tau)}$ being positive semi-definite for every $\tau\in I$,
    as can be seen by noting that
    with computations analogous to \eqref{eq:proof:lem:parabolicposdefTB:1}
    it holds
    \begin{equation}       
    \begin{split}
        (\widehat{u},T_{B(\mu^N_\tau)}\widehat{u})_{L_2(D_T)}
        &= \int_0^T\!\!\!\int_D \widehat{u}(t,x) \int_0^T\!\!\!\int_D B(t,x,t',x';\mu^N_\tau)\widehat{u}(t',x') \,dx'dt'\,dxdt\\
        &\geq \int \left(\int_0^T\!\!\!\int_D\sigma(w^tt+w^Tx+\eta)\widehat{u}(t,x)\,dxdt\right)^2d\mu^N_{\tau,(w^t,w,\eta)}(w^t,w,\eta) \\
        &\geq0,
    \end{split}
    \end{equation}
    which concludes the proof.
\end{proof}

An immediate consequence of the loss $\JN{\theta_\tau}$ being monotonically non-increasing,
are uniform (in the training time $\tau$) bounds on the $L_2$ norm of the PDE solution $\uN{\theta_\tau}$ and the $L_2([0,T],H^1(D))$ and $L_\infty([0,T],L_2(D))$ norm of the adjoint $\uNhat{\theta_\tau}$.

\begin{lemma}
    \label{lem:parabolic_uNL2}
    Assume that the learning rate satisfies additionally $\int_{0}^\infty\alpha_\tau^{4/3}\,d\tau<\infty$.
    Let $((\uN{\theta_\tau},\uNhat{\theta_\tau}))_{\tau\in I}\in \CC\left(I,\CS\times\CS\right)$ denote the unique weak solution to the PDE system \eqref{eq:parabolicPDEN_plain}\,\&\,\eqref{eq:parabolicadjointN_plain} coupled with the gradient descent update~\eqref{eq:GD} in the sense of \Cref{lem:parabolic_wellposedness_N} on the training time interval $I$.
    Then the solution~$\uN{\theta_\tau}$ is uniformly (in the training time $\tau$) bounded in $L_2(D_T)$ on that interval $I$, i.e., it holds
    \begin{equation}
        \sup_{\tau\in I}\N{\uN{\theta_\tau}}_{L_2(D_T)}\leq C^{u^N}
    \end{equation}
    for the constant $C^{u^N}=4\JN{\theta_0} + 2\N{h}_{L_2(D_T)}^2$.
\end{lemma}

\begin{proof}
    Using \Cref{lem:parabolictimeevolutionJtN}, the proof follows the one of \Cref{lem:parabolic_uL2}.
\end{proof}

\begin{lemma}
    \label{lem:parabolic_uNhatL2}
    Assume that the learning rate satisfies additionally $\int_{0}^\infty\alpha_\tau^{4/3}\,d\tau<\infty$.
    Let $((\uN{\theta_\tau},\uNhat{\theta_\tau}))_{\tau\in I}\in \CC\left(I,\CS\times\CS\right)$ denote the unique weak solution to the PDE system \eqref{eq:parabolicPDEN_plain}\,\&\,\eqref{eq:parabolicadjointN_plain} coupled with the gradient descent update~\eqref{eq:GD} in the sense of \Cref{lem:parabolic_wellposedness_N} on the training time interval $I$.
    Then the adjoint~$\uNhat{\theta_\tau}$ in \eqref{eq:parabolicadjointN_plain} is uniformly (in the training time $\tau$) bounded in $L_2([0,T],H^1(D))$ and $L_\infty([0,T],L_2(D))$ on that interval $I$,
    i.e., it holds
    \begin{equation}
        \sup_{\tau\in I} \left(\N{\uNhat{\theta_\tau}}_{L_2([0,T],H^1(D))} + \N{\uNhat{\theta_\tau}}_{L_\infty([0,T],L_2(D))} \right)
        \leq C^{\widehat{u}^N}
    \end{equation}
    for a constant $C^{\widehat{u}^N}=C^{\widehat{u}^N}(T,\CL,\J{0})$.
\end{lemma}

\begin{proof}
    Using \Cref{lem:parabolictimeevolutionJtN}, the proof follows the one of \Cref{lem:parabolic_uhatL2}.
\end{proof}

\addcontentsline{toc}{section}{References}
\bibliographystyle{abbrv}
\bibliography{biblio.bib}

@book {brezis2011functional,
    AUTHOR = {Brezis, Haim},
     TITLE = {Functional analysis, {S}obolev spaces and partial differential
              equations},
    SERIES = {Universitext},
 PUBLISHER = {Springer, New York},
      YEAR = {2011},
     PAGES = {xiv+599},
      isbn_deac = {978-0-387-70913-0},
   MRCLASS = {35-01 (46-01 46E35 46N20 47F05)},
  MRNUMBER = {2759829},
MRREVIEWER = {Vicen\c tiu\ D.\ R\u adulescu},
}

@book {evans2010partial,
    AUTHOR = {Evans, Lawrence C.},
     TITLE = {Partial differential equations},
    SERIES = {Graduate Studies in Mathematics},
    VOLUME = {19},
   EDITION = {Second},
 PUBLISHER = {American Mathematical Society, Providence, RI},
      YEAR = {2010},
     PAGES = {xxii+749},
      isbn_deac = {978-0-8218-4974-3},
   MRCLASS = {35-01},
  MRNUMBER = {2597943},
MRREVIEWER = {Diego\ M.\ Maldonado},
       DOI_deac = {10.1090/gsm/019},
       URL_deac = {https://doi.org/10.1090/gsm/019},
}

@book {ladyzhenskaia1968linear,
    AUTHOR = {Lady\v{z}enskaja, O. A. and Solonnikov, V. A. and Ural'ceva, N. N.},
     TITLE = {Linear and quasilinear equations of parabolic type},
    SERIES = {Translations of Mathematical Monographs},
    VOLUME = {Vol. 23},
      NOTE = {Translated from the Russian by S. Smith},
 PUBLISHER = {American Mathematical Society, Providence, RI},
      YEAR = {1968},
     PAGES = {xi+648},
   MRCLASS = {35.62},
  MRNUMBER = {241822},
MRREVIEWER = {B.\ Frank\ Jones, Jr.},
}

@article {bertsekas2000gradient,
    AUTHOR = {Bertsekas, Dimitri P. and Tsitsiklis, John N.},
     TITLE = {Gradient convergence in gradient methods with errors},
   JOURNAL = {SIAM J. Optim.},
  FJOURNAL = {SIAM Journal on Optimization},
    VOLUME = {10},
      YEAR = {2000},
    NUMBER = {3},
     PAGES = {627--642},
      ISSN_deac = {1052-6234,1095-7189},
   MRCLASS = {90C52 (62L20 90C30)},
  MRNUMBER = {1741189},
MRREVIEWER = {G.\ Pflug},
       DOI_deac = {10.1137/S1052623497331063},
       URL_deac = {https://doi.org/10.1137/S1052623497331063},
}

@article{sirignano2018dgm,
    AUTHOR = {Sirignano, Justin and Spiliopoulos, Konstantinos},
     TITLE = {D{GM}: a deep learning algorithm for solving partial
              differential equations},
   JOURNAL = {J. Comput. Phys.},
  FJOURNAL = {Journal of Computational Physics},
    VOLUME = {375},
      YEAR = {2018},
     PAGES = {1339--1364},
      ISSN_deac = {0021-9991,1090-2716},
   MRCLASS = {65M99 (68T05)},
  MRNUMBER = {3874585},
       DOI_deac = {10.1016/j.jcp.2018.08.029},
       URL_deac = {https://doi.org/10.1016/j.jcp.2018.08.029},
}

@article {sirignano2023pde,
    AUTHOR = {Sirignano, Justin and MacArt, Jonathan and Spiliopoulos, Konstantinos},
     TITLE = {P{DE}-constrained models with neural network terms:
              optimization and global convergence},
   JOURNAL = {J. Comput. Phys.},
  FJOURNAL = {Journal of Computational Physics},
    VOLUME = {481},
      YEAR = {2023},
     PAGES = {Paper No. 112016, 35},
      ISSN_deac = {0021-9991,1090-2716},
   MRCLASS = {65M06 (65K10 68T07)},
  MRNUMBER = {4559355},
       DOI_deac = {10.1016/j.jcp.2023.112016},
       URL_deac = {https://doi.org/10.1016/j.jcp.2023.112016},
}

@article {sirignano2022online,
    AUTHOR = {Sirignano, Justin and Spiliopoulos, Konstantinos},
     TITLE = {Online adjoint methods for optimization of {PDE}s},
   JOURNAL = {Appl. Math. Optim.},
  FJOURNAL = {Applied Mathematics and Optimization},
    VOLUME = {85},
      YEAR = {2022},
    NUMBER = {2},
     PAGES = {Paper No. 18, 29},
      ISSN_deac = {0095-4616,1432-0606},
   MRCLASS = {49M41 (65M06 68W27)},
  MRNUMBER = {4409810},
MRREVIEWER = {Wei\ Wei},
       DOI_deac = {10.1007/s00245-022-09852-5},
       URL_deac = {https://doi.org/10.1007/s00245-022-09852-5},
}

@article{sirignano2020dpm,
    AUTHOR = {Sirignano, Justin and MacArt, Jonathan F. and Freund, Jonathan
              B.},
     TITLE = {D{PM}: a deep learning {PDE} augmentation method with
              application to large-eddy simulation},
   JOURNAL = {J. Comput. Phys.},
  FJOURNAL = {Journal of Computational Physics},
    VOLUME = {423},
      YEAR = {2020},
     PAGES = {109811, 21},
      ISSN_deac = {0021-9991,1090-2716},
   MRCLASS = {76F65 (68T05)},
  MRNUMBER = {4156941},
       DOI_deac = {10.1016/j.jcp.2020.109811},
       URL_deac = {https://doi.org/10.1016/j.jcp.2020.109811},
}

@article{hornik1991approximation,
  author       = {Kurt Hornik},
  title        = {Approximation capabilities of multilayer feedforward networks},
  journal      = {Neural Networks},
  volume       = {4},
  number       = {2},
  pages        = {251--257},
  year         = {1991},
  url_deac          = {https://doi.org/10.1016/0893-6080(91)90009-T},
  doi_deac          = {10.1016/0893-6080(91)90009-T},
  biburl_deac       = {https://dblp.org/rec/journals/nn/Hornik91.bib},
  bibsource    = {dblp computer science bibliography, https://dblp.org}
}

@article{cybenko1989approximation,
  author       = {George Cybenko},
  title        = {Approximation by superpositions of a sigmoidal function},
  journal      = {Math. Control. Signals Syst.},
  volume       = {2},
  number       = {4},
  pages        = {303--314},
  year         = {1989},
  url_deac          = {https://doi.org/10.1007/BF02551274},
  doi_deac          = {10.1007/BF02551274},
  biburl_deac       = {https://dblp.org/rec/journals/mcss/Cybenko89.bib},
  bibsource    = {dblp computer science bibliography, https://dblp.org}
}

@incollection {bosse2014one,
    AUTHOR = {Bosse, Torsten and Gauger, Nicolas R. and Griewank, Andreas
              and G\"unther, Stefanie and Schulz, Volker},
     TITLE = {One-shot approaches to design optimization},
 BOOKTITLE = {Trends in {PDE} constrained optimization},
    SERIES = {Internat. Ser. Numer. Math.},
    VOLUME = {165},
     PAGES = {43--66},
 PUBLISHER = {Birkh\"auser/Springer, Cham},
      YEAR = {2014},
      isbn_deac = {978-3-319-05082-9; 978-3-319-05083-6},
   MRCLASS = {49M37 (65K10 90C30)},
  MRNUMBER = {3328970},
       DOI_deac = {10.1007/978-3-319-05083-6\_5},
       URL_deac = {https://doi.org/10.1007/978-3-319-05083-6_5},
}

@incollection {brandenburg2009continuous,
    AUTHOR = {Brandenburg, Christian and Lindemann, Florian and Ulbrich,
              Michael and Ulbrich, Stefan},
     TITLE = {A continuous adjoint approach to shape optimization for
              {N}avier {S}tokes flow},
 BOOKTITLE = {Optimal control of coupled systems of partial differential
              equations},
    SERIES = {Internat. Ser. Numer. Math.},
    VOLUME = {158},
     PAGES = {35--56},
 PUBLISHER = {Birkh\"auser Verlag, Basel},
      YEAR = {2009},
      isbn_deac = {978-3-7643-8922-2},
   MRCLASS = {49Q10 (49K20 76D07)},
  MRNUMBER = {2588547},
MRREVIEWER = {Igor\ Bock},
       DOI_deac = {10.1007/978-3-7643-8923-9\_2},
       URL_deac = {https://doi.org/10.1007/978-3-7643-8923-9_2},
}

@article{bueno2012continuous,
  title={Continuous adjoint approach for the Spalart-Allmaras model in aerodynamic optimization},
  author={Bueno-Orovio, Alfonso and Castro, Carlos and Palacios, Francisco and Zuazua, Enrique},
  journal={AIAA journal},
  volume={50},
  number={3},
  pages={631--646},
  year={2012}
}

@incollection {gauger2012automated,
    AUTHOR = {Gauger, Nicolas and Griewank, Andreas and Hamdi, Adel and
              Kratzenstein, Claudia and \"Ozkaya, Emre and Slawig, Thomas},
     TITLE = {Automated extension of fixed point {PDE} solvers for optimal
              design with bounded retardation},
 BOOKTITLE = {Constrained optimization and optimal control for partial
              differential equations},
    SERIES = {Internat. Ser. Numer. Math.},
    VOLUME = {160},
     PAGES = {99--122},
 PUBLISHER = {Birkh\"auser/Springer Basel AG, Basel},
      YEAR = {2012},
      isbn_deac = {978-3-0348-0132-4; 978-3-0348-0133-1},
   MRCLASS = {47J25 (49Q10 65N22 90C30)},
  MRNUMBER = {3060470},
       DOI_deac = {10.1007/978-3-0348-0133-1\_6},
       URL_deac = {https://doi.org/10.1007/978-3-0348-0133-1_6},
}

@article{giles2000introduction,
  title={An introduction to the adjoint approach to design},
  author={Giles, Michael B and Pierce, Niles A},
  journal={Flow, turbulence and combustion},
  volume={65},
  pages={393--415},
  year={2000},
  publisher={Springer}
}

@article {giles2010convergence1,
    AUTHOR = {Giles, Mike and Ulbrich, Stefan},
     TITLE = {Convergence of linearized and adjoint approximations for
              discontinuous solutions of conservation laws. {P}art 1:
              {L}inearized approximations and linearized output functionals},
   JOURNAL = {SIAM J. Numer. Anal.},
  FJOURNAL = {SIAM Journal on Numerical Analysis},
    VOLUME = {48},
      YEAR = {2010},
    NUMBER = {3},
     PAGES = {882--904},
      ISSN_deac = {0036-1429,1095-7170},
   MRCLASS = {65M06 (65M12)},
  MRNUMBER = {2669394},
MRREVIEWER = {Alessandra\ Jannelli},
       DOI_deac = {10.1137/080727464},
       URL_deac = {https://doi.org/10.1137/080727464},
}

@article {giles2010convergence2,
    AUTHOR = {Giles, Mike and Ulbrich, Stefan},
     TITLE = {Convergence of linearized and adjoint approximations for
              discontinuous solutions of conservation laws. {P}art 2:
              {A}djoint approximations and extensions},
   JOURNAL = {SIAM J. Numer. Anal.},
  FJOURNAL = {SIAM Journal on Numerical Analysis},
    VOLUME = {48},
      YEAR = {2010},
    NUMBER = {3},
     PAGES = {905--921},
      ISSN_deac = {0036-1429,1095-7170},
   MRCLASS = {65M06 (65M12 76L05)},
  MRNUMBER = {2669395},
MRREVIEWER = {Alessandra\ Jannelli},
       DOI_deac = {10.1137/09078078X},
       URL_deac = {https://doi.org/10.1137/09078078X},
}

@article{knopoff2013adjoint,
  title={Adjoint method for a tumor growth {PDE}-constrained optimization problem},
  author={Knopoff, Dami{\'a}n A and Fern{\'a}ndez, Dami{\'a}n R and Torres, Germ{\'a}n Ariel and Turner, Cristina Vilma},
  journal={Computers \& Mathematics with Applications},
  volume={66},
  number={6},
  pages={1104--1119},
  year={2013},
  publisher={Elsevier}
}

@article {kochkov2021machine,
    AUTHOR = {Kochkov, Dmitrii and Smith, Jamie A. and Alieva, Ayya and
              Wang, Qing and Brenner, Michael P. and Hoyer, Stephan},
     TITLE = {Machine learning-accelerated computational fluid dynamics},
   JOURNAL = {Proc. Natl. Acad. Sci. USA},
  FJOURNAL = {Proceedings of the National Academy of Sciences of the United
              States of America},
    VOLUME = {118},
      YEAR = {2021},
    NUMBER = {21},
     PAGES = {Paper No. e2101784118, 8},
      ISSN_deac = {0027-8424,1091-6490},
   MRCLASS = {76-10},
  MRNUMBER = {4301254},
       DOI_deac = {10.1073/pnas.2101784118},
       URL_deac = {https://doi.org/10.1073/pnas.2101784118},
}

@article{hazra2007direct,
  title={Direct treatment of state constraints in aerodynamic shape optimization using simultaneous pseudo-time-stepping},
  author={Hazra, Subhendu Bikash},
  journal={AIAA journal},
  volume={45},
  number={8},
  pages={1988--1997},
  year={2007}
}

@article {sirignano2023deep,
    AUTHOR = {Sirignano, Justin and MacArt, Jonathan F.},
     TITLE = {Deep learning closure models for large-eddy simulation of
              flows around bluff bodies},
   JOURNAL = {J. Fluid Mech.},
  FJOURNAL = {Journal of Fluid Mechanics},
    VOLUME = {966},
      YEAR = {2023},
     PAGES = {Paper No. A26, 23},
      ISSN_deac = {0022-1120,1469-7645},
   MRCLASS = {76F65},
  MRNUMBER = {4610170},
MRREVIEWER = {Antonella\ Abb\`a},
       DOI_deac = {10.1017/jfm.2023.446},
       URL_deac = {https://doi.org/10.1017/jfm.2023.446},
}

@inproceedings{hickling2024large,
  title={Large Eddy Simulation of Airfoil Flows Using Adjoint-Trained Deep Learning Closure Models},
  author={Hickling, Tom and Sirignano, Justin and MacArt, Jonathan F},
  booktitle={AIAA SCITECH 2024 Forum},
  pages={0296},
  year={2024}
}

@article{kakka2025neural,
  title={Neural network-augmented eddy viscosity closures for turbulent premixed jet flames},
  author={Kakka, Priyesh and MacArt, Jonathan F},
  journal={arXiv preprint arXiv:2503.03880},
  year={2025}
}

@article{duraisamy2019turbulence,
  title={Turbulence modeling in the age of data},
  author={Duraisamy, Karthik and Iaccarino, Gianluca and Xiao, Heng},
  journal={Annual review of fluid mechanics},
  volume={51},
  number={1},
  pages={357--377},
  year={2019},
  publisher={Annual Reviews}
}

@article{brunton2020machine,
  title={Machine learning for fluid mechanics},
  author={Brunton, Steven L and Noack, Bernd R and Koumoutsakos, Petros},
  journal={Annual review of fluid mechanics},
  volume={52},
  number={1},
  pages={477--508},
  year={2020},
  publisher={Annual Reviews}
}

@article {cagnetti2013adjoint,
    AUTHOR = {Cagnetti, Filippo and Gomes, Diogo and Tran, Hung Vinh},
     TITLE = {Adjoint methods for obstacle problems and weakly coupled
              systems of {PDE}},
   JOURNAL = {ESAIM Control Optim. Calc. Var.},
  FJOURNAL = {ESAIM. Control, Optimisation and Calculus of Variations},
    VOLUME = {19},
      YEAR = {2013},
    NUMBER = {3},
     PAGES = {754--779},
      ISSN_deac = {1292-8119,1262-3377},
   MRCLASS = {35F30 (35B25 49L25)},
  MRNUMBER = {3092361},
       DOI_deac = {10.1051/cocv/2012032},
       URL_deac = {https://doi.org/10.1051/cocv/2012032},
}

@article {kaland2014one,
    AUTHOR = {Kaland, L. and De Los Reyes, J. C. and Gauger, N. R.},
     TITLE = {One-shot methods in function space for {PDE}-constrained
              optimal control problems},
   JOURNAL = {Optim. Methods Softw.},
  FJOURNAL = {Optimization Methods \& Software},
    VOLUME = {29},
      YEAR = {2014},
    NUMBER = {2},
     PAGES = {376--405},
      ISSN_deac = {1055-6788,1029-4937},
   MRCLASS = {49M25 (93C20)},
  MRNUMBER = {3175493},
       DOI_deac = {10.1080/10556788.2013.774397},
       URL_deac = {https://doi.org/10.1080/10556788.2013.774397},
}

@article {pierce2000adjoint,
    AUTHOR = {Pierce, Niles A. and Giles, Michael B.},
     TITLE = {Adjoint recovery of superconvergent functionals from {PDE}
              approximations},
   JOURNAL = {SIAM Review},
  FJOURNAL = {SIAM Review},
    VOLUME = {42},
      YEAR = {2000},
    NUMBER = {2},
     PAGES = {247--264},
      ISSN_deac = {0036-1445,1095-7200},
   MRCLASS = {65N12 (76M12 76M20)},
  MRNUMBER = {1778357},
MRREVIEWER = {Dennis\ C.\ Jespersen},
       DOI_deac = {10.1137/S0036144598349423},
       URL_deac = {https://doi.org/10.1137/S0036144598349423},
}

@article{macart2021embedded,
  title={Embedded training of neural-network subgrid-scale turbulence models},
  author={MacArt, Jonathan F and Sirignano, Justin and Freund, Jonathan B},
  journal={Physical Review Fluids},
  volume={6},
  number={5},
  pages={050502},
  year={2021},
  publisher={APS}
}

@article{hazra2004simultaneous,
  title={Simultaneous pseudo-timestepping for {PDE}-model based optimization problems},
  author={Hazra, Subhendu Bikash and Schulz, Volker},
  journal={Bit Numerical Mathematics},
  volume={44},
  pages={457--472},
  year={2004},
  publisher={Springer}
}

@book{hinze2008optimization,
  title={Optimization with {PDE} constraints},
  author={Hinze, Michael and Pinnau, Ren{\'e} and Ulbrich, Michael and Ulbrich, Stefan},
  volume={23},
  year={2008},
  publisher={Springer Science \& Business Media}
}

@inproceedings{nair2024adjoint,
  title={Adjoint-Trained Deep-Learning Closures of the {N}avier--{S}tokes Equations for 2D Nonequilibrium Flows},
  author={Nair, Ashish S and Waidmann, Den and Sirignano, Justin and Singh, Narendra and Panesi, Marco and MacArt, Jonathan F},
  booktitle={AIAA SCITECH 2024 Forum},
  pages={2860},
  year={2024}
}

@article{nair2023deep,
  title={Deep learning closure of the {N}avier--{S}tokes equations for transition-continuum flows},
  author={Nair, Ashish S and Sirignano, Justin and Panesi, Marco and MacArt, Jonathan F},
  journal={AIAA journal},
  volume={61},
  number={12},
  pages={5484--5497},
  year={2023},
  publisher={American Institute of Aeronautics and Astronautics}
}

@inproceedings{holland2019towards,
  title={Towards integrated field inversion and machine learning with embedded neural networks for RANS modeling},
  author={Holland, Jonathan R and Baeder, James D and Duraisamy, Karthik},
  booktitle={AIAA Scitech 2019 forum},
  pages={1884},
  year={2019}
}

@article{srivastava2021generalizable,
  title={Generalizable physics-constrained modeling using learning and inference assisted by feature-space engineering},
  author={Srivastava, Vishal and Duraisamy, Karthik},
  journal={Physical Review Fluids},
  volume={6},
  number={12},
  pages={124602},
  year={2021},
  publisher={APS}
}

@article{duraisamy2021perspectives,
  title={Perspectives on machine learning-augmented {R}eynolds-averaged and large eddy simulation models of turbulence},
  author={Duraisamy, Karthik},
  journal={Physical Review Fluids},
  volume={6},
  number={5},
  pages={050504},
  year={2021},
  publisher={APS}
}

@article{duta2002harmonic,
  title={The harmonic adjoint approach to unsteady turbomachinery design},
  author={Duta, MC and Giles, MB and Campobasso, MS},
  journal={International Journal for Numerical Methods in Fluids},
  volume={40},
  number={3-4},
  pages={323--332},
  year={2002},
  publisher={Wiley Online Library}
}

@article{jameson2003aerodynamic,
  title={Aerodynamic shape optimization using the adjoint method},
  author={Jameson, Antony},
  journal={Lectures at the Von Karman Institute, Brussels},
  volume={6},
  year={2003}
}

@article{jameson1998optimum,
  title={Optimum aerodynamic design using the {N}avier--{S}tokes equations},
  author={Jameson, Antony and Martinelli, Luigi and Pierce, Niles A},
  journal={Theoretical and computational fluid dynamics},
  volume={10},
  number={1},
  pages={213--237},
  year={1998},
  publisher={Springer}
}

@inproceedings{jameson2003reduction,
  title={Reduction of the adjoint gradient formula in the continuous limit},
  author={Jameson, Antony and Kim, Sangho},
  booktitle={41st Aerospace Sciences Meeting and Exhibit},
  pages={40},
  year={2003}
}

@inproceedings{nadarajah2000comparison,
  title={A comparison of the continuous and discrete adjoint approach to automatic aerodynamic optimization},
  author={Nadarajah, Siva and Jameson, Antony},
  booktitle={38th Aerospace sciences meeting and exhibit},
  pages={667},
  year={2000}
}

@inproceedings{nadarajah2001studies,
  title={Studies of the continuous and discrete adjoint approaches to viscous automatic aerodynamic shape optimization},
  author={Nadarajah, Siva and Jameson, Antony},
  booktitle={15th AIAA computational fluid dynamics conference},
  pages={2530},
  year={2001}
}

@inproceedings{reuther1996aerodynamic,
  title={Aerodynamic shape optimization of complex aircraft configurations via an adjoint formulation},
  author={Reuther, James and Jameson, Antony and Farmer, James and Martinelli, Luigi and Saunders, David},
  booktitle={34th aerospace sciences meeting and exhibit},
  pages={94},
  year={1996}
}

@article{gierjatowicz2020robust,
  title={Robust pricing and hedging via neural {SDEs}},
  author={Gierjatowicz, Patryk and Sabate-Vidales, Marc and {\v{S}}i{\v{s}}ka, David and Szpruch, Lukasz and {\v{Z}}uri{\v{c}}, {\v{Z}}an},
  journal={arXiv preprint arXiv:2007.04154},
  year={2020}
}

@article{fan2024machine,
  title={Machine Learning Methods for Pricing Financial Derivatives},
  author={Fan, Lei and Sirignano, Justin},
  journal={arXiv preprint arXiv:2406.00459},
  year={2024}
}

@article{cohen2023arbitrage,
  title={Arbitrage-free neural-{SDE} market models},
  author={Cohen, Samuel N and Reisinger, Christoph and Wang, Sheng},
  journal={Applied Mathematical Finance},
  volume={30},
  number={1},
  pages={1--46},
  year={2023},
  publisher={Taylor \& Francis}
}

@article{goswami2021data,
  title={Data-driven option pricing using single and multi-asset supervised learning},
  author={Goswami, Anindya and Rajani, Sharan and Tanksale, Atharva},
  journal={International Journal of Financial Engineering},
  volume={8},
  number={02},
  pages={2141001},
  year={2021},
  publisher={World Scientific}
}

@article{dong2024descent,
  title={A descent algorithm for the optimal control of ReLU neural network informed {PDEs} based on approximate directional derivatives},
  author={Dong, Guozhi and Hinterm{\"u}ller, Michael and Papafitsoros, Kostas},
  journal={SIAM Journal on Optimization},
  volume={34},
  number={3},
  pages={2314--2349},
  year={2024},
  publisher={SIAM}
}

@article {dong2022optimization,
    AUTHOR = {Dong, Guozhi and Hinterm\"uller, Michael and Papafitsoros,
              Kostas},
     TITLE = {Optimization with learning-informed differential equation
              constraints and its applications},
   JOURNAL = {ESAIM Control Optim. Calc. Var.},
  FJOURNAL = {ESAIM. Control, Optimisation and Calculus of Variations},
    VOLUME = {28},
      YEAR = {2022},
     PAGES = {Paper No. 3, 44},
      ISSN_deac = {1292-8119,1262-3377},
   MRCLASS = {49M15 (35J61 49J27 49M41 65J15 65J20 65K05 65K10)},
  MRNUMBER = {4362194},
MRREVIEWER = {Daniela\ Fu\ss eder},
       DOI_deac = {10.1051/cocv/2021100},
       URL_deac = {https://doi.org/10.1051/cocv/2021100},
}

@inproceedings{jacot2018neural,
  author       = {Arthur Jacot and
                  Cl{\'{e}}ment Hongler and
                  Franck Gabriel},
  editor       = {Samy Bengio and
                  Hanna M. Wallach and
                  Hugo Larochelle and
                  Kristen Grauman and
                  Nicol{\`{o}} Cesa{-}Bianchi and
                  Roman Garnett},
  title        = {Neural Tangent Kernel: Convergence and Generalization in Neural Networks},
  booktitle    = {Advances in Neural Information Processing Systems 31: Annual Conference
                  on Neural Information Processing Systems 2018, NeurIPS 2018, December
                  3-8, 2018, Montr{\'{e}}al, Canada},
  pages        = {8580--8589},
  year         = {2018},
  url_deac          = {https://proceedings.neurips.cc/paper/2018/hash/5a4be1fa34e62bb8a6ec6b91d2462f5a-Abstract.html},
  biburl_deac       = {https://dblp.org/rec/conf/nips/JacotHG18.bib},
  bibsource    = {dblp computer science bibliography, https://dblp.org}
}

@inproceedings{chizat2019lazy,
  author       = {L{\'{e}}na{\"{\i}}c Chizat and
                  Edouard Oyallon and
                  Francis R. Bach},
  editor       = {Hanna M. Wallach and
                  Hugo Larochelle and
                  Alina Beygelzimer and
                  Florence d'Alch{\'{e}}{-}Buc and
                  Emily B. Fox and
                  Roman Garnett},
  title        = {On Lazy Training in Differentiable Programming},
  booktitle    = {Advances in Neural Information Processing Systems 32: Annual Conference
                  on Neural Information Processing Systems 2019, NeurIPS 2019, December
                  8-14, 2019, Vancouver, BC, Canada},
  pages        = {2933--2943},
  year         = {2019},
  url_deac          = {https://proceedings.neurips.cc/paper/2019/hash/ae614c557843b1df326cb29c57225459-Abstract.html},
  biburl_deac       = {https://dblp.org/rec/conf/nips/ChizatOB19.bib},
  bibsource    = {dblp computer science bibliography, https://dblp.org}
}

@article{lecun2015deep,
  author       = {Yann LeCun and
                  Yoshua Bengio and
                  Geoffrey E. Hinton},
  title        = {Deep learning},
  journal      = {Nat.},
  volume       = {521},
  number       = {7553},
  pages        = {436--444},
  year         = {2015},
  url_deac          = {https://doi.org/10.1038/nature14539},
  doi_deac          = {10.1038/NATURE14539},
  biburl_deac       = {https://dblp.org/rec/journals/nature/LeCunBH15.bib},
  bibsource    = {dblp computer science bibliography, https://dblp.org}
}

@article{jumper2021highly,
  title={Highly accurate protein structure prediction with {A}lpha{F}old},
  author={Jumper, John and Evans, Richard and Pritzel, Alexander and Green, Tim and Figurnov, Michael and Ronneberger, Olaf and Tunyasuvunakool, Kathryn and Bates, Russ and {\v{Z}}{\'\i}dek, Augustin and Potapenko, Anna and others},
  journal={Nature},
  volume={596},
  number={7873},
  pages={583--589},
  year={2021},
  publisher={Nature Publishing Group}
}

@incollection{boulle2023mathematical,
title = {Chapter 3 -- {A} mathematical guide to operator learning},
editor = {Siddhartha Mishra and Alex Townsend},
series = {Handbook of Numerical Analysis},
publisher = {Elsevier},
volume = {25},
pages = {83-125},
year = {2024},
booktitle = {Numerical Analysis Meets Machine Learning},
issn_dead = {1570-8659},
doi_deac = {https://doi.org/10.1016/bs.hna.2024.05.003},
url_deac = {https://www.sciencedirect.com/science/article/pii/S1570865924000036},
author = {Nicolas Boullé and Alex Townsend}
}

@article{karniadakis2021physics,
  title={Physics-informed machine learning},
  author={Karniadakis, George Em and Kevrekidis, Ioannis G and Lu, Lu and Perdikaris, Paris and Wang, Sifan and Yang, Liu},
  journal={Nature Reviews Physics},
  volume={3},
  number={6},
  pages={422--440},
  year={2021},
  publisher={Nature Publishing Group UK London}
}

@article {yu2018deep,
    AUTHOR = {E, Weinan and Yu, Bing},
     TITLE = {The deep {R}itz method: a deep learning-based numerical
              algorithm for solving variational problems},
   JOURNAL = {Commun. Math. Stat.},
  FJOURNAL = {Communications in Mathematics and Statistics},
    VOLUME = {6},
      YEAR = {2018},
    NUMBER = {1},
     PAGES = {1--12},
      ISSN_deac = {2194-6701,2194-671X},
   MRCLASS = {35J20 (35J91 65N30)},
  MRNUMBER = {3767958},
       DOI_deac = {10.1007/s40304-018-0127-z},
       URL_deac = {https://doi.org/10.1007/s40304-018-0127-z},
}

@article{raissi2019physics,
  title={Physics-informed neural networks: A deep learning framework for solving forward and inverse problems involving nonlinear partial differential equations},
  author={Raissi, Maziar and Perdikaris, Paris and Karniadakis, George E},
  journal={Journal of Computational physics},
  volume={378},
  pages={686--707},
  year={2019},
  publisher={Elsevier}
}

@article{lu2021deepxde,
  title={DeepXDE: A deep learning library for solving differential equations},
  author={Lu, Lu and Meng, Xuhui and Mao, Zhiping and Karniadakis, George Em},
  journal={SIAM Review},
  volume={63},
  number={1},
  pages={208--228},
  year={2021},
  publisher={SIAM}
}

@article{wang2023expert,
  title={An expert's guide to training physics-informed neural networks},
  author={Wang, Sifan and Sankaran, Shyam and Wang, Hanwen and Perdikaris, Paris},
  journal={arXiv preprint arXiv:2308.08468},
  year={2023}
}

@inproceedings{vaswani2017attention,
  author       = {Ashish Vaswani and
                  Noam Shazeer and
                  Niki Parmar and
                  Jakob Uszkoreit and
                  Llion Jones and
                  Aidan N. Gomez and
                  Lukasz Kaiser and
                  Illia Polosukhin},
  editor       = {Isabelle Guyon and
                  Ulrike von Luxburg and
                  Samy Bengio and
                  Hanna M. Wallach and
                  Rob Fergus and
                  S. V. N. Vishwanathan and
                  Roman Garnett},
  title        = {Attention is All you Need},
  booktitle    = {Advances in Neural Information Processing Systems 30: Annual Conference
                  on Neural Information Processing Systems 2017, December 4-9, 2017,
                  Long Beach, CA, {USA}},
  pages        = {5998--6008},
  year         = {2017},
  url_deac          = {https://proceedings.neurips.cc/paper/2017/hash/3f5ee243547dee91fbd053c1c4a845aa-Abstract.html},
  biburl_deac       = {https://dblp.org/rec/conf/nips/VaswaniSPUJGKP17.bib},
  bibsource    = {dblp computer science bibliography, https://dblp.org}
}

@inproceedings{krizhevsky2012imagenet,
  author       = {Alex Krizhevsky and
                  Ilya Sutskever and
                  Geoffrey E. Hinton},
  editor       = {Peter L. Bartlett and
                  Fernando C. N. Pereira and
                  Christopher J. C. Burges and
                  L{\'{e}}on Bottou and
                  Kilian Q. Weinberger},
  title        = {ImageNet Classification with Deep Convolutional Neural Networks},
  booktitle    = {Advances in Neural Information Processing Systems 25: 26th Annual
                  Conference on Neural Information Processing Systems 2012. Proceedings
                  of a meeting held December 3-6, 2012, Lake Tahoe, Nevada, United States},
  pages        = {1106--1114},
  year         = {2012},
  url_deac          = {https://proceedings.neurips.cc/paper/2012/hash/c399862d3b9d6b76c8436e924a68c45b-Abstract.html},
  biburl_deac       = {https://dblp.org/rec/conf/nips/KrizhevskySH12.bib},
  bibsource    = {dblp computer science bibliography, https://dblp.org}
}

@article{hinton2012deep,
  title        = {Deep Neural Networks for Acoustic Modeling in Speech Recognition:
                  The Shared Views of Four Research Groups},
  author = {Hinton, Geoffrey and Deng, Li and Yu, Dong and Dahl, George E and Mohamed, Abdel-rahman and Jaitly, Navdeep and Senior, Andrew and Vanhoucke, Vincent and Nguyen, Patrick and Sainath, Tara N and Kingsbury, Brian},
  journal      = {{IEEE} Signal Process. Mag.},
  volume       = {29},
  number       = {6},
  pages        = {82--97},
  year         = {2012},
  url_deac          = {https://doi.org/10.1109/MSP.2012.2205597},
  doi_deac          = {10.1109/MSP.2012.2205597},
  biburl_deac       = {https://dblp.org/rec/journals/spm/X12a.bib},
  bibsource    = {dblp computer science bibliography, https://dblp.org}
}

@article{lu2021learning,
  author       = {Lu Lu and
                  Pengzhan Jin and
                  Guofei Pang and
                  Zhongqiang Zhang and
                  George Em Karniadakis},
  title        = {Learning nonlinear operators via DeepONet based on the universal approximation
                  theorem of operators},
  journal      = {Nat. Mach. Intell.},
  volume       = {3},
  number       = {3},
  pages        = {218--229},
  year         = {2021},
  url_deac          = {https://doi.org/10.1038/s42256-021-00302-5},
  doi_deac          = {10.1038/S42256-021-00302-5},
  biburl_deac       = {https://dblp.org/rec/journals/natmi/LuJPZK21.bib},
  bibsource    = {dblp computer science bibliography, https://dblp.org}
}

@article{kovachki2023neural,
  author       = {Nikola B. Kovachki and
                  Zongyi Li and
                  Burigede Liu and
                  Kamyar Azizzadenesheli and
                  Kaushik Bhattacharya and
                  Andrew M. Stuart and
                  Anima Anandkumar},
  title        = {Neural Operator: Learning Maps Between Function Spaces With Applications
                  to {PDEs}},
  journal      = {J. Mach. Learn. Res.},
  volume       = {24},
  pages        = {89:1--89:97},
  year         = {2023},
  url_deac          = {https://jmlr.org/papers/v24/21-1524.html},
  biburl_deac       = {https://dblp.org/rec/journals/jmlr/KovachkiLLABSA23.bib},
  bibsource    = {dblp computer science bibliography, https://dblp.org}
}

@article{brunton2016discovering,
  title={Discovering governing equations from data by sparse identification of nonlinear dynamical systems},
  author={Brunton, Steven L and Proctor, Joshua L and Kutz, J Nathan},
  journal={Proceedings of the national academy of sciences},
  volume={113},
  number={15},
  pages={3932--3937},
  year={2016},
  publisher={National Academy of Sciences}
}

@article{champion2019data,
  title={Data-driven discovery of coordinates and governing equations},
  author={Champion, Kathleen and Lusch, Bethany and Kutz, J Nathan and Brunton, Steven L},
  journal={Proceedings of the National Academy of Sciences},
  volume={116},
  number={45},
  pages={22445--22451},
  year={2019},
  publisher={National Academy of Sciences}
}

@article {cuomo2022scientific,
    AUTHOR = {Cuomo, Salvatore and Schiano Di Cola, Vincenzo and Giampaolo,
              Fabio and Rozza, Gianluigi and Raissi, Maziar and Piccialli,
              Francesco},
     TITLE = {Scientific machine learning through physics-informed neural
              networks: where we are and what's next},
   JOURNAL = {J. Sci. Comput.},
  FJOURNAL = {Journal of Scientific Computing},
    VOLUME = {92},
      YEAR = {2022},
    NUMBER = {3},
     PAGES = {Paper No. 88, 62},
      ISSN_deac = {0885-7474,1573-7691},
   MRCLASS = {65M99 (65N99 68T07)},
  MRNUMBER = {4457972},
       DOI_deac = {10.1007/s10915-022-01939-z},
       URL_deac = {https://doi.org/10.1007/s10915-022-01939-z},
}

@article{robbins1951stochastic,
  title={A stochastic approximation method},
  author={Robbins, Herbert and Monro, Sutton},
  journal={The annals of mathematical statistics},
  pages={400--407},
  year={1951},
  publisher={JSTOR}
}

@article{kumar2025zclip,
  title={ZClip: Adaptive Spike Mitigation for LLM Pre-Training},
  author={Kumar, Abhay and Owen, Louis and Chowdhury, Nilabhra Roy and G{\"u}ra, Fabian},
  journal={arXiv preprint arXiv:2504.02507},
  year={2025}
}

@article{zhao2025convergence,
  title={Convergence Guarantees for Gradient-Based Training of Neural {PDE} Solvers: From Linear to Nonlinear {PDEs}},
  author={Zhao, Wei and Luo, Tao},
  journal={arXiv preprint arXiv:2505.14002},
  year={2025}
}

@article {schaeffer2017learning,
    AUTHOR = {Schaeffer, Hayden},
     TITLE = {Learning partial differential equations via data discovery and
              sparse optimization},
   JOURNAL = {Proc. R. Soc. A.},
  FJOURNAL = {Proceedings A},
    VOLUME = {473},
      YEAR = {2017},
    NUMBER = {2197},
     PAGES = {20160446, 20},
      ISSN_deac = {1364-5021,1471-2946},
   MRCLASS = {35Q53 (37M99)},
  MRNUMBER = {3615300},
       DOI_deac = {10.1098/rspa.2016.0446},
       URL_deac = {https://doi.org/10.1098/rspa.2016.0446},
}

@inproceedings{sun2020neupde,
  author       = {Yifan Sun and
                  Linan Zhang and
                  Hayden Schaeffer},
  editor       = {Jianfeng Lu and
                  Rachel A. Ward},
  title        = {{NeuPDE}: Neural Network Based Ordinary and Partial Differential Equations
                  for Modeling Time-Dependent Data},
  booktitle    = {Proceedings of Mathematical and Scientific Machine Learning, {MSML}
                  2020, 20-24 July 2020, Virtual Conference / Princeton, NJ, {USA}},
  series       = {Proceedings of Machine Learning Research},
  volume       = {107},
  pages        = {352--372},
  publisher    = {{PMLR}},
  year         = {2020},
  url_deac          = {http://proceedings.mlr.press/v107/sun20a.html},
  biburl_deac       = {https://dblp.org/rec/conf/msml/SunZS20.bib},
  bibsource    = {dblp computer science bibliography, https://dblp.org}
}

@inproceedings{cohen2021gradient,
  author       = {Jeremy Cohen and
                  Simran Kaur and
                  Yuanzhi Li and
                  J. Zico Kolter and
                  Ameet Talwalkar},
  title        = {Gradient Descent on Neural Networks Typically Occurs at the Edge of
                  Stability},
  booktitle    = {9th International Conference on Learning Representations, {ICLR} 2021,
                  Virtual Event, Austria, May 3-7, 2021},
  publisher_deac    = {OpenReview.net},
  year         = {2021},
  URL_deac         = {https://openreview.net/forum?id=jh-rTtvkGeM},
  bibURL_deac      = {https://dblp.org/rec/conf/iclr/CohenKLKT21.bib},
  bibsource    = {dblp computer science bibliography, https://dblp.org}
}

@article{cohen2022adaptive,
  title={Adaptive gradient methods at the edge of stability},
  author={Cohen, Jeremy M and Ghorbani, Behrooz and Krishnan, Shankar and Agarwal, Naman and Medapati, Sourabh and Badura, Michal and Suo, Daniel and Cardoze, David and Nado, Zachary and Dahl, George E and others},
  journal={arXiv preprint arXiv:2207.14484},
  year={2022}
}

@inproceedings{damian2022self,
  author       = {Alex Damian and
                  Eshaan Nichani and
                  Jason D. Lee},
  title        = {Self-Stabilization: The Implicit Bias of Gradient Descent at the Edge
                  of Stability},
  booktitle    = {The Eleventh International Conference on Learning Representations,
                  {ICLR} 2023, Kigali, Rwanda, May 1-5, 2023},
  publisher_deac    = {OpenReview.net},
  year         = {2023},
  URL_deac          = {https://openreview.net/forum?id=nhKHA59gXz},
  bibURL_deac       = {https://dblp.org/rec/conf/iclr/DamianNL23.bib},
  bibsource    = {dblp computer science bibliography, https://dblp.org}
}

@article{holler2024uniqueness,
  title={On uniqueness in structured model learning},
  author={Holler, Martin and Morina, Erion},
  journal={arXiv preprint arXiv:2410.22009},
  year={2024}
}

@article {aarset2023learning,
    AUTHOR = {Aarset, Christian and Holler, Martin and Nguyen, Tram Thi
              Ngoc},
     TITLE = {Learning-informed parameter identification in nonlinear
              time-dependent {PDE}s},
   JOURNAL = {Appl. Math. Optim.},
  FJOURNAL = {Applied Mathematics and Optimization},
    VOLUME = {88},
      YEAR = {2023},
    NUMBER = {3},
     PAGES = {Paper No. 76, 53},
      issn_deac= {0095-4616,1432-0606},
   MRCLASS = {49N45 (35K57 49J20)},
  MRNUMBER = {4632153},
MRREVIEWER = {Natesan\ Barani Balan},
       DOI_deac= {10.1007/s00245-023-10044-y},
       URL_deac= {https://doi.org/10.1007/s00245-023-10044-y},
}

@inproceedings{li2020fourier,
  author       = {Zongyi Li and
                  Nikola Borislavov Kovachki and
                  Kamyar Azizzadenesheli and
                  Burigede Liu and
                  Kaushik Bhattacharya and
                  Andrew M. Stuart and
                  Anima Anandkumar},
  title        = {Fourier Neural Operator for Parametric Partial Differential Equations},
  booktitle    = {9th International Conference on Learning Representations, {ICLR} 2021,
                  Virtual Event, Austria, May 3-7, 2021},
  publisher_deac    = {OpenReview.net},
  year         = {2021},
  URL_deac         = {https://openreview.net/forum?id=c8P9NQVtmnO},
  bibURL_deac      = {https://dblp.org/rec/conf/iclr/LiKALBSA21.bib},
  bibsource    = {dblp computer science bibliography, https://dblp.org}
}

@article{jiang2023global,
  title={Global convergence of deep galerkin and pinns methods for solving partial differential equations},
  author={Jiang, Deqing and Sirignano, Justin and Cohen, Samuel N},
  journal={arXiv preprint arXiv:2305.06000},
  year={2023}
}

@book{Brunton_Kutz_2019, place={Cambridge}, title={Data-Driven Science and Engineering: Machine Learning, Dynamical Systems, and Control}, publisher={Cambridge University Press}, author={Brunton, Steven L. and Kutz, J. Nathan}, year={2019}}

@article {cohen2023neural,
    AUTHOR = {Cohen, Samuel N. and Jiang, Deqing and Sirignano, Justin},
     TITLE = {Neural {Q}-learning for solving {PDE}s},
   JOURNAL = {J. Mach. Learn. Res.},
  FJOURNAL = {Journal of Machine Learning Research (JMLR)},
    VOLUME = {24},
      YEAR = {2023},
     PAGES = {Paper No. [236], 49},
      issn_deac= {1532-4435,1533-7928},
   MRCLASS = {65M75 (35-04 60H30 62M45 68T07)},
  MRNUMBER = {4633963},
}

@inproceedings{tracey2015machine,
  title={A machine learning strategy to assist turbulence model development},
  author={Tracey, Brendan D and Duraisamy, Karthikeyan and Alonso, Juan J},
  booktitle={53rd AIAA aerospace sciences meeting},
  pages={1287},
  year={2015}
}

\end{document}